\documentclass[letter]{article}

\usepackage{fullpage}

\usepackage[utf8]{inputenc} 
\usepackage[T1]{fontenc}    
\usepackage{lmodern}
\usepackage{hyperref}       
\usepackage{url}            
\usepackage{booktabs}       
\usepackage{amsfonts}       
\usepackage{nicefrac}       
\usepackage{microtype}      

\usepackage{enumitem}
\usepackage{graphicx}
\usepackage{subfig}
\usepackage{color}

\usepackage{algorithm}
\usepackage{algorithmic}
\usepackage{amsmath}
\usepackage{amsthm}
\usepackage{amssymb}
\usepackage{bm}
\DeclareMathOperator*{\argmin}{arg\,min}

\newcommand{\reg}{\psi}
\newcommand{\mtd}{\mathcal{M}}

\newcommand{\eg}{{\it e.g.}}

\newcommand{\R}{\mathbb{R}}

\newcommand{\E}{\mathbb{E}}

\newcommand{\C}{\mathcal{C}}

\renewcommand\leqslant\leq
\renewcommand\geqslant\geq
\renewcommand\epsilon\varepsilon
\renewcommand\ln\log
\newcommand\st{~~~{\text{s.t.}}~~~}
\renewcommand\star{*}
\def\defin{\triangleq}
\def\Real{{\mathbb{R}}}

\newcommand{\qning}{QNing}
\newcommand{\qningsp}{QNing~}
\newcommand{\gradest}[1]{\text{\textsf{ApproxGradient}}\left({#1}\right)}
\newcommand{\gradestb}{\text{\textsf{ApproxGradient}}}
 \newtheorem{example}{Example}
 \newtheorem{definition}{Definition}
 
\makeatletter
\providecommand\phantomcaption{\caption@refstepcounter\@captype}
\makeatother

\newtheorem{proposition}{Proposition}
\newtheorem{lemma}{Lemma}
\newtheorem{cor}{Corollary}
\newtheorem{rem}{Remark}

\newcommand{\TheTitle}{An Inexact Variable Metric Proximal Point Algorithm \\ for Generic Quasi-Newton Acceleration} 
\graphicspath{{./results/}}

\title{{\TheTitle}}

\author{
   Hongzhou Lin \\
   MIT\thanks{Computer Science and Artificial Intelligence Laboratory
   Cambridge, MA 02139, USA. This work was performed in large parts while Hongzhou Lin was at Inria.}\\
   \texttt{hongzhou@mit.edu} \\
   \and
   Julien Mairal \\ 
   Inria\thanks{Univ. Grenoble Alpes, Inria, CNRS, Grenoble INP, LJK, Grenoble, 38000,
   France.}\\
   \texttt{julien.mairal@inria.fr} \\
   \and
   Zaid Harchaoui \\
   University of Washington\thanks{Department of Statistics Seattle, WA 98195,
   USA.} \\
   \texttt{zaid@uw.edu} \\
}

\begin{document}
\maketitle

\begin{abstract}
   We propose an inexact variable metric proximal point algorithm to accelerate
gradient-based optimization algorithms.
The proposed scheme, called \qning,
can be notably applied to incremental first-order methods such as the stochastic
variance-reduced gradient descent algorithm (SVRG) and other randomized incremental optimization algorithms. 
\qningsp is also compatible with composite
objectives, meaning that it has the ability to provide exactly sparse solutions
when the objective involves a sparsity-inducing regularization. When combined with limited-memory BFGS rules,
 \qningsp is particularly effective to solve high-dimensional optimization problems, while enjoying a worst-case linear
convergence rate for strongly convex problems. We present experimental results
where \qningsp gives significant improvements over competing methods for 
training machine learning methods on large samples and in high dimensions. 

\end{abstract}


\section{Introduction}
Convex composite optimization arises in many scientific fields, such as image
and signal processing or machine learning. It consists of minimizing a real-valued 
function composed of two convex terms:
\begin{equation} \label{eq:general}
   \min_{x \in \R^d} \left\{ f(x) \defin f_0(x) + \reg(x) \right\},
\end{equation} 
where~$f_0$ is smooth with Lipschitz continuous derivatives, and $\reg$
is a regularization function which is not necessarily differentiable.
A typical example from the signal and image processing literature is the $\ell_1$-norm $\reg(x) =
\|x\|_1$, which encourages sparse solutions~\cite{elad2010,mairal2014sparse}; composite minimization
also encompasses constrained minimization 
when considering extended-valued indicator functions~$\reg$ that may take the
value $+\infty$ outside of a convex set~$\C$ and~$0$ inside (see
\cite{hiriart1996convex}). In general, algorithms that are dedicated to composite optimization only require to be able to compute efficiently the proximal
operator of~$\psi$:
\begin{equation*} \label{eq:prox}
   p_{\psi}(y) \defin \argmin_{x \in \Real^d} \left\{ \psi(x) + \frac{1}{2}\|x-y\|^2 \right\},
\end{equation*}
where $\Vert \cdot \Vert$ denotes the Euclidean norm. Note that when $\psi$ is
an indicator function, the proximal operator corresponds to the
simple Euclidean projection.

To solve~(\ref{eq:general}), significant efforts have been devoted to (i) extending
techniques for smooth optimization to deal with composite terms; (ii)
exploiting the underlying structure of the problem---is $f$ a finite sum of independent terms?
Is~$\psi$ separable in different blocks of coordinates? (iii) exploiting the local curvature of the
smooth term~$f$ to achieve faster convergence than gradient-based approaches when the dimension~$d$ is large.
Typically, the first point is well understood in the context of optimal
first-order methods, see~\cite{fista,nesterov2013gradient}, and the third point
is tackled with effective heuristics such as L-BFGS when the problem is smooth~\cite{liu1989limited,lbfgs}. 
Yet, addressing all these challenges at the same time is difficult, which is precisely
the focus of this paper.

In particular, a problem of interest that initially motivated our work is that of
empirical risk minimization (ERM); the problem
arises in machine learning and can be formulated as the minimization of a
composite function $f: \R^d \rightarrow \R$:
\begin{equation}\label{eq:obj}
   \min_{x \in \Real^d} \left\{ f(x) \defin \frac{1}{n} \sum_{i=1}^n f_i(x) + \reg(x) \right\}, 
\end{equation}
where the functions $f_i$ are convex and smooth with Lipschitz continuous derivatives, and
$\reg$ is a composite term, possibly non-smooth. The function $f_i$ measures the fit of some model parameters~$x$ to a
specific data point indexed by~$i$, and $\reg$ is a regularization penalty to prevent
over-fitting. To exploit the sum structure of~$f$, a large number of randomized
incremental gradient-based techniques have been proposed, such as
SAG~\cite{sag}, SAGA~\cite{saga}, SDCA~\cite{accsdca}, SVRG~\cite{proxsvrg},
Finito~\cite{finito}, or MISO~\cite{miso}. These approaches access a single
gradient~$\nabla f_i(x)$ at every iteration instead of the full gradient $(1/n)
\sum_{i=1}^n \nabla f_i(x) $ and achieve lower computational complexity in
expectation than optimal first-order methods~\cite{fista,nesterov2013gradient} under a few assumptions. 
Yet, these methods are unable to exploit the curvature of the objective
function; this is indeed also the case for variants that are accelerated in the sense of
Nesterov~\cite{frostig,catalyst,accsdca}.

To tackle~(\ref{eq:obj}), dedicated first-order methods are often the default
choice in machine learning, but it is also known that standard Quasi-Newton
approaches can sometimes be surprisingly effective in the smooth case---that is when
$\psi=0$, see, e.g., \cite{sag} for extensive benchmarks. Since the dimension
of the problem~$d$ is typically very large ($d \geq 10\,000$), ``limited
memory'' variants of these algorithms, such as L-BFGS, are necessary to
achieve the desired scalability~\cite{liu1989limited,lbfgs}.
The theoretical guarantees offered by L-BFGS are somewhat limited, meaning that
it does not outperform accelerated first-order methods in terms of worst-case
convergence rate, and also it is not guaranteed to correctly approximate
the Hessian of the objective. Yet, L-BFGS remains one of the greatest practical
success of smooth optimization. Adapting L-BFGS to composite and 
structured problems, such as the finite sum of functions (\ref{eq:obj}), is of utmost importance nowadays.

For instance, there have been several attempts to develop a proximal Quasi-Newton
method~\cite{successiveapprox, lee2012proximal,scheinberg2014practical,yu2008quasi}. 
These algorithms typically require
computing many times the proximal operator of~$\psi$ with respect to a
variable metric. Quasi-Newton steps were also incorporated 
as local search steps
into accelerated first-order methods 
to further enhance their numerical performance~\cite{ghadimi:lan:2015}. 
More related to our work,
L-BFGS is combined with SVRG for minimizing smooth finite
sums in \cite{richtarik2016}. The scope of our approach is broader
beyond the case of SVRG. We present a generic Quasi-Newton scheme, 
applicable to a large-class of first-order
methods for composite optimization, including other 
incremental algorithms~\cite{saga,finito,miso,sag,accsdca} 
and block coordinate descent methods~\cite{razaviyayn2,richtarik2014}

More precisely, the main contribution of this paper is a generic meta-algorithm, called \qningsp (the letters ``Q'' and ``N'' stand for Quasi-Newton), which
uses a given optimization method to solve a sequence of auxiliary
problems up to some appropriate accuracy, resulting in faster global convergence
in practice. QNing falls into the class of inexact proximal point algorithms
with variable metric  and \emph{may be seen
as applying a Quasi-Newton algorithm with inexact (but accurate enough) gradients to
the Moreau envelope of the objective.}
As a result, our approach is (i) generic, as stated previously; 
(ii) despite the smoothing of the objective, 
the sub-problems that we solve are composite ones, 
which may lead to exactly sparse iterates when a sparsity-inducing
regularization is involved, e.g., the $\ell_1$-norm; 
(iii) when used with L-BFGS rules, it admits a worst-case linear convergence rate for strongly convex
problems similar to that of gradient descent, which is typically the best guarantees obtained for L-BFGS schemes in the literature. 

The idea of combining second-order or Quasi-Newton methods with Moreau envelope is in fact relatively old. It may be traced back to variable
metric proximal bundle methods \cite{fukushima1999,
fukushima1996globally,mifflin1996quasi}, which aims to incorporate curvature information into the bundle methods. Our approach revisits this principle 
with \emph{a limited-memory variant} (to deal with large dimension $d$), with \emph{a simple
line search scheme}, with \emph{several warm start strategies for the sub-problems} and with \emph{a global complexity analysis} that is more relevant
than convergence rates that do not take into account the cost per iteration.

To demonstrate the effectiveness of our scheme in practice, we evaluate
\qningsp on regularized logistic regression and regularized least-squares, 
with smooth and nonsmooth regularization penalities such as the Elastic-Net~\cite{zou2005regularization}.
We use large-scale machine learning datasets and show that \qningsp
performs at least as well as the recently proposed accelerated incremental algorithm Catalyst~\cite{catalyst}, and other Quasi-Newton baselines such as proximal Quasi-Newton methods~\cite{lee2012proximal} and Stochastic L-BFGS~\cite{slbfgs} in all numerical experiments, and
significantly outperforms them in many cases.

The paper is organized as follows: Section~\ref{sec:related} presents related
work on Quasi-Newton methods such as L-BFGS; we introduce \qningsp
in Section~\ref{sec:approach} and its 
convergence analysis in Section~\ref{sec:theory}; Section~\ref{sec:exp} is
devoted to numerical experiments and Section~\ref{sec:ccl}
concludes the paper.  

\section{Related work and preliminaries}\label{sec:related}
The history of Quasi-Newton methods can be traced back to the
1950's~\cite{Bonnans:2006,hiriart_urruty_lemarechal_1993ii,nocedalbook}.
Quasi-Newton methods often lead to significantly faster convergence in practice
compared to simpler gradient-based methods for solving smooth optimization problems~\cite{schmidt2011}. 
Yet, a
theoretical analysis of Quasi-Newton methods that explains their impressive
empirical behavior is still an open topic.
Here, we briefly review the well-known BFGS algorithm in
Section~\ref{subsec:lbfgs}, its limited memory variant~\cite{lbfgs}, and a few
recent extensions in Section~\ref{subsec:proximal qn}. Then, we present earlier works that combine proximal point algorithm and Quasi-Newton methods in Section~\ref{subsec:ppaqn}.

\subsection{Quasi-Newton methods for smooth optimization}\label{subsec:lbfgs}
The most popular Quasi-Newton method is probably BFGS, named after its inventors (Broyden-Fletcher-Goldfarb-Shanno), and its limited variant L-BFGS~\cite{nocedalbook}. These approaches will be the workhorses of the \qningsp meta-algorithm in practice. 
Consider a smooth convex objective $f$ to be minimized, the BFGS method constructs at iteration~$k$ a couple $(x_k, B_k)$ with the following update:
 \begin{equation}\label{Bkupdate}
 x_{k+1}  = x_k - \alpha_k B_k^{-1} \nabla f(x_k) \quad \text{ and } \quad  B_{k+1} = B_k  - \frac{B_k s_k s_k^\top B_k}{s_k^\top B_k s_k} + \frac{y_k y_k^\top}{y_k^\top s_k}, 
 \end{equation}
 where $\alpha_k$ is a suitable stepsize and 
 \begin{equation*}\label{skyk}
 s_k = x_{k+1} -x_k, \quad y_k = \nabla f(x_{k+1}) - \nabla f(x_k).
 \end{equation*}
 The matrix $B_k$ aims to approximate the Hessian matrix at the iterate $x_k$. When $f$ is strongly convex, the positive definiteness of~$B_k$ is guaranteed, as well as the condition $y_k^\top s_k> 0$, which ensures that (\ref{Bkupdate}) is well defined. The stepsize $\alpha_k$ is usually determined by a line search strategy. For instance, applying Wolfe's line-search strategy provides linear convergence rate for strongly convex objectives. 
Moreover, under stronger conditions that the objective $f$ is twice differentiable and its Hessian is
 Lipschitz continuous, the algorithm can asymptotically achieve superlinear convergence rate~\cite{nocedalbook}.
  
However, when the dimension $d$ is large, storing the $d$-by-$d$ matrix $B_k$ is infeasible. The limited memory variant L-BFGS~\cite{lbfgs} overcomes this issue by restricting the matrix $B_k$ to be low rank. More precisely, instead of storing the full matrix, a ``generating list'' of at most $l$ pairs of vectors~$\{(s_i^k,y_i^k)\}_{i=0 ... j}$ is kept in memory. The low rank matrix $B_k$ can then be recovered by performing the matrix update recursion in~(\ref{Bkupdate}) involving
all pairs of the generating list.
Between iteration~$k$ and~$k+1$, the generating list is
incrementally updated, by removing the oldest pair in the list (when $j=l$) and adding a new
one. What makes the approach appealing is the ability of computing the matrix-vector product $H_kz = B_k^{-1}z$ with only $O(l d)$ floating
point operations for any vector $z$. This procedure entirely relies on vector-vector product which does not explicitly construct the $d$-by-$d$ matrix $B_k$ or $H_k$. 
The price to pay is that superlinear convergence becomes out of reach.

L-BFGS is thus appropriate for high-dimensional problems~(when $d$ is large), 
but it still requires computing the full gradient at each iteration,
which may be cumbersome in the large sum setting~(\ref{eq:obj}). This
motivated stochastic counterparts of the Quasi-Newton methods (SQN)~\cite{sqn,mokhtari2015global,schraudolph2007stochastic}. Unfortunately, 
directly substituting the full gradient $\nabla f(x_k)$ by its stochastic counterpart does not lead to a convergent scheme. 
Instead, the SQN method~\cite{sqn} uses updates with sub-sampled Hessian-vector products, which leads to a sublinear convergence rate. Later, a linearly-convergent SQN algorithm is proposed by exploiting a variance reduction scheme~\cite{richtarik2016,slbfgs}. However, it is unclear how to extend these techniques into the composite setting.  

\subsection{Quasi-Newton methods for composite optimization}\label{subsec:proximal qn}
Different approaches have been proposed to extend Quasi-Newton methods to composite optimization problems. A first approach consists in minimizing successive quadratic approximations, also called proximal Quasi-Newton methods~\cite{successiveapprox,ghanbari2016proximal,lee2018inexact,lee2012proximal,liu2017inexact,scheinberg2014practical}. More concretely, a local quadratic approximation $q_k$ is minimized at each iteration:
\begin{equation}\label{eq:proxqn}
	q_k(x) \defin f_0(x_k) + \langle \nabla f_0(x_k), x- x_k \rangle + \frac{1}{2}(x-x_k)^T B_k (x-x_k) + \psi(x),
\end{equation}
where $B_k$ is a Hessian approximation based on Quasi-Newton methods. The  minimizer of $q_k$ provides a descent direction, which is subsequently used to build the next iterate.  
However, a closed form solution of~(\ref{eq:proxqn}) is usually not available since $B_k$ changes over the iterations. Thus, one needs to apply an optimization algorithm to approximately solve~(\ref{eq:proxqn}). 
The composite structure of the subproblem
naturally leads to choosing a first-order optimization algorithm, such as randomized coordinate descent algorithms.
Then, superlinear complexity becomes out of reach since it requires 
 the subproblems (\ref{eq:proxqn}) to be solved with ``high accuracy''~\cite{lee2012proximal}.
The global convergence rate of this inexact variant has been for instance analyzed in \cite{scheinberg2014practical}, where
a sublinear convergence rate is obtained for convex problems; later, the analysis has been extended to strongly convex problems in \cite{liu2017inexact}, where
linear convergence rate is achieved.
 
A second approach of extending Quasi-Newton methods to composite optimization problems is based on smoothing techniques. More precisely, a Quasi-Newton method is applied to a smoothed version of the objective. For instance, one may use the forward-backward envelope \cite{becker2012quasi,stella2017forward}. The idea is to mimic forward-backward splitting methods and apply Quasi-Newton instead of gradient steps on top of the envelope. 
Another well known smoothing technique is to apply the Moreau-Yosida regularization~\cite{moreau1962fonctions, yosida} which gives the smoothed function called Moreau envelope. Then apply Quasi-Newton methods on it leads to the family of variable metric proximal point algorithms~\cite{Burke2000,fukushima1999,fuentes2012,fukushima1996globally}. Our method pursues this line of work by developing a practical inexact variant with global complexity guarantees.

\subsection{Combining the proximal point algorithm and Quasi-Newton methods}\label{subsec:ppaqn}
We briefly recall the definition of the Moreau envelope and its basic properties. 
\begin{definition}
	Given an objective function~$f$ and a smoothing parameter $\kappa > 0$, the {\bf Moreau envelope} of~$f$ is the function $F$ obtained by performing the infimal convolution
\begin{equation} \label{definMY}
   F(x)  \defin \min_{z \in \R^d} \left \{  f(z) + \frac{\kappa}{2} \Vert z-x \Vert^2  \right  \}.
\end{equation}
\end{definition}
When $f$ is convex, the sub-problem defined in (\ref{definMY}) is strongly convex which provides an unique minimizer, called the proximal point of~$x$, which we denote by $p(x)$.  
\begin{proposition}[\bfseries Basic properties of the Moreau Envelope]\label{propMY}
If $f$ is convex, the Moreau envelope~$F$ defined in~(\ref{definMY}) satisfies 
\begin{enumerate}
\item  $F$ has the same minimum as $f$, i.e.
\begin{equation*}
\min_{x \in \R^d} F(x) = \min_{x \in \R^d} f(x), 
\end{equation*}
and the solution set of the two above problems coincide with each other.
\item $F$ is continuously differentiable even when $f$ is not and 
   \begin{equation}\label{eq:gradF}
\nabla F(x) = \kappa(x- p(x)).
\end{equation}
      Moreover the gradient $\nabla F$ is Lipschitz continuous with constant $L_F=\kappa$.
\item $F$ is convex; moreover, when $f$ is $\mu$-strongly convex with respect to the Euclidean norm, $F$ is $\mu_F$-strongly convex with $\mu_F=\frac{\mu \kappa}{\mu+\kappa}.$  
\item $F$ is upper-bounded by $f$. More precisely, for any $x \in \R^d$, 
\begin{equation}\label{eq:F<f}
	F(x) + \frac{1}{2\kappa}\Vert \nabla F(x) \Vert^2 \leq f(x).
\end{equation}
\end{enumerate}
\end{proposition}
Interestingly, $F$ inherits all the convex properties of $f$ and more
importantly it is always continuously differentiable, see~\cite{lemarechal1997practical} for elementary proofs. Moreover, the condition number of $F$ is given by 
\begin{equation}\label{eq:condition nb}
	q = \frac{L_F}{\mu_F} = \frac{\mu+\kappa}{\mu},
\end{equation}
which may be adjusted by the regularization parameter $\kappa$. Then, a naive approach to overcome the non-smoothness of the function~$f$ is to transfer the optimization problem to its Moreau envelope~$F$. More concretely, we may apply an optimization algorithm to minimize $F$ and use the obtained solution as a solution to the original problem, since both functions share the same minimizers. This yields the following well-known algorithms.

\paragraph{Proximal point algorithm.}
Consider the gradient descent method with constant step size $1/L_F$: 
\begin{equation*}
x_{k+1} = x_k - \frac{1}{\kappa} \nabla F(x_k). 
\end{equation*}
By rewriting the gradient $\nabla F(x_k)$ as $\kappa (x_k - p(x_k)) $, we obtain the proximal point algorithm~\cite{rockafellarppa}:
\begin{equation}\label{eq:auxiliary} 
x_{k+1} = p(x_k) = \argmin_{z \in \R^d} \left \{ f(z) + \frac{\kappa}{2} \Vert z- x_k \Vert^2 \right \} .
\end{equation}

\paragraph{Accelerated proximal point algorithm.}
Since gradient descent on~$F$ yields the proximal point algorithm, it is natural to apply an accelerated first-order method to get faster convergence.
To that effect, Nesterov's algorithm~\cite{nesterov1983} uses a two-stage update, along with a specific extrapolation parameter $\beta_{k+1}$:
\begin{equation*}
 x_{k+1} = y_k - \frac{1}{\kappa} \nabla F (y_k) \quad \text{ and } \quad  y_{k+1} = x_{k+1} + \beta_{k+1} (x_{k+1} -x_k),
\end{equation*}
and, given~(\ref{eq:gradF}), we obtain that
\begin{equation*}
 x_{k+1} = p(y_k) \quad \text{ and } \quad  y_{k+1} = x_{k+1} + \beta_{k+1} (x_{k+1} -x_k),
\end{equation*}
This is known as the accelerated proximal point algorithm introduced by
G\"uler~\cite{ppa}, which was recently extended in~\cite{catalyst,catalyst2}.

\paragraph{Variable metric proximal point algorithm.}
One can also apply Quasi-Newton methods on $F$, which yields 
\begin{equation}
 x_{k+1} = x_k - \alpha_k B_k^{-1} \nabla F(x_k) , \label{eq:myqn}
\end{equation}
where~$B_k$ is the Hessian approximation of $F$ based on Quasi-Newton methods.
This is known as the variable metric proximal point algorithm~\cite{Burke2000,fukushima1999,fuentes2012,fukushima1996globally}. 

\paragraph{Towards an inexact variable metric proximal point algorithm.} Quasi-Newton approaches have been applied after inexact Moreau envelope in various ways \cite{Burke2000,fukushima1999,fuentes2012,fukushima1996globally}. In particular, it is shown in~\cite{fukushima1999} that if the sub-problems~(\ref{definMY}) are solved up to high enough accuracy, then the inexact variable metric proximal point algorithm preserves the superlinear convergence rate. However, the complexity for solving the sub-problems with high accuracy is typically not taken into account in these previous works.

In the unrealistic case where $p(x_k)$ can be obtained at no cost, proximal point algorithm can afford much larger step sizes than classical gradient methods, thus are more effective.
For instance, when $f$ is strongly convex, the Moreau envelope~$F$ can be made arbitrarily well-conditioned by making $\kappa$ arbitrarily small, according to (\ref{eq:condition nb}). Then, a single gradient step on $F$ is enough to be arbitrarily close to the optimum.
In practice, however, sub-problems are solved only approximately and the complexity of solving the sub-problems is directly related to the smoothing parameter $\kappa$. This leaves an important question: how to choose the smoothing parameter $\kappa$. 
A small $\kappa$ makes the smoothed
function~$F$ better conditioned, while a large $\kappa$ is needed to improve the conditioning of the
sub-problem~(\ref{definMY}).

The main contribution of our paper is to close this gap by providing a global complexity analysis which takes into account the complexity of solving the subproblems. 
More concretely, in the proposed \qningsp algorithm, we provide i) a practical stopping criterion for the sub-problems; ii) several warm-start strategies; iii) a simple line-search strategy which guarantees sufficient descent in terms of function value. These three components together yield the global convergence analysis, which allows us to use first-order method as a subproblem solver. Moreover, the global complexity we developed depends on the smoothing parameter $\kappa$, which provides some insight about how to practically choose this parameter.   

\paragraph{Solving the subproblems with first-order algorithms.}
In the composite setting, both proximal Quasi-Newton methods and the variable metric proximal point algorithm require solving sub-problems, namely~(\ref{eq:proxqn}) and~(\ref{definMY}), respectively.
In the general case, when a generic first-order method---\eg, proximal gradient
descent---is used, our worst-case complexity analysis does not provide
a clear winner, and our experiments in Section~\ref{subsec:gd} confirm that 
both approaches perform similarly.
However, when it is possible to exploit the specific structure of the sub-problems
in one case, but not in the other one, the conclusion may differ. 

For instance, when the problem has a finite sum (\ref{eq:obj}) structure, the proximal point algorithm approach leads to
sub-problems that can be solved in $O(n \log(1/\varepsilon))$ iterations with first-order incremental methods such as 
SVRG~\cite{proxsvrg}, SAGA~\cite{saga} or MISO~\cite{miso}, by using the same choice of smoothing parameter $\kappa=2L/n$ as Catalyst~\cite{catalyst2}.
Assuming that computing a gradient of a function $f_i$ and computing the
proximal operator of~$\psi$ are both feasible in $O(d)$ floating point
operations, our approach solves each sub-problem with enough accuracy in $\tilde{O}(nd)$ operations.\footnote{ The notation $\tilde{O}$ hides logarithmic quantities.}
On the other hand, we cannot naively apply SVRG to solve the proximal
Quasi-Newton update~(\ref{eq:proxqn}) at the same cost due to the following reasons. First, the variable metric matrix~$B_k$ does not admit a natural finite sum decomposition. The naive way of writing it into $n$ copies results an increase of computational complexity for evaluating the incremental gradients. More precisely, when~$B_k$ has rank~$l$, computing a single gradient now requires to compute a matrix-vector product with cost at least~$O(dl)$, resulting in $l$-fold increase per iteration. Second, the previous iteration-complexity $O(n \log(1/\varepsilon))$ for solving the sub-problems would require the subproblems to be well-conditioned, i.e. $B_k \succeq (L/n) I$, forcing the Quasi-Newton
metric to be potentially more isotropic.
For these reasons, existing attempts to combine SVRG with Quasi-Newton principles have adopted 
other directions~\cite{richtarik2016,slbfgs}.

\section{\qning: a Quasi-Newton meta-algorithm}\label{sec:approach}
We now present the \qning~method in Algorithm~\ref{alg:newtonizer}, which
consists of applying variable metric algorithms on the smoothed objective~$F$ with inexact
gradients. Each gradient approximation is the result
of a minimization problem tackled with the algorithm~$\mtd$, used as a sub-routine.  
The outer loop of the algorithm performs
Quasi-Newton updates.  The method~$\mtd$ can be any algorithm of the user's
choice, as long as it enjoys linear convergence rate for strongly convex
problems. More technical details are given in Section~\ref{subsec:qnizer}.

\begin{algorithm}[hbtp]
   \caption{\qning: a Quasi-Newton meta-algorithm}\label{alg:newtonizer}
   \begin{algorithmic}[1]
      \INPUT Initial point~$x_0$ in~$\Real^d$; number of iterations~$K$;
      smoothing parameter $\kappa > 0$; optimization algorithm~$\mtd$; optionally, budget $T_{\mtd}$ for solving the sub-problems.
      \STATE Initialization: $(g_0,F_0,z_0) = \gradest{x_0,\mtd}$; $H_0=\frac{1}{\kappa} I$.
      \FOR{$k=0,\ldots,K-1$}
      \STATE Initialize $\eta_k=1$. 
      \STATE Perform the Quasi-Newton step 
      \begin{displaymath}
         x_{\text{test}} = x_{k} - \left (\eta_k H_{k} + (1-\eta_k)H_0 \right ) g_{k}.
      \end{displaymath}
      \STATE Estimate the gradient and function value of the Approximate Moreau envelope at $x_{\text{test}}$
      \begin{displaymath}
         (g_{\text{test}},F_{\text{test}},z_{\text{test}}) = \gradest{x_{\text{test}},\mtd} \, .
      \end{displaymath}
      \WHILE{$F_{\text{test}} > F_k - \frac{1}{4 \kappa} \Vert g_k \Vert^2$}
      \STATE Decrease the value of the line search parameter $\eta_k$ in $[0,1]$ and re-evaluate $x_{\text{test}}$. 
      \vspace*{0.1cm}
      \STATE Re-evaluate $(g_{\text{test}},F_{\text{test}},z_{\text{test}}) = \gradest{x_{\text{test}},\mtd} .$
      \vspace*{0.1cm}
      \ENDWHILE
      \vspace*{0.1cm}
      \STATE Accept the new iterate: $(x_{k+1},g_{k+1},F_{k+1},z_{k+1}) =(x_{\text{test}},g_{\text{test}},F_{\text{test}},z_{\text{test}})$.
	  \vspace*{0.1cm}
      \STATE Update $H_{k+1}$ (for example, use $\textsf{L-BFGS}$ update with $s_k=x_{k+1}-x_k$, and $y_k= g_{k+1}-g_k$).
      \ENDFOR
      \OUTPUT inexact proximal point $z_{K}$ (solution).
   \end{algorithmic}
\end{algorithm}
\begin{algorithm}[hbtp]
   \caption{Generic procedure \textsf{ApproxGradient}}\label{alg:grad}
   \begin{algorithmic}[1]
      \INPUT Current point~$x$ in~$\Real^d$; smoothing parameter $\kappa > 0$; optionally, budget $T_{\mtd}$.
      \STATE Compute the approximate proximal mapping using an optimization method~$\mtd$:
      \begin{equation}\label{eq:approxF}
         z \approx \argmin_{w \in \R^d} \left \{ h(w) \defin f(w) + \frac{\kappa}{2} \Vert w -x \Vert^2   \right \},
      \end{equation}
      using one of the following stopping criteria:
      \begin{itemize}
      	\item Stop when the approximate solution $z$ satisfies       
      	\begin{equation}\label{eq:stop condition}
      	h(z) - h^* \leq \frac{\kappa}{36} \Vert z-x \Vert^2.
      	\end{equation}
      	\item Stop when we reach the pre-defined constant budget $T_{\mtd}$ (for instance one pass over the data).
      \end{itemize}
      \STATE Estimate the gradient~$\nabla F(x)$ of the Moreau envelope using
      \begin{equation*}
         g = \kappa (x - z).
      \end{equation*}
      \OUTPUT gradient estimate $g$, objective value estimate $F_a \defin h(z)$, proximal mapping estimate $z$.  
   \end{algorithmic}
\end{algorithm}

\subsection{The main algorithm}\label{subsec:qnizer}
We now discuss the main algorithm components and its main features.
\paragraph{Outer-loop: inexact variable metric proximal point algorithm.}
We apply variable metric algorithms with a simple line search strategy similar to~\cite{scheinberg2014practical} on the Moreau envelope $F$. 
Given a positive definite matrix~$H_k$ and a step size $\eta_k$ in $[0,1]$, the algorithm computes the update
\begin{equation}\label{eq:LS}
	 x_{k+1} = x_k - (\eta_k H_k + (1-\eta_k) H_0) g_k, \tag{LS}
\end{equation}
where $H_0 = \frac{1}{\kappa} I$. 
When $\eta_k=1$, the update uses the metric $H_k$, and when
$\eta_k=0$, it uses an inexact proximal point update $x_{k+1} = x_k -
(1/\kappa) g_k$. In other words, when the quality of the metric $H_k$ is not good enough,
due to the inexactness of the gradients used in its construction, 
the update is corrected towards a simple proximal point update, whose
convergence is well understood when the gradients are inexact.

In order to determine the stepsize $\eta_k$, we introduce the following descent condition, 
\begin{equation}\label{suffdescent}
F_{k+1} \leq F_k - \frac{1}{4 \kappa} \Vert g_k \Vert^2.
\end{equation}
We show that the descent condition~(\ref{suffdescent}) is always satisfied when $\eta_k=0$, thus the finite termination of the line search follows, see Section~\ref{subsec:linesearch} for more details.
In our experiments, we observed empirically that the stepsize $\eta_k=1$ was almost 
always selected. In practice, we try the values $\eta_k$ in $\{1,1/2,1/4,1/8,0\}$
starting from the largest one and stops whenever condition~(\ref{suffdescent}) is
satisfied.

\paragraph{Example of variable metric algorithm: inexact L-BFGS method.}
The L-BFGS rule we consider is the
standard one and consists in updating incrementally a generating list of
vectors $\{ (s_i,y_i)\}_{i =1\ldots j}$, which implicitly defines the L-BFGS
matrix. We use here the two-loop recursion detailed in~\cite[Algorithm
7.4]{nocedalbook} and use skipping steps when the condition $s_i^\top y_i > 0$
is not satisfied, in order to ensure the positive-definiteness of the L-BFGS
matrix~$H_k$ (see \cite{friedlander2012hybrid}). 

\paragraph{Inner-loop: approximate Moreau envelope.} The inexactness of our scheme comes from the approximation of the Moreau envelope $F$ and its gradient.
The procedure $\gradest{}$ calls an
minimization algorithm $\mtd$ and apply $\mtd$ to minimize the sub-problem~(\ref{eq:approxF}).
When the problem is solved exactly, the function returns the exact values
$g=\nabla F(x)$, $F_a = F(x)$, and $z=p(x)$. However, this is infeasible in practice and we can only expect approximate solutions. In particular, a stopping criterion should be specified. We consider the following variants:
\begin{itemize}
   \item[(a)] we define an adaptive stopping criterion based on function values and stop~$\mtd$ when the approximate solution satisfies the inequality (\ref{eq:stop condition}). In contrast to standard stopping criterion where the accuracy is an absolute constant, our stopping criterion is adaptive since the righthand side of (\ref{eq:stop condition}) also depends on the current iterate $z$. More detailed theoretical insights will be given in Section~\ref{sec:theory}. Typically, checking whether or not the criterion is satisfied requires computing a duality gap, as in Catalyst~\cite{catalyst2}.
   \item[(b)] using a pre-defined budget $T_{\mtd}$ in terms of number of iterations of the method~$\mtd$, where $T_{\mtd}$ is a constant independent of $k$.
\end{itemize}
Note that such an adaptive stopping criterion is relatively classical in the literature of inexact gradient-based methods~\cite{byrd}. 
As we will see later in Section~\ref{sec:theory}, when $T_{\mtd}$ is large enough, criterion (\ref{eq:stop condition}) is guaranteed.

\paragraph{Requirements on~$\mtd$.}
To apply \qning, the optimization method $\mtd$ needs to have linear convergence rates for strongly-convex problems. More precisely, for any strongly-convex objective $h$, the method~$\mtd$ should be able to  
generate a sequence of iterates $(w_t)_{t \geq 0}$ such that
\begin{equation}\label{eq:assumption}
   h(w_t) - h^\star \leq C_\mtd (1- \tau_{\mtd})^t (h(w_0)- h^\star)~~\text{for some constants}~~C_{\mtd} > 0~\text{and}~ 1 > \tau_{\mtd} > 0,
\end{equation}
where~$w_0$ is the initial point given to~$\mtd$. The notion of linearly-convergent methods extends naturally to non-deterministic methods where~(\ref{eq:assumption}) is satisfied in expectation:
\begin{equation}
  \E[h(w_t) - h^\star] \leq C_\mtd (1- \tau_{\mtd})^t (h(w_0)- h^\star).
\end{equation}
 The linear convergence condition typically holds for many primal gradient-based optimization techniques, including classical full gradient descent methods, block-coordinate descent algorithms~\cite{nesterov2012,richtarik2014}, or variance reduced incremental algorithms~\cite{saga,sag,proxsvrg}. In particular, our method provides a generic way to combine incremental algorithms with Quasi-Newton methods which are suitable for large-scale optimization problems. For the simplicity of the presentation, we only consider the deterministic variant (\ref{eq:assumption}) in the analysis. However, it is possible to show that the same complexity results still hold for non-deterministic methods in expectation, as discussed in Section~\ref{subsec:global}.
We emphasize that we do not assume any convergence guarantee of~$\mtd$ on non-strongly convex problems since our sub-problems are always strongly convex.

\paragraph{Warm starts for the sub-problems.}  
Employing an adequate initialization for solving 
each sub-problem plays an important rule in our analysis. 
The warm start strategy we proposed here ensures that the stopping criterion in each subproblem can be
achieved in a constant number of iterations:\\
Consider the minimization of a sub-problem   
\begin{equation*}
         \min_{w \in \R^d} \left \{ h(w) \defin f(w) + \frac{\kappa}{2} \Vert w -x \Vert^2   \right \}.
\end{equation*}
Then, our warm start strategy depends on the nature of $f$:
\begin{itemize}
	\item when $f$ is smooth, we initialize with $w_0 =x$; 
	\item when $f = f_0 + \psi$ is composite, we initialize with  
	\begin{equation*}
      w_0 = \argmin_{w \in \R^d} \left \{ f_0(x)+ \langle \nabla f_0(x),w-x \rangle + \frac{L+\kappa}{2} \Vert w-x \Vert^2 + \psi(w)  \right \},
   \end{equation*}
   which performs an additional proximal step comparing to the smooth case. 
\end{itemize}

\paragraph{Handling composite objective functions.}
In machine learning or signal processing, convex composite
objectives~(\ref{eq:general}) with a non-smooth penalty~$\psi$ are typically formulated to
encourage solutions with specific characteristics; in particular,
the~$\ell_1$-norm is known to provide sparsity. Smoothing
techniques~\cite{nesterov2005smooth} may allow us to solve the optimization
problem up to some chosen accuracy, but they provide solutions that do not
inherit the properties induced by the non-smoothness of the objective. To
illustrate what we mean by this statement, we may consider smoothing the $\ell_1$-norm, leading
to a solution vector with small coefficients, but not with exact zeroes.  When
the goal is to perform model selection---that is, understanding which variables
are important to explain a phenomenon, exact sparsity is seen as an asset,
and optimization techniques dedicated to composite problems such as
FISTA~\cite{fista} are often preferred (see \cite{mairal2014sparse}). 

Then, one might be concerned that our scheme operates on the smoothed
objective~$F$, leading to iterates~$(x_k)_{k \geq 0}$ that may suffer from the above
``non-sparse'' issue, assuming that~$\psi$ is the~$\ell_1$-norm. Yet, our approach does not directly output the iterates~$(x_k)_{k \geq 0}$ but their proximal mappings~$(z_k)_{k \geq 0}$. In particular, the $\ell_1$-regularization is encoded in the proximal mapping~(\ref{eq:approxF}), thus the approximate proximal point $z_k$ may be sparse.
For this reason, our theoretical
analysis presented in Section~\ref{sec:theory} studies the convergence of the
sequence $(f(z_k))_{k \geq 0}$ to the solution~$f^\star$.

\section{Convergence and complexity analysis}\label{sec:theory}
In this section, we study the convergence of the \qningsp algorithm---that
is, the rate of convergence of the quantities $(F(x_k)-F^\star)_{k \geq 0}$ and~$(f(z_k)-f^\star)_{k \geq 0}$, and also the computational complexity due to solving the sub-problems~(\ref{eq:approxF}).  
We start by stating the main properties
of the gradient approximation in Section~\ref{subsec:gradproperties}. Then,
we analyze the convergence of the outer loop algorithm in Section~\ref{subsec:outerloop},
and Section~\ref{subsec:linesearch} is devoted to the properties of the line search strategy. After that, we provide the cost of solving the sub-problems in Section~\ref{subsec:innerloop} and derive the global complexity analysis in Section~\ref{subsec:global}.

\subsection{Properties of the gradient approximation}\label{subsec:gradproperties}
The next lemma is classical and
provides approximation guarantees about the quantities returned by the
\gradestb~procedure (Algorithm~\ref{alg:grad}); see~\cite{bertsekas:2015,fukushima1996globally}. 
We recall here the proof
for completeness.

\begin{lemma}[\bfseries Approximation quality of the gradient approximation]\label{lemma:basic}
   Consider a vector~$x$ in~$\R^d$, a positive scalar~$\varepsilon$ and 
   an approximate proximal point  
   \begin{equation*}
         z \approx \argmin_{w \in \R^d} \left \{ h(w) \defin f(w) + \frac{\kappa}{2} \Vert w -x \Vert^2   \right \},
      \end{equation*}
      such that 
      $$h(z) - h^\star \leq \varepsilon,$$ 
      where $h^\star = \min_{w \in \R^d} h(w)$.
      As in Algorithm~\ref{alg:grad}, we define the gradient estimate $g=\kappa (x-z)$ and the function value estimate $F_a=h(z)$.
   Then, the following inequalities hold
 \begin{eqnarray}
    && F(x) \leq F_{\text{\normalfont a}} \leq F(x)+\varepsilon, \label{ineqF} \\
 && \Vert z - p(x) \Vert \leq \sqrt{\frac{2\varepsilon}{\kappa}},  \label{ineqprox} \\
 && \Vert g - \nabla F(x) \Vert \leq \sqrt{2\kappa \varepsilon}  \label{ineqgrad}.
 \end{eqnarray}
 Moreover, $F_a$ is related to $f$ by the following relationship  
 \begin{equation}
	f(z) = F_{\text{\normalfont a}} - \frac{1}{2\kappa} \Vert  g \Vert^2 \label{eqprox}.
\end{equation} 
\end{lemma}
\begin{proof}
   (\ref{ineqF}) and~(\ref{eqprox}) are straightforward by definition of $h(z)$.
 Since $f$ is convex, the function $h$ is $\kappa$-strongly convex, and (\ref{ineqprox}) follows from
 \begin{equation*}
    \frac{\kappa}{2} \Vert z - p(x) \Vert^2 \leq h(z)- h(p(x)) = h(z) - h^\star \leq \varepsilon,
 \end{equation*}
   where we recall that $p(x)$ minimizes~$h$.
 Finally, we obtain (\ref{ineqgrad}) from
 \begin{equation*}
    g - \nabla F(x) = \kappa (x-z) - \kappa(x - p(x)) = \kappa(p(x)-z),
 \end{equation*}
   by using the definitions of~$g$ and the property~(\ref{eq:gradF}).
\end{proof}
This lemma allows us to quantify the quality of the gradient and function value approximations, which is crucial to control the error accumulation of inexact proximal point methods. Moreover, the relation~(\ref{eqprox}) establishes a link between the approximate function value of $F$ and the function value of the original objective~$f$; as a consequence, it is possible to relate the convergence rate of~$f$ from the convergence rate of~$F$. Finally, the following result is a direct consequence of Lemma~\ref{lemma:basic}:

\begin{lemma}[\bfseries Bounding the exact gradient by its approximation]
\label{lemma:simple}
   Consider the same quantities introduced in Lemma \ref{lemma:basic}. Then,
 \begin{eqnarray}
 \frac{1}{2} \Vert g \Vert^2 - 2 \kappa \varepsilon \leq \Vert \nabla F(x) \Vert^2 \leq 2(\Vert g \Vert^2 + 2\kappa \varepsilon)  \label{ineqgrad2}.
 \end{eqnarray}
\end{lemma}
 \begin{proof}
    The right-hand side of Eq.~(\ref{ineqgrad2}) follows from
 \begin{eqnarray*}
 \Vert \nabla F(x) \Vert^2 &\leq& 2(\Vert \nabla F(x)- g \Vert^2 + \Vert  g \Vert^2  ) \\
                                   &\leq& 2(2\kappa \varepsilon+ \Vert  g \Vert^2  ) ~~~~~~~~(\text{from}~(\ref{ineqgrad})).
 \end{eqnarray*}
 Interchanging $\nabla F(x)$ and $g$ gives the left-hand side inequality.
\end{proof}
\begin{cor}\label{cor:1}
If $\epsilon \leq \frac{c}{\kappa} \Vert g \Vert^2$ with $c < \frac{1}{4}$, then 
\begin{equation}\label{eq:cor1}
	\frac{1-4c}{2} \leq \frac{\Vert \nabla F(x) \Vert^2 }{\Vert g \Vert^2} \leq 2(1+2c).
\end{equation}	
\end{cor}
This corollary is important since it allows to replace the unknown exact gradient $\Vert \nabla F(x) \Vert$ by its approximation $\Vert g \Vert $, at the cost of a constant factor, as long as the condition $\epsilon \leq \frac{c}{\kappa} \Vert g \Vert^2$ is satisfied.

\subsection{Convergence analysis of the outer loop}\label{subsec:outerloop}
We are now in shape to establish the convergence of the \qning~meta-algorithm, without
considering yet the cost of solving the sub-problems~(\ref{eq:approxF}). 
At iteration $k$, an approximate proximal point is evaluated:
\begin{equation}
   (g_k,F_k,z_k)=\text{\textsf{ApproxGradient}}(x_{k},\mtd) \; . \label{eq:approx_eq}
\end{equation}
The following lemma characterizes the expected descent in terms of objective function value.
\begin{lemma}[\bfseries Approximate descent property]\label{lemmadescent}
At iteration $k$, if the sub-problem (\ref{eq:approx_eq}) is solved up to accuracy $\varepsilon_k$ in the sense of Lemma~\ref{lemma:basic} and the next iterate $x_{k+1}$ satisfies the descent condition (\ref{suffdescent}), then,
 \begin{equation}\label{lemma2}
     F(x_{k+1}) \leq F(x_k) - \frac{1}{8\kappa} \Vert \nabla F(x_k) \Vert^2 + \frac{3}{2}\varepsilon_k. 
 \end{equation}
\end{lemma}
\begin{proof}
From~(\ref{ineqF}) and~(\ref{suffdescent}),
\begin{eqnarray*}
  F(x_{k+1}) & \leq &   F_{k+1} \leq F_k - \frac{1}{4\kappa} \Vert g_k \Vert^2 \\
	 &\leq& F(x_k)+\varepsilon_k - \left( \frac{1}{8\kappa} \Vert \nabla F(x_k) \Vert^2 - \frac{\varepsilon_k}{2} \right) \quad \text{(from (\ref{ineqF}) and (\ref{ineqgrad2}))}\\
	 &= & F(x_k) - \frac{1}{8\kappa} \Vert \nabla F(x_k) \Vert^2 + \frac{3}{2}\varepsilon_k.
 \end{eqnarray*}  
\end{proof}

This lemma gives us a first intuition about the natural choice of the accuracy
$\varepsilon_k$, which should be in the same order as
$\Vert \nabla F(x_k) \Vert^2$. In particular, if 
   \begin{equation}
      \varepsilon_k \leq \frac{1}{16\kappa} \Vert \nabla F(x_k) \Vert^2, \label{eq:condition_eps}
   \end{equation}
then we have 
\begin{equation}\label{eq:gd}
	F(x_{k+1}) \leq F(x_k) - \frac{1}{32\kappa} \Vert \nabla F(x_k) \Vert^2,
\end{equation} 
which is a typical inequality used for analyzing gradient descent methods. Before presenting the convergence result, we remark that condition (\ref{eq:condition_eps}) cannot be used directly since it requires knowing the exact gradient~$\Vert \nabla
F(x_k) \Vert$. A more practical choice consists of replacing it by the approximate gradient.
\begin{lemma}[\bfseries{Practical choice of $\varepsilon_k$}]\label{prop:eps}
   The following condition implies inequality~(\ref{eq:condition_eps}):
 \begin{equation}\label{eq: eps < g}
 	 \varepsilon_k \leq \frac{1}{36 \kappa} \Vert g_k \Vert^2.
 \end{equation}
\end{lemma}
\begin{proof}
From Corollary~\ref{cor:1}, Equation (\ref{eq: eps < g}) implies  
$$  \Vert g_k \Vert^2 \leq \frac{2}{1- \frac{4}{36}}\Vert \nabla F(x_k) \Vert^2 = \frac{9}{4}\Vert \nabla F(x_k) \Vert^2 \quad \text{and thus } \quad 
\varepsilon_k \leq \frac{1}{36 \kappa} \Vert g_k \Vert^2 \leq \frac{1}{16\kappa} \Vert \nabla F(x_k) \Vert^2.$$
\end{proof}
This is the first stopping criterion~(\ref{eq:stop condition}) in Algorithm~\ref{alg:grad}. 
Finally, we obtain the following convergence result for strongly convex problems, which is classical in the literature of inexact gradient methods~(see Section 4.1 of \cite{byrd} for a similar result).

\begin{proposition}[\bfseries Convergence of Algorithm~\ref{alg:newtonizer}, strongly-convex objectives] 
\label{prop:stronglyconvex}
   Assume that~$f$ is $\mu$-strongly convex. 
   Let $(x_k)_{k \geq 0}$ be the sequences generated by Algorithm~\ref{alg:newtonizer} where the stopping criterion (\ref{eq:stop condition}) is used. Then,
 \begin{eqnarray*}
 F(x_{k}) - F^* \leq \left (1-\frac{1}{16q} \right)^{k} (F(x_0) -F^*), \quad \text{with} \quad q= \frac{\mu+\kappa}{\mu} .
\end{eqnarray*}
\end{proposition}
\begin{proof}
The proof follows directly from (\ref{eq:gd}) and the standard analysis of the gradient descent algorithm for the $\mu_F$-strongly convex and $L_F$-smooth function~$F$ by remarking that $L_F = \kappa$ and $\mu_F = \frac{\mu \kappa}{\mu+\kappa}$.
\end{proof}
\begin{cor}
	Under the conditions of Proposition~\ref{prop:stronglyconvex}, we have 
	\begin{equation}
		f(z_k) - f^* \leq \left (1-\frac{1}{16q} \right )^k (f(x_0) -f^*).
	\end{equation}
\end{cor}
\begin{proof}
From (\ref{eqprox}) and (\ref{eq: eps < g}), we have 
\begin{equation*}
	f(z_k) = F_k - \frac{1}{2\kappa} \Vert g_k \Vert^2 \leq F(x_k)+\epsilon_k - \frac{1}{2\kappa} \Vert g_k \Vert^2 \leq F(x_k).
\end{equation*}
Moreover, $F(x_0)$ is upper-bounded by $f(x_0)$ following (\ref{eq:F<f}).
\end{proof}
It is worth pointing out that our analysis establishes a linear convergence rate whereas one would expect a superlinear convergence rate as for classical variable metric methods.
 The tradeoff lies in the choice of the accuracy~$\varepsilon_k$. In order to achieve superlinear convergence, the approximation error $\varepsilon_k$ needs to decrease superlinearly, as shown in \cite{fukushima1999}. However, a fast decreasing sequence~$\varepsilon_k$ requires an increasing effort in solving the sub-problems, which will dominate the global complexity. In other words, the global complexity may become worse even though we achieve faster convergence in the outer-loop. This will become clearer when we discuss the inner loop complexity in Section~\ref{subsec:innerloop}.

Next, we show that under a bounded level set condition, QNing enjoys the classical sublinear $O(1/k)$ convergence rate when the objective is convex but not strongly convex.
\begin{proposition}[\bfseries Convergence of Algorithm~\ref{alg:newtonizer} for convex, but not strongly-convex objectives]
\label{prop:convex}
 Let $f$ be a convex function with bounded level sets. 
   Then, there exists a constant $R>0$, which depends on the initialization point~$x_0$, such that the sequences $(x_k)_{k \geq 0}$ and $(z_k)_{k \geq 0}$ generated by Algorithm~\ref{alg:newtonizer} with stopping criterion (\ref{eq:stop condition}), satisfies
 $$ F(x_k) -F^*  \leq \frac{32 \kappa R^2}{k} ~~~\text{and}~~~  f(z_k) - f^* \leq \frac{32 \kappa R^2}{k}.$$
\end{proposition}
\begin{proof}
We defer the proof and the proper definition of the bounded level set assumption to Appendix~\ref{app:prop6}.  
\end{proof}
So far, we have assumed in our analysis that the iterates satisfy the descent condition~(\ref{suffdescent}), which means the line search strategy will always terminate. We prove this is indeed the case in the next section and provide some additional conditions under which a non-zero step size will be selected.

\subsection{Conditions for non-zero step sizes $\eta_k$ and termination of the line search}\label{subsec:linesearch}
At iteration $k$, a line search is performed on the stepsize $\eta_k$ to find the next iterate
\begin{equation*}
	x_{k+1} = x_k - (\eta_k H_k + (1-\eta_k)H_0) g_k,
\end{equation*}
such that $x_{k+1}$ satisfies the descent condition (\ref{suffdescent}). 
We first show that the descent condition holds when $\eta_k=0$ before giving a more general result.
\begin{lemma}
If the sub-problems are solved up to accuracy $\varepsilon_k \leq \frac{1}{36 \kappa} \Vert g_k \Vert^2$, then the descent condition~(\ref{suffdescent}) holds when $\eta_k=0$.
\end{lemma}
\begin{proof}
	When $\eta_k=0$, $x_{k+1} = x_k - \frac{1}{\kappa} g_k = z_k$. Then, 
	\begin{align*}
		F_{k+1} & \leq F(x_{k+1}) + \frac{1}{36 \kappa} \Vert g_{k+1}\Vert^2 \quad \text{ (from Eq.~(\ref{ineqF}))} \\
		& \leq F(x_{k+1}) + \frac{1}{36\kappa} \frac{2}{1-\frac{4}{36}} \Vert \nabla F(x_{k+1}) \Vert^2 \quad \text{(from Eq. (\ref{eq:cor1}) )}\\
		&< F(x_{k+1}) + \frac{1}{2\kappa} \Vert \nabla F(z_{k+1}) \Vert^2 \\
		& \leq f(x_{k+1}) = f(z_k) \quad \text{ (from Eq. (\ref{eq:F<f}))}\\
		& = F_k - \frac{1}{2\kappa}\Vert g_k \Vert^2 \quad \text{ (from Eq. (\ref{eqprox}))}.
	\end{align*} 
\end{proof}
Therefore, it is theoretically sound to take the trivial step size $\eta_k=0$, which implies the termination of our line search strategy. In other words, the descent condition always holds by taking an inexact gradient step on the Moreau envelope~$F$, which corresponds to the update of the proximal point algorithm.  However, the purpose of using variable metric method is to exploit the curvature of the function,  which is not the case when $\eta_k=0$. Thus, the trivial step size should be only considered as a backup plan and we show in the following some sufficient conditions for taking non-zero step sizes and even stronger, unit step sizes.  
\begin{lemma}[\bf{A sufficient condition for satisfying the descent condition (\ref{suffdescent})}] \label{lem:suffdescent}
	If the sub-problems are solved up to accuracy $\varepsilon_k \leq \frac{1}{36 \kappa} \Vert g_k \Vert^2$, then the sufficient condition (\ref{suffdescent}) holds for any $x_{k+1} = x_k - A_k g_k$ where $A_k$ is a positive definite matrix satisfying $ \frac{1-\alpha}{\kappa}I \preceq A_k \preceq \frac{1+\alpha}{\kappa} I$ with $\alpha \leq \frac{1}{3}$.
\end{lemma}

As a consequence, a line search strategy consisting of finding the largest $\eta_k$ of the form $\gamma^i$, with $i=1,\ldots,+\infty$ and $\gamma$ in $(0,1)$ always terminates in a bounded number of iterations if the sequence of variable metric $(H_k)_{k \geq 0}$ is bounded, i.e. there exists 
$0< m <M$ such that for any $k$, $mI \preceq H_k \preceq MI$. This is the case for L-BFGS update:
\begin{lemma}[\bf{Boundedness of L-BFGS metric matrix, see Chap 8,9 of \cite{nocedalbook}}]
	The variable metric matrices $(B_k)_{k}$ constructed by the L-BFGS rule are positive definite and bounded.
\end{lemma}
\begin{proof}[Proof of Lemma~\ref{lem:suffdescent}]
First, we recall that $z_k = x_k - \frac{1}{\kappa} g_k$ and we rewrite
\begin{equation*}
	F_{k+1} = \underbrace{F_{k+1} - F(x_{k+1})}_{\defin E_1} + \underbrace{F(x_{k+1}) -F(z_k)}_{\defin E_2} + F(z_k).
\end{equation*}  
We are going to bound the two error terms $E_1$ and $E_2$ by some factors of $\Vert g_k \Vert^2$. Noting that the subproblems are solved up to $\varepsilon_k \leq \frac{c}{\kappa} \Vert g_k \Vert^2$ with $c=  \frac{1}{36}$, we obtain by construction 
\begin{equation}
	E_1 = F_{k+1} - F(x_{k+1}) \leq \epsilon_{k+1} \leq \frac{c}{\kappa} \Vert g_{k+1} \Vert^2 \leq \frac{2c}{(1-4c)\kappa} \Vert \nabla F(x_{k+1}) \Vert^2 ,
\end{equation}
where the last inequality comes from Corollary~\ref{cor:1}. Moreover, 
\begin{align}
	\Vert \nabla F(x_{k+1}) \Vert & \leq \Vert \nabla F(z_k)\Vert +\Vert \nabla F(x_{k+1}) - \nabla F(z_k) \Vert  \nonumber \\
	& \leq \Vert \nabla F(z_k) \Vert +\kappa \Vert x_{k+1} -z_k \Vert . \nonumber   
\end{align}
Since $x_{k+1} - z_k = (\frac{1}{\kappa} - A_k) g_k$, we have $\Vert x_{k+1} - z_k \Vert \leq \frac{\alpha}{\kappa}\Vert g_k \Vert$. This implies that 
\begin{equation}
	\Vert \nabla F(x_{k+1}) \Vert \leq \Vert \nabla F(z_k) \Vert + \alpha \Vert g_k \Vert , \label{eq:bound grad} 
\end{equation}
and thus 
\begin{equation}\label{eq:E1}
	E_1 \leq \frac{4c}{(1-4c)\kappa} \left ( \Vert \nabla F(z_k) \Vert^2 +  \alpha^2 \Vert g_k \Vert^2 \right ) .
\end{equation}
Second, by the $\kappa$-smoothness of $F$, we have 
\begin{align}
	E_2 & = F(x_{k+1}) -F(z_k) \nonumber \\
	    & \leq \langle \nabla F(z_k),x_{k+1} -z_k \rangle + \frac{\kappa}{2} \Vert x_{k+1} - z_k \Vert^2 \nonumber\\
	    & \leq \frac{1}{4\kappa} \Vert \nabla F(z_k) \Vert^2 + \kappa\Vert x_{k+1} -z_k \Vert^2 + \frac{\kappa}{2} \Vert x_{k+1}-z_k \Vert^2 \nonumber \\ 
	    & \leq \frac{1}{4\kappa}\Vert \nabla F(z_k) \Vert^2 + \frac{3\alpha^2}{2\kappa} \Vert g_k \Vert^2, \label{eq:E2}
\end{align}
where the last inequality follows from $\Vert x_{k+1} - z_k \Vert \leq \frac{\alpha}{\kappa}\Vert g_k \Vert$.
Combining (\ref{eq:E1}) and (\ref{eq:E2}) yields 
\begin{equation}
	E_1 + E_2 \leq  \left [ \frac{4c}{1-4c} + \frac{1}{4}\right ] \frac{1}{\kappa} \Vert \nabla F(z_k) \Vert^2 + \left [ \frac{4c}{1-4c} + \frac{3}{2}\right ] \frac{\alpha^2}{\kappa} \Vert g_k \Vert^2.
\end{equation}
When $c \leq \frac{1}{36}$ and $\alpha \leq \frac{1}{3}$, we have 
$$E_1 + E_2 \leq \frac{1}{2 \kappa} \Vert \nabla F(z_k) \Vert^2+ \frac{1}{4 \kappa} \Vert g_k \Vert^2 .$$ 
Therefore,
\begin{align}
		F_{k+1} & \leq F(z_k) + E_1 +E_2  \nonumber \\
				& \leq F(z_k) + \frac{1}{2 \kappa} \Vert \nabla F(z_k) \Vert^2 +\frac{1}{4 \kappa} \Vert g_k \Vert^2 \nonumber \\
				& \leq f(z_k)+  \frac{1}{4 \kappa} \Vert g_k \Vert^2 \nonumber \\
				& = F_k - \frac{1}{4\kappa} \Vert g_k \Vert^2, \label{eq:F zk}
\end{align}
where the last equality follows from (\ref{eqprox}), this completes the proof. 
\end{proof}
Note that in practice, we consider a set of step sizes $\eta_k = \gamma^i$ for $i \leq i_{\text{max}}$ or $\eta_k=0$, which naturally upper-bounds the number of line search iterations to $i_{\text{max}}$. More precisely, all experiments performed in this paper use $\gamma = 1/2$ and $i_{\text{max}}=3$. Moreover, we observe that the unit stepsize is very often sufficient for the descent condition to hold, as empirically studied in Appendix~\ref{subsec:unit}.

The following result shows that under a specific assumption on the Moreau envelope~$F$,
the unit stepsize is indeed selected when the iterate are close to the optimum. 
The condition, called  Dennis-Mor\'e criterion \cite{dennis1974characterization}, is classical in the literature of Quasi-Newton methods. Even though we
cannot formally show that the criterion holds for the Moreau envelope~$F$, since it requires $F$ to be twice continuously differentiable,
which is not true in general, see~\cite{lemarechal1997practical}, it provides a sufficient condition for the unit step size.  
Therefore, the lemma below should not be seen as an formal
explanation for the choice of step size $\eta_k=1$, 
but simply as a reasonable condition that leads to this choice.

\begin{lemma}[\bf{A sufficient condition for unit stepsize}]\label{lem:unit stepsize}
	Assume that $f$ is strongly convex and $F$ is twice-continuously differentiable with Lipschitz continuous Hessian $\nabla^2 F$. If the sub-problems are solved up to accuracy $\varepsilon_k \leq \frac{\mu^2}{128 \kappa (\mu+\kappa)^2} \Vert g_k \Vert^2$ and the Dennis-Mor\'e criterion \cite{dennis1974characterization} is satisfied, i.e. 
	\begin{equation}\label{eq:DM1}
		\lim_{k \rightarrow \infty} \frac{\Vert (B_k^{-1} - \nabla^2 F(x^*)^{-1}) g_k\Vert }{\Vert g_k \Vert } = 0,  \tag{DM}
	\end{equation}  
	where $x^*$ is the minimizer of the problem and $B_k = H_k^{-1}$ is the variable metric matrix, then the descent condition~(\ref{suffdescent}) is satisfied with $\eta_k=1$ when $k$ is large enough.
\end{lemma}
We remark that the Dennis-Mor\'e criterion we use here is slightly different from the standard one since our criterion is based on approximate gradients $g_k$. If the $g_k$'s are indeed the exact gradients and the variable metric $B_k$ are bounded, then our criterion is equivalent to the standard Dennis-Mor\'e criterion.
The proof of the lemma is close to that of similar lemmas appearing in the proximal Quasi-Newton literature \cite{lee2012proximal}, and is relegated to the appendix.
Interestingly, this proof also suggests that a stronger stopping criterion $\varepsilon_k$ such that $\varepsilon_k = o(\Vert g_k \Vert^2 )$ could lead to superlinear convergence. However, such a choice of $\varepsilon_k$ would significantly increase the complexity for solving the sub-problems, and overall degrade the global complexity.

\subsection{Complexity analysis of the inner loop}\label{subsec:innerloop}
In this section, we evaluate the complexity of solving the
sub-problems~(\ref{eq:approxF}) up to the desired accuracy using a linearly convergent method~$\mtd$. 
Our main result is that all sub-problems can be solved in a constant number~$T_\mtd$ of iterations (in expectation if the method is non-deterministic) using the proposed warm start strategy. Let us consider the sub-problem with an arbitrary prox center $x$, 
\begin{equation}\label{eq:sub prob}
	\min_{w \in \R^d} \left \{ h(w) = f(w) + \frac{\kappa}{2} \Vert w -x \Vert^2 \right \}. 
\end{equation}
The number of iterations needed is determined by the ratio between the initial gap $h(w_0) -h^*$ and the desired accuracy. We are going to bound this ratio by a constant factor. 

\begin{lemma}[\bfseries Warm start for primal methods - smooth case]\label{lemma:restart}
   If $f$ is differentiable with $L$-Lipschitz continuous gradients, we  initialize the method~$\mtd$ with $w_0 = x$. 
   Then, we have the guarantee that
   \begin{equation}
   h(w_0) - h^* \leq \frac{L+\kappa}{2\kappa^2} \Vert \nabla F(x) \Vert^2.
\end{equation}
\end{lemma}
\begin{proof}
Denote by $w^*$ the minimizer of $h$. Then, we have the optimality condition $\nabla f(w^*) + \kappa (w^* - x) = 0$. As a result,  
\begin{equation*}
   \begin{split}
   h(w_0) -h^* & =  f(x) - \left(f(w^*) + \frac{\kappa}{2} \left\| w^* - x \right\|^2 \right) \\
   & \leq  f(w^*) + \langle \nabla f(w^*), x - w^* \rangle + \frac{L}{2} \| x-w^* \|^2 - \left(f(w^*) + \frac{\kappa}{2} \| w^* - x \|^2 \right) \\
  & =  \frac{L+\kappa}{2} \| w^* - x \|^2 \\
  & =  \frac{L+\kappa}{2\kappa^2} \| \nabla F(x) \|^2. 
   \end{split}
\end{equation*}
\end{proof}

The inequality in the above proof 
relies on the smoothness of~$f$, which does not hold for composite problems.
The next lemma addresses this issue.

\begin{lemma}[\bfseries Warm start for primal methods - composite case]\label{lemma:restart composite}
   Consider the composite optimization problem $f= f_0 +\psi$, where $f_0$ is $L$-smooth.  By initializing  with 
   \begin{equation}
      w_0 = \argmin_{w \in \R^d} \left \{ f_0(x)+ \langle \nabla f_0(x),w-x \rangle + \frac{L+\kappa}{2} \Vert w-x \Vert^2 + \psi(w)  \right \}, \label{eq:restart2}
   \end{equation}
we have, $$ h(w_0) - h^* \leq \frac{L+\kappa}{2\kappa^2} \Vert \nabla F(x) \Vert^2 . $$
\end{lemma}

\begin{proof}
 We use the inequality corresponding to Lemma 2.3 in \cite{fista}: for any $w$,
 \begin{equation}\label{eqfista}
    h(w) - h(w_0) \geq \frac{L'}{2} \Vert w_0 - x \Vert^2 + L' \langle w_0 -x, x-w \rangle,
 \end{equation}
   with $L'=L+\kappa$. Then, we apply this inequality to $w = w^*$, and
 \begin{displaymath}
 h(w_0) -h^*  \leq - \frac{L'}{2} \Vert w_0 - x \Vert^2 - L' \langle w_0 -x,x-w^* \rangle
 \leq  \frac{L'}{2} \Vert x - w^* \Vert^2 
 =  \frac{L+\kappa}{2\kappa^2}\Vert \nabla F(x) \Vert^2.
 \end{displaymath}
\end{proof} 
We get an initialization of the same quality in the composite case as in the smooth case, by performing an additional proximal step. It is important to remark that the above analysis do not require strong convexity of $f$, which allows us to derive the desired inner-loop complexity.
\begin{proposition}[\bfseries Inner-loop complexity for Algorithm~\ref{alg:newtonizer}]\label{remarkepsilon1}
    Consider Algorithm~\ref{alg:newtonizer} with the warm start strategy described
   in Lemma~\ref{lemma:restart} or in Lemma~\ref{lemma:restart composite}. 
   Assume that the optimization method~$\mtd$ applied in the inner loop produces a sequence $(w_t)_{t \geq 0}$ for
   each sub-problem~(\ref{eq:sub prob}) such that
   \begin{equation}
      h(w_t) - h^\star \leq C_\mtd (1- \tau_{\mtd})^t (h(w_0)- h^\star)~~\text{for some constants}~~C_{\mtd}, \tau_{\mtd} > 0. \label{eq:primal2}
   \end{equation}
   Then, the stopping criterion $\varepsilon \leq \frac{1}{72\kappa} \Vert g \Vert^2$ is achieved 
   in at most $T_{\mtd}$ iterations with
   \begin{displaymath}
     T_{\mtd} = \frac{1}{\tau_{\mtd}} \log \left ( 38 C_\mtd \frac{L+\kappa}{\kappa}  \right ) .
   \end{displaymath}

\end{proposition}
\begin{proof}
   Consider at iteration $k$, we apply $\mtd$ to approximate the proximal mapping according to $x$. With the given $T_{\mtd}$ (which we abbreviate by $T$), we have 
   \begin{equation*}
   \begin{split}
   h(w_{T}) -h^* & \leq C_\mtd (1- \tau_{\mtd})^T (h(w_0)- h^\star) \\
		 & \leq C_\mtd e^{-\tau_{\mtd} T} (h(w_0)- h^\star) \\
		 & \leq C_\mtd e^{-\tau_{\mtd} T} \frac{L+\kappa}{2\kappa^2} \| \nabla F(x) \|^2 \quad (\text{By Lemma~\ref{lemma:restart} and Lemma~\ref{lemma:restart composite}}) \\
		 & = \frac{1}{76 \kappa} \| \nabla F(x) \|^2 \\
		 & \leq \frac{1}{36 \kappa} \| g \|^2,
   \end{split}
  \end{equation*}
  where the last inequality follows from Lemma~\ref{lemma:simple}.
\end{proof}

Next, we extend the previous result obtained with deterministic methods~$\mtd$ to randomized ones, where linear convergence is
only achieved in expectation. The proof is a simple application of Lemma~C.1 in \cite{catalyst} (see also~\cite{cartis2017global} for related results on the expected complexity of randomized algorithms).
\begin{rem}[\bf When $\mtd$ is non-deterministic] 
Assume that the optimization method~$\mtd$ applied to each sub-problem~(\ref{eq:sub prob}) produces a sequence $(w_t)_{t \geq 0}$ such that
   \begin{equation*}
     \E[ h(w_t) - h^\star] \leq C_\mtd (1- \tau_{\mtd})^t (h(w_0)- h^\star)~~\text{for some constants}~~C_{\mtd}, \tau_{\mtd} > 0. 
   \end{equation*}  
   We define the stopping time $T_\mtd$ by  
   \begin{equation}
   	T_\mtd = \inf \left \{ t\geq 1 \,\, | \,\, h(w_t)-h^*\leq \frac{1}{36\kappa} \Vert g_t \Vert^2 \right \}, \quad \text{where} \quad g_t = \kappa (x - w_t),
   \end{equation}
   which is the random variable corresponding to the minimum number of iterations to guarantee the stopping condition (\ref{eq:stop condition}). Then, when the warm start strategy described in Lemma~\ref{lemma:restart} or in Lemma~\ref{lemma:restart composite} is applied, the expected number of iterations satisfies
   \begin{equation}\label{eq:expectedT}
   	\E[T_\mtd] \leq \frac{1}{\tau_\mtd} \log \left ( 76 C_\mtd \frac{L+\kappa}{\tau_\mtd \kappa} \right ) +1.
   \end{equation}
   \end{rem}
   
   \begin{rem} [\bf Checking the stopping criterium]
   	It is worth to notice that the stopping criterium (\ref{eq:stop condition}), i.e. $h(w) - h^* \leq \frac{\kappa}{36} \Vert w-x \Vert^2$, can not be directly checked since $h^*$ is unknown. Nevertheless, an upper bound on the optimality gap $h(w) - h^*$ is usually available. In particular,
   	\begin{itemize}
   		\item When $f$ is smooth, which implies $h$ is smooth, we have 
   		\begin{equation}
   			h(w) - h^* \leq \frac{1}{2(\mu+\kappa)}\Vert \nabla h(w) \Vert^2.
   		\end{equation}
   		\item Otherwise, we can evaluate the Fenchel conjugate function, which is a natural lower bound of $h^*$, see Section D.2.3 in~\cite{mairal2010sparse}.  
   	\end{itemize}
   \end{rem}

\subsection{Global complexity of~\qning}\label{subsec:global}
Finally, we can use the previous results to upper-bound the complexity of the
\qningsp algorithm in terms of iterations of the method~$\mtd$ for minimizing~$f$ up to~$\varepsilon$.

\begin{proposition}[\bfseries Worst-case global complexity for Algorithm~\ref{alg:newtonizer}]\label{global complexity}
Given a linearly-convergent  method~$\mtd$ satisfying (\ref{eq:assumption}), we apply $\mtd$ to solve the sub-problems of Algorithm~\ref{alg:newtonizer} with the warm start strategy given in Lemma~\ref{lemma:restart} or Lemma~\ref{lemma:restart composite} up to accuracy $\varepsilon_k \leq \frac{1}{36\kappa} \Vert g_k \Vert^2$. Then, the number of iterations of the method~$\mtd$ 
   to guarantee the optimality condition $f(z_k)-f^\star \leq \varepsilon$ is
\begin{itemize}
 \item for $\mu$-strongly-convex problems:
 \begin{equation*}
   O \left( T_{\mtd} \times \frac{\mu+\kappa}{\mu} \log \left ( \frac{f(x_0)-f^*}{\varepsilon} \right) \right ) = O \left( \frac{\mu+\kappa}{\tau_\mtd \mu } \log \left ( \frac{f(x_0)-f^*}{\varepsilon} \right)  \log \left ( 38 C_\mtd \frac{L+\kappa}{\kappa}  \right )  \right ).
 \end{equation*}
 \item for convex problems with bounded level sets:
  \begin{equation*}
   O \left ( T_{\mtd} \times \frac{2\kappa R^2}{\varepsilon} \right ) = O\left ( \frac{2\kappa R^2}{\tau_\mtd \varepsilon} \log \left ( 38 C_\mtd \frac{L+\kappa}{\kappa}  \right ) \right ) .
 \end{equation*}
\end{itemize}
\end{proposition}
\begin{proof}
The total number of calls of method $\mtd$ is simply $T_{\mtd}$ times the
   number of outer-loop iterations times the potential number of line search steps at each iteration (which is hidden in the $O(.)$ notation since this number can be made arbitrarily small).
\end{proof}

\begin{rem}
	For non-deterministic methods, applying (\ref{eq:expectedT}) yields a  global complexity in expectation similar to the previous result with additional constant $2/\tau_\mtd$ in the last $\log$ factor.
\end{rem}

As we shall see, the global complexity of our algorithm is mainly controlled by the smoothing parameter~$\kappa$. Unfortunately, under the current analysis, our algorithm \qningsp does not lead to an improved convergence rate in terms of the worst-case complexity bounds. It is worthwhile to underline, though, that this result is not
surprising since it is often the case for L-BFGS-type methods, for which an
important gap remains between theory and practice.  Indeed, 
L-BFGS often outperforms the vanilla gradient descent method in many practical cases, but never in theory, which turns out to be the bottleneck in our analysis. 

We give below the worst-case global complexity of \qningsp when applied to two optimization methods~$\mtd$ of interest. Proposition~\ref{global complexity} and its application to the two examples show that, in terms of worse-case complexity, the \qningsp scheme leaves
the convergence rate almost unchanged. 

\begin{example}
Consider gradient descent with fixed constant step-size $1/L$ as the optimization method $\mtd$. 
Directly applying gradient descent (GD) to minimize $f$ requires 
$$~O(L/\mu \log(1/\epsilon))~$$ 
iterations to  achieve $\varepsilon$ accuracy. 
The complexity to achieve the same accuracy with \qning-GD is in the worst case
$$~\tilde{O}((L+\kappa)/\mu \log(1/\epsilon)).$$ 
\end{example}

\begin{example}
Consider the stochastic variance-reduced gradient (SVRG) as the optimization method $\mtd$. 
SVRG minimizes $f$ to $\varepsilon$ accuracy in 
\begin{equation*}
O \left (\max \left \{n, \frac{L}{\mu} \right \} \log \left ( \frac{1}{\varepsilon} \right ) \right ) 
\end{equation*}
iterations in expectation. 
\qning-SVRG achieves the same result with the worst-case expected complexity
\begin{equation*}
\tilde{O} \left (\max \left \{ \frac{\mu+\kappa}{\mu} n, \frac{L+\kappa}{\mu} \right \} \log \left ( \frac{1}{\varepsilon} \right ) \right ).\; 
\end{equation*}
\end{example}

\paragraph{Choice of $\bm \kappa$.}
Minimizing the above worst-case complexity respect to $\kappa$ suggests that $\kappa$ should be chosen as small as possible. However, such statement is based on the pessimistic theoretical analysis of L-BFGS-type method, which is not better than standard gradient descent methods. Noting that for smooth functions, L-BFGS method often outperforms Nesterov's accelerated gradient method, it is reasonable to expect they achieve a similar complexity bound. In other words, the choice of $\kappa$ may be substantially different if one is able to show that L-BFGS-type method enjoys an accelerated convergence rate.  

In order to illustrate the difference, we heuristically assume that L-BFGS method enjoys a similar convergence rarte as Nesterov's accelerated gradient method. Then, the global complexity of our algorithm \qningsp matches the complexity of the related Catalyst acceleration scheme~\cite{catalyst}, which will be
 \begin{equation*}
   \tilde{O} \left( \frac{1}{\tau_\mtd} \sqrt{\frac{\mu+\kappa}{\mu}} \log \left ( \frac{1}{\varepsilon} \right) \right ) ,
 \end{equation*} 
for $\mu$-strongly-convex problems. In such case, the complexity of \qning-GD and \qning-SVRG will be 
$$~\tilde{O}\left(\frac{L+\kappa}{\sqrt{(\mu+\kappa) \mu}}\log(1/\epsilon) \right ) \quad \text{ and } \quad  \tilde{O} \left (\max \left \{ \sqrt{\frac{\mu+\kappa}{\mu}} n, \frac{L+\kappa}{\sqrt{(\mu+\kappa)\mu}} \right \} \log \left ( \frac{1}{\varepsilon} \right ) \right ),$$ 
which do enjoy acceleration by taking $\kappa = O(L)$ and $\kappa=O(L/n)$ respectively. In the following section, we will experiment with this heuristic, as if L-BFGS method enjoys an accelerated convergence rate. More precisely, we will choose the smoothing parameter $\kappa$ as in the related Catalyst acceleration scheme~\cite{catalyst}, and we present empirical evidence in support of this heuristic.

\section{Experiments and practical details}\label{sec:exp}
In this section, we present the experimental results obtained by applying
\qningsp to several first-order optimization algorithms. We start the section by presenting various benchmarks and practical parameter-tuning choices.
 Then, we study the performance of \qningsp applied to SVRG (Section~\ref{subsec:svrg}) and to the proximal gradient algorithm ISTA (Section~\ref{subsec:gd}), which reduces to gradient descent (GD) in the smooth case. We demonstrate that \qningsp can be viewed as an acceleration scheme: by applying \qning~to an optimization algorithm $\mtd$, we achieve better performance than when applying $\mtd$ directly to the problem. Besides, we also compare \qningsp to existing stochastic variants of L-BFGS algorithm in Section \ref{subsec:svrg}. Finally, we study the behavior of \qningsp under different choice of parameters in 
Section~\ref{subsec:param}. The code used for
      all the experiments is available at
      \url{https://github.com/hongzhoulin89/Catalyst-QNing/}.

\subsection{Formulations and datasets}\label{subsec:data}
We consider three common optimization problems in machine learning
and signal processing, including logistic regression, Lasso and linear regression with Elastic-Net regularization. These three formulations all admit the composite finite-sum structure but differ in terms of smoothness and strong-convexness. 
Specifically, the three formulations are listed below.
\begin{itemize}
   \item {\bfseries $\ell_2^2$-regularized Logistic Regression}:
    \begin{equation*}
    \min_{x \in \R^d} \quad \frac{1}{n} \sum_{i=1}^n \log\left(1+\exp(-b_i \,a_i^{T} x)\right) + \frac{\mu}{2} \Vert x \Vert^2,
    \end{equation*}
      which leads to a $\mu$-strongly convex smooth optimization problem.   \item {\bfseries $\ell_1$-regularized Linear Regression (LASSO)}: 
  \begin{equation*}
    \min_{x \in \R^d} \quad \frac{1}{2n} \sum_{i=1}^n   ( b_i -a_i^{T}x)^2 + \lambda \Vert x \Vert_1,
 \end{equation*}
      which is convex and non-smooth, but not strongly convex.
  \item {\bfseries $\ell_1-\ell_2^2$-regularized Linear Regression (Elastic-Net)}:
 \begin{equation*}
    \min_{x \in \R^d} \quad \frac{1}{2n} \sum_{i=1}^n ( b_i -a_i^{T}x )^2 + \lambda \Vert x \Vert_1+  \frac{\mu}{2} \Vert x \Vert^2,
 \end{equation*}
      which is based on the Elastic-Net regularization~\cite{zou2005regularization} leading to non-smooth strongly-convex problems.
\end{itemize}

For each formulation, 
we consider a training set $(a_i,b_i)_{i=1}^n$ of $n$ data points, where
the $b_i$'s are scalars in $\{-1,+1\}$ and the~$a_i$'s are feature vectors in~$\Real^d$.
Then, the goal is to fit a linear model $x$ in~$\Real^d$ such that the scalar $b_i$ can be well
predicted by the inner-product $ a_i^\top x$, or by its sign. Since we normalize the feature vectors~$a_i$, a natural upper-bound on the
Lipschitz constant~$L$ of the unregularized objective can be easily obtained
with $L_{\text{logistic}} = 1/4$, $L_{\text{elastic-net}} =1$ and
$L_{\text{lasso}} =1$.  

In the experiments, we consider relatively
ill-conditioned problems with the regularization parameter $\mu=1/(100n)$.
The $\ell_1$-regularization parameter is set to $\lambda = 1/n$ for the
Elastic-Net formulation; for the Lasso problem, we consider a  
logarithmic grid $10^i/n$, with $i=-3,-2,\ldots,3$, and we select the
parameter~$\lambda$ that provides a sparse optimal solution closest to $10\%$
non-zero coefficients.

\paragraph{Datasets.}
We consider five standard machine learning datasets with different
characteristics in terms of size and dimension, which are described below:
\vspace*{0.1cm}
 \begin{center}
    \begin{tabular}{|l|c|c|c|c|c|c|}
       \hline
       name & \textsf{covtype} & \textsf{alpha} & \textsf{real-sim}  & \textsf{MNIST-CKN} & \textsf{CIFAR-CKN} \\
       \hline
       $n$ & $581\,012$ & $250\,000$ & $72\,309$ & $60\,000$ & $50\,000$ \\
       \hline
       $d$ & $54$ & $500$ & $20\,958$ & $2\,304$ & $9\,216$ \\
       \hline
       \end{tabular}
 \end{center}
\vspace*{0.1cm}
The first three data sets are standard machine learning data sets from LIBSVM \cite{chang2011libsvm}. We normalize the features, which provides a natural estimate of the Lipschitz constant as mentioned previously. The last two data sets are coming from computer vision applications. MNIST and CIFAR-10 are two image classification data sets involving 10 classes. The feature representation of each image was computed  using an unsupervised convolutional kernel network~\cite{mairal2016end}. We focus here on the task of classifying class \#1 vs. other classes.

\subsection{Choice of hyper-parameters and variants}\label{subsec:variants}
We now discuss the choice of default parameters used in the experiments as
well as the different variants. First, to deal with the high-dimensional nature
of the data, we systematically use the L-BFGS metric~$H_k$ and maintain the
positive definiteness by skipping updates when necessary (see~\cite{friedlander2012hybrid}).

\paragraph{Choice of method~$\mtd$.} We apply \qningsp to 
proximal SVRG algorithm~\cite{proxsvrg} and proximal gradient algorithm. The proximal SVRG algorithm is an incremental algorithm that is able to exploit the finite-sum structure of the
objective and can deal with the composite regularization. We also consider the gradient descent algorithm and its proximal variant ISTA,
which allows us to perform a comparison with the natural baselines FISTA~\cite{fista} and L-BFGS.

\paragraph{Stopping criterion for the inner loop.} 
The default stopping criterion consists of solving each sub-problem with accuracy
$\varepsilon_k \leq \frac{1}{36} \Vert g_k \Vert^2$. 
Although we have shown that such accuracy is attainable in some constant $T=\tilde{O}(n)$ number of
iterations for SVRG with the choice $\kappa = L/2n$,
a natural heuristic proposed in Catalyst~\cite{catalyst2} consists of performing exactly
one pass over the data $T=n$ in the inner loop without checking any stopping
criterion. In particular, for gradient descent or ISTA, one pass over the data
means a single gradient step, because the evaluation of the full gradient
requires passing through the entire dataset.
When applying \qning~to SVRG and ISTA, we call the default algorithm using stopping criterion~(\ref{eq:condition_eps}) \textbf{\qning-SVRG}, \textbf{\qning-ISTA} and the one-pass variant
\textbf{\qning-SVRG1}, \textbf{\qning-ISTA1}, respectively. 

\paragraph{Choice of regularization parameter $\kappa$.}  We choose $\kappa$ as in the Catalyst algorithm~\cite{catalyst2}, which is $\kappa = L$ for gradient descent/ISTA and $\kappa = L/2n$ for SVRG. Indeed, convergence of L-BFGS is hard to characterize and its theoretical rate of convergence can be pessimistic as shown in our theoretical analysis. Noting that for smooth functions, L-BFGS often outperforms Nesterov's accelerated gradient method, it is reasonable to expect \qning~achieves a similar complexity bound as Catalyst. Later in Section~\ref{subsec:param}, we make a comparison between different values of $\kappa$ to demonstrate the effectiveness of this strategy.

\paragraph{Choice of limited memory parameter $l$.} The default setting is
$l=100$. We show later in Section~\ref{subsec:param} a comparison with
different values to study the influence of this parameter.

\paragraph{Implementation of the line search.} As mentioned earlier, we consider the stepsizes $\eta_k$ in the set $\{1,1/2,1/4,1/8,0\}$ and select the largest one that satisfies the descent condition.

\paragraph{Evaluation metric.} For all experiments, we use the number of gradient evaluations as a measure of complexity,
assuming this is the computational bottleneck of all methods considered. This
is indeed the case here since the L-BFGS step cost $O(l d)$ floating-point operations~\cite{nocedalbook},
whereas evaluating the gradient of the full objective costs $O(n d)$, with $l \ll n$. 

\subsection{\qning-SVRG for minimizing large sums of functions}\label{subsec:svrg}
We now apply \qning~to SVRG and compare different variants.
\begin{itemize}
   \item \textbf{SVRG}: the Prox-SVRG algorithm of \cite{proxsvrg} with default parameters $m=1$ and $\eta =1/L$, where~$L$ is the upper-bound on Lipschitz constant of the gradient, as described in the Section~\ref{subsec:data}.
   \item \textbf{Catalyst-SVRG}: The Catalyst meta-algorithm of~\cite{catalyst2} applied to Prox-SVRG, using the variant (C3) that performs best among the different variants of Catalyst.
   \item \textbf{L-BFGS/Orthant}: Since implementing effectively L-BFGS with a line-search algorithm is a bit involved, 
   we use the implementation of Mark Schmidt\footnote{available here \url{http://www.cs.ubc.ca/~schmidtm/Software/minFunc.html}}, 
   which has been widely used in other comparisons~\cite{sag}. In particular, the Orthant-wise method follows the algorithm developed in \cite{andrew2007scalable}.  We use L-BFGS for the logistic regression experiment and the Orthant-wise method \cite{andrew2007scalable} for elastic-net and lasso experiments. The limited memory parameter $l$ is  set to $100$.

\item \textbf{\qning-SVRG}: the algorithm according to the theory by solving the sub-problems until $\varepsilon_k \leq \frac{1}{36} \Vert g_k \Vert^2$.
\item \textbf{\qning-SVRG1}: the one-pass heuristic. 
\end{itemize}

\begin{figure}[hbtp!]
   \centering
   ~~\includegraphics[width=0.30\linewidth]{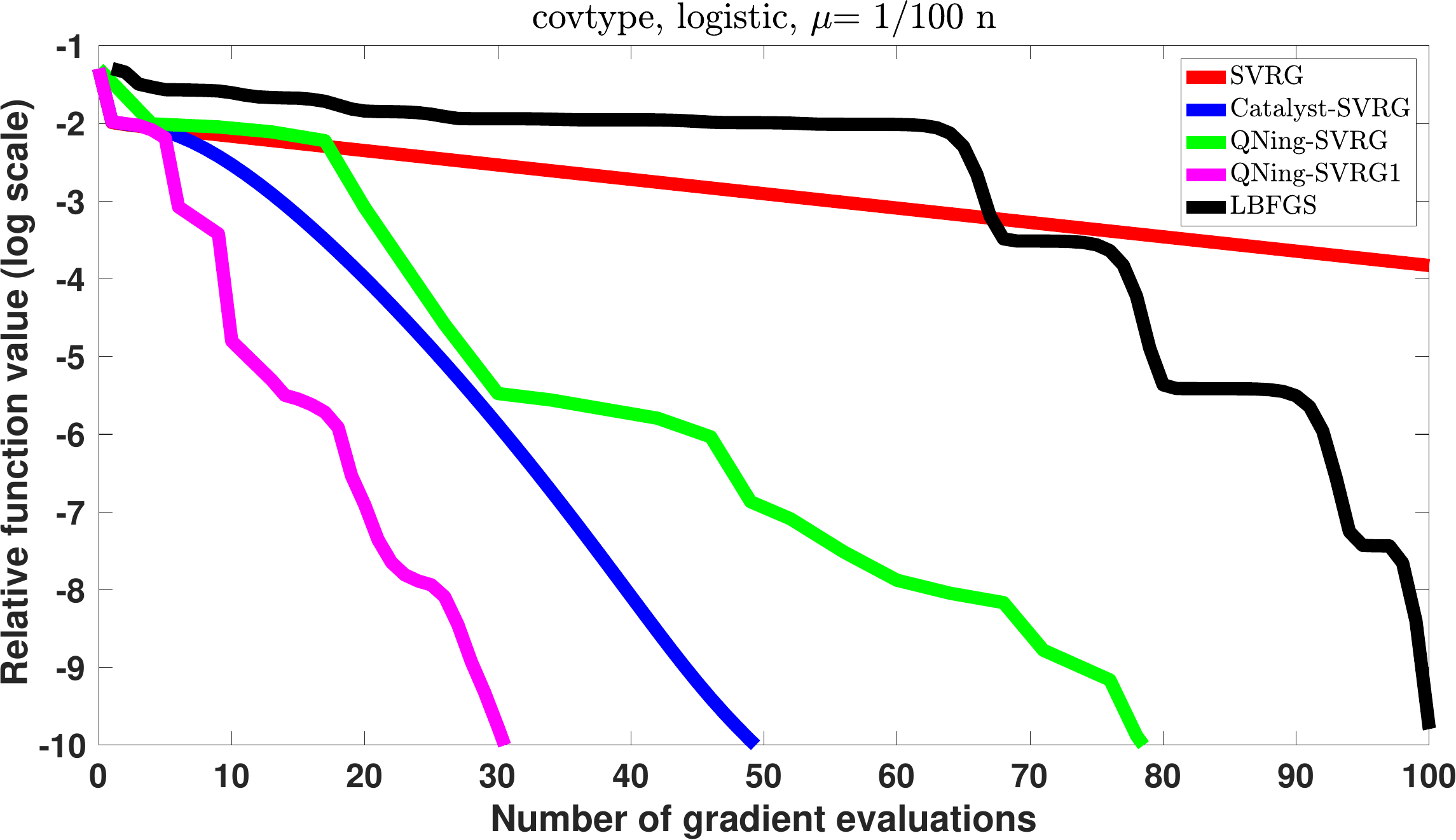} ~ 
   ~~\includegraphics[width=0.30\linewidth]{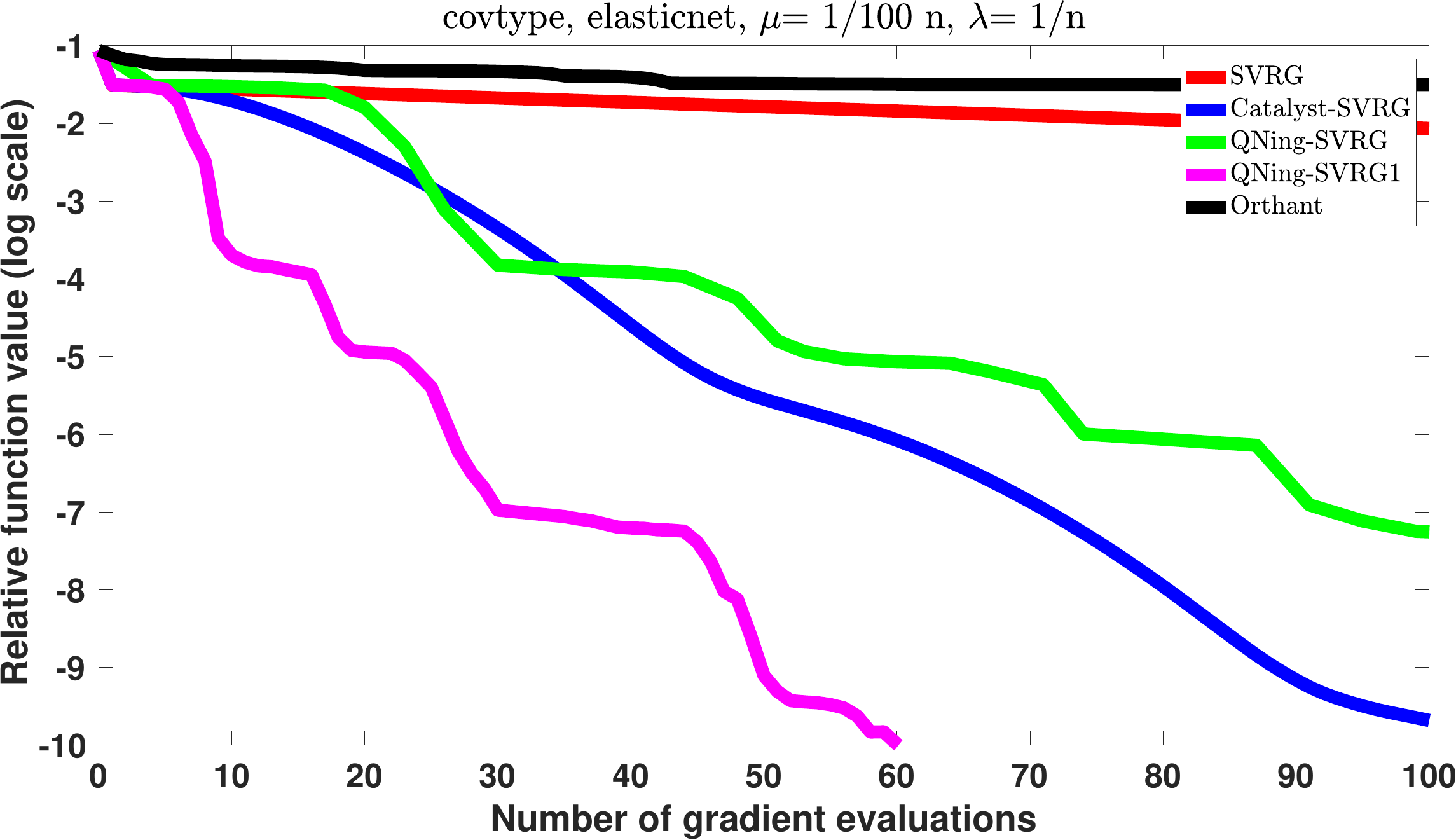} ~ 
   ~~\includegraphics[width=0.30\linewidth]{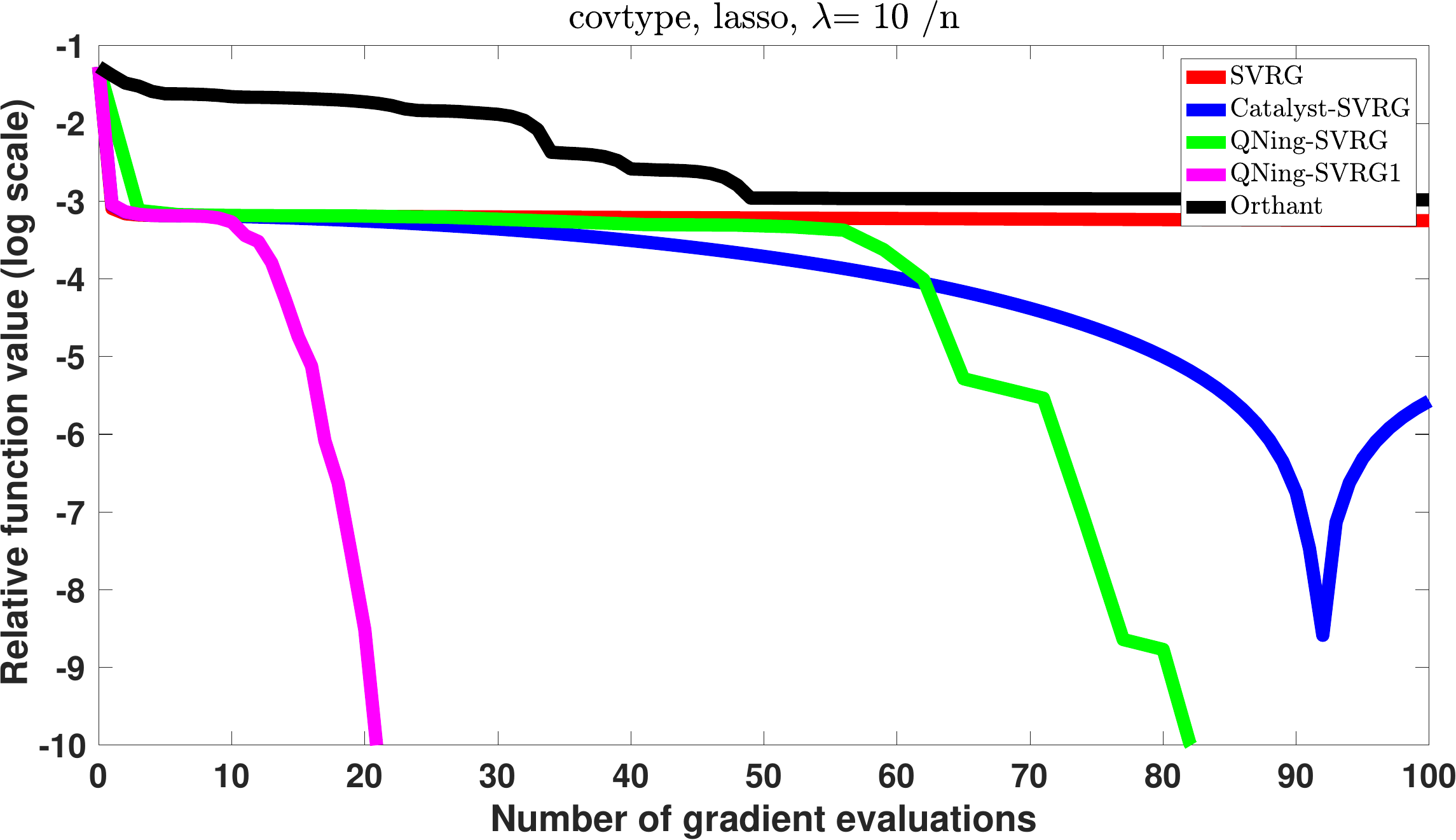} \\
   ~~\includegraphics[width=0.30\linewidth]{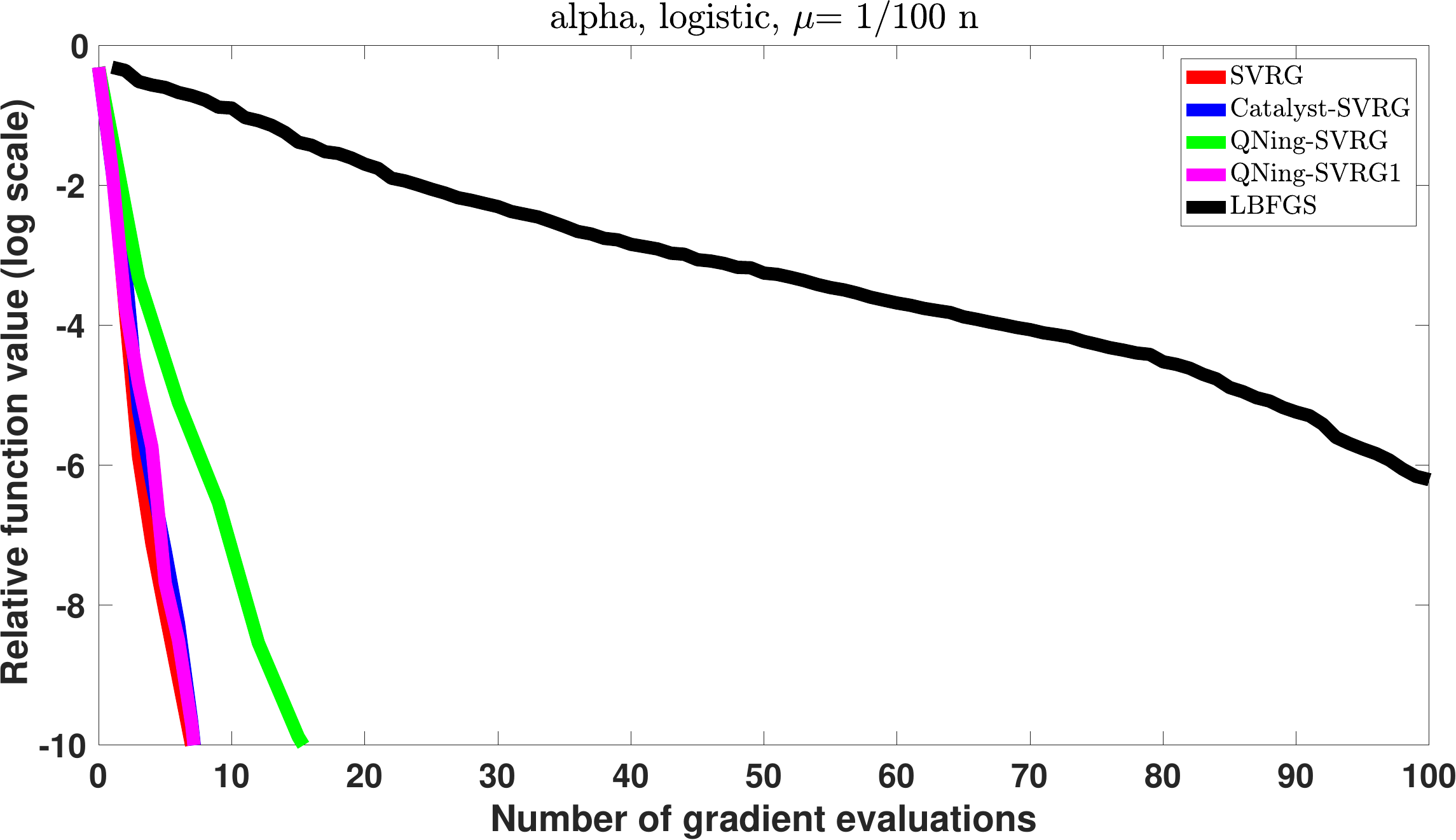} ~ 
   ~~\includegraphics[width=0.30\linewidth]{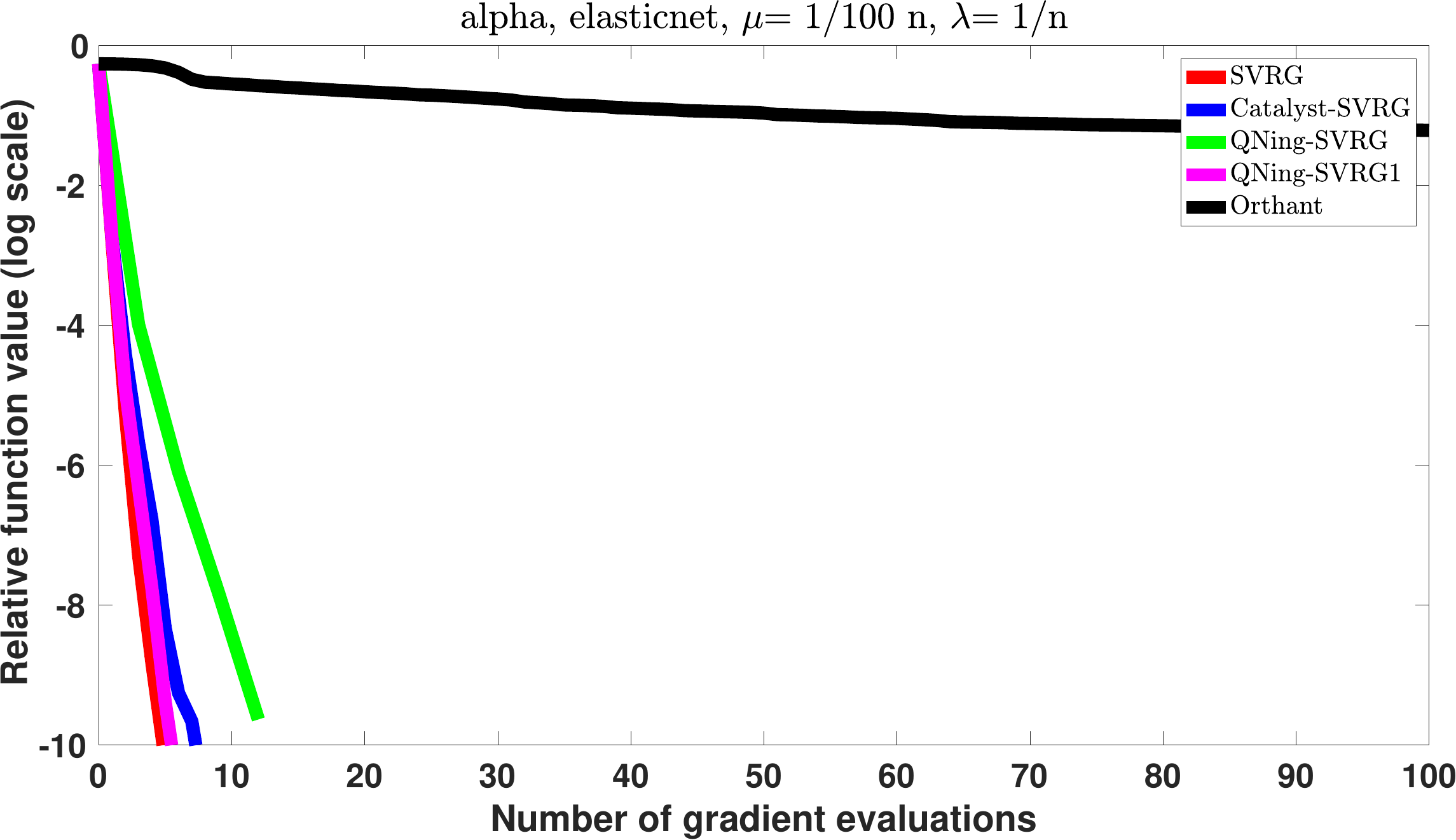} ~ 
   ~~\includegraphics[width=0.30\linewidth]{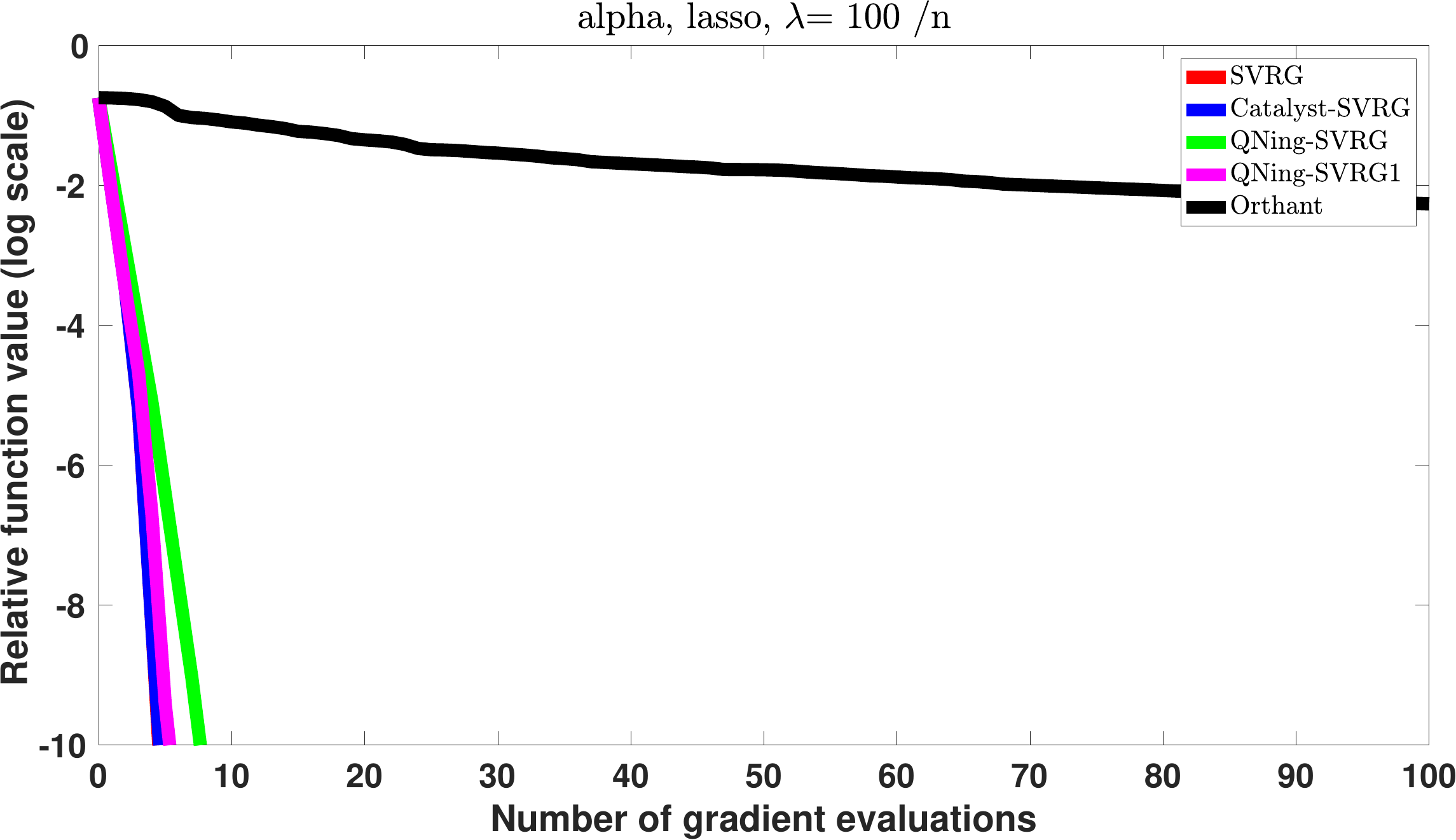}\\
   ~~\includegraphics[width=0.30\linewidth]{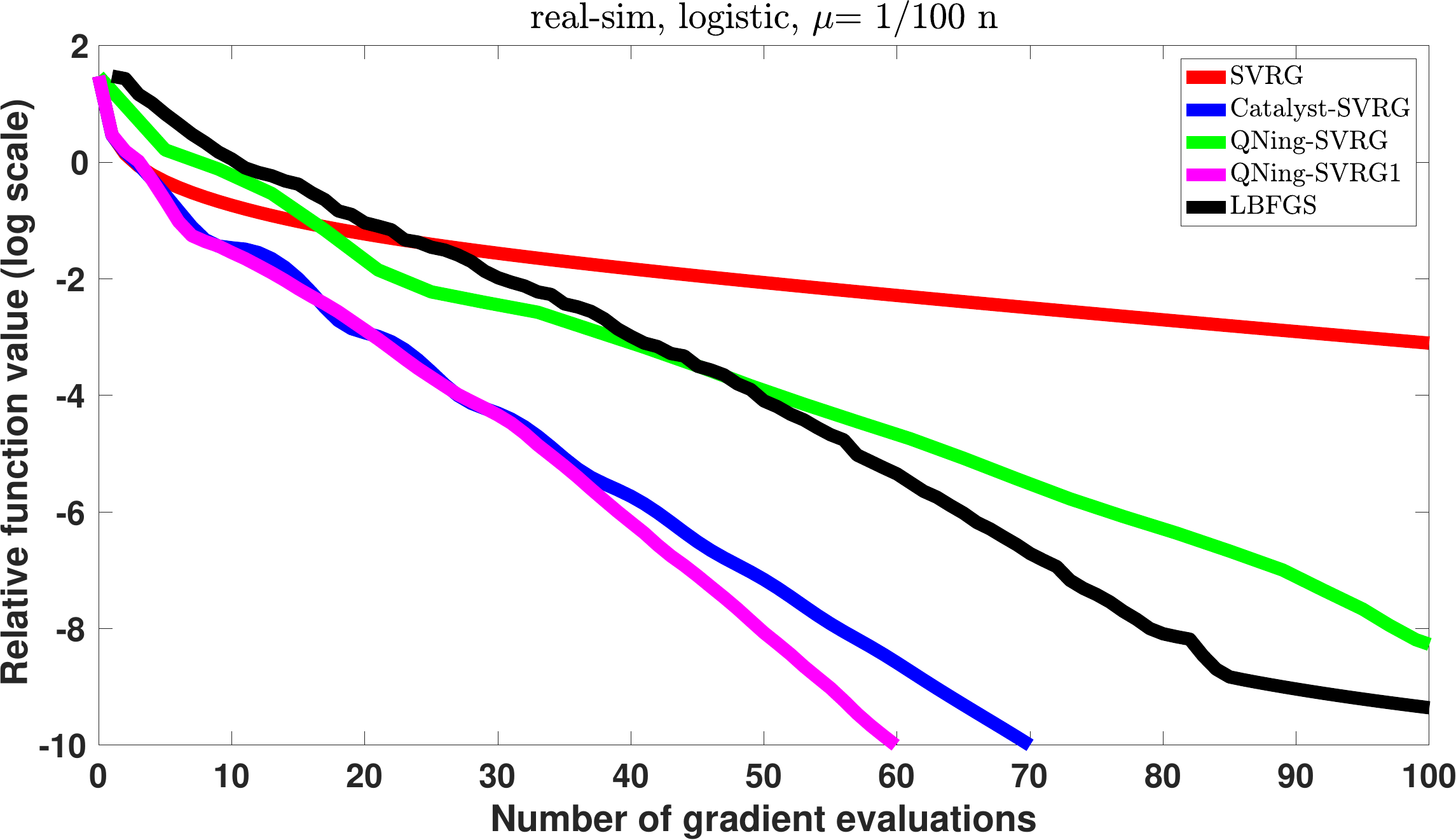} ~ 
   ~~\includegraphics[width=0.30\linewidth]{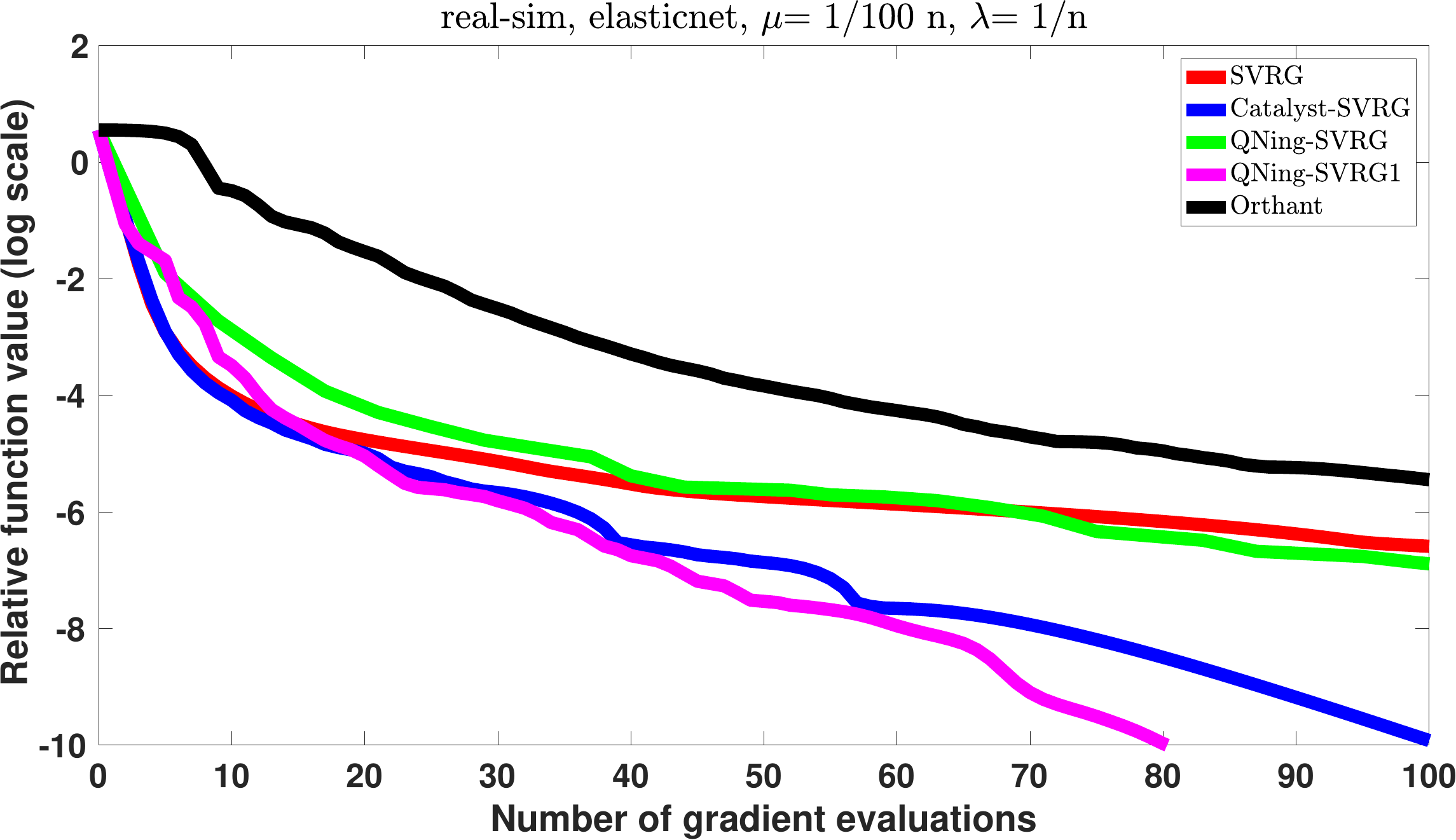} ~ 
   ~~\includegraphics[width=0.30\linewidth]{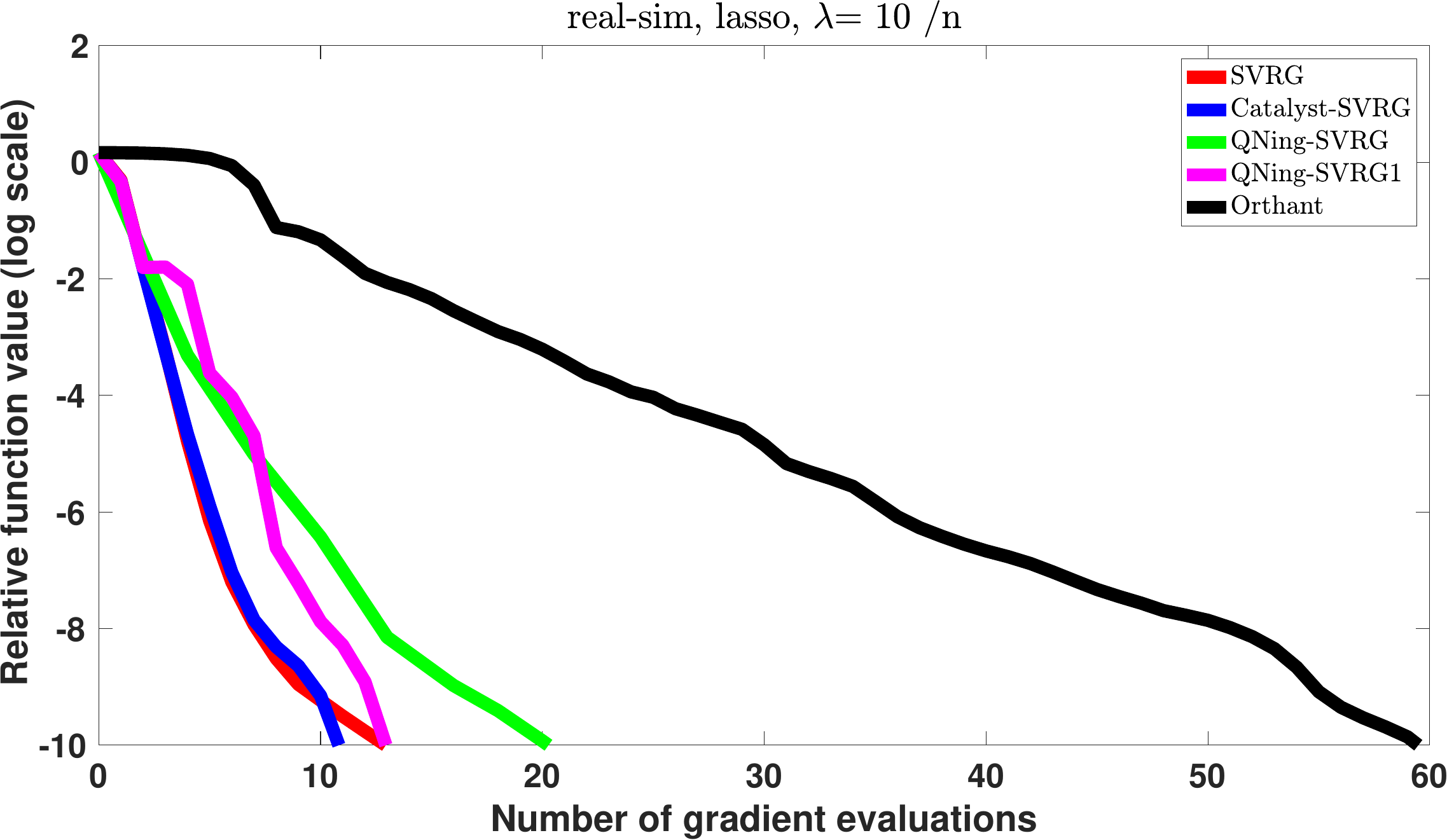} \\ 
   ~~\includegraphics[width=0.30\linewidth]{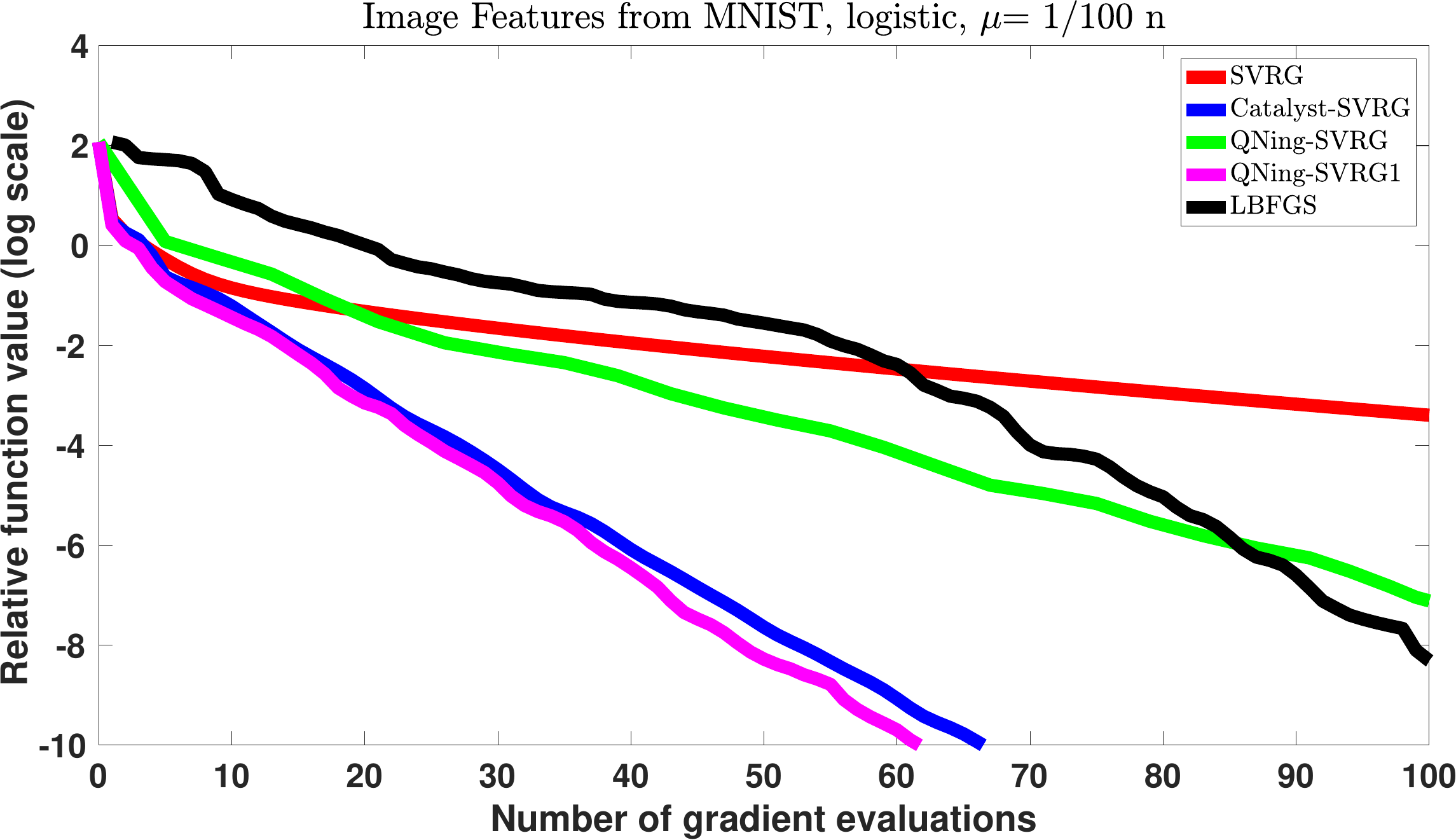} ~ 
   ~~\includegraphics[width=0.30\linewidth]{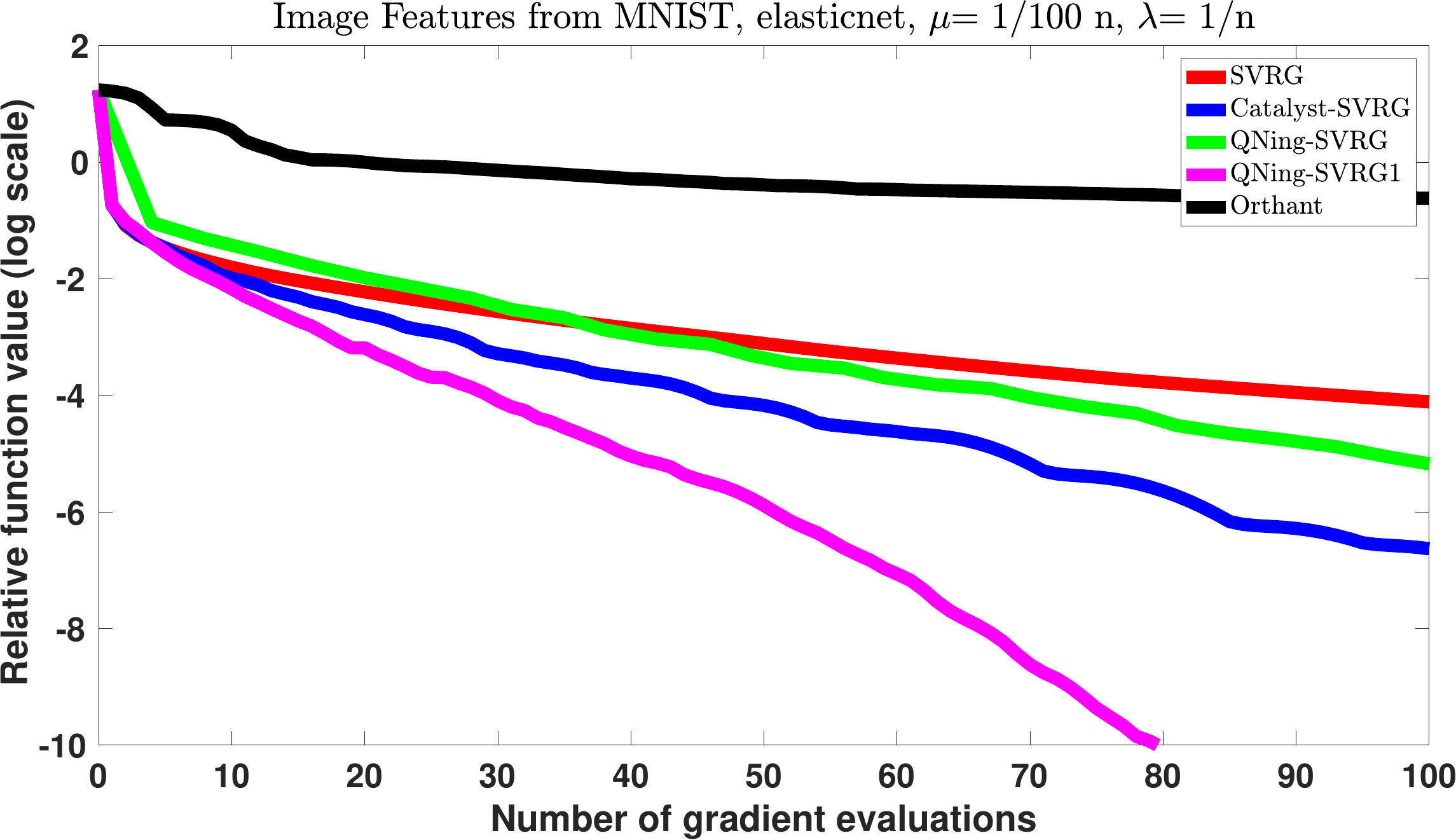} ~ 
   ~~\includegraphics[width=0.30\linewidth]{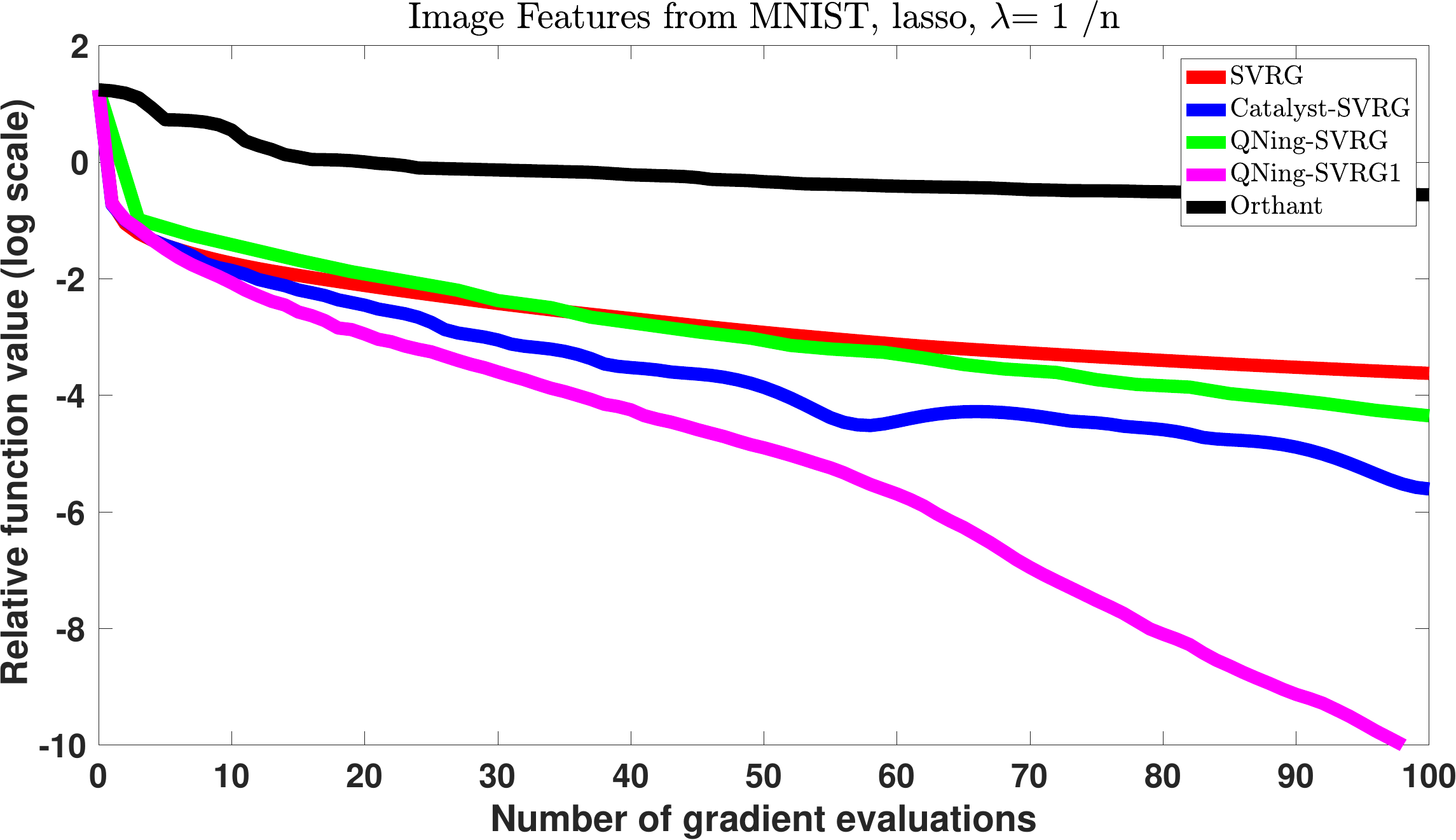} \\
   ~~\includegraphics[width=0.30\linewidth]{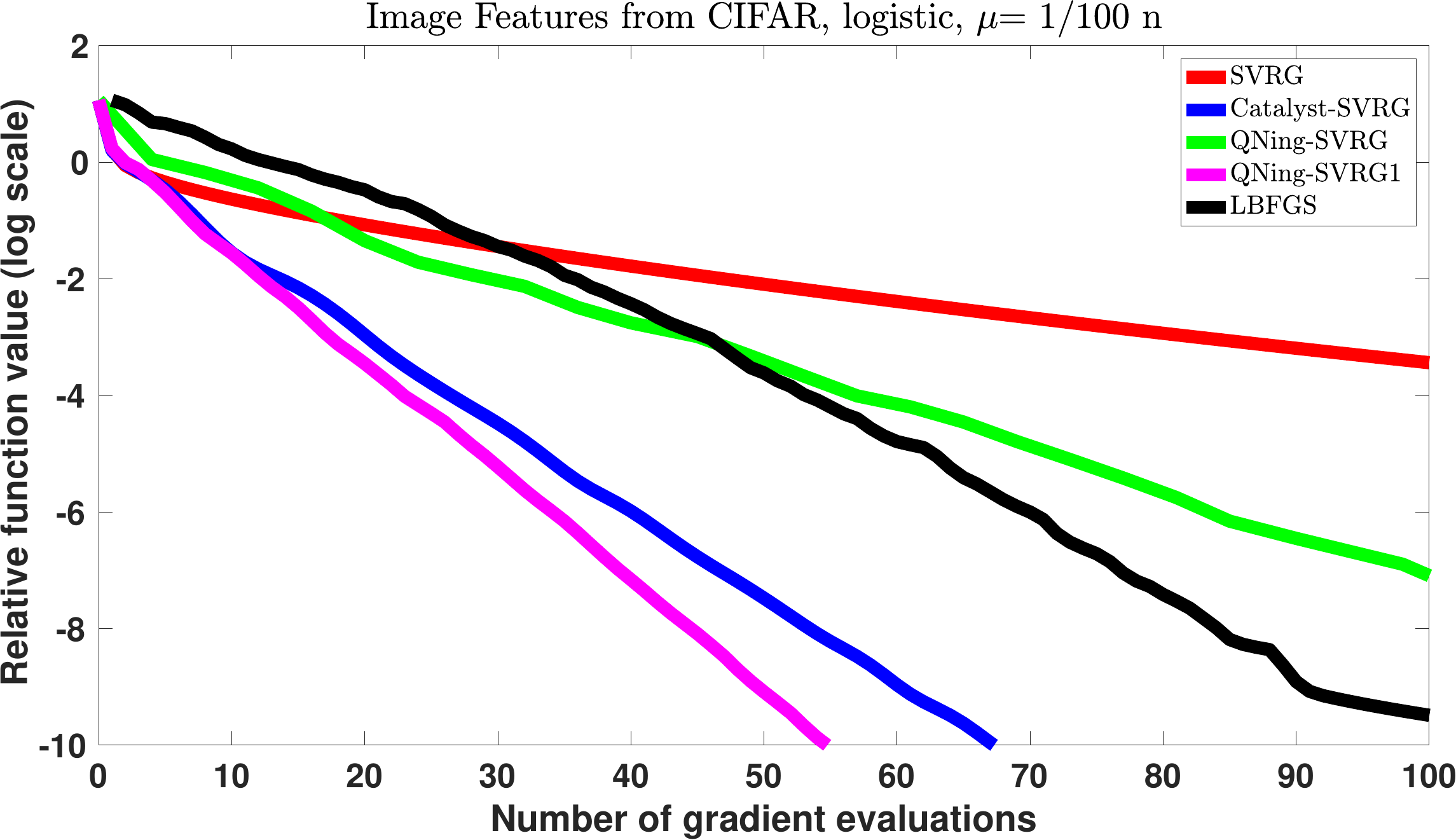} ~ 
   ~~\includegraphics[width=0.30\linewidth]{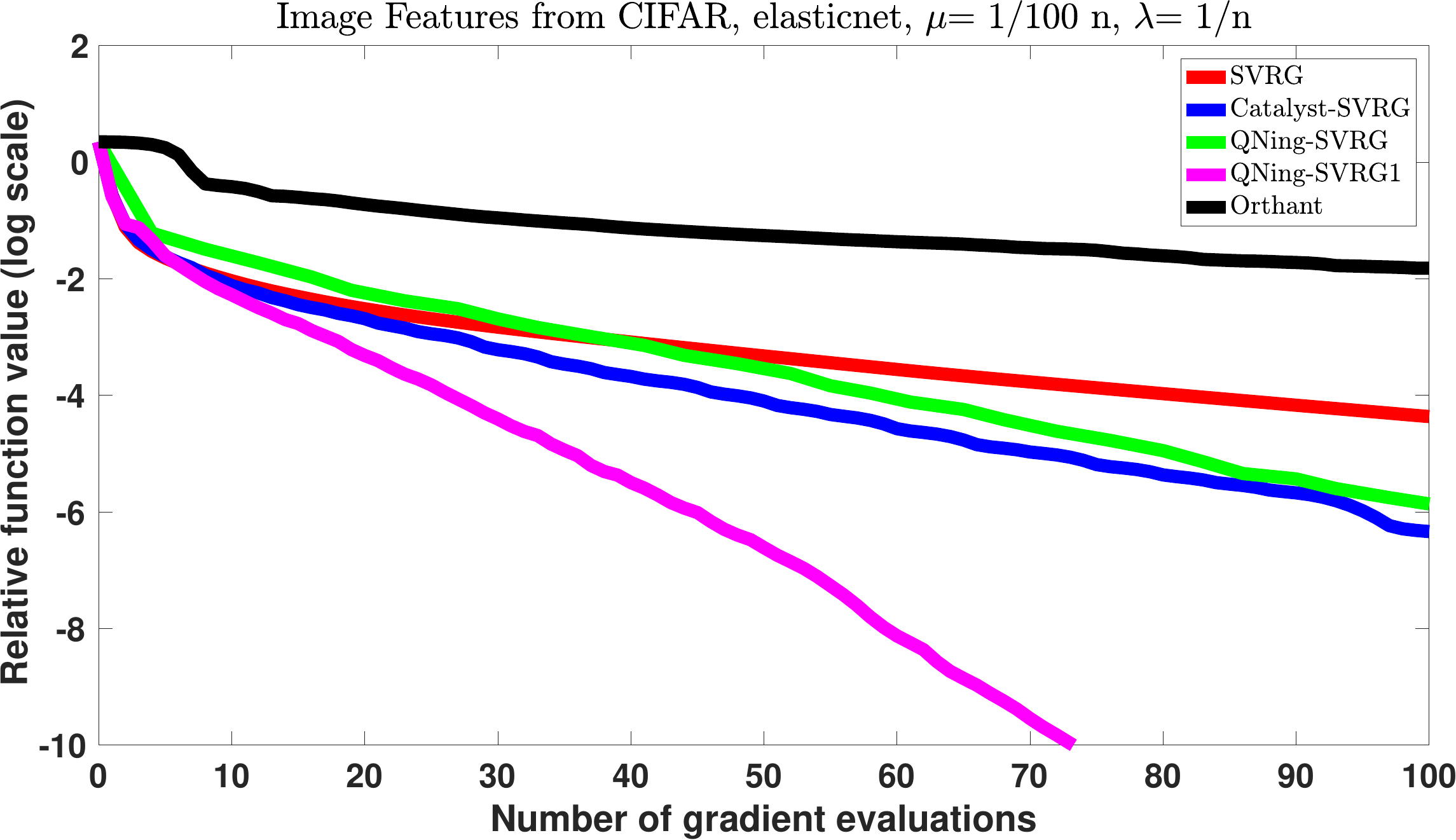} ~ 
   ~~\includegraphics[width=0.30\linewidth]{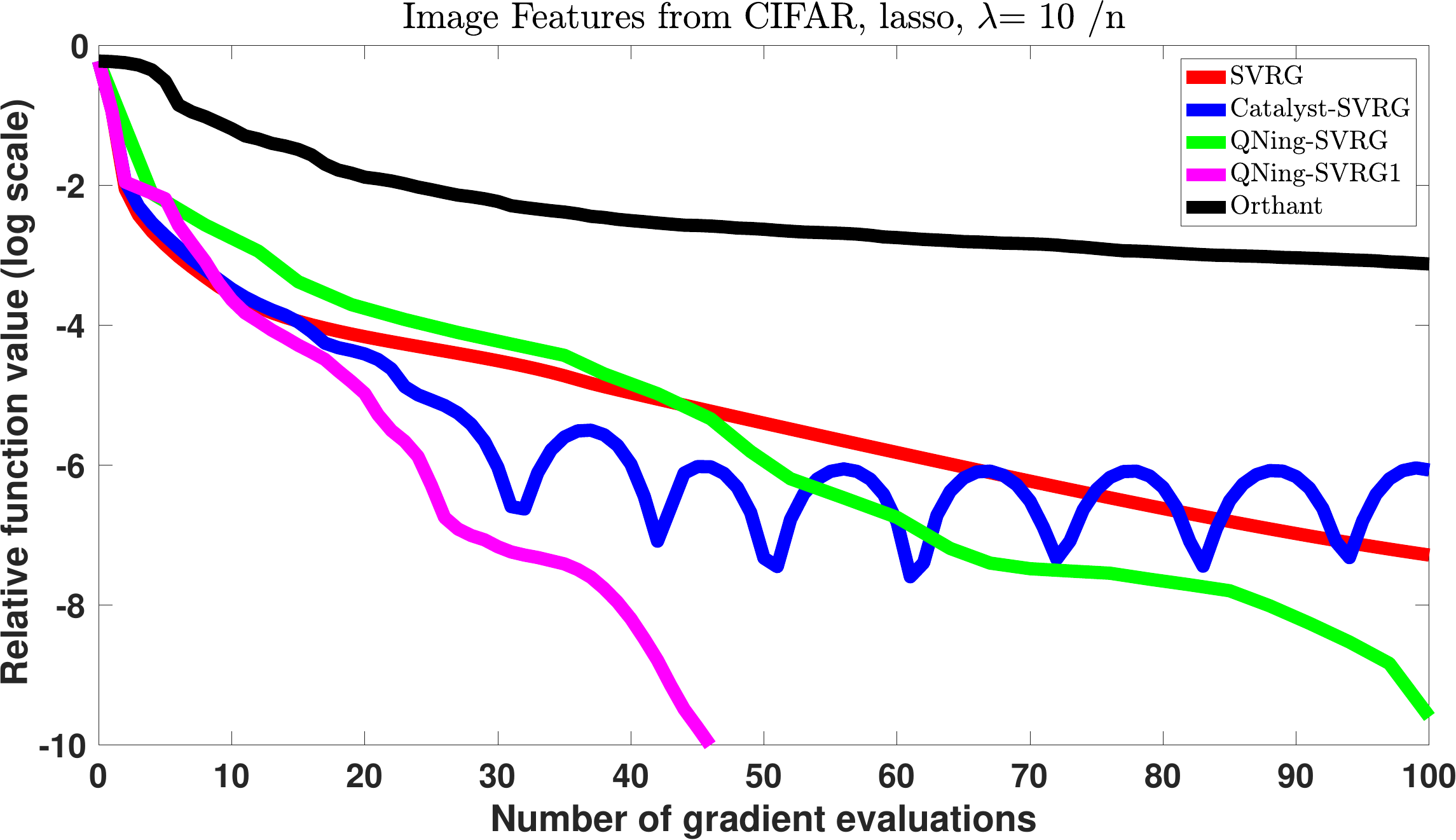} \\
   \caption{Experimental study of the performance of \qning-SVRG for
   minimizing large sums of functions. We plot the value~$F(x_k)/F^\star-1$ as
   a function of the number of gradient evaluations, on a logarithmic scale;
   the optimal value $F^\star$ is estimated with a duality gap.  }\label{fig:svrg}
\end{figure}

The result of the comparison is presented in Figure~\ref{fig:svrg} and
leads to the conclusions below, showing that \textbf{\qning-SVRG1} is a
safe heuristic, which never decreases the speed of the method \textbf{SVRG}:
\begin{itemize}
   \item \textbf{L-BFGS/Orthant} is less competitive than other approaches that
      exploit the sum structure of the objective, except on the dataset
      \textsf{real-sim}; the difference in performance with the SVRG-based
      approaches can be important (see dataset \textsf{alpha}).
   \item \textbf{\qning-SVRG1} is significantly faster than or on par with \textbf{SVRG} and \textbf{\qning-SVRG}. 
   \item \textbf{\qning-SVRG} is significantly faster than, or on par with, or only slightly slower than \textbf{SVRG}. 
   \item \textbf{\qning-SVRG1} is significantly faster, or on par with \textbf{Catalyst-SVRG}. This justifies our choice of~$\kappa$ which assumes ``a priori'' that L-BFGS performs as well as Nesterov's method.
\end{itemize}

So far, we have shown that applying QNing with SVRG provides a significant speedup compared to the original SVRG algorithm or other acceleration scheme such as Catalyst. Now we compare our algorithm to other variable metric approaches including Proximal L-BFGS \cite{lee2012proximal} and Stochastic L-BFGS \cite{slbfgs}:

\begin{itemize}
	\item \textbf{Proximal L-BFGS:} We apply the Matlab package PNOPT\footnote{available here \url{https://web.stanford.edu/group/SOL/software/pnopt}} implemented by \cite{lee2012proximal}. The sub-problems are solved by the default algorithm up to desired accuracy. We consider one sub-problem as one gradient evaluation in our plot, even though it often requires multiple passes.    
	\item \textbf{Stochastic L-BFGS} (for smooth objectives): We apply the Matlab package StochBFGS\footnote{available here \url{https://perso.telecom-paristech.fr/rgower/software.html}} implemented by \cite{slbfgs}. We consider the 'prev' variant which has the best practical performance.  
\end{itemize}

\begin{figure}[hbtp!]
\centering
   ~~\includegraphics[width=0.30\linewidth]{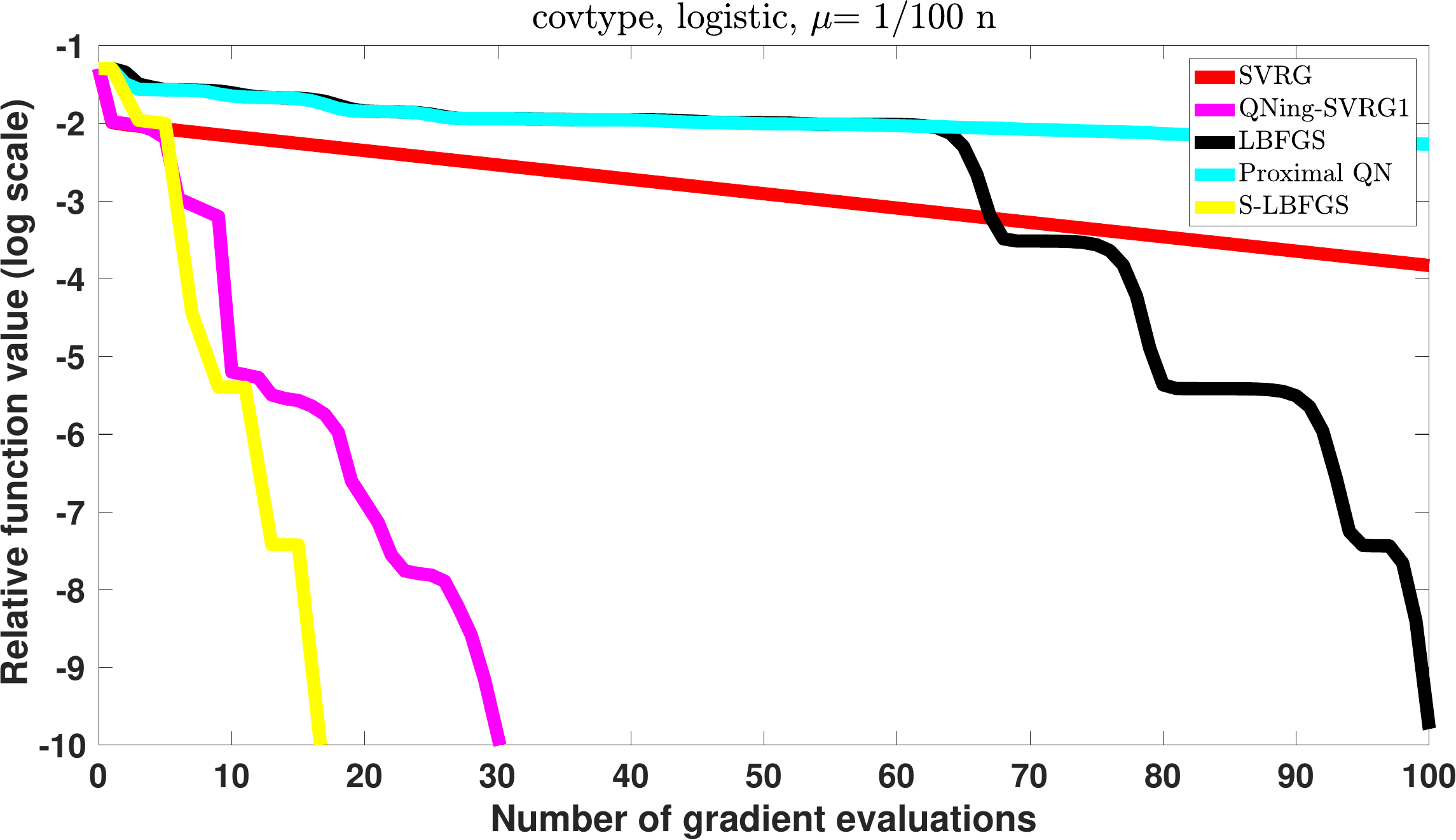} ~ 
   ~~\includegraphics[width=0.30\linewidth]{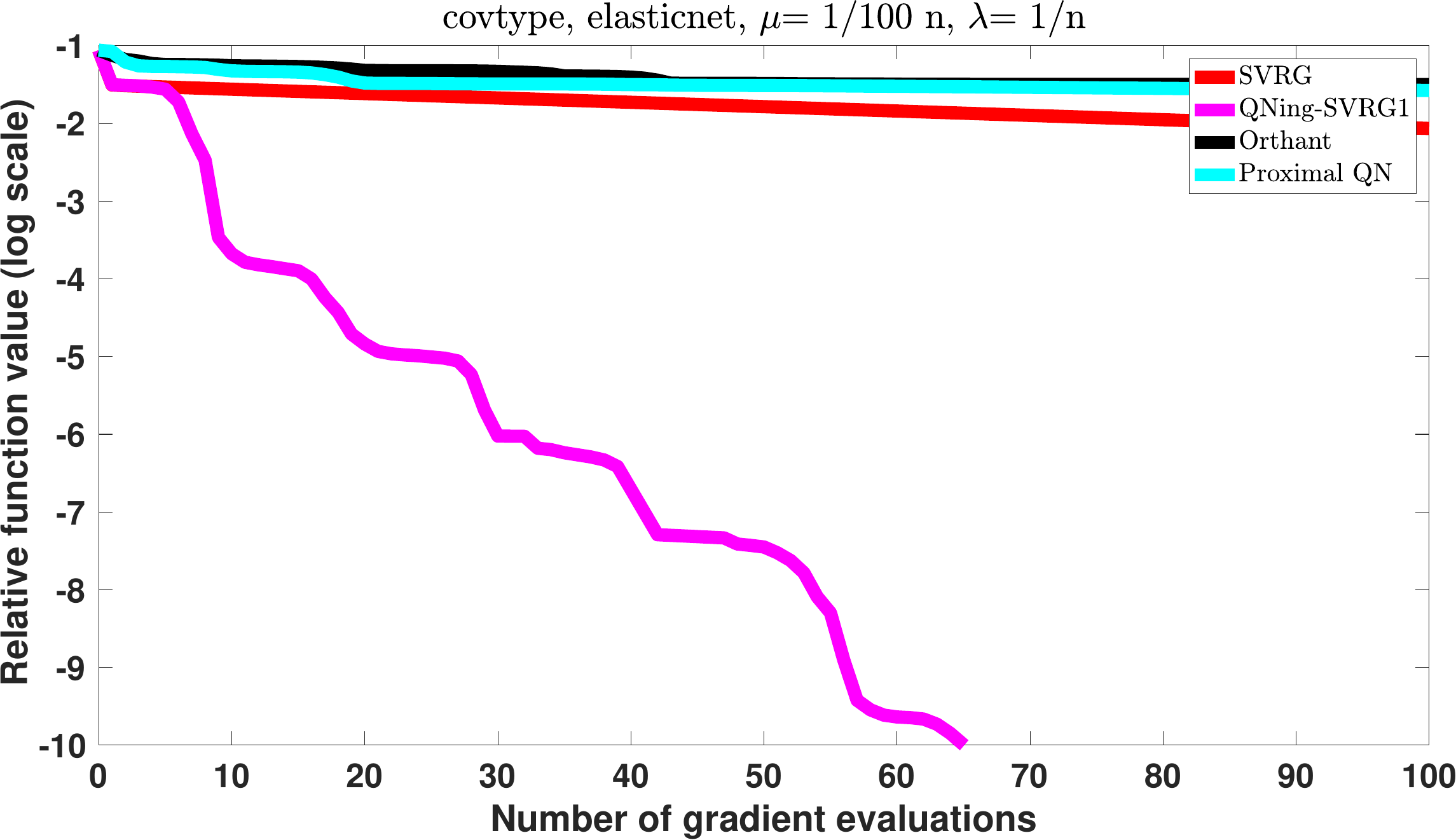} ~ 
   ~~\includegraphics[width=0.30\linewidth]{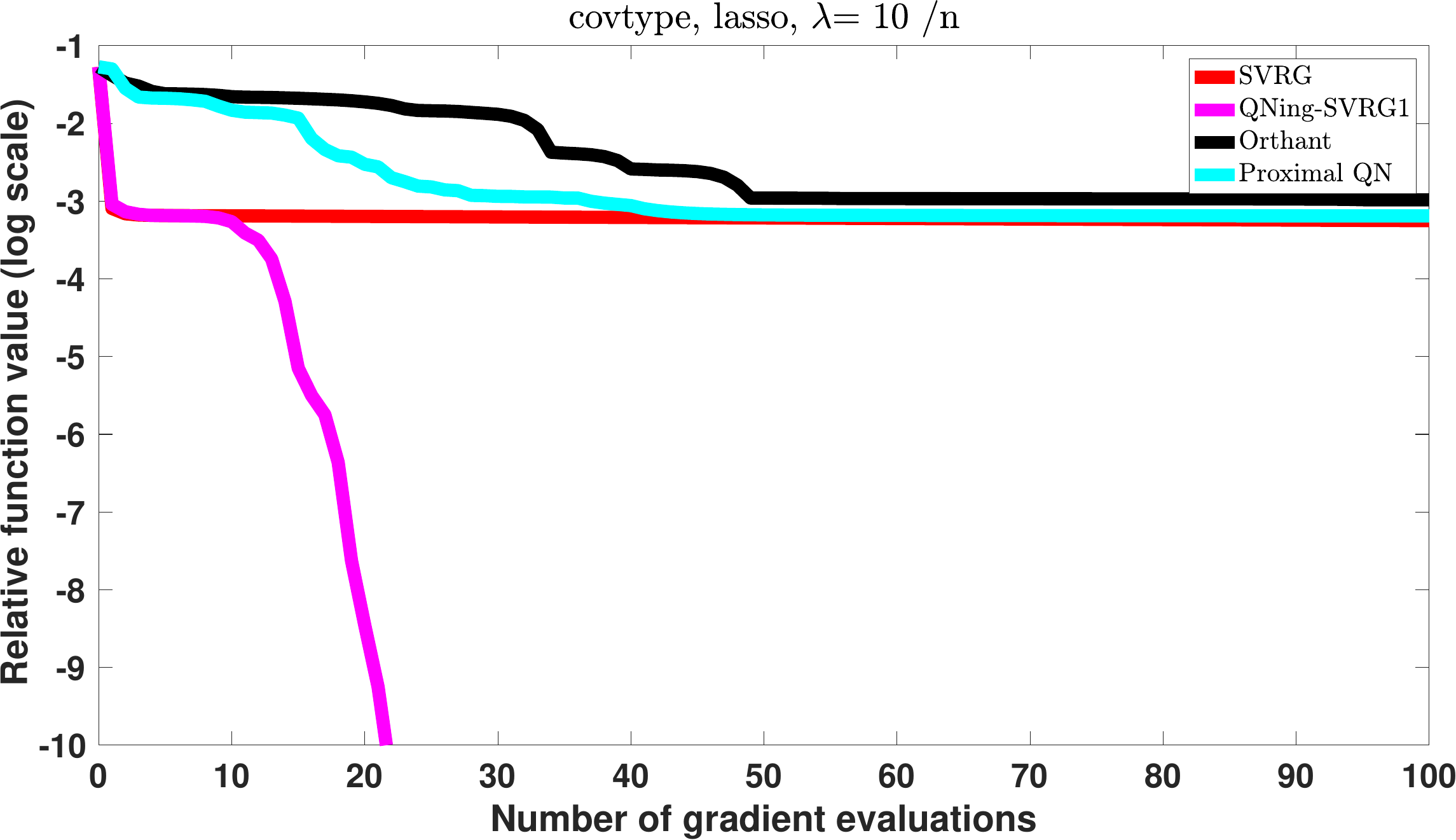} \\
   ~~\includegraphics[width=0.30\linewidth]{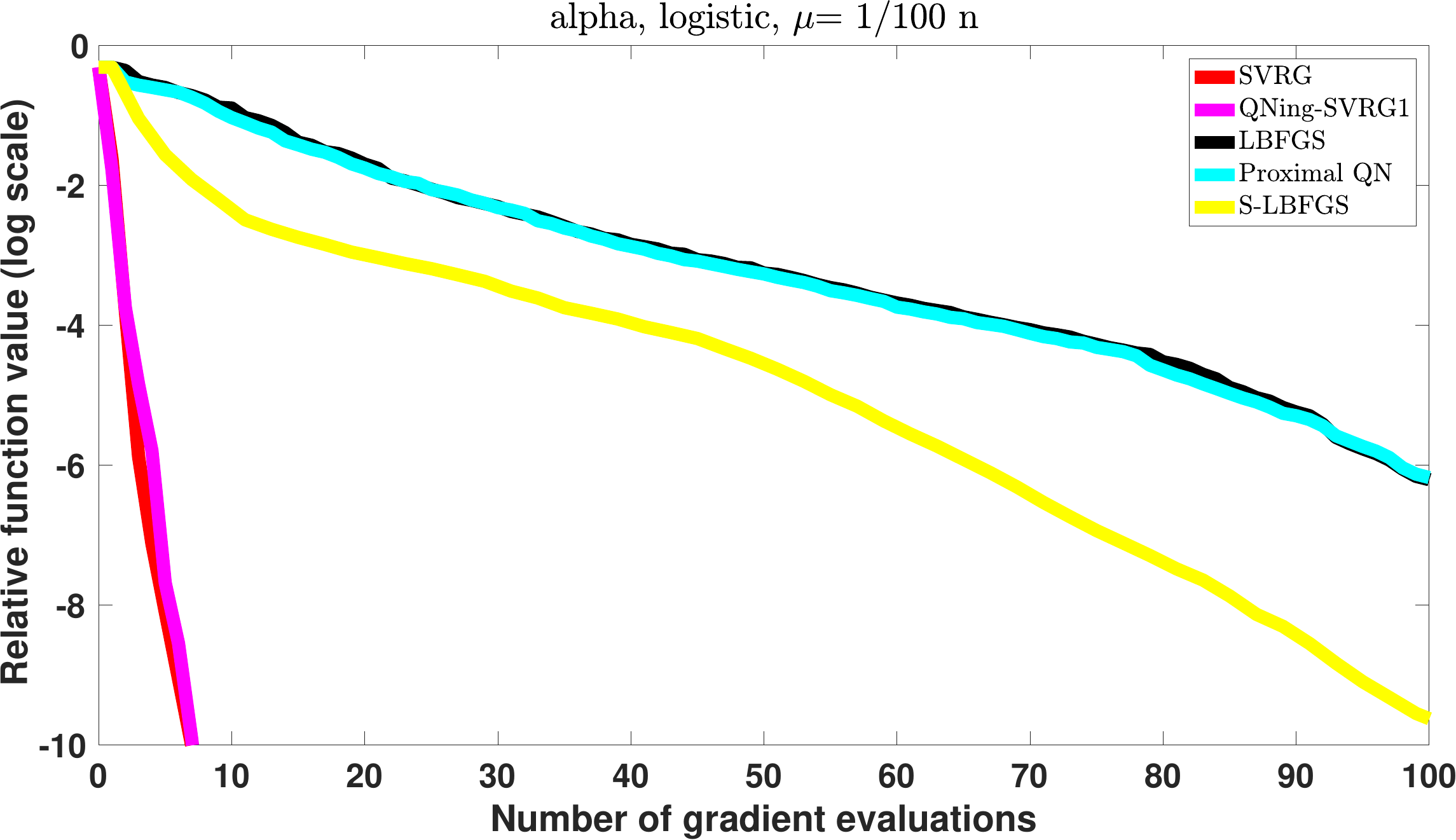} ~ 
   ~~\includegraphics[width=0.30\linewidth]{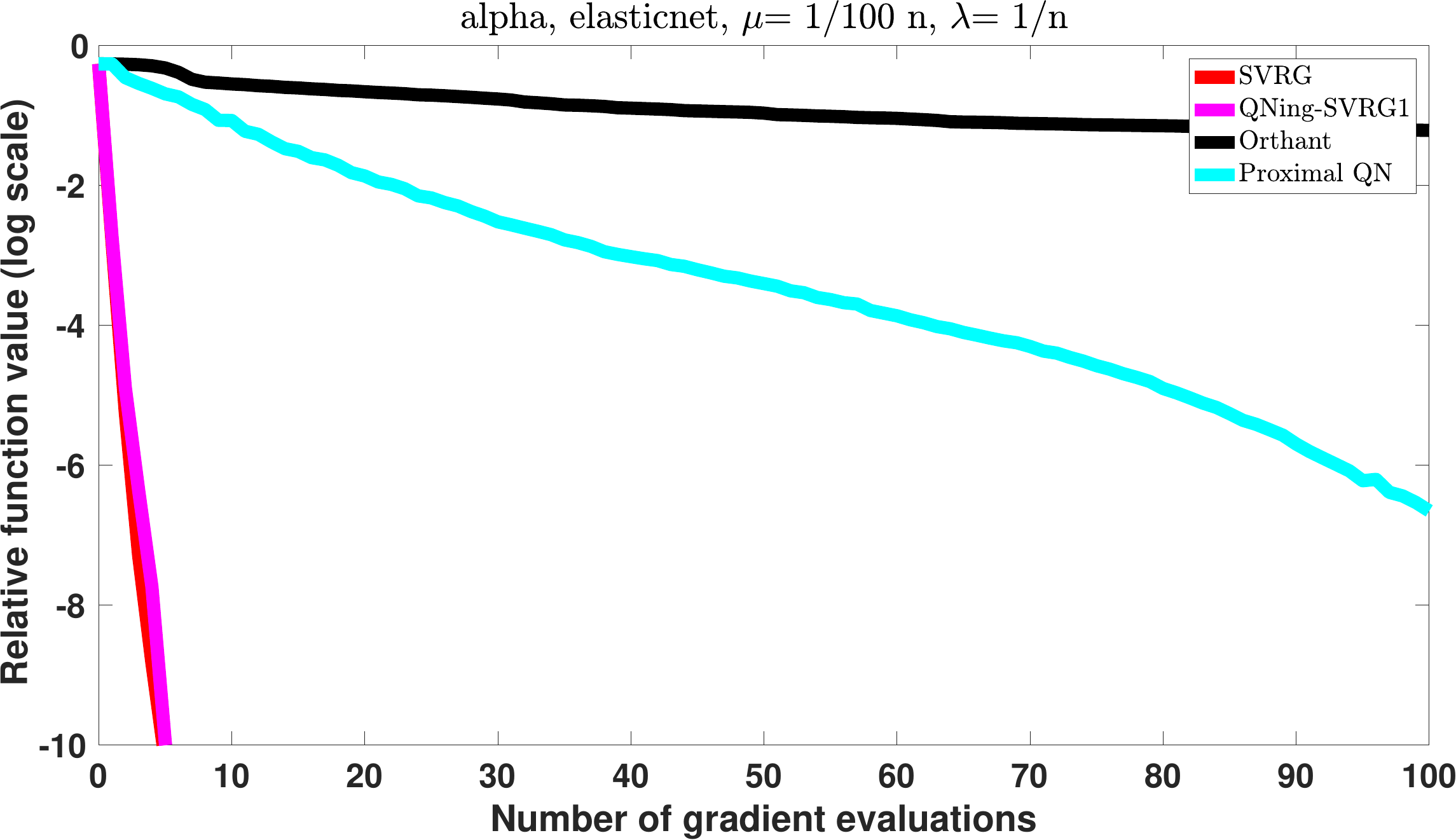} ~ 
   ~~\includegraphics[width=0.30\linewidth]{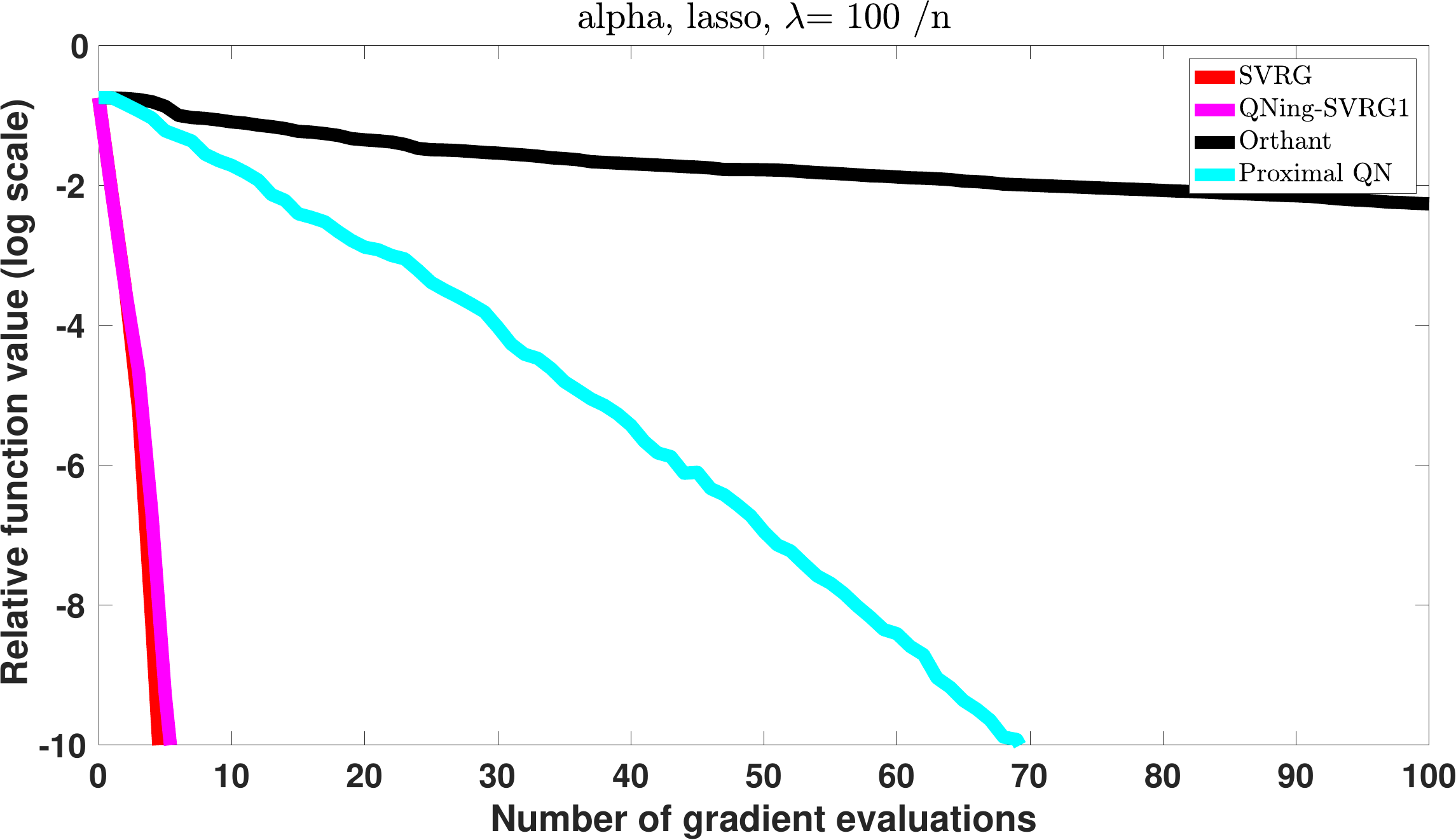} \\ 
   ~~\includegraphics[width=0.30\linewidth]{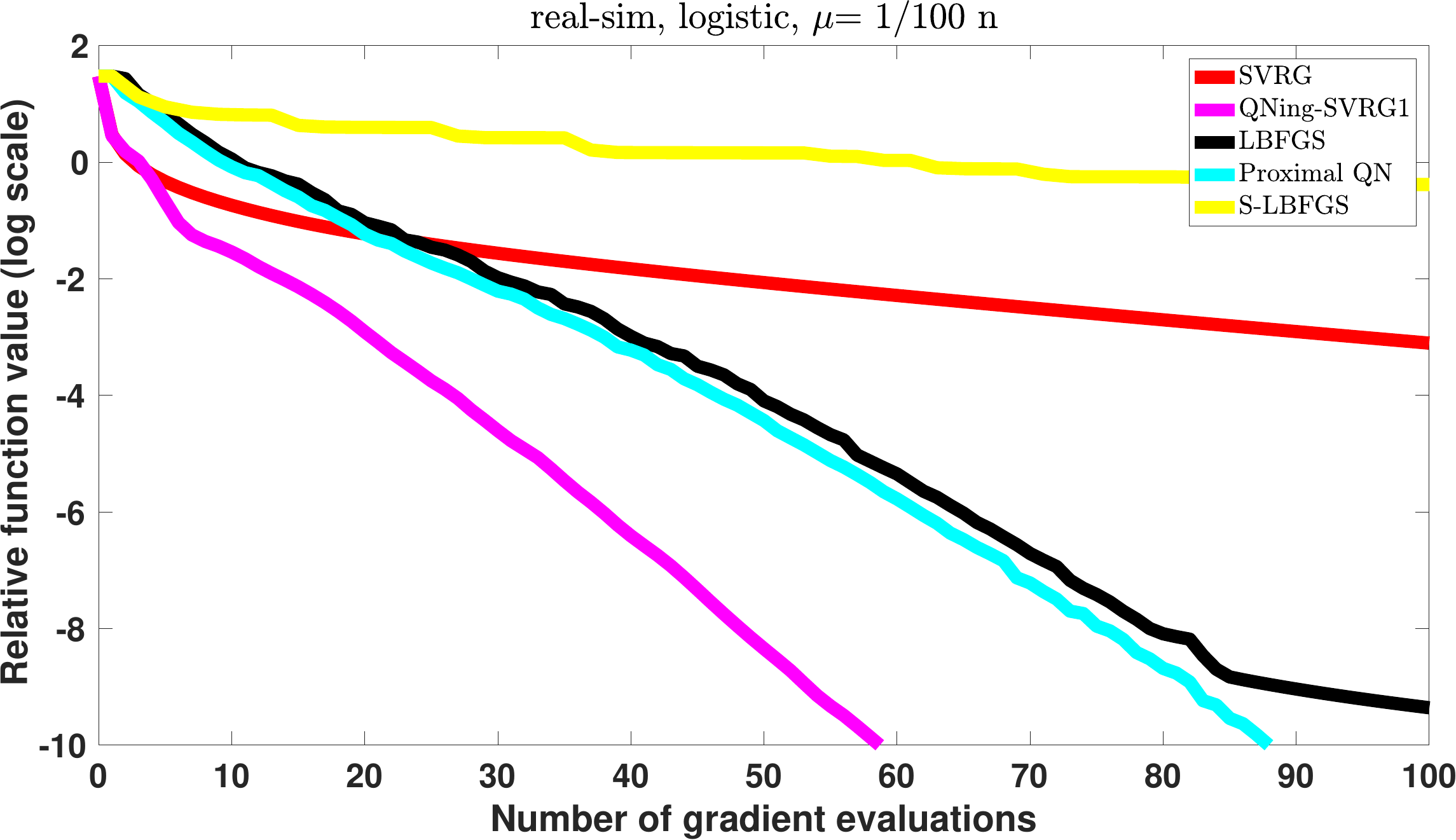} ~ 
   ~~\includegraphics[width=0.30\linewidth]{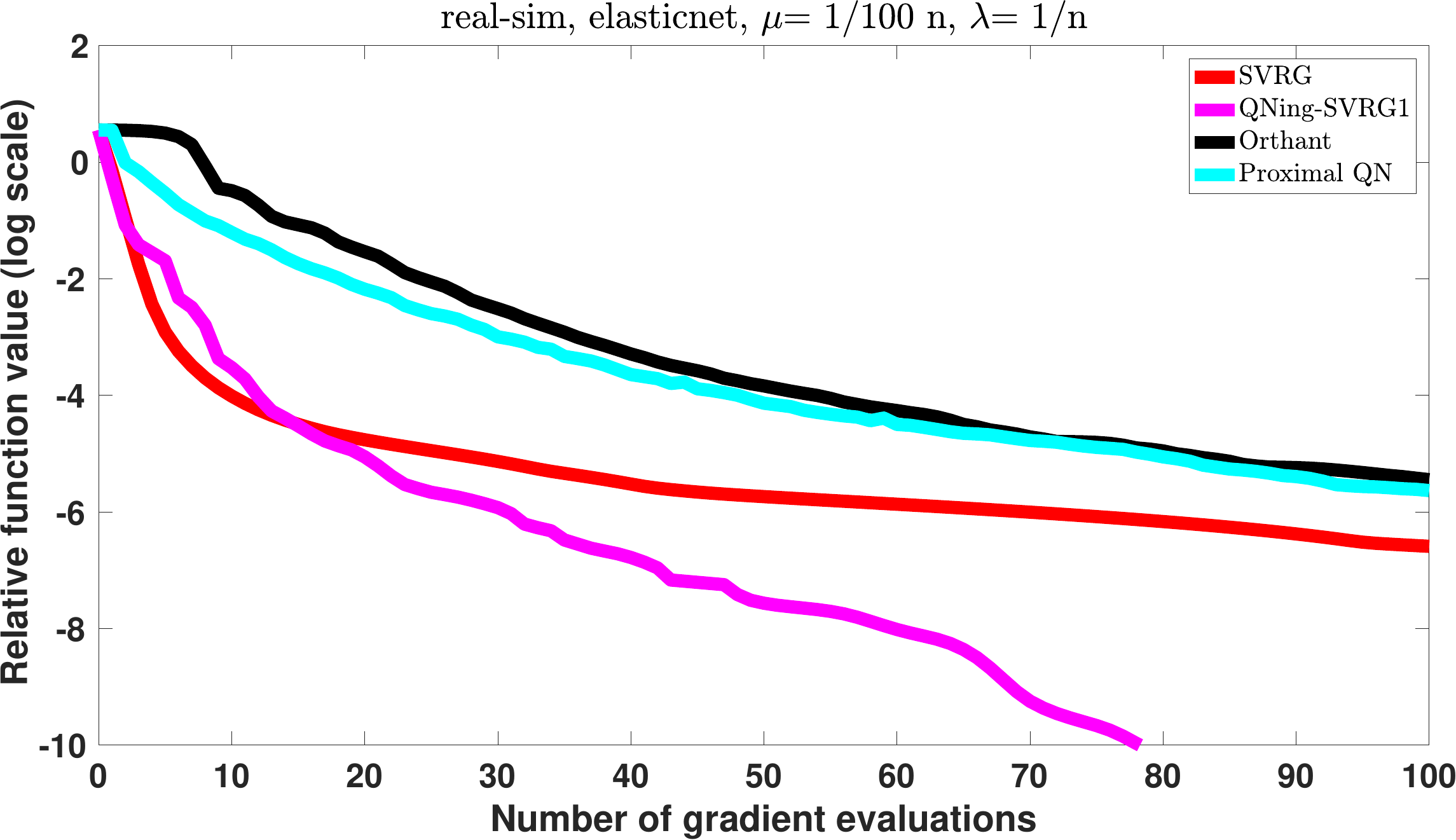} ~ 
   ~~\includegraphics[width=0.30\linewidth]{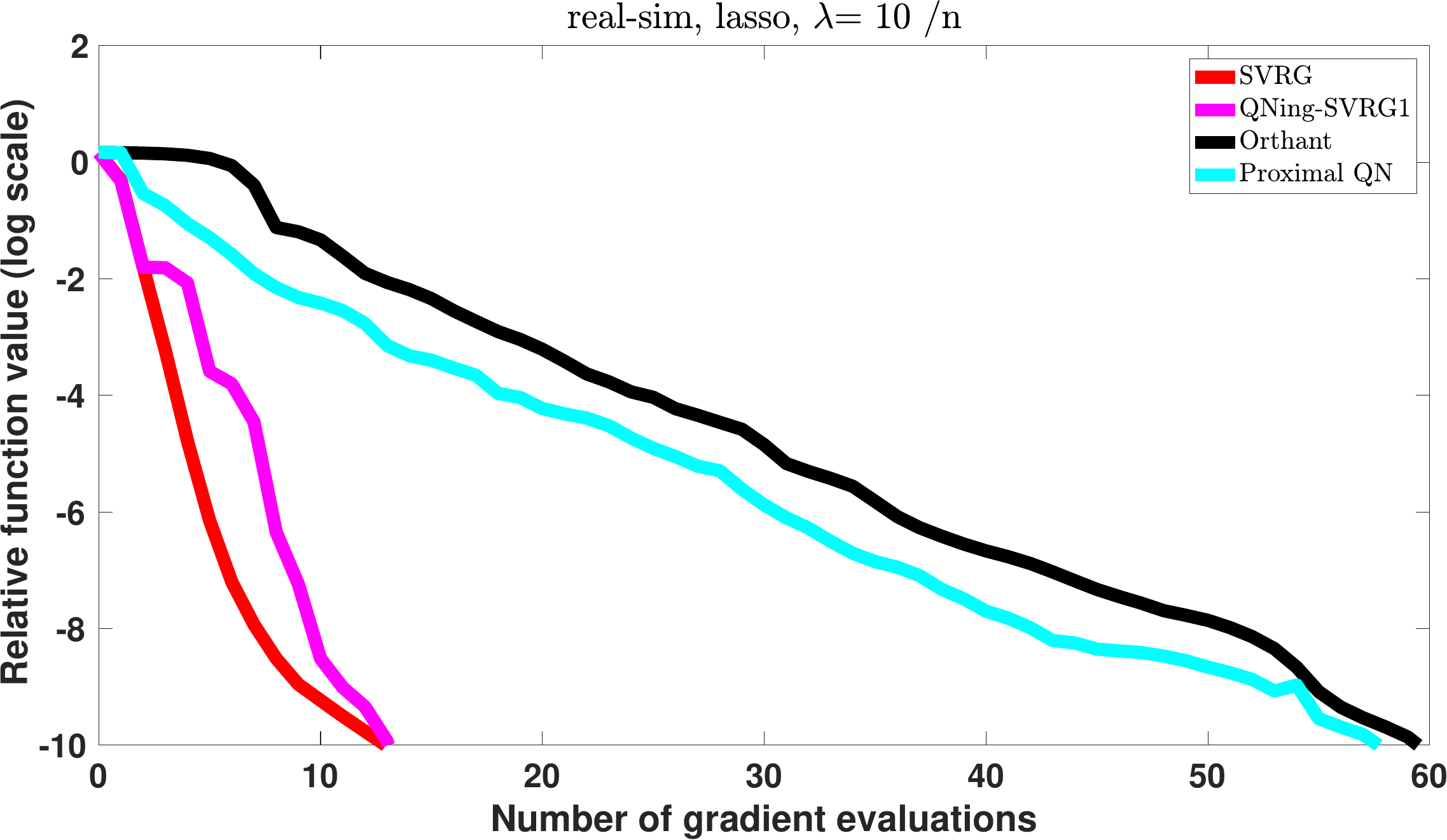} \\
   ~~\includegraphics[width=0.30\linewidth]{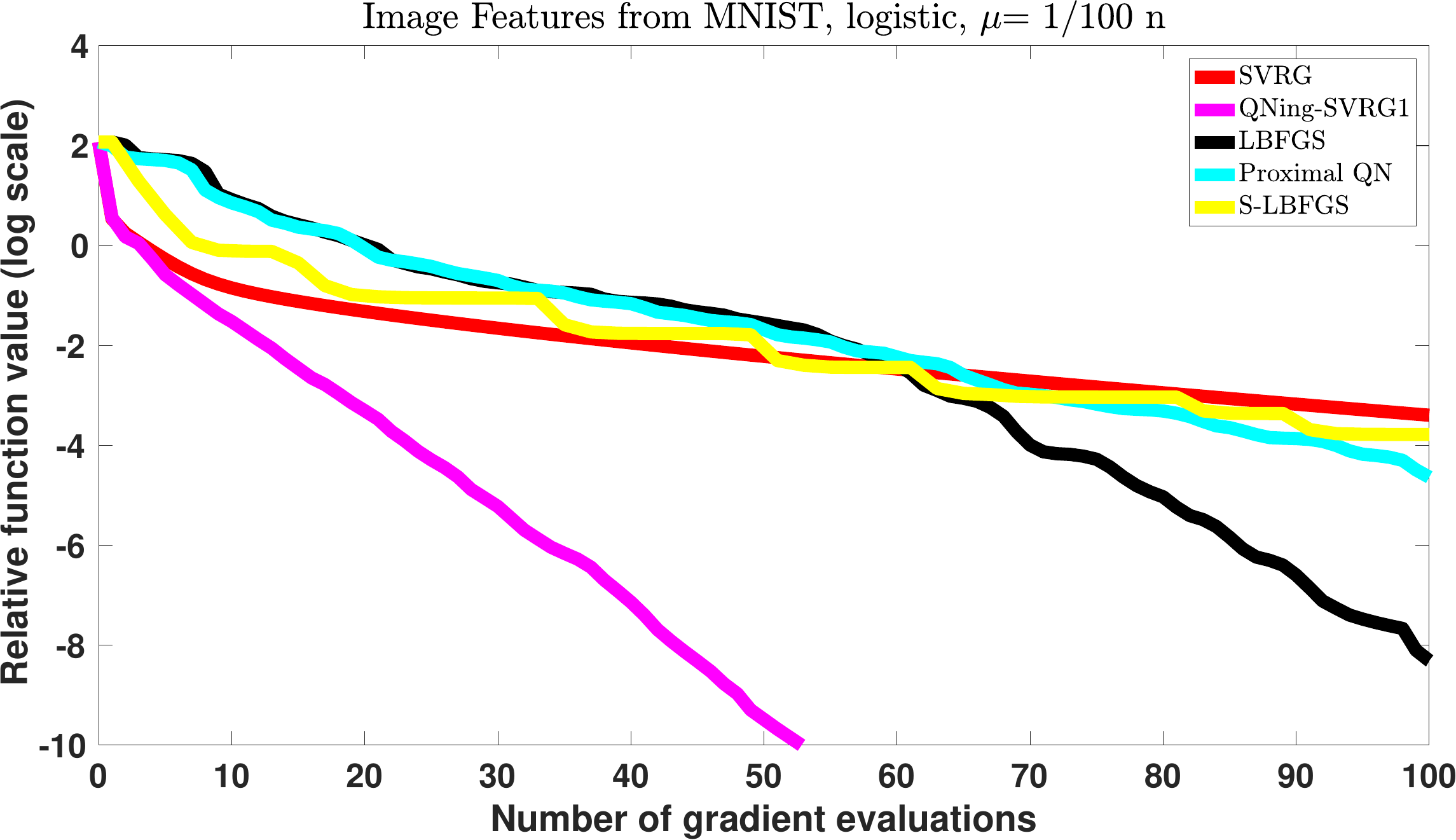} ~ 
   ~~\includegraphics[width=0.30\linewidth]{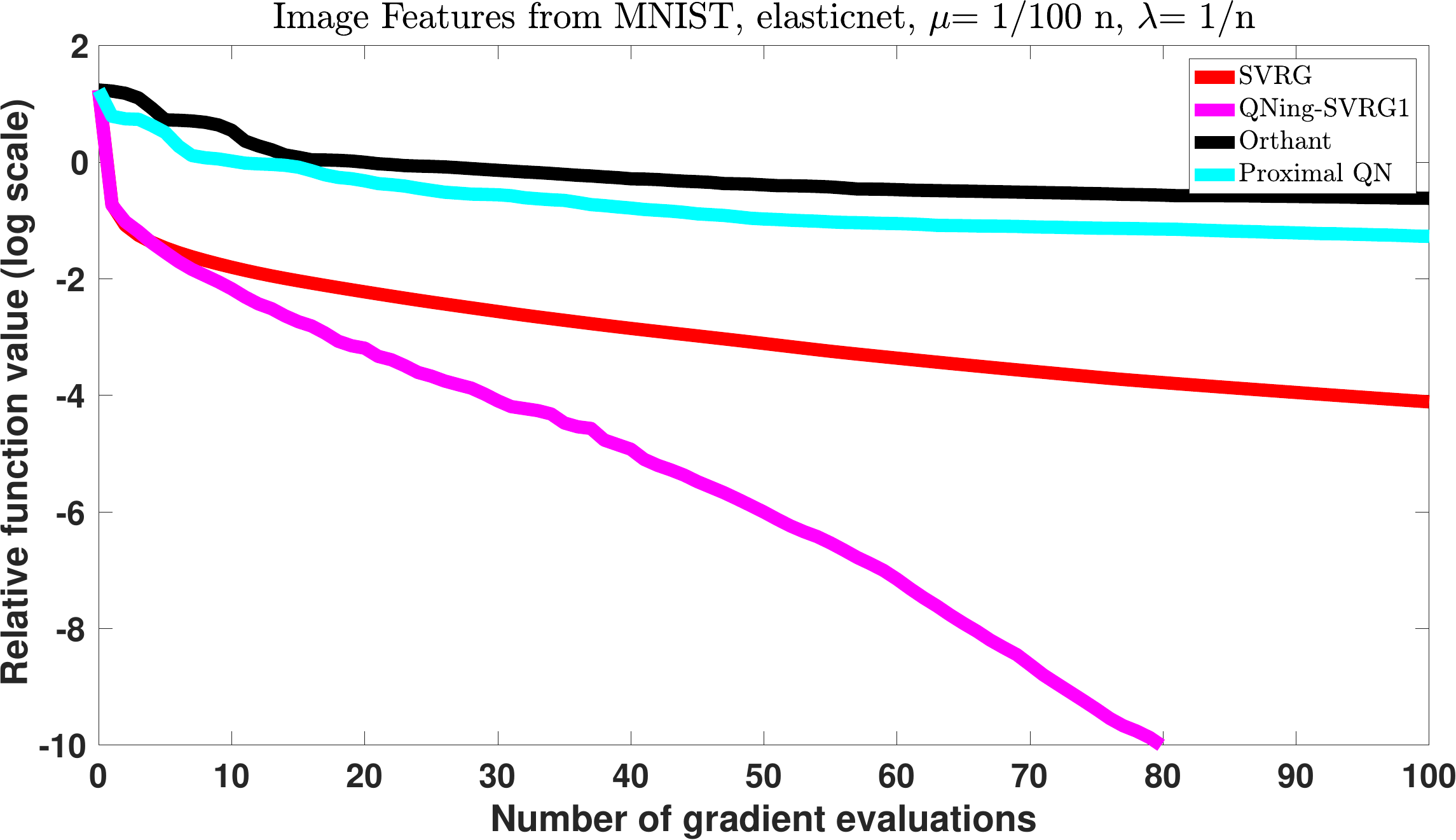} ~ 
   ~~\includegraphics[width=0.30\linewidth]{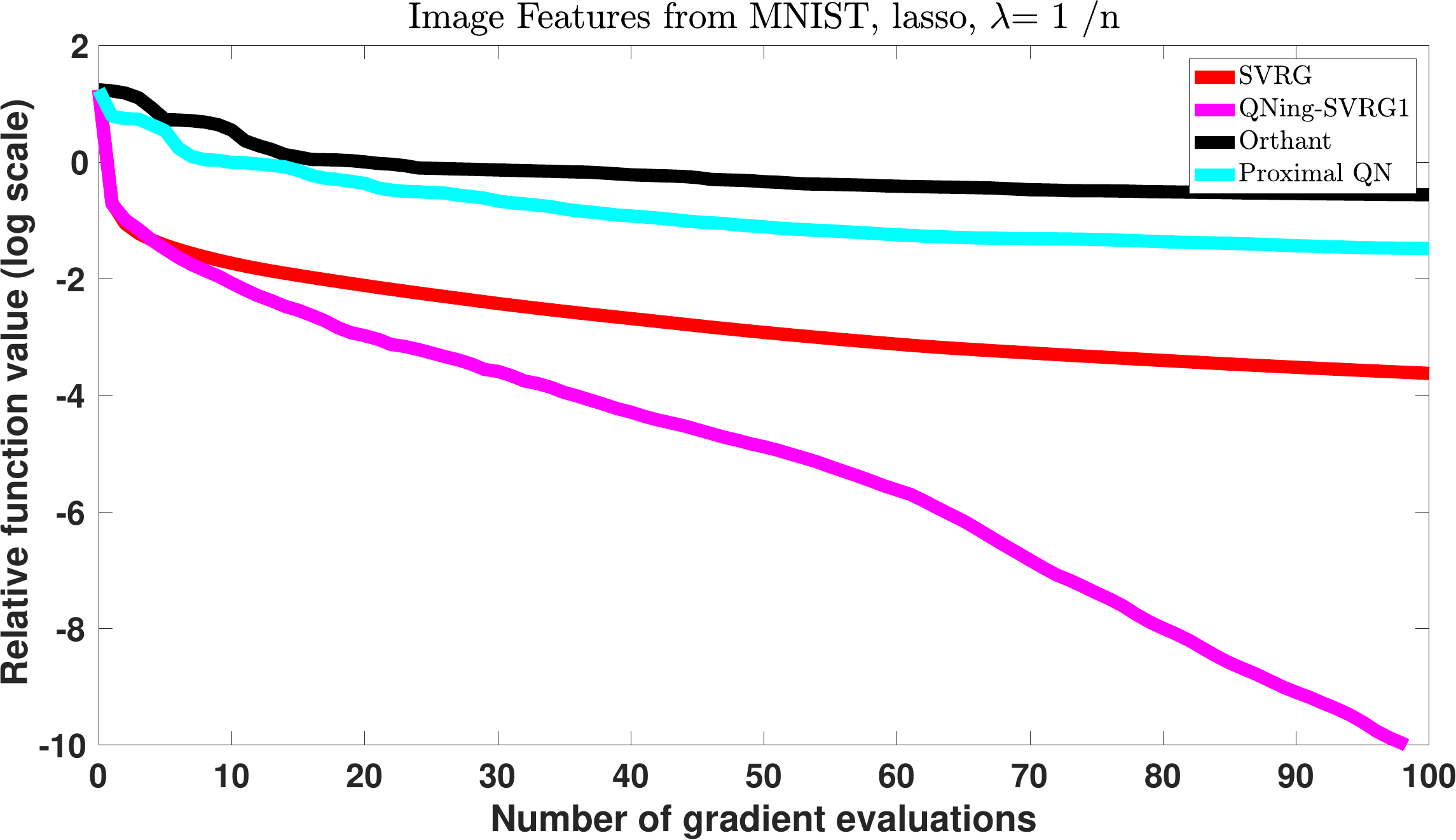} \\
   ~~\includegraphics[width=0.30\linewidth]{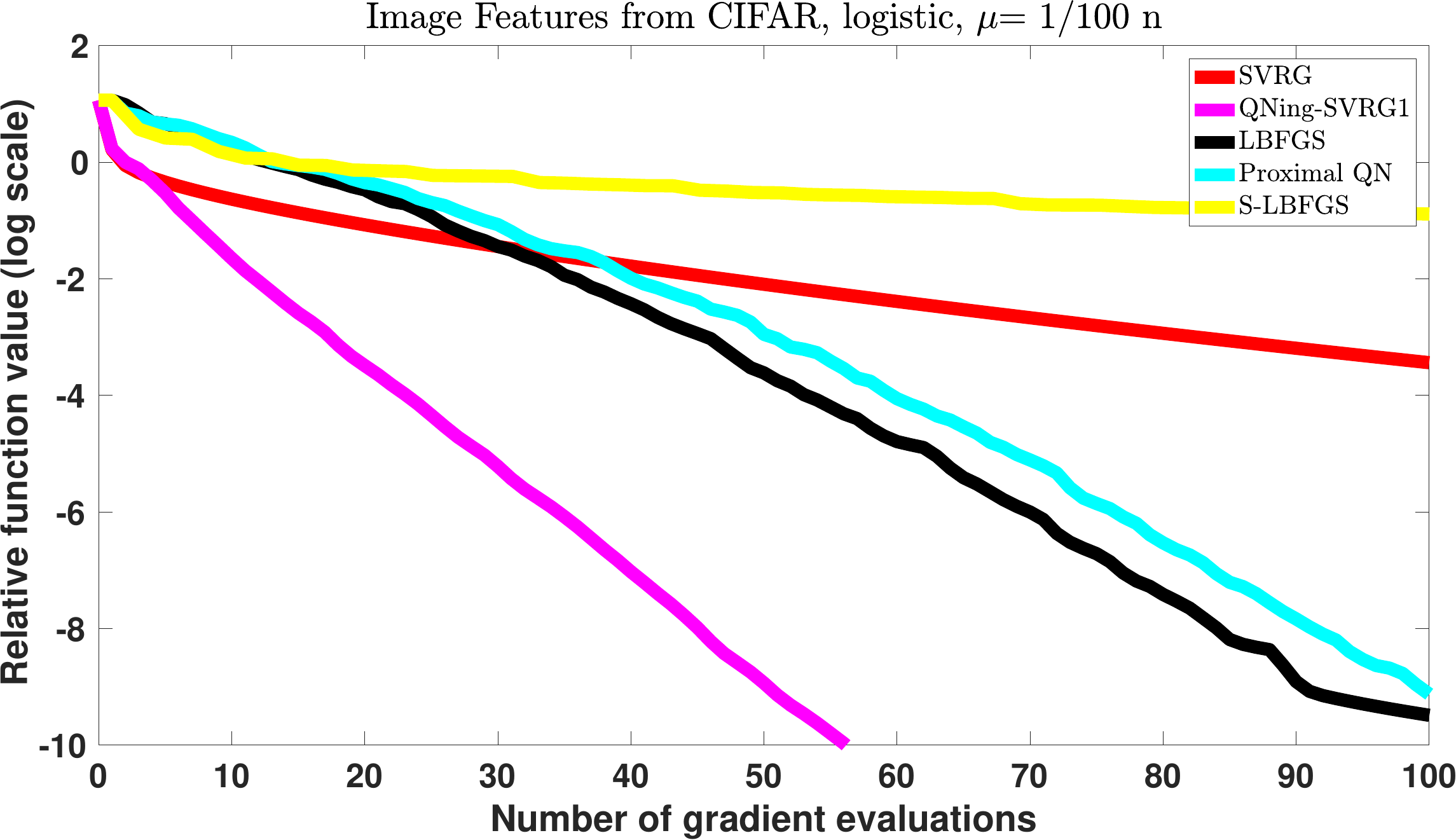} ~ 
   ~~\includegraphics[width=0.30\linewidth]{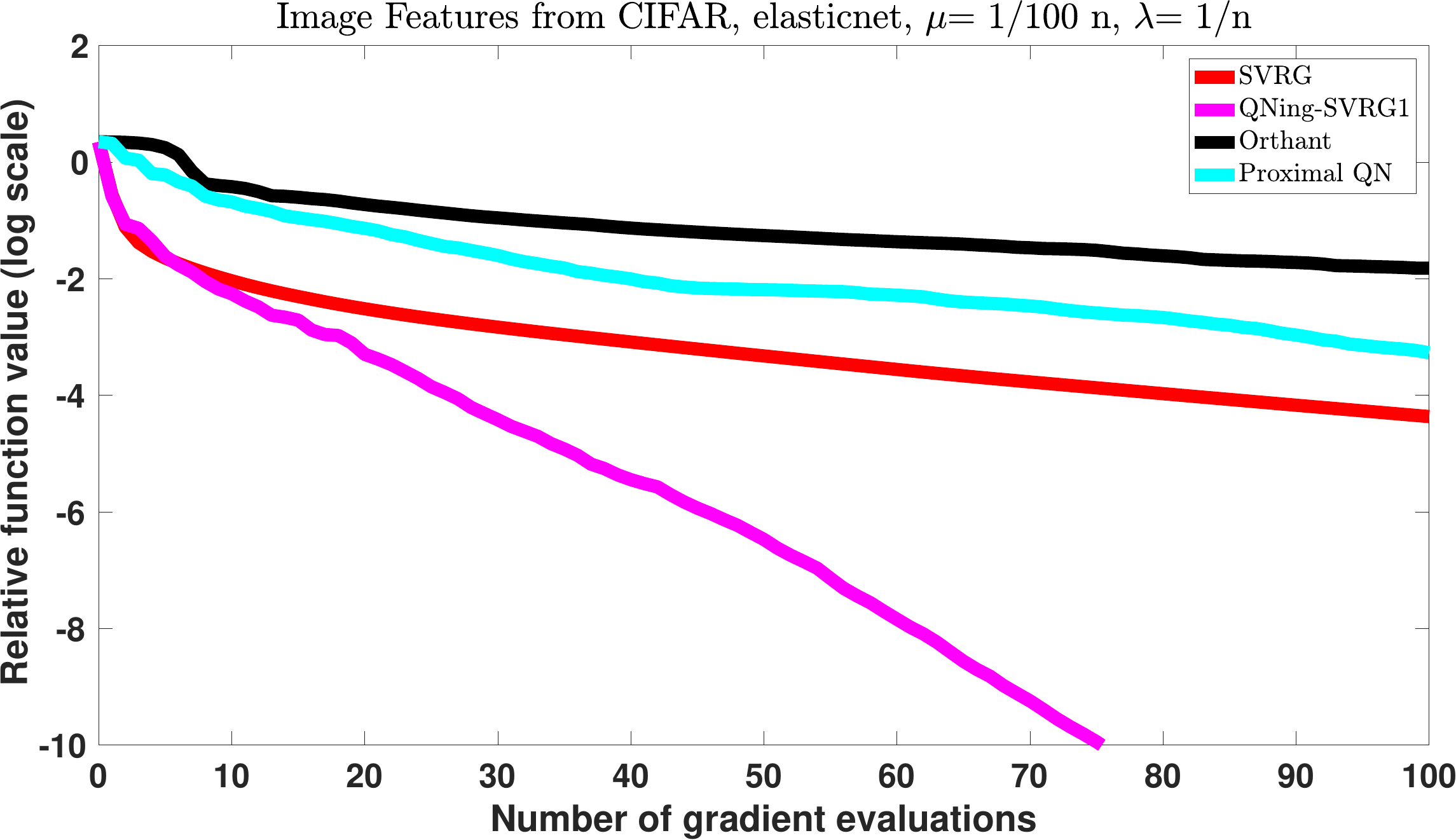} ~ 
   ~~\includegraphics[width=0.30\linewidth]{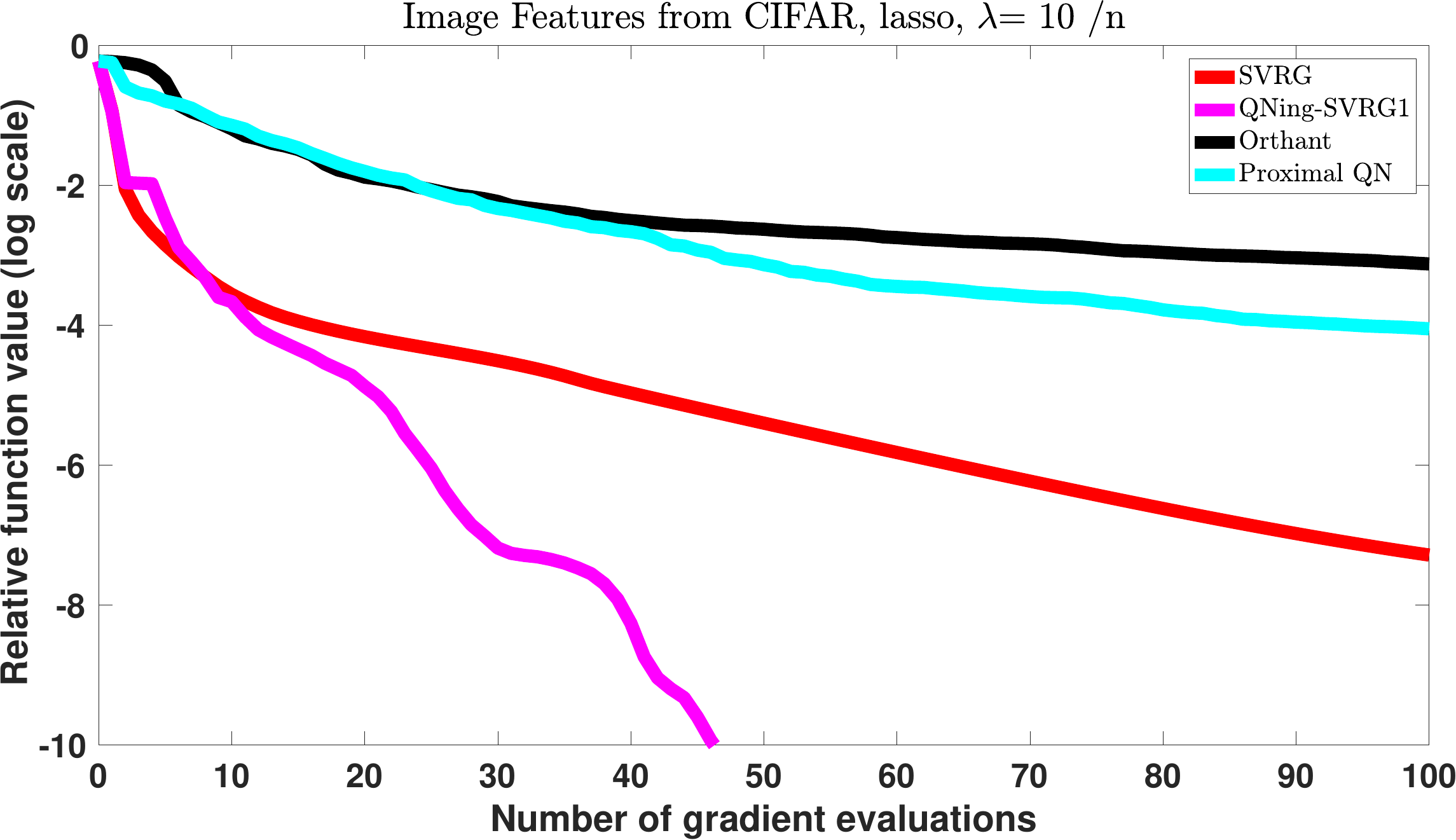} \\
   \caption{Comparison to proximal L-BFGS and Stochastic L-BFGS. We plot the value~$F(x_k)/F^\star-1$ as
   a function of the number of gradient evaluations, on a logarithmic scale;
   the optimal value $F^\star$ is estimated with a duality gap.}\label{fig:qning_vs_proxqn}
\end{figure}

The result of the comparison is presented in Figure~\ref{fig:qning_vs_proxqn} and we observe that \textbf{\qning-SVRG1} is significantly faster than Proximal L-BFGS and Stochastic L-BFGS:
\begin{itemize}
   \item \textbf{Proximal L-BFGS} often outperforms \textbf{Orthant}-based methods but it is less competitive than QNing.
   \item \textbf{Stochastic L-BFGS} is sensitive to parameters and data since the variable metric is based on stochastic information which may have high variance. It performs well on dataset \textsf{covtype} but becomes less competitive on other datasets. Moreover, it only applies to smooth problems.
\end{itemize}

The previous results are complemented by Appendix~\ref{appendix:iterations},
which also presents some comparison in terms of outer-loop iterations,
regardless of the cost of the inner-loop.

\subsection{\qning-ISTA and comparison with L-BFGS}\label{subsec:gd}
The previous experiments have included a comparison between L-BFGS and approaches that 
are able to exploit the sum structure of the objective.
It is then interesting to study the behavior of \qning~when
applied to a basic proximal gradient descent algorithm such as ISTA.
Specifically, we now consider
\begin{itemize}
   \item \textbf{GD/ISTA}: the classical proximal gradient descent algorithm ISTA~\cite{fista}
      with back-tracking line-search to automatically adjust the 
      Lipschitz constant of the gradient objective;
   \item \textbf{Acc-GD/FISTA}: the accelerated variant of ISTA from~\cite{fista}. 
   \item \textbf{\qning-ISTA}, and \textbf{\qning-ISTA1}, as in the previous section replacing SVRG by GD/ISTA.
\end{itemize}

\begin{figure}[hbtp!]
   \centering
    ~~\includegraphics[width=0.30\linewidth]{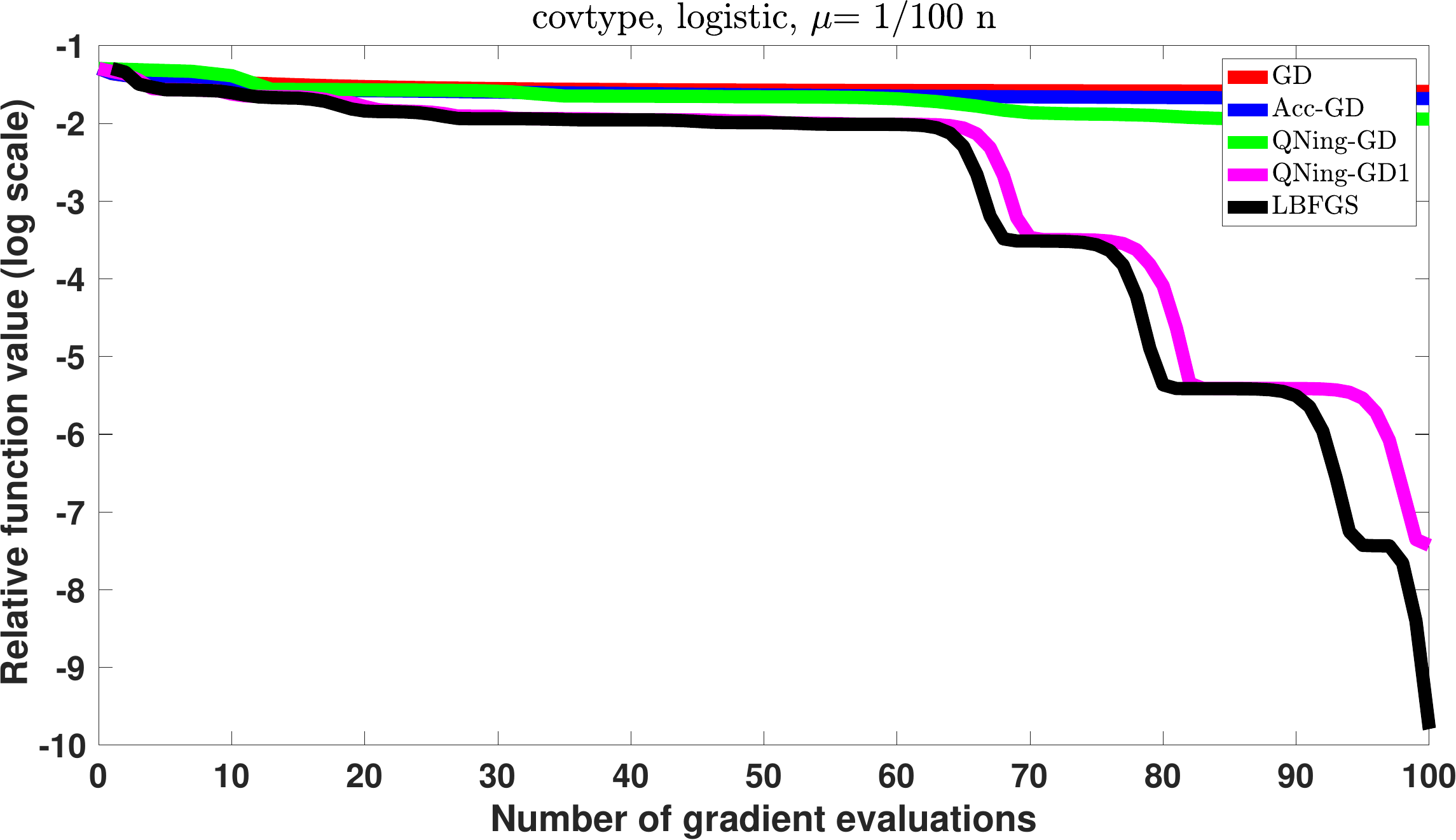} ~ 
   ~~\includegraphics[width=0.30\linewidth]{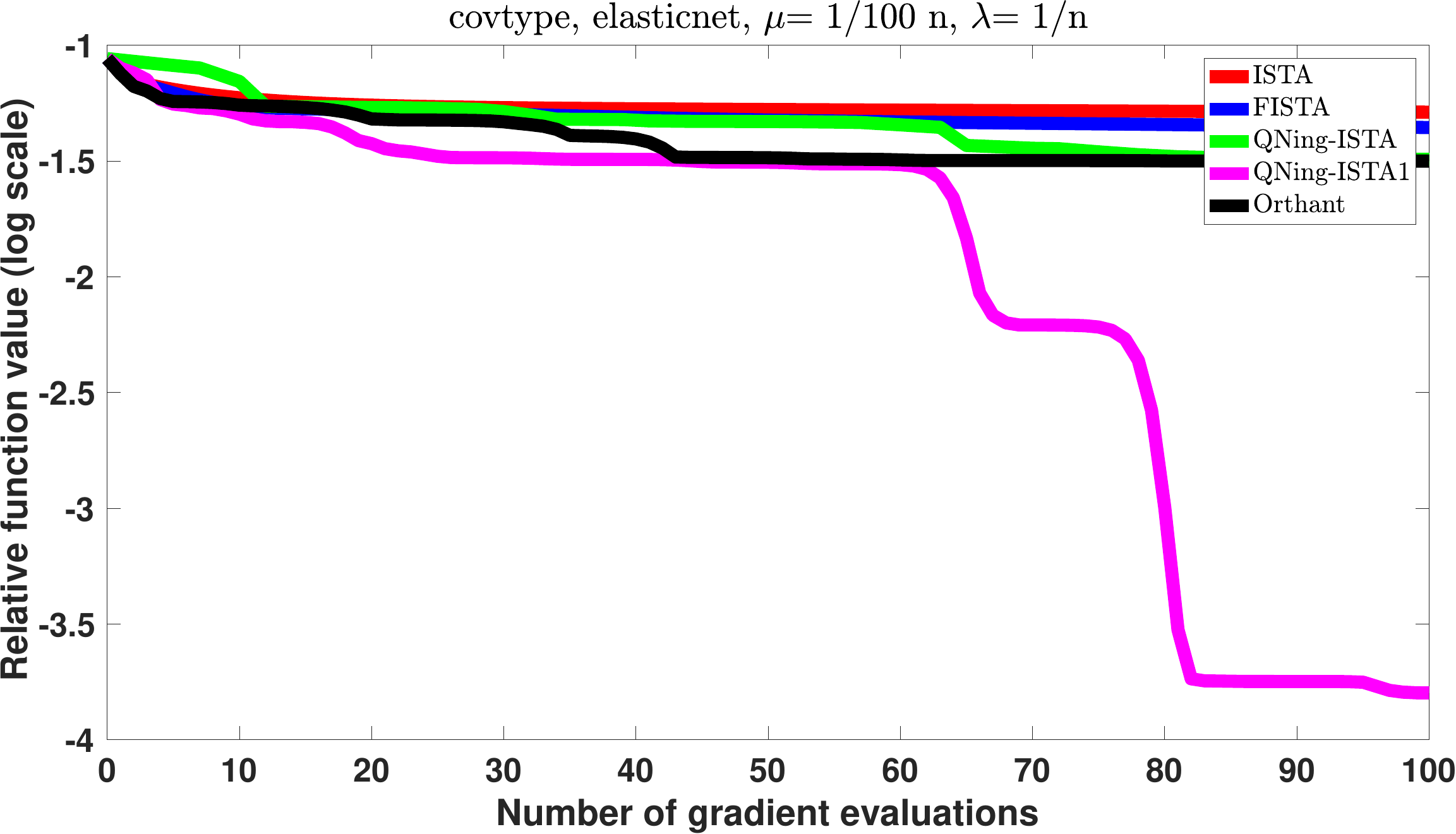} ~ 
   ~~\includegraphics[width=0.30\linewidth]{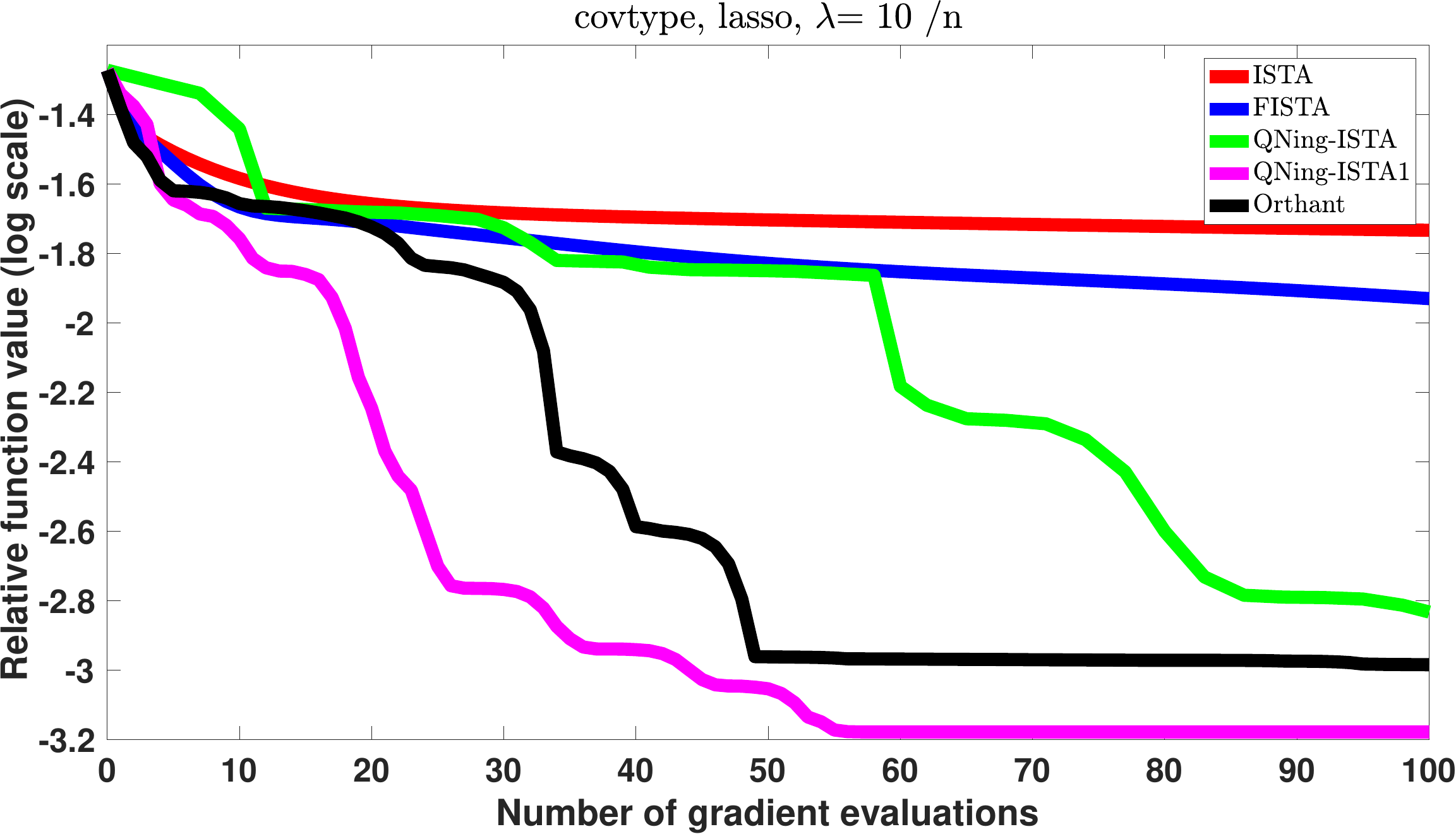} \\
   ~~\includegraphics[width=0.30\linewidth]{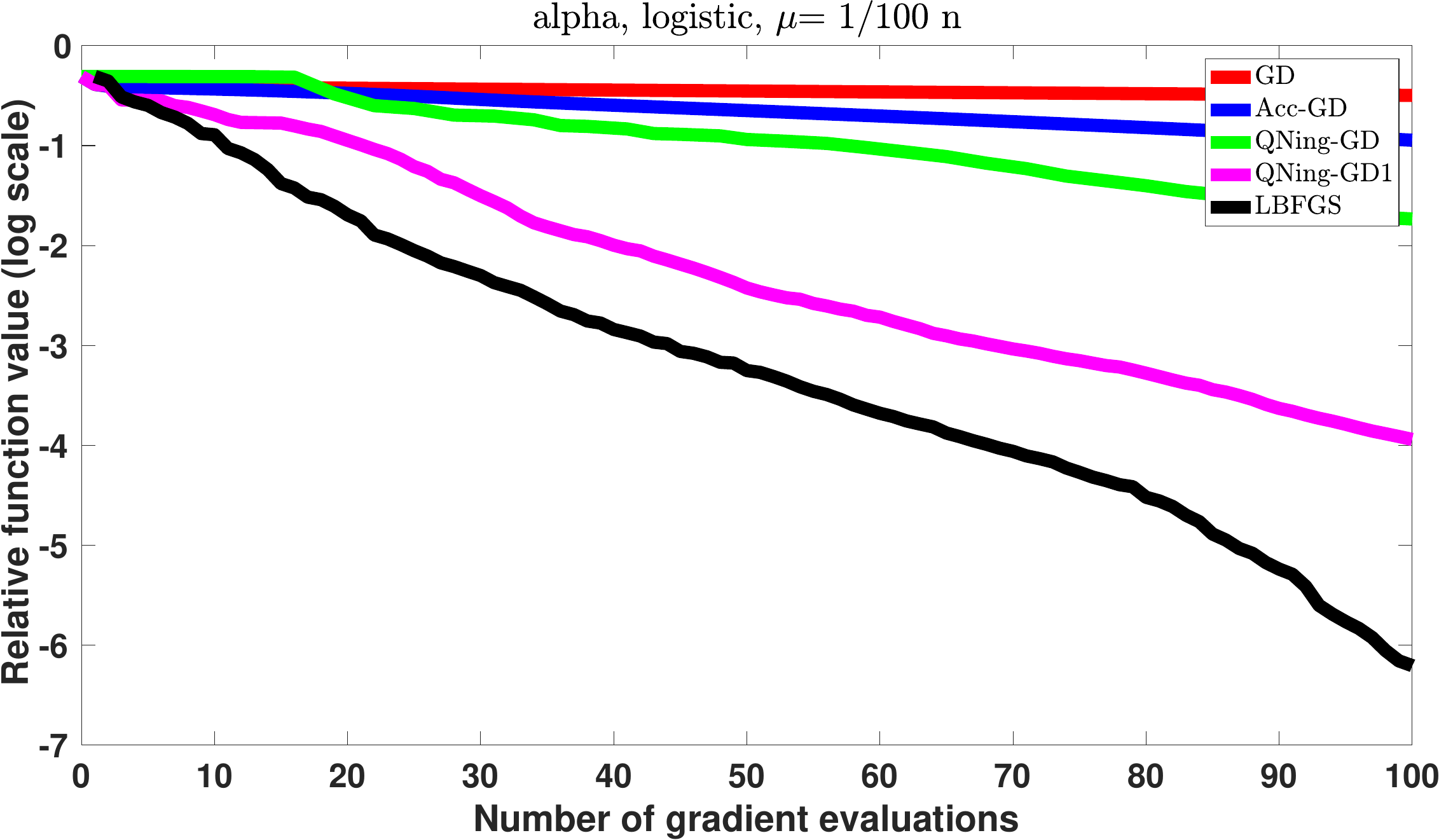} ~ 
   ~~\includegraphics[width=0.30\linewidth]{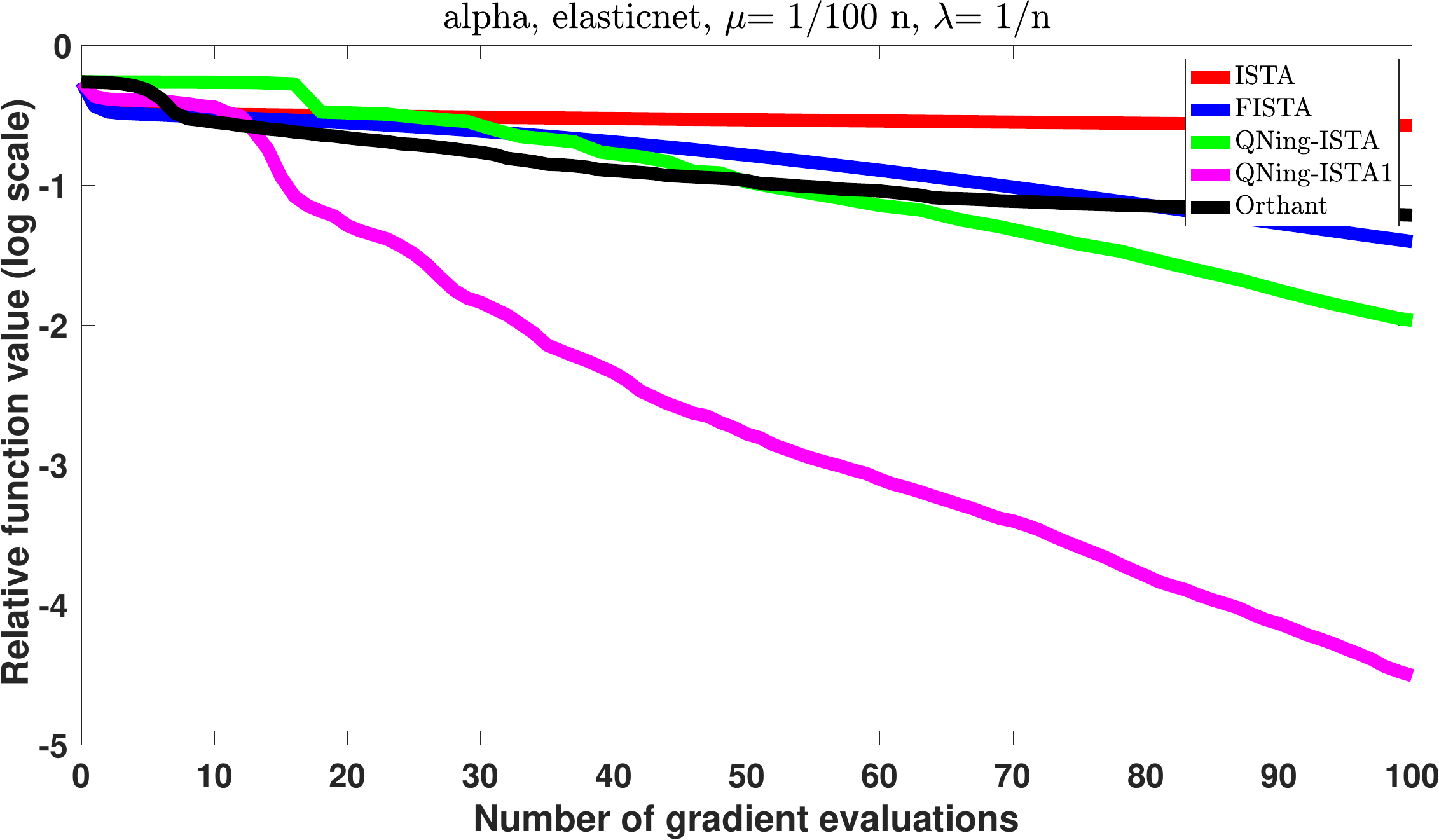} ~ 
   ~~\includegraphics[width=0.30\linewidth]{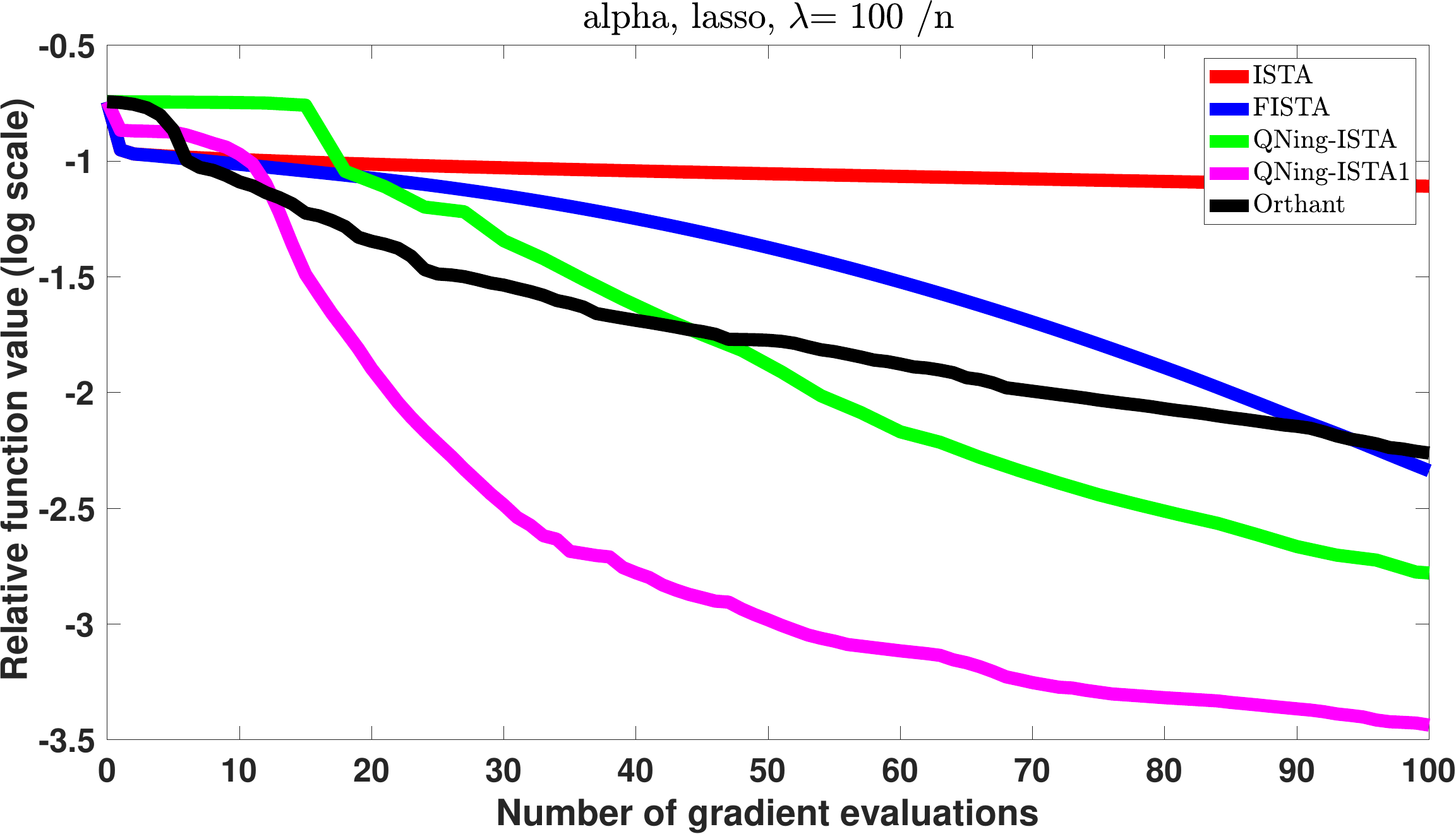} \\  
   ~~\includegraphics[width=0.30\linewidth]{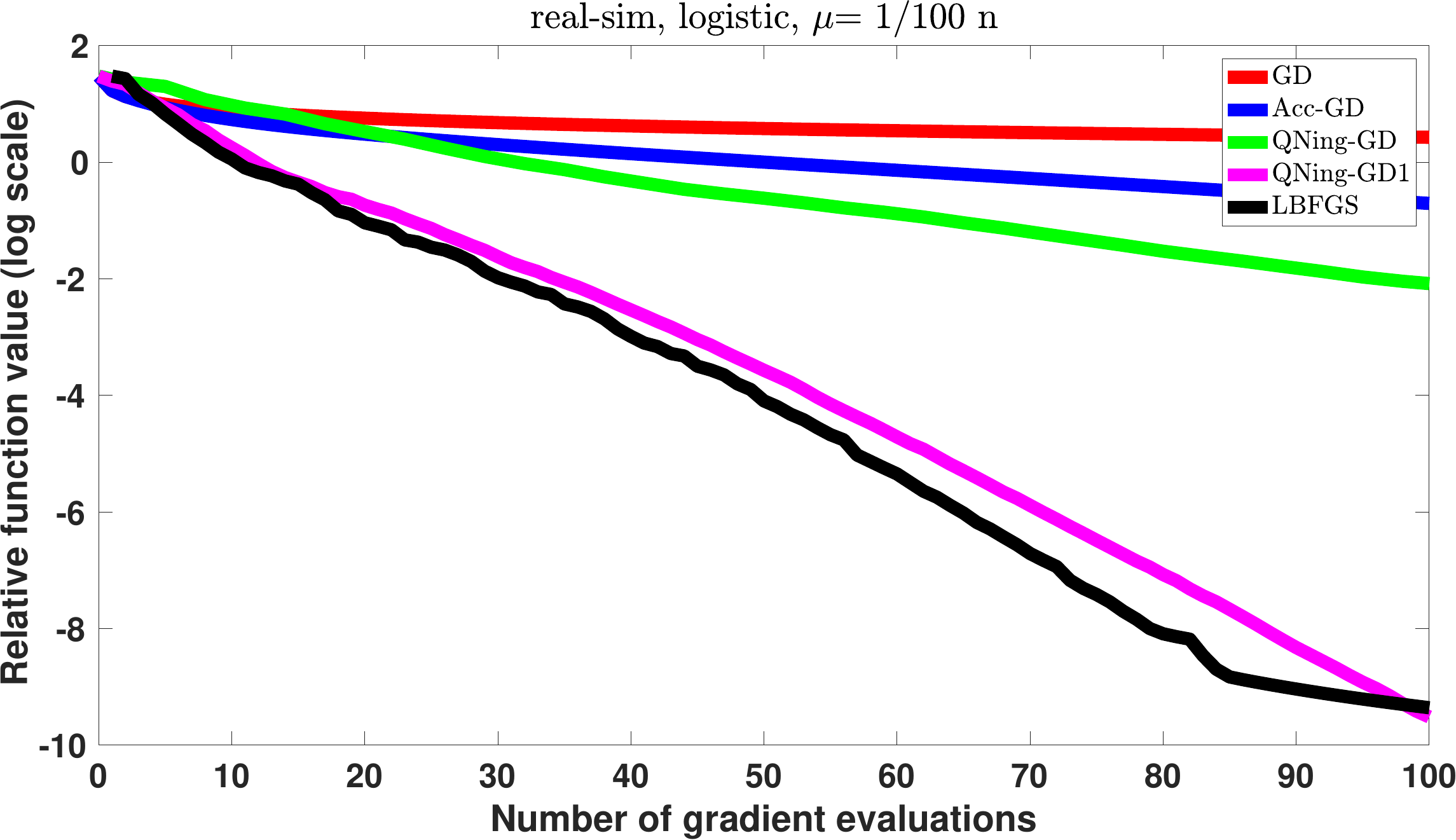} ~ 
   ~~\includegraphics[width=0.30\linewidth]{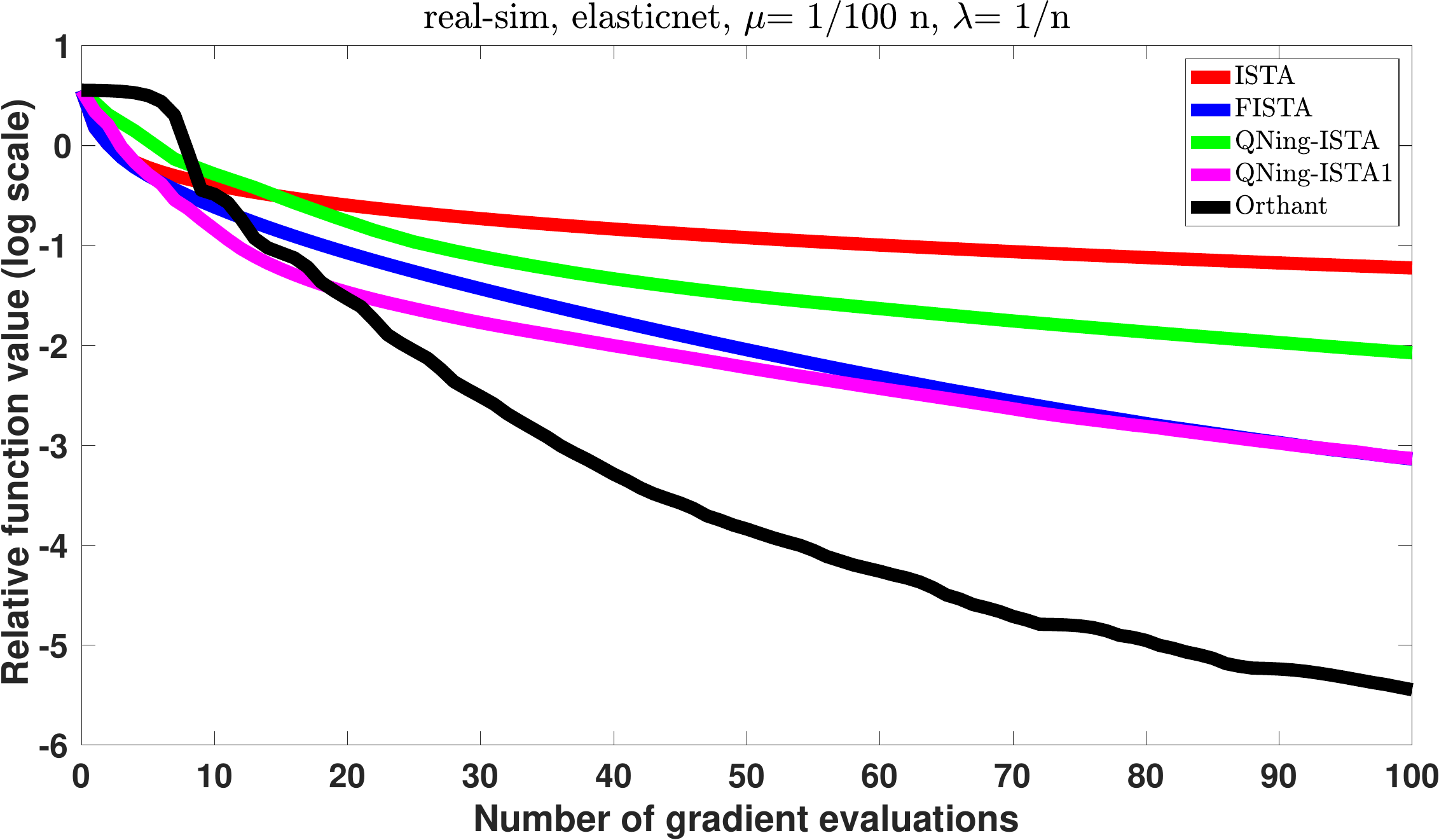} ~ 
   ~~\includegraphics[width=0.30\linewidth]{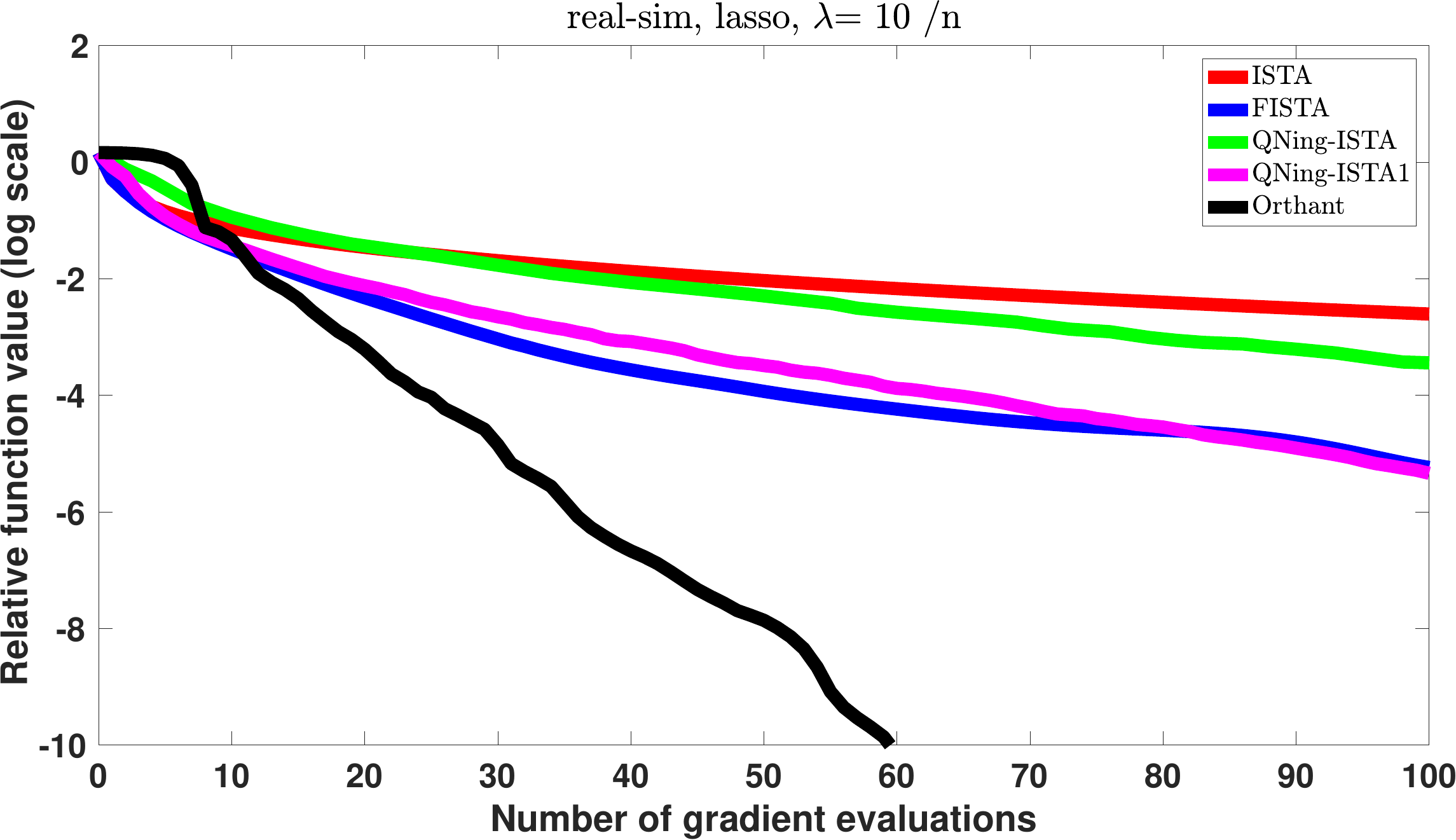} \\ 
   ~~\includegraphics[width=0.30\linewidth]{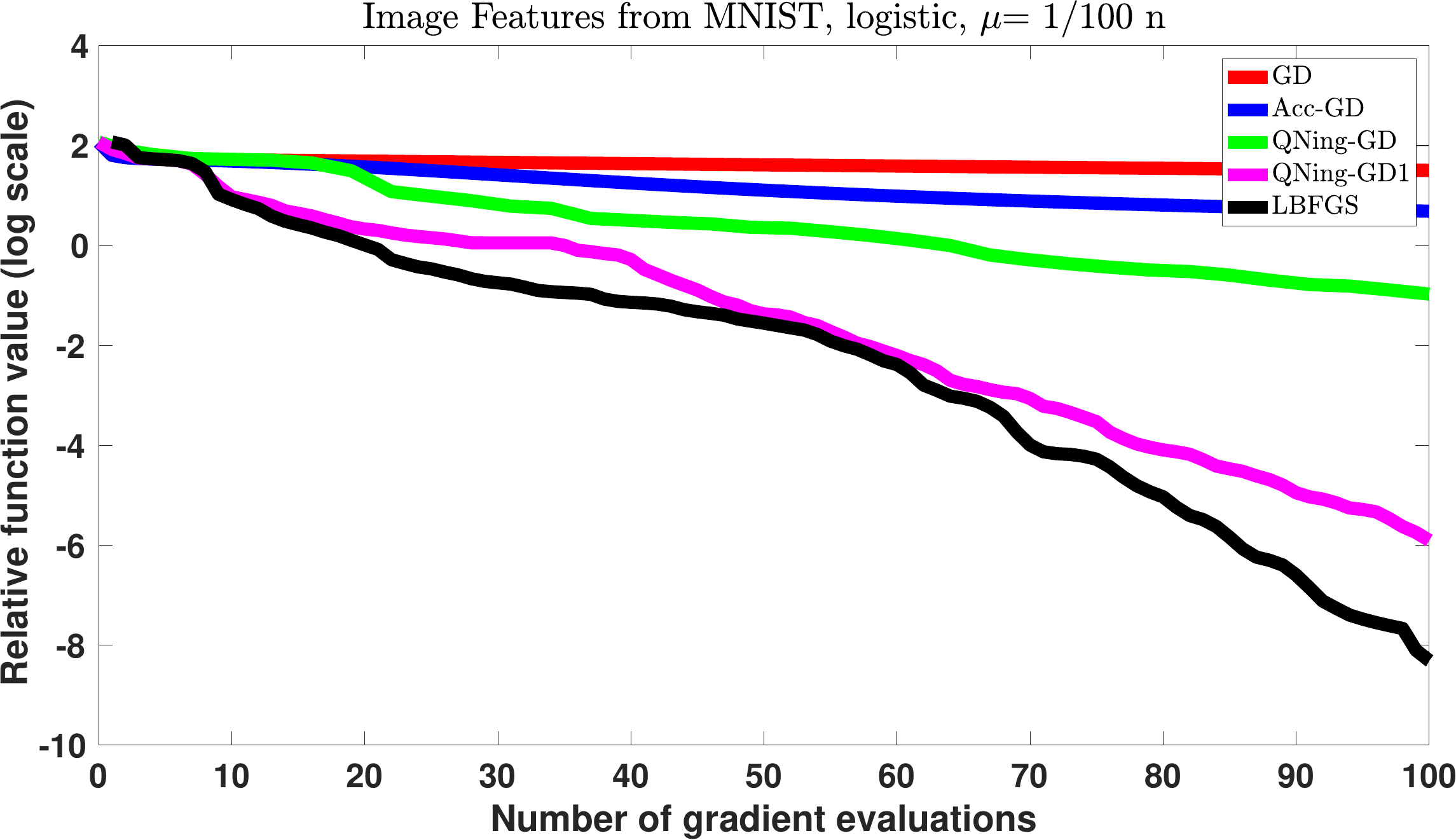} ~ 
   ~~\includegraphics[width=0.30\linewidth]{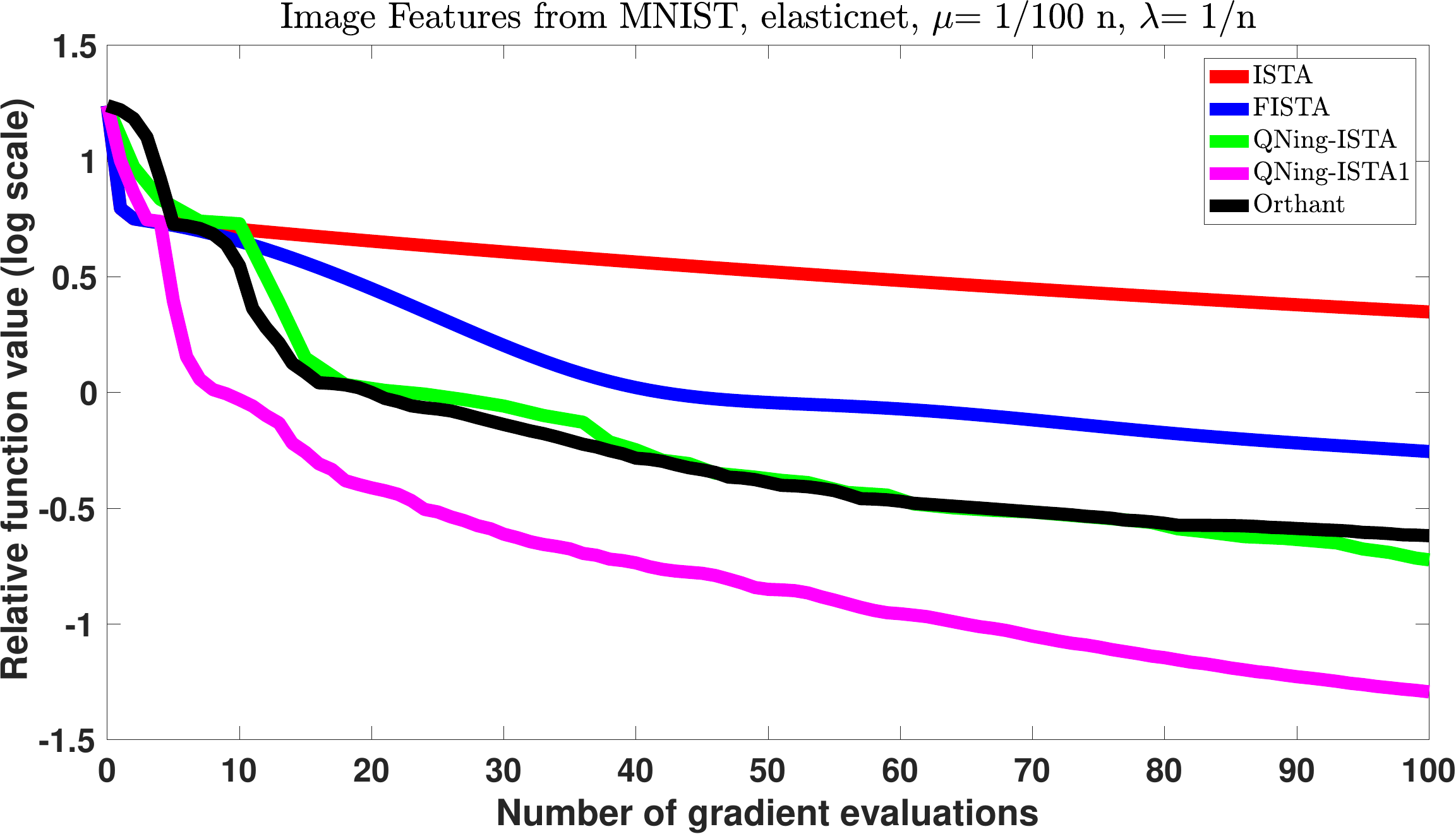} ~ 
   ~~\includegraphics[width=0.30\linewidth]{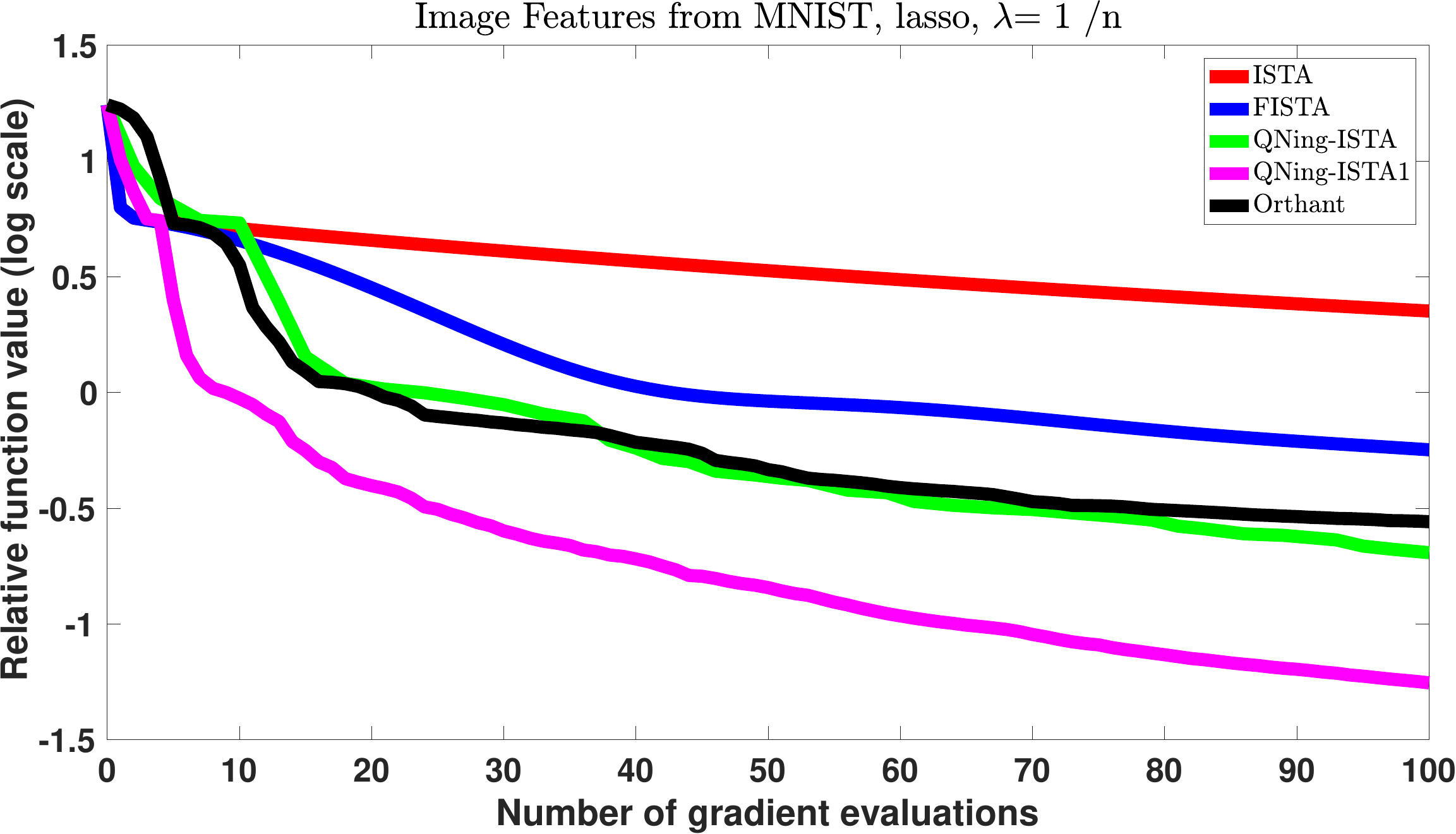} \\
   ~~\includegraphics[width=0.30\linewidth]{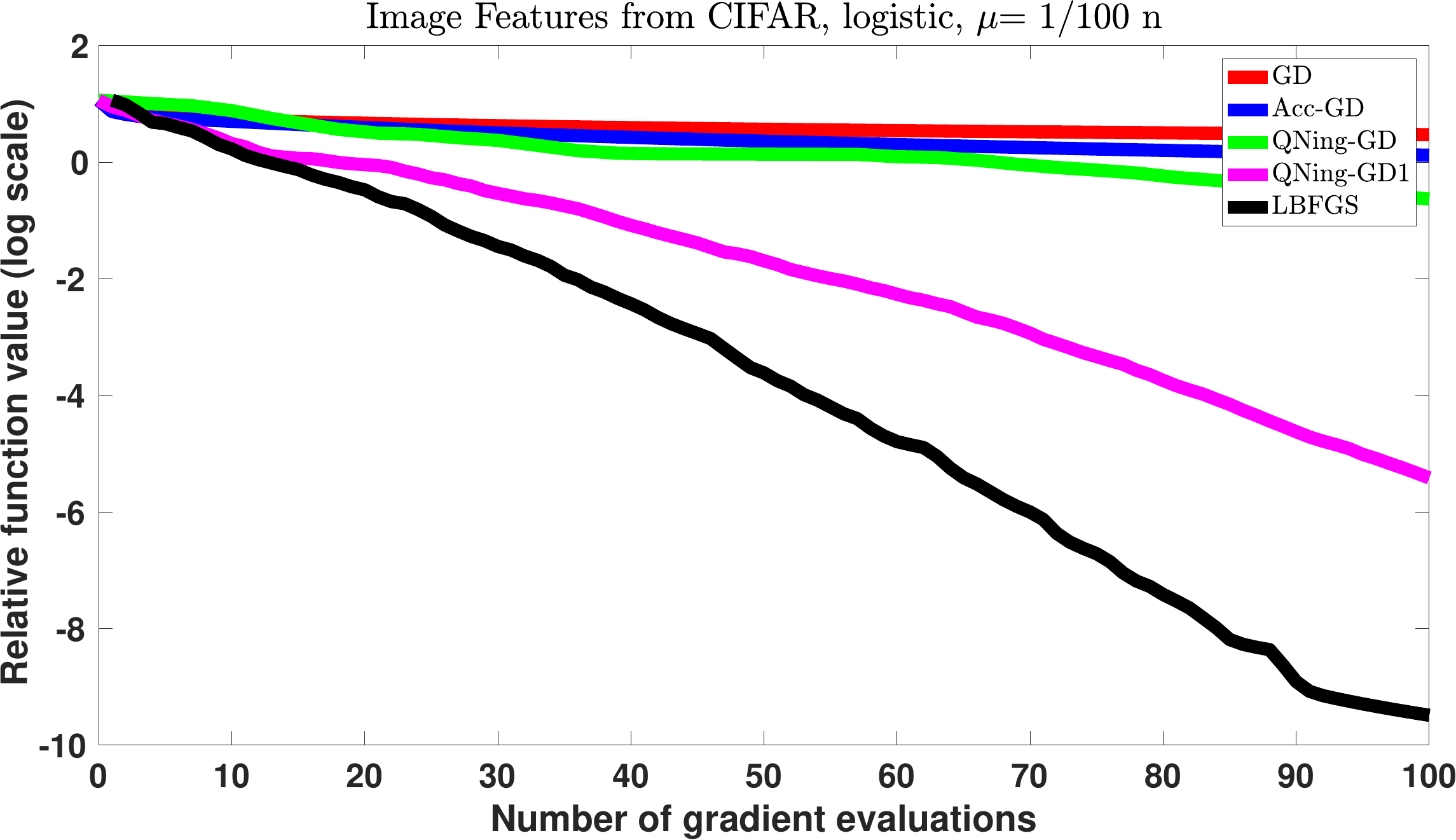} ~ 
   ~~\includegraphics[width=0.30\linewidth]{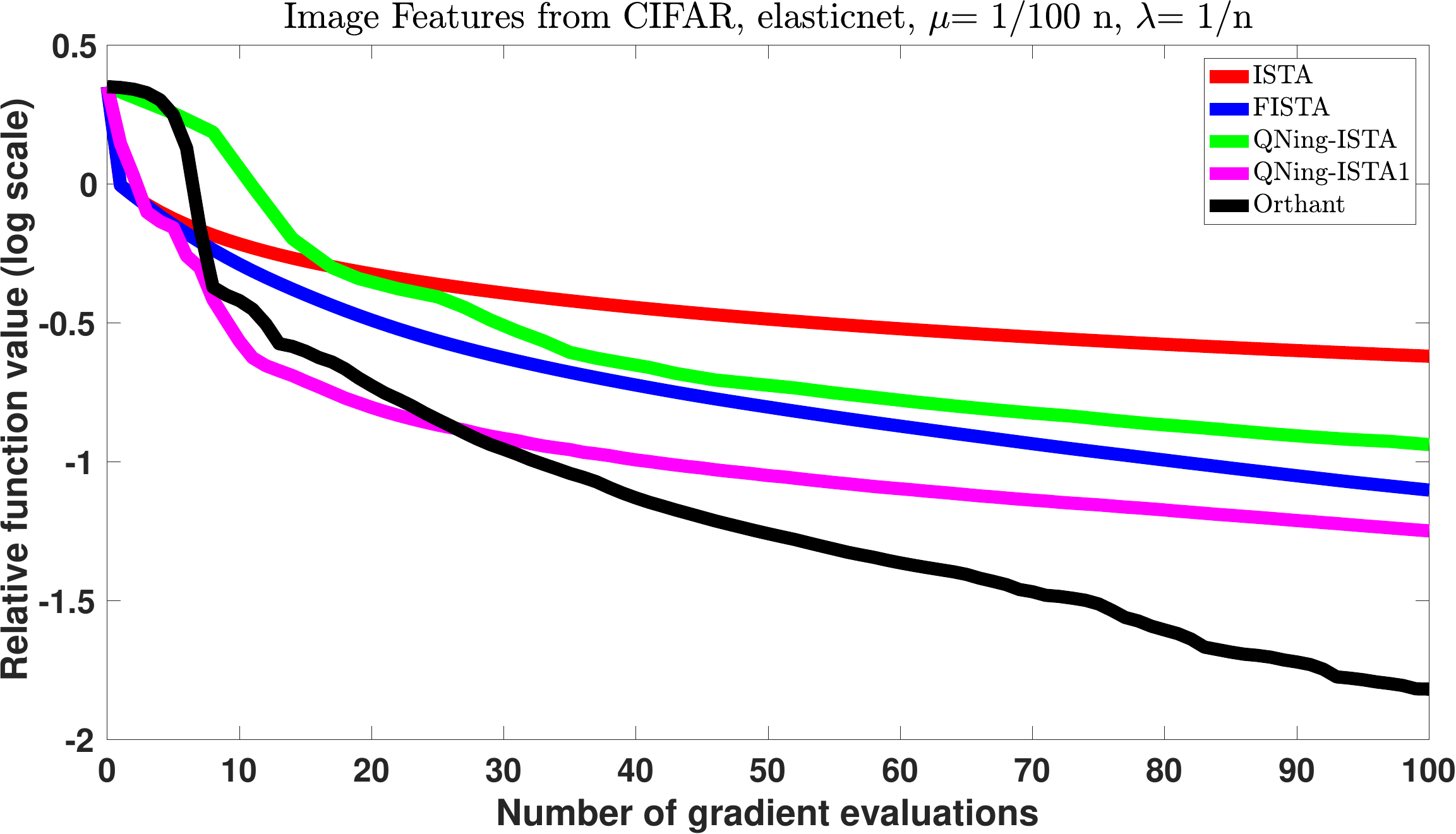} ~ 
   ~~\includegraphics[width=0.30\linewidth]{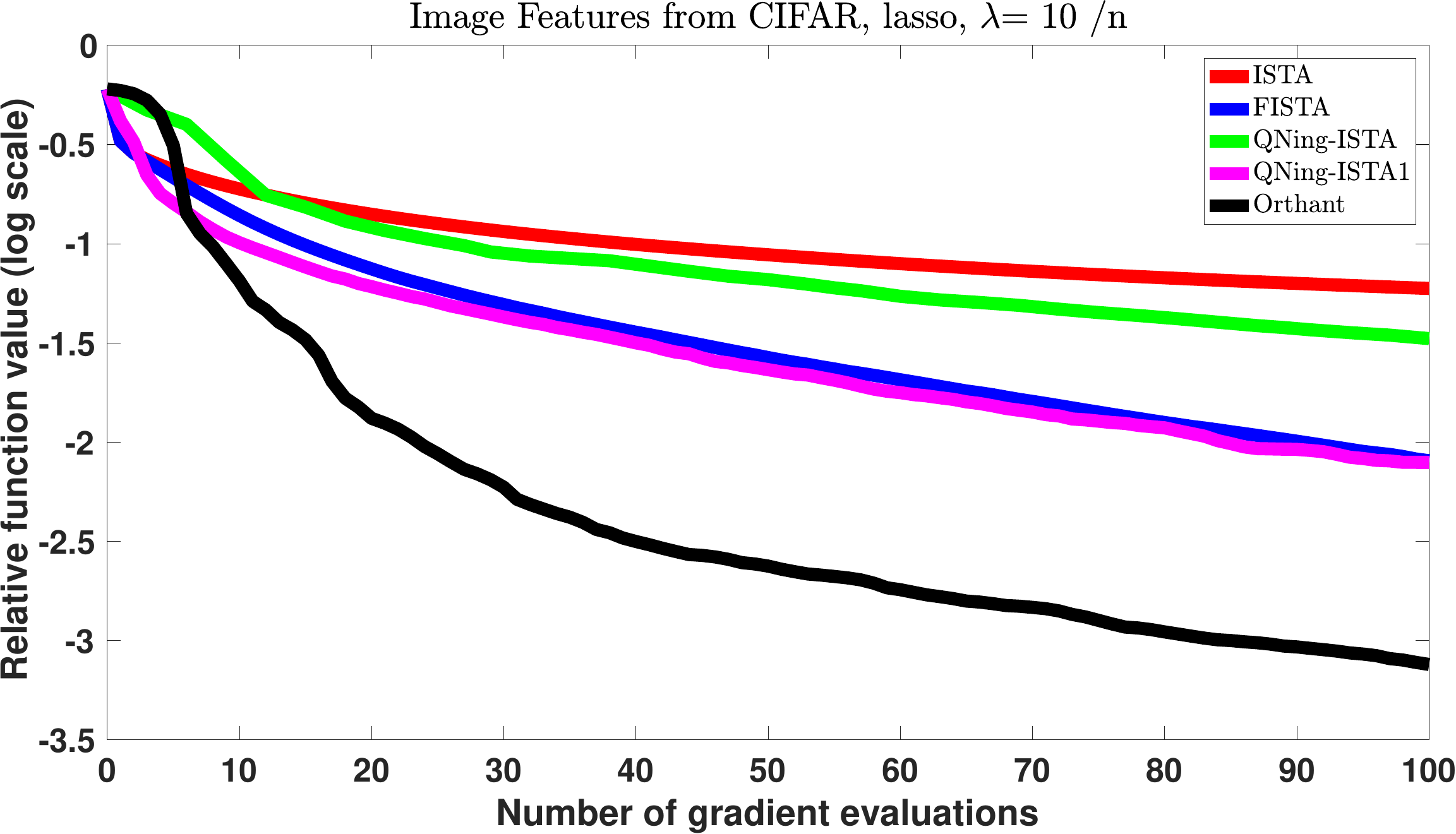} \\
   \caption{Experimental study of the performance of \qning-ISTA. We plot the value~$F(x_k)/F^\star-1$ as
   a function of the number of gradient evaluations, on a logarithmic scale;
   the optimal value $F^\star$ is estimated with a duality gap.  }\label{fig:gd mu=100}
\end{figure}
The results are reported in Figure~\ref{fig:gd mu=100} and lead to the following conclusions
\begin{itemize}
   \item \textbf{L-BFGS} is slightly better on average than \textbf{\qning-ISTA1} for smooth problems, which is not surprising since we use a state-of-the-art implementation with a well-calibrated line search.
   \item \textbf{\qning-ISTA1} is always significantly faster than \textbf{ISTA} and \textbf{\qning-ISTA}. 
   \item The \textbf{\qning-ISTA} approaches are significantly faster than \textbf{FISTA} in 12 cases out of 15.
   \item There is no clear conclusion regarding the performance of the \textbf{Orthant-wise} method vs other approaches. For three datasets, \textsf{covtype, alpha} and  \textsf{mnist}, \textbf{\qning-ISTA} is significantly better than \textbf{Orthant-wise}. However, on the other two datasets, the behavior is different \textbf{Orthant-wise} method outperforms \textbf{\qning-ISTA}. 
\end{itemize}

\subsection{Experimental study of hyper-parameters~$l$ and~$\kappa$}\label{subsec:param}
In this section, we study the influence of the limited memory parameter $l$ and
of the regularization parameter~$\kappa$ in \qning. More precisely, we
start with the parameter~$l$ and try the method \textbf{\qning-SVRG1} with
the values $l=1,2,5,10,20,100$. Note that all previous experiments were
conducted with $l=100$, which is the most expensive in terms of memory and
computational cost for the L-BFGS step. The results are presented in
Figure~\ref{fig:studyl}.  Interestingly, the experiment suggests that having a
large value for~$l$ is not necessarily the best choice, especially for composite problems where
the solution is sparse, where~$l=10$ seems to perform reasonably well.

\begin{figure}[hbtp!]
\centering
   ~~\includegraphics[width=0.30\linewidth]{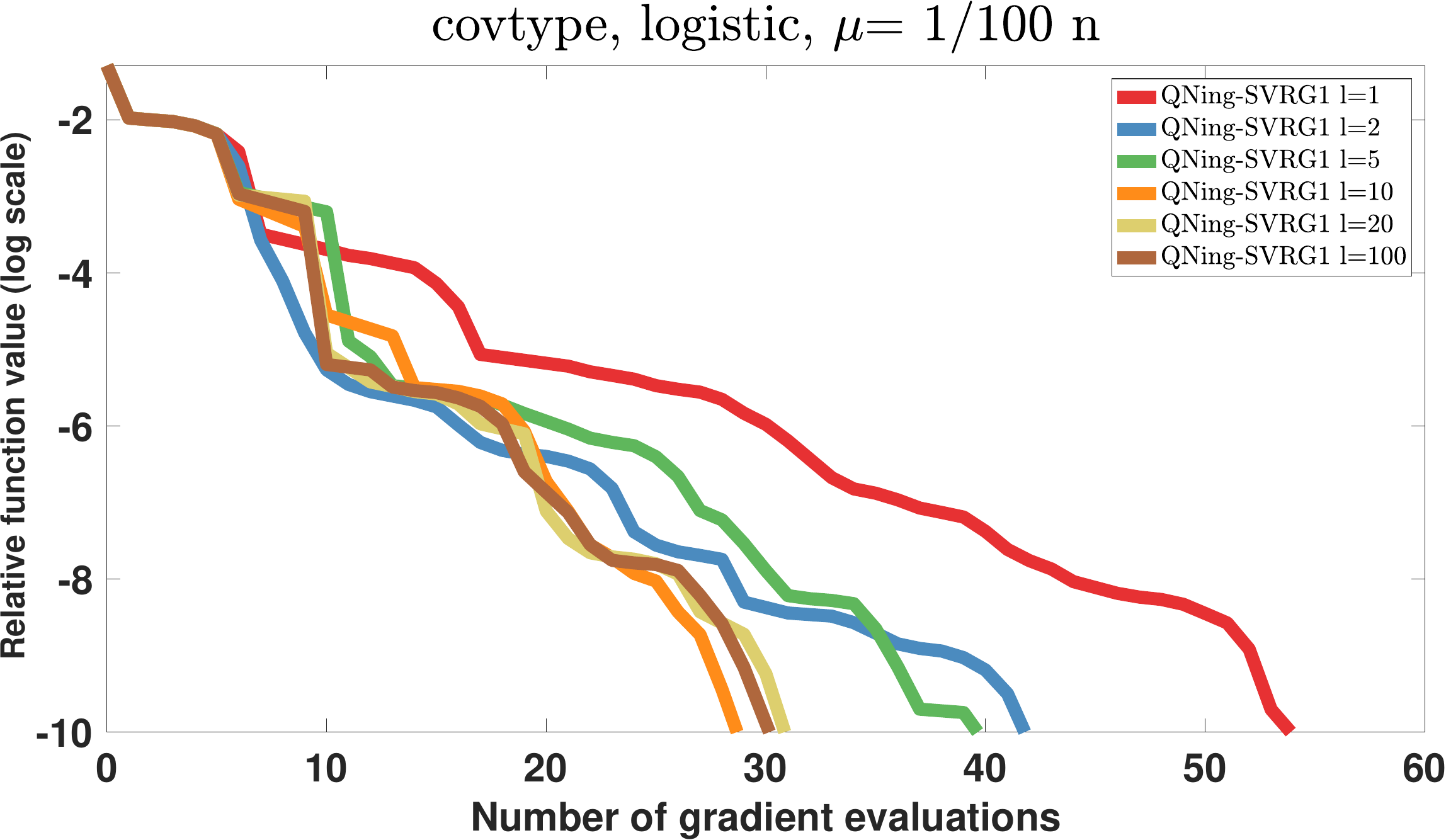} ~ 
   ~~\includegraphics[width=0.30\linewidth]{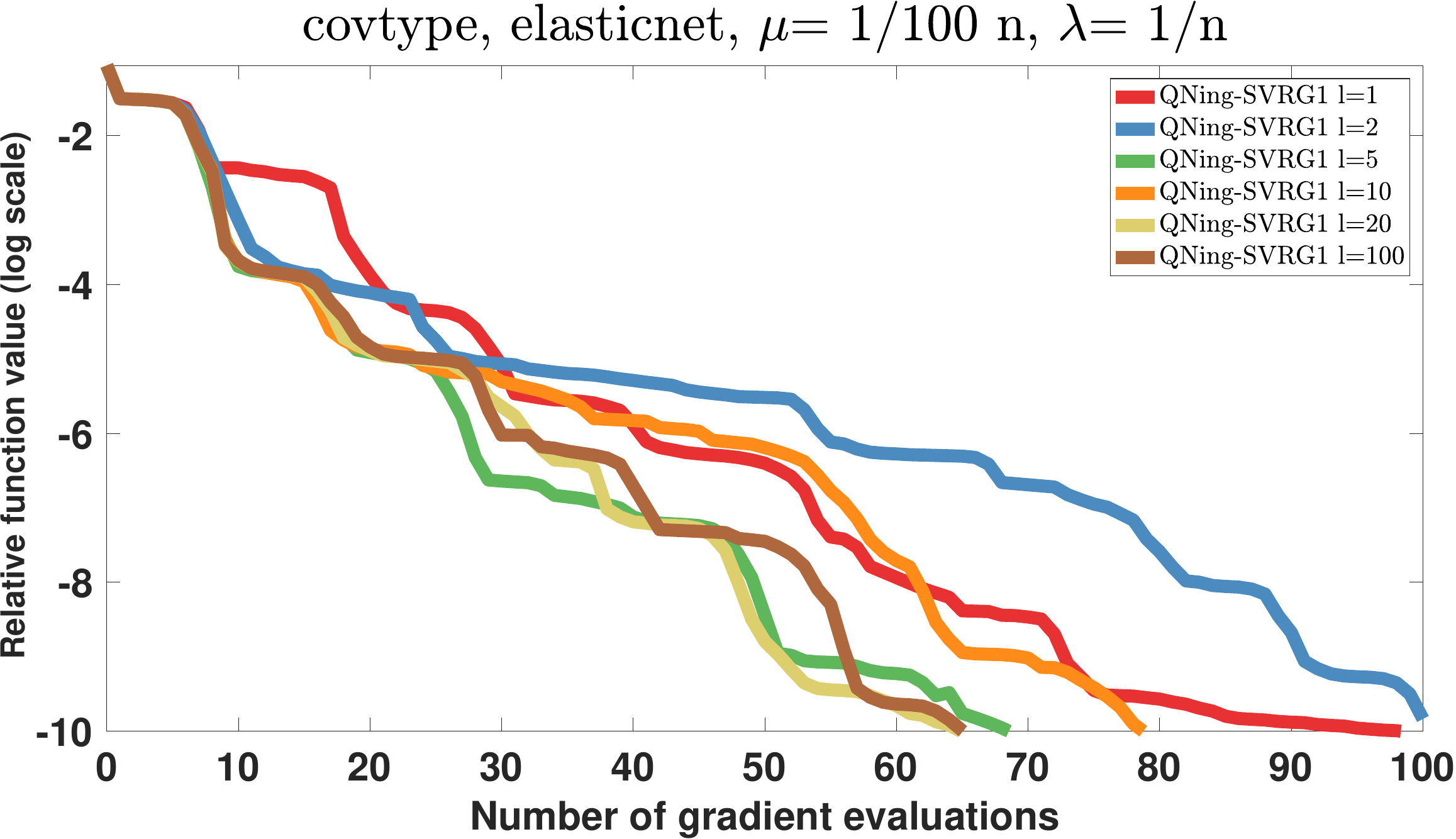} ~ 
   ~~\includegraphics[width=0.30\linewidth]{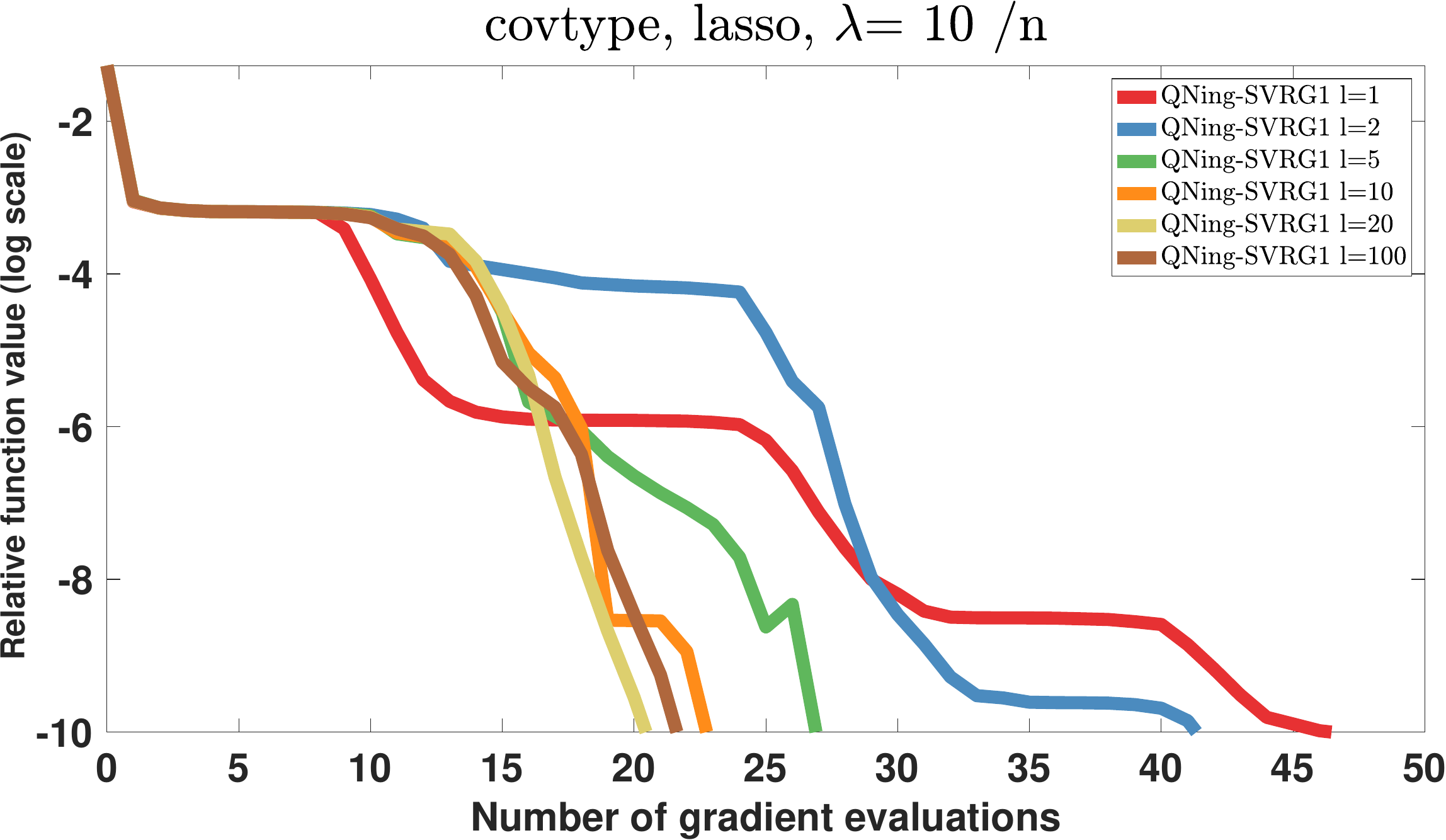}\\  
   ~~\includegraphics[width=0.30\linewidth]{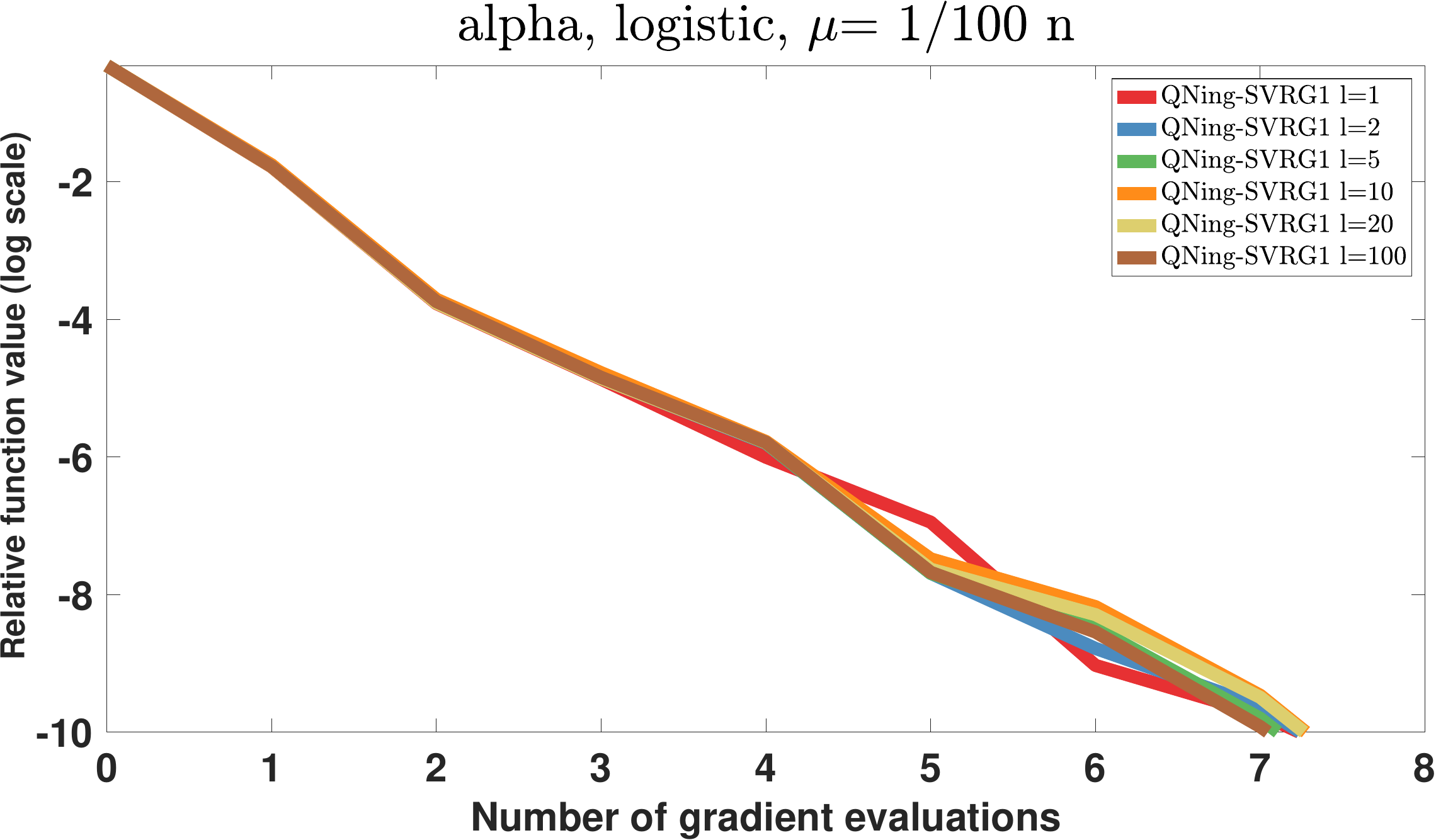} ~ 
   ~~\includegraphics[width=0.30\linewidth]{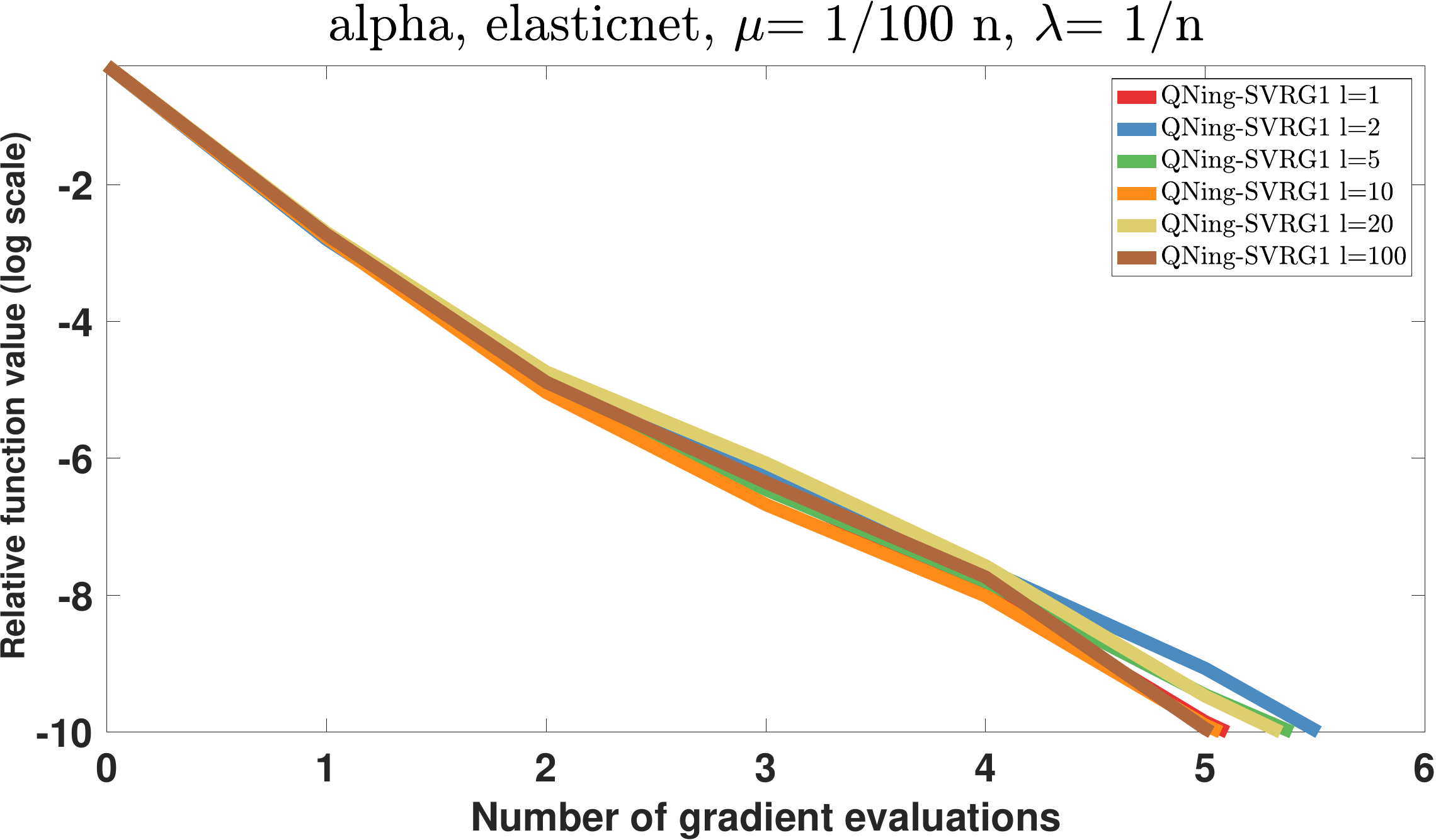} ~ 
   ~~\includegraphics[width=0.30\linewidth]{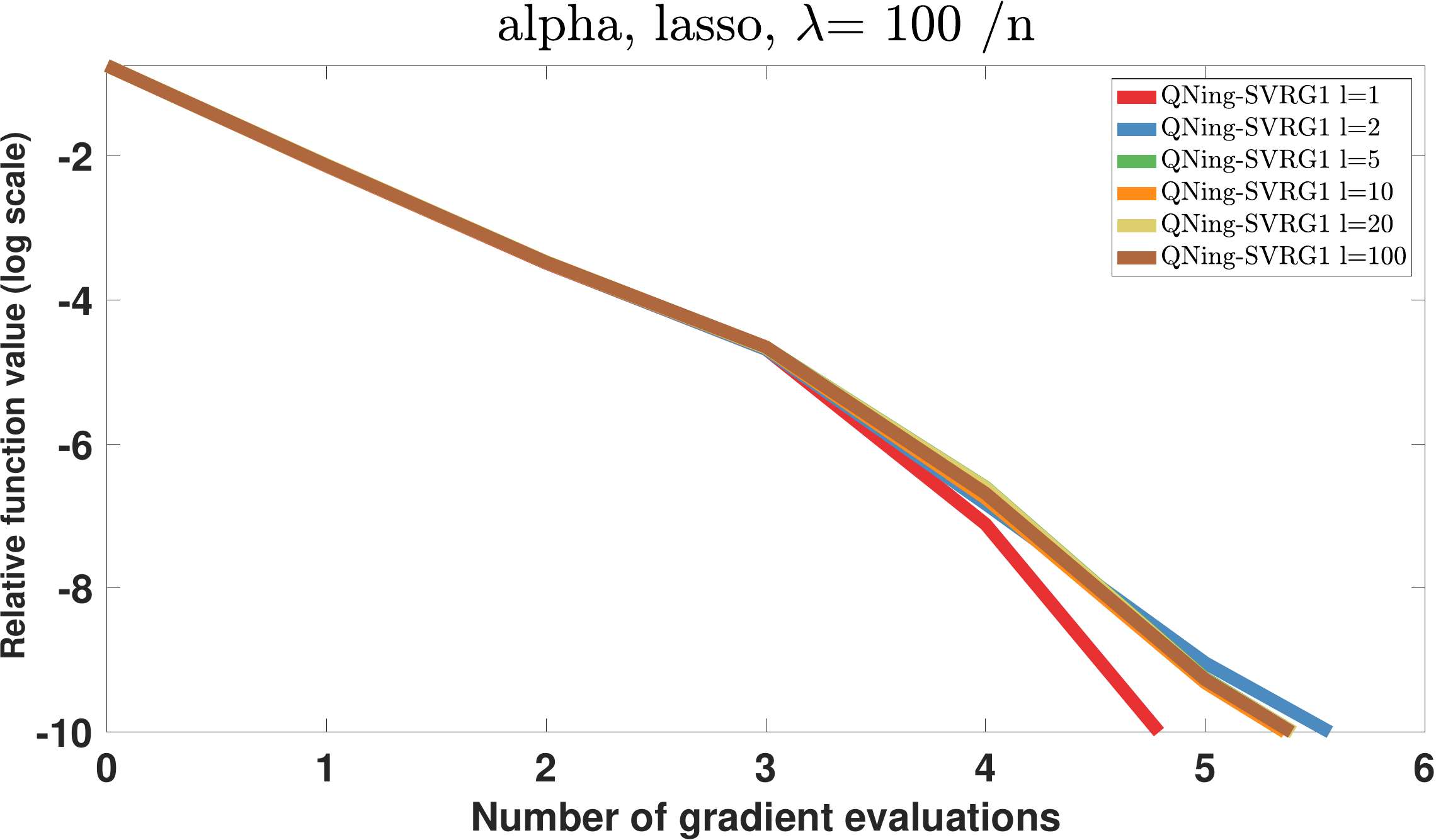}\\
   ~~\includegraphics[width=0.30\linewidth]{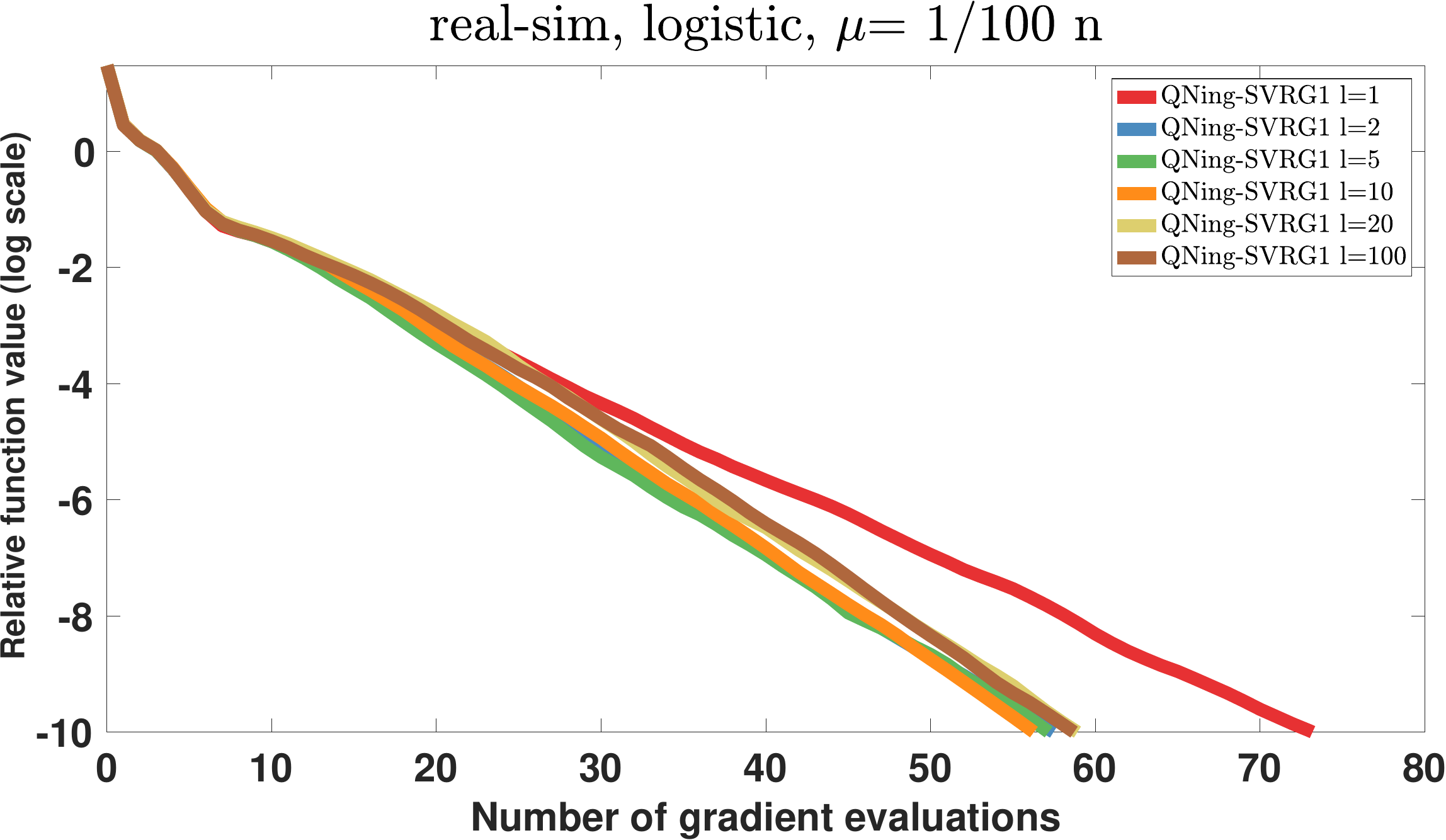} ~ 
   ~~\includegraphics[width=0.30\linewidth]{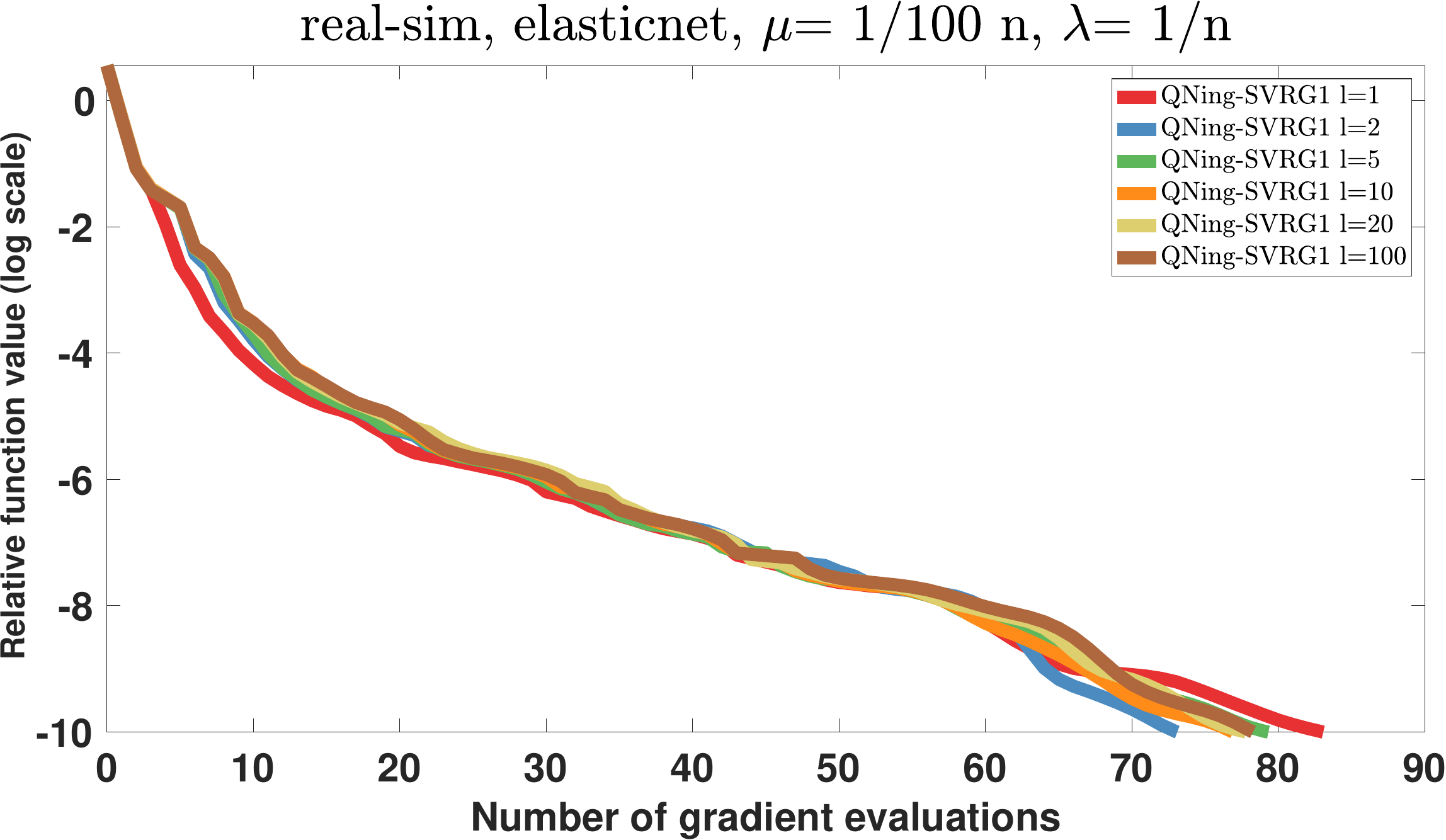} ~ 
   ~~\includegraphics[width=0.30\linewidth]{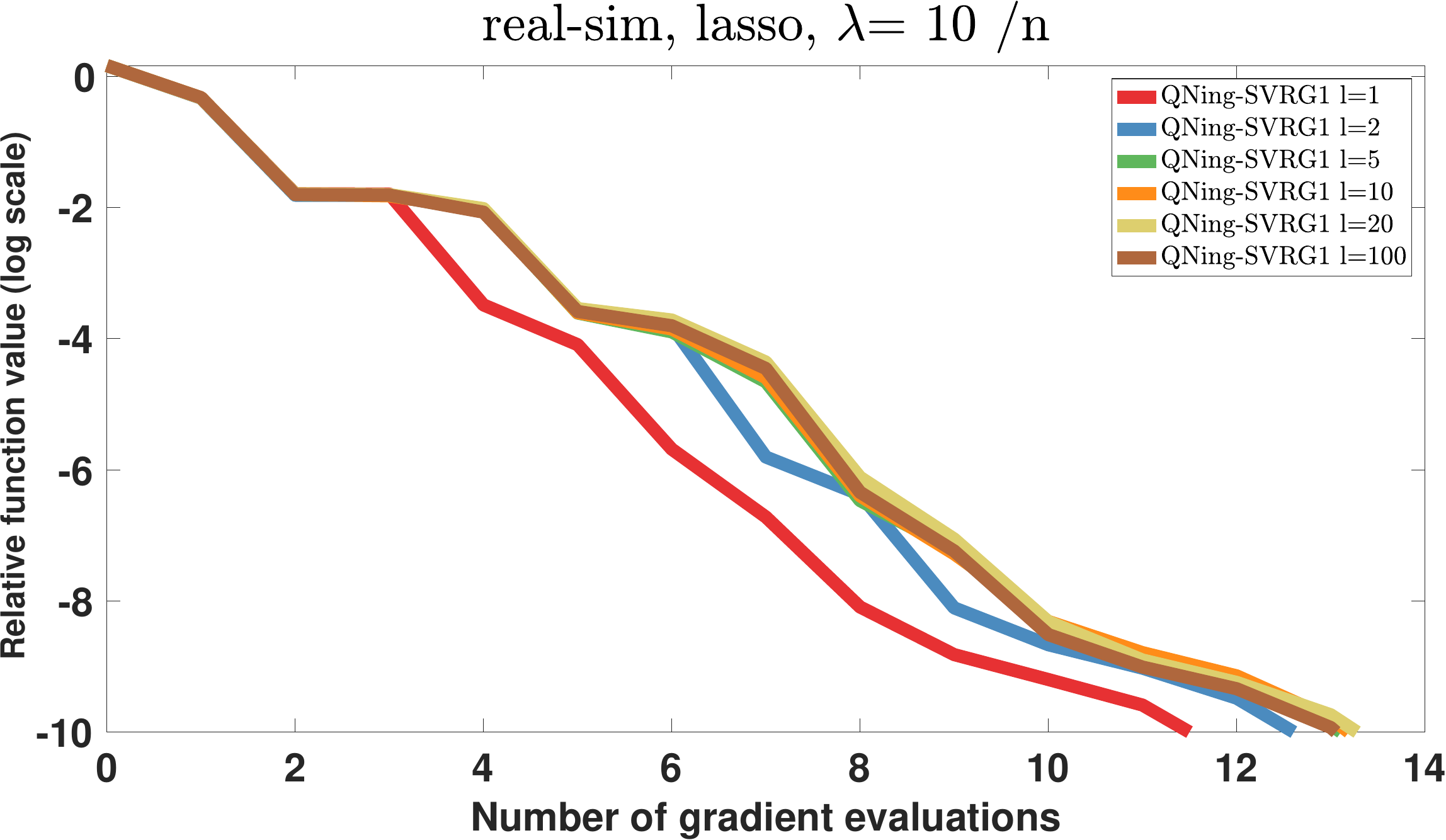} \\
   ~~\includegraphics[width=0.30\linewidth]{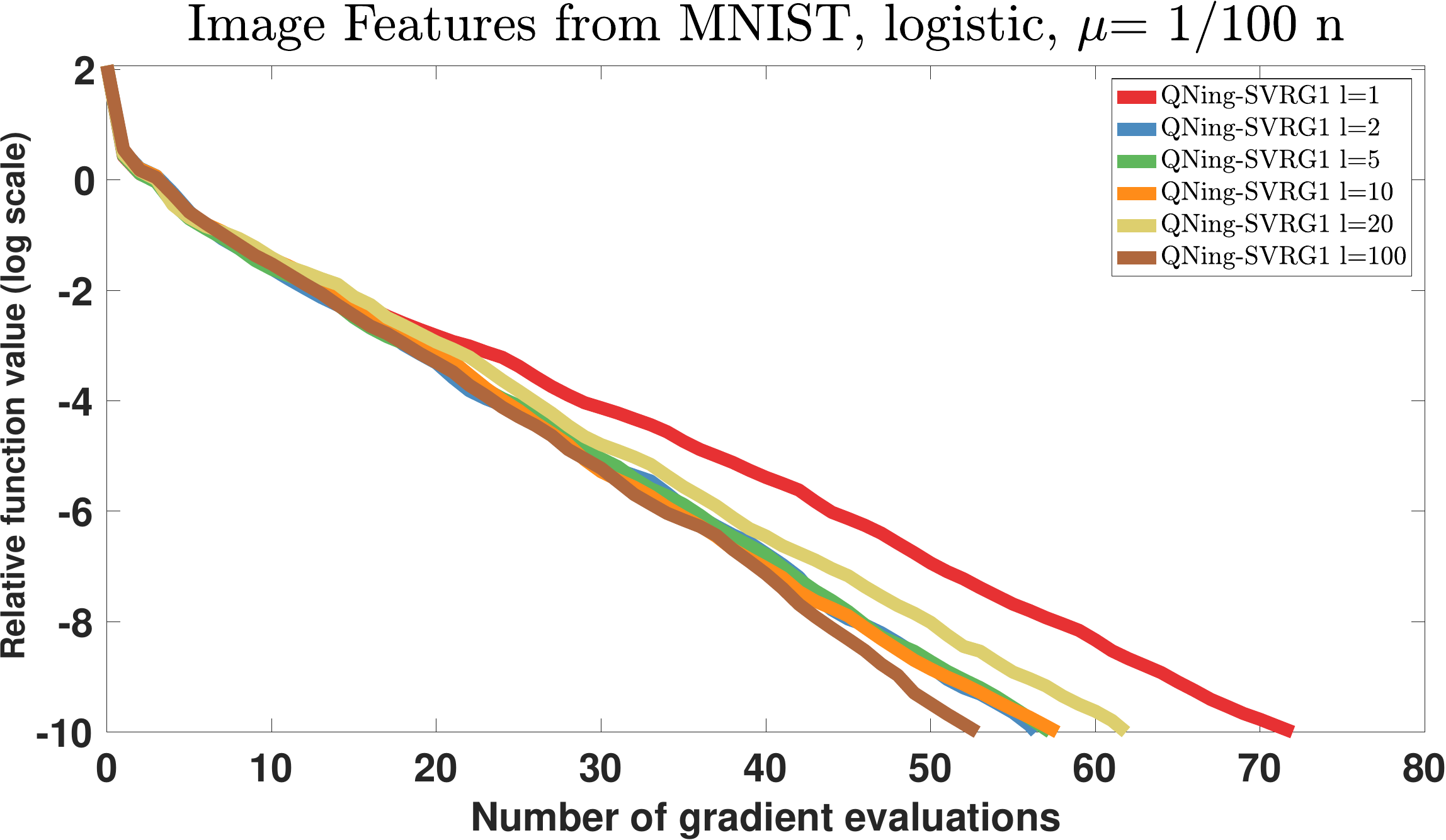} ~ 
   ~~\includegraphics[width=0.30\linewidth]{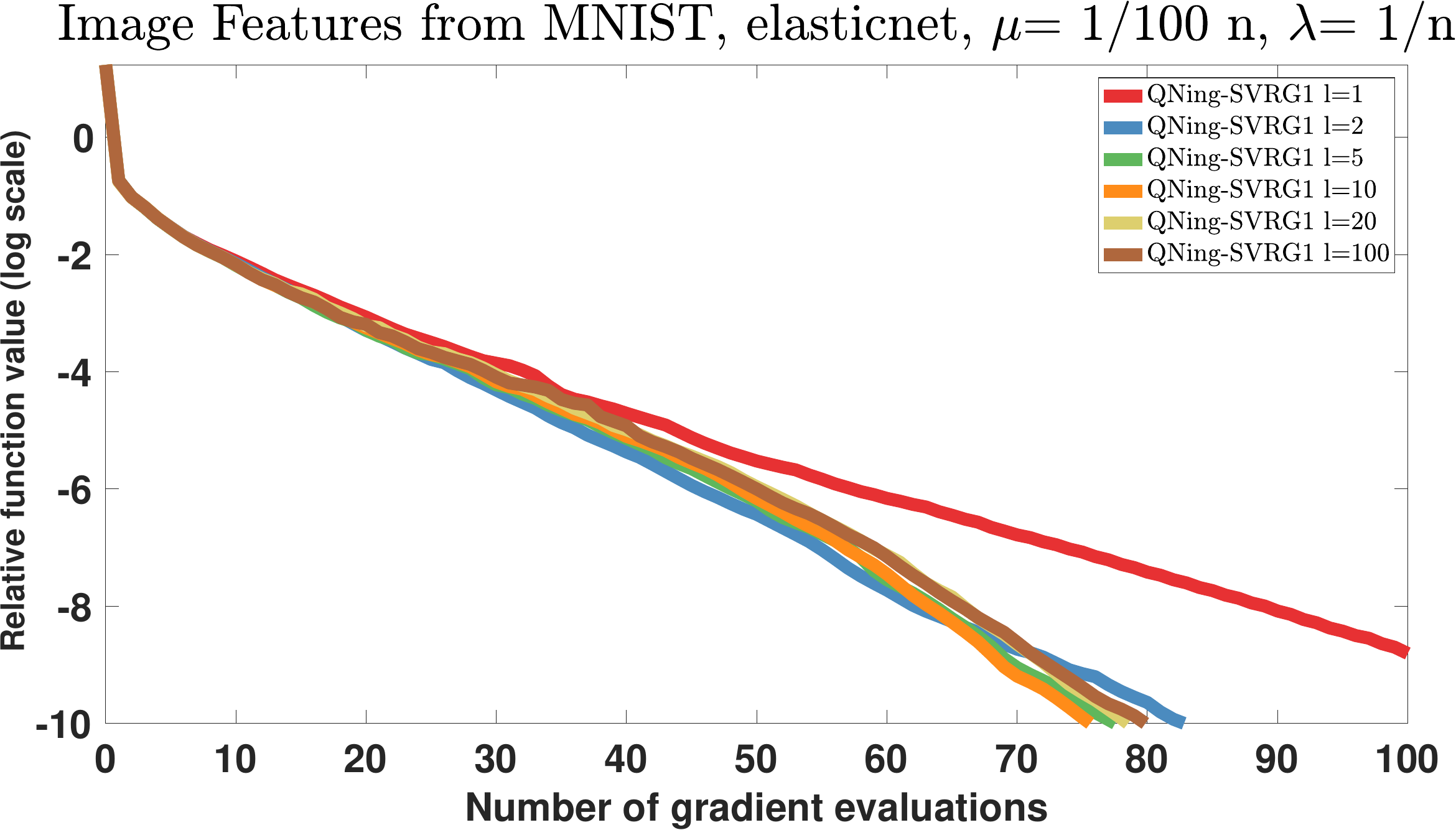} ~ 
   ~~\includegraphics[width=0.30\linewidth]{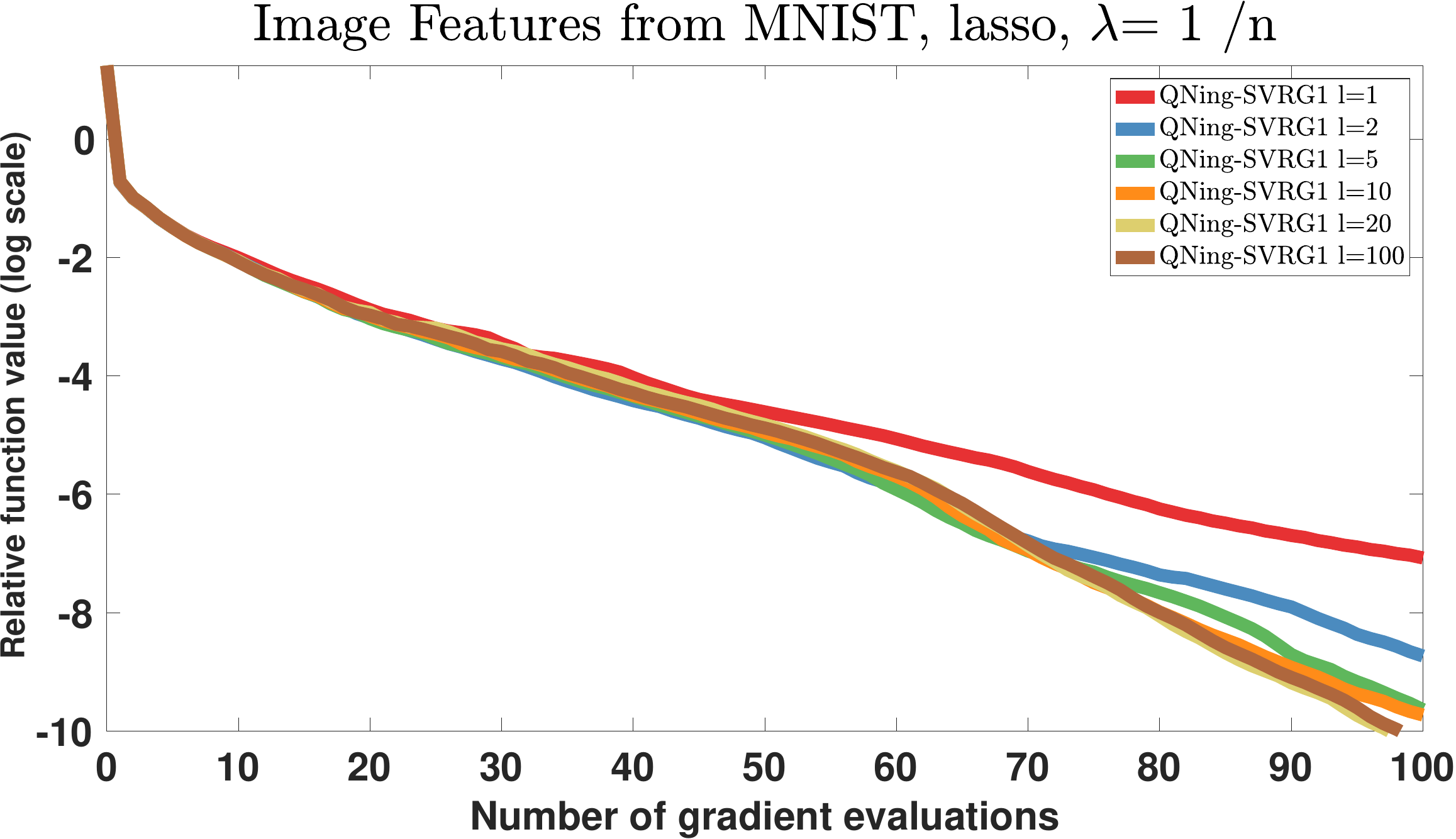}\\
   ~~\includegraphics[width=0.30\linewidth]{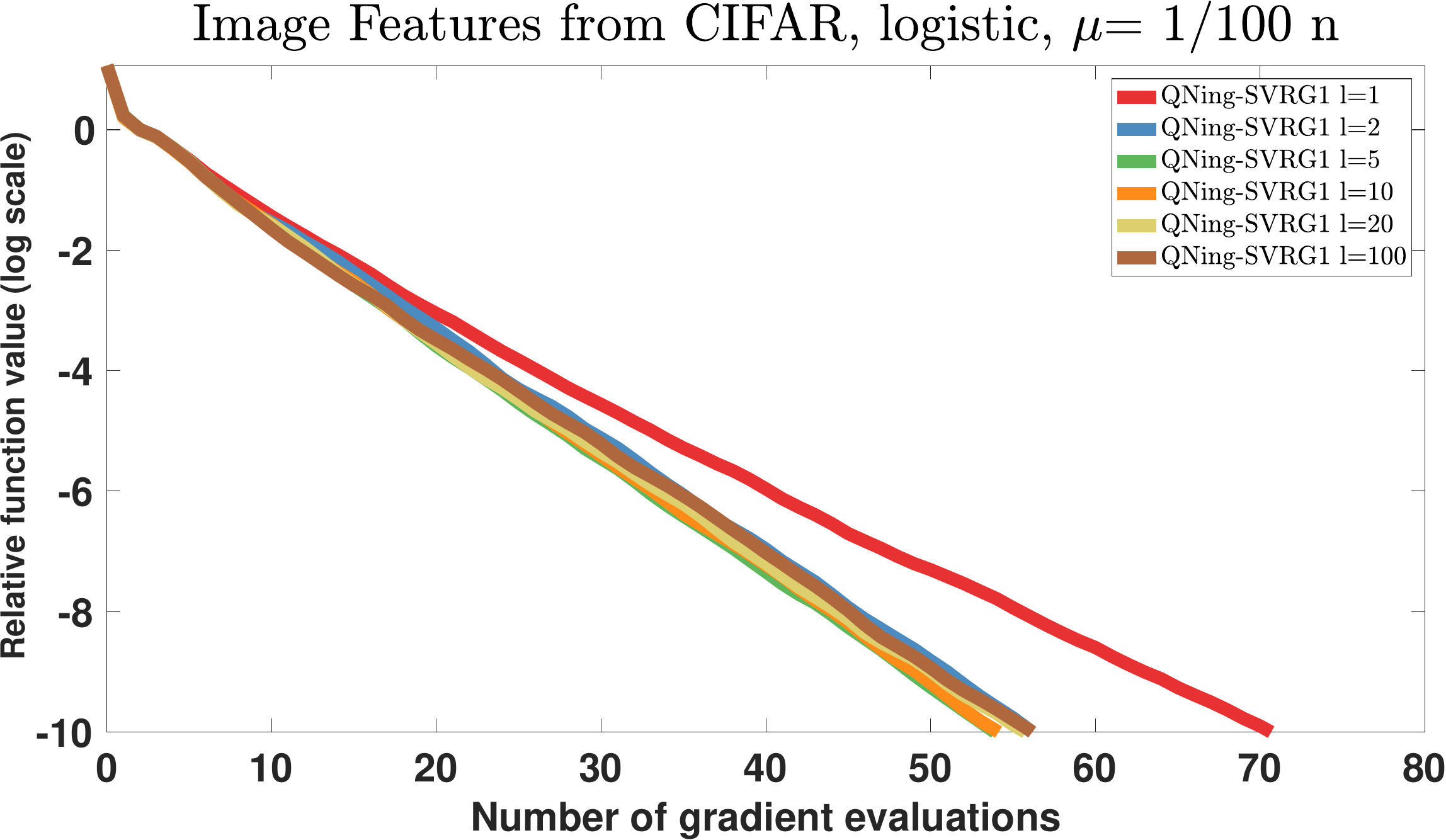} ~ 
   ~~\includegraphics[width=0.30\linewidth]{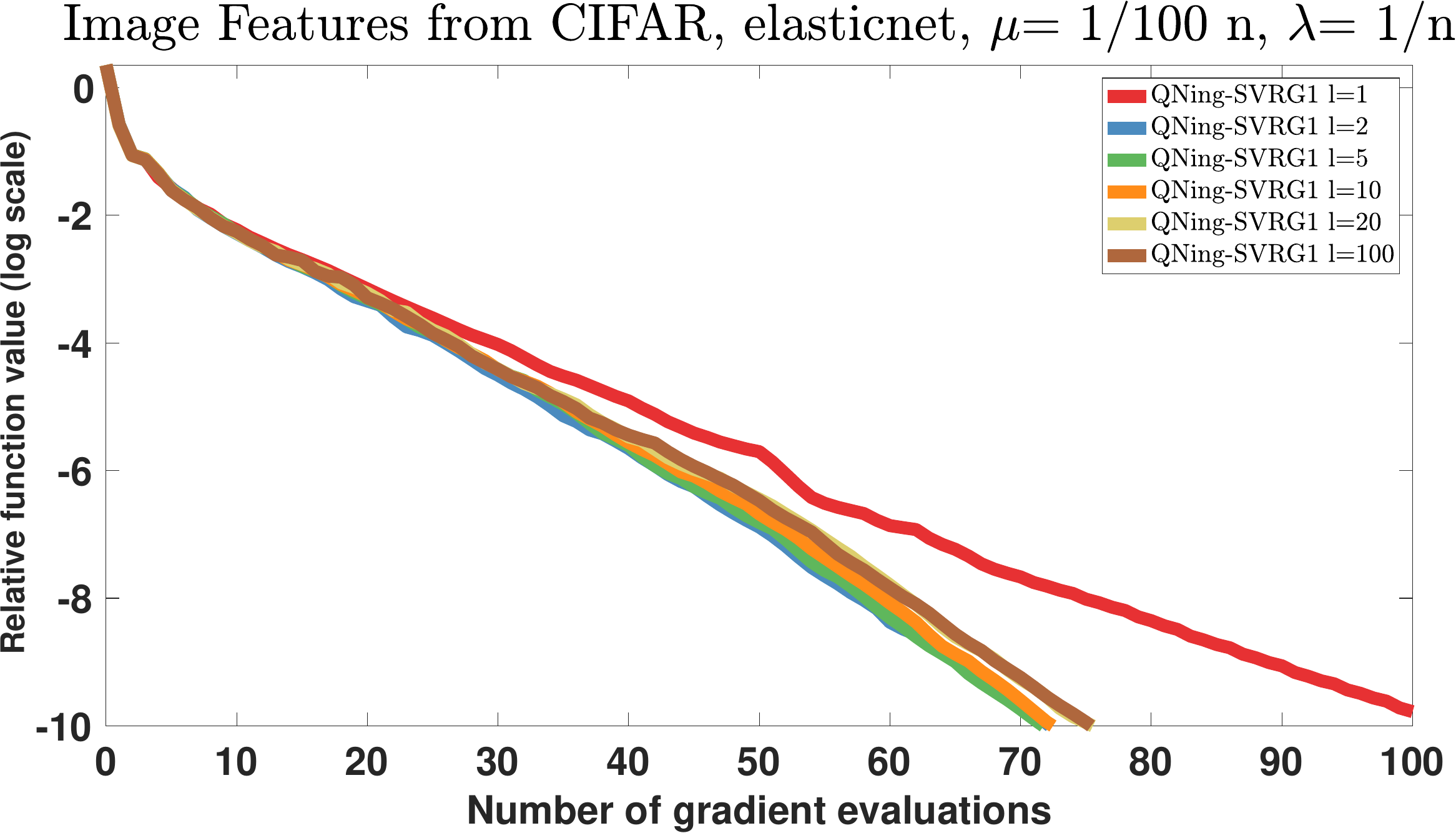} ~ 
   ~~\includegraphics[width=0.30\linewidth]{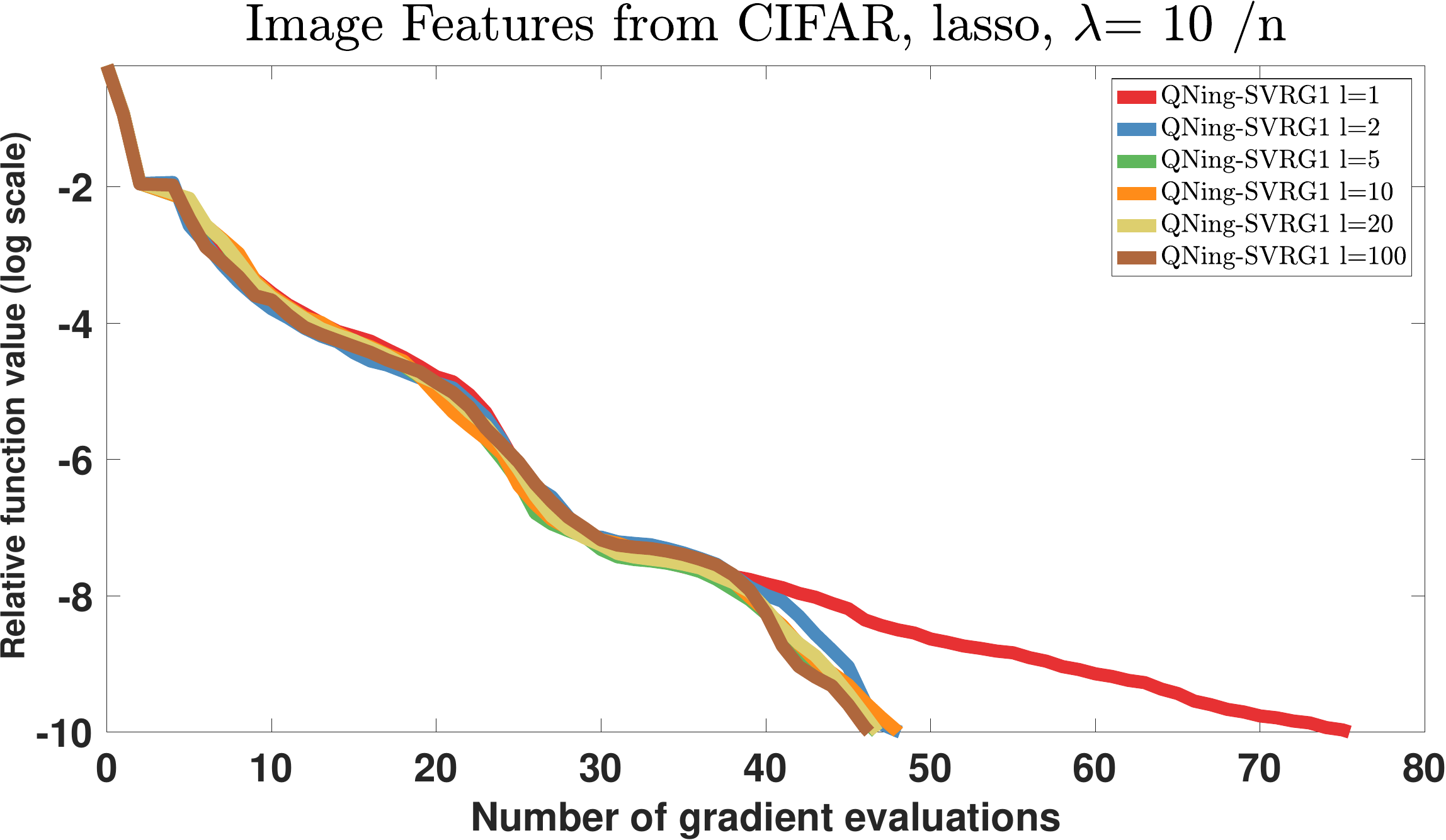} \\
   \caption{Experimental study of influence of the limited-memory parameter~$l$ for \qning-SVRG1. We plot the value~$F(x_k)/F^\star-1$ as
   a function of the number of gradient evaluations, on a logarithmic scale;
   the optimal value $F^\star$ is estimated with a duality gap.}\label{fig:studyl}
\end{figure}

The next experiment consists of studying the robustness of \qningsp to the
smoothing parameter~$\kappa$. We present in Figure~\ref{fig:kappa} an
experiment by trying the values $\kappa=10^i \kappa_0$, for
$i=-3,-2,\ldots,2,3$, where $\kappa_0$ is the default parameter that we used in the previous experiments.
The conclusion is clear: \qningsp clearly slows down when using a larger
smoothing parameter than~$\kappa_0$, but it is robust to small values
of~$\kappa$ (and in fact it even performs better for smaller values than
$\kappa_0$). 

\begin{figure}[hbtp!]
   \centering
   ~~\includegraphics[width=0.30\linewidth]{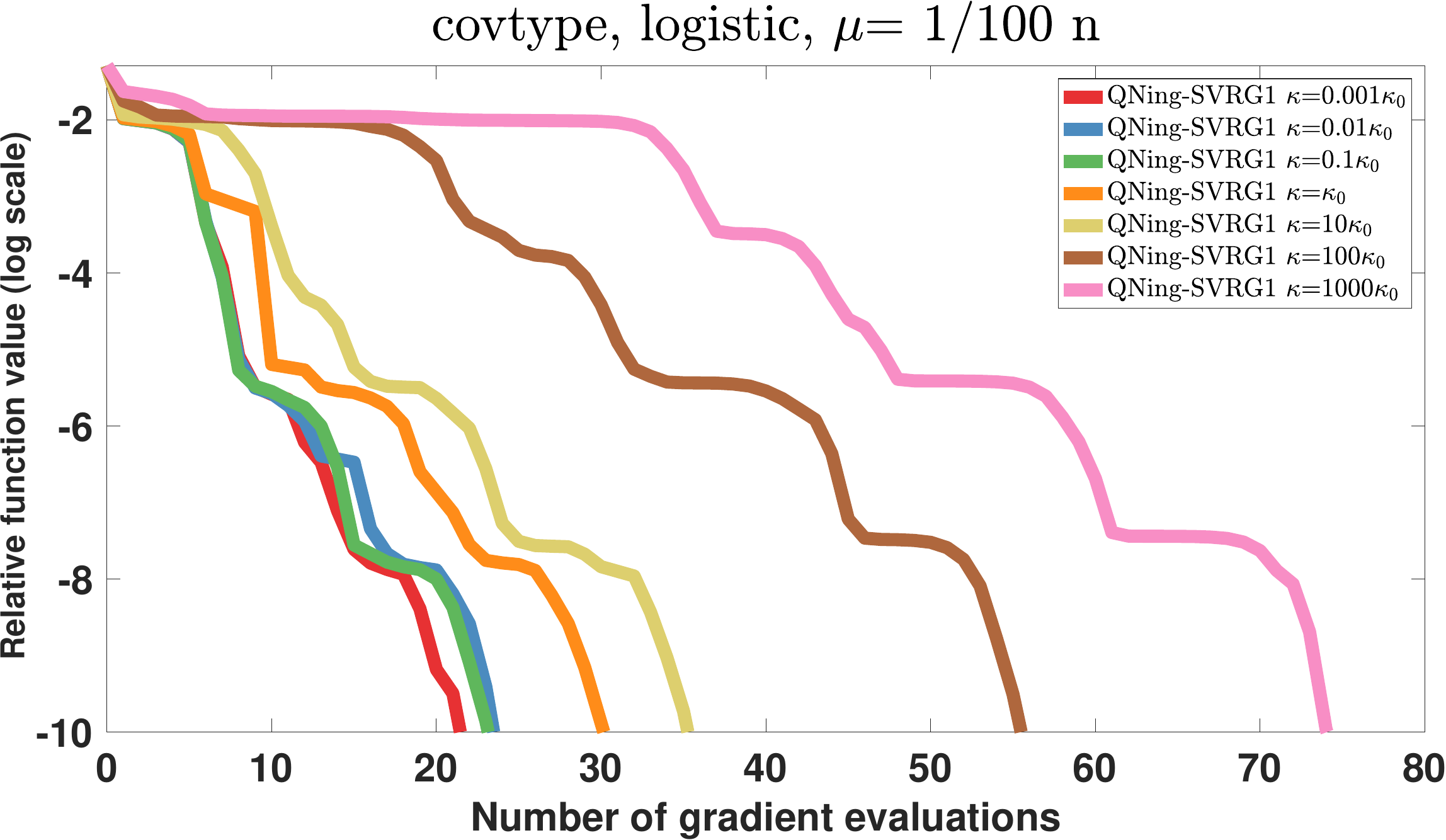} ~ 
   ~~\includegraphics[width=0.30\linewidth]{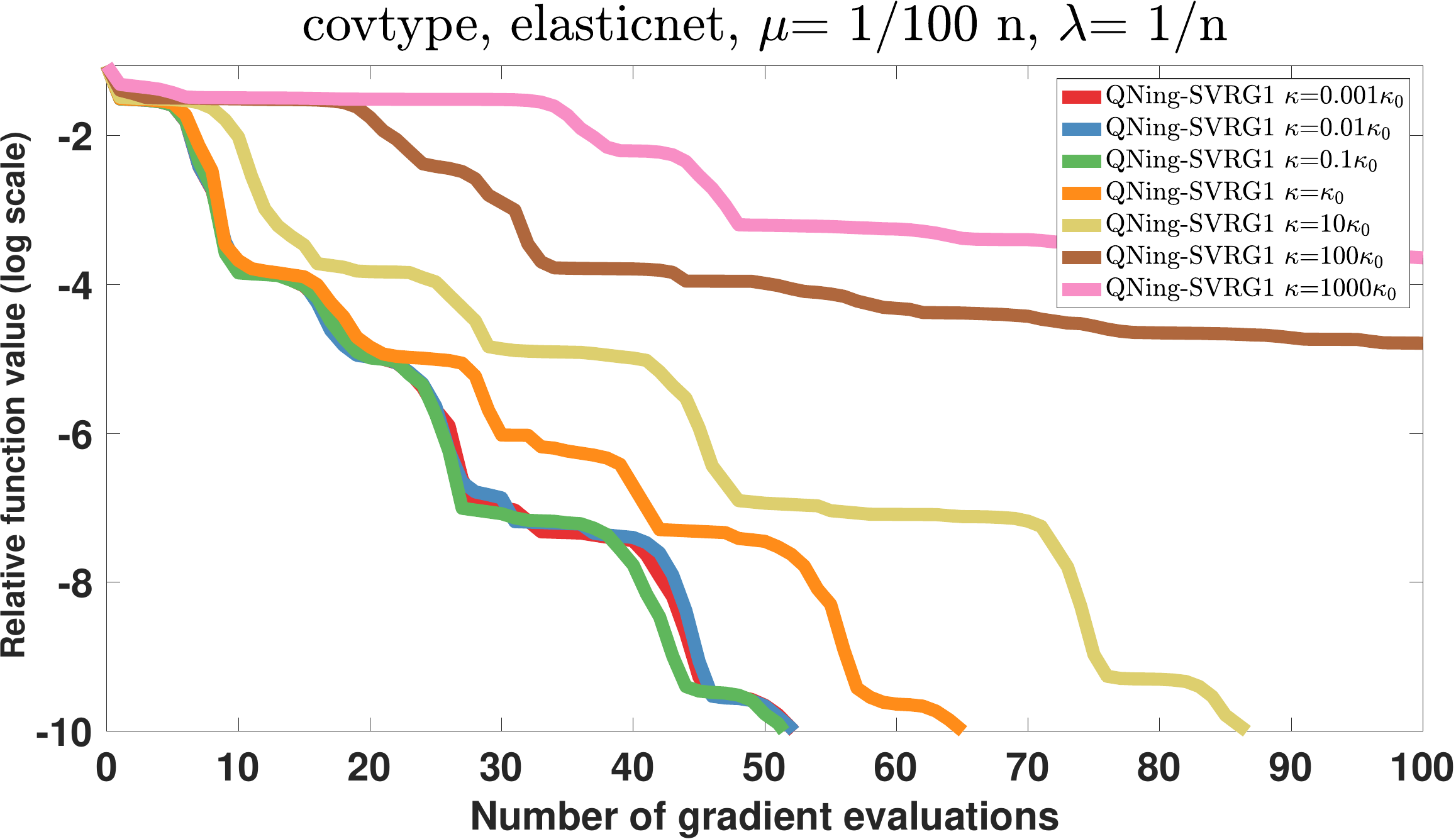} ~ 
   ~~\includegraphics[width=0.30\linewidth]{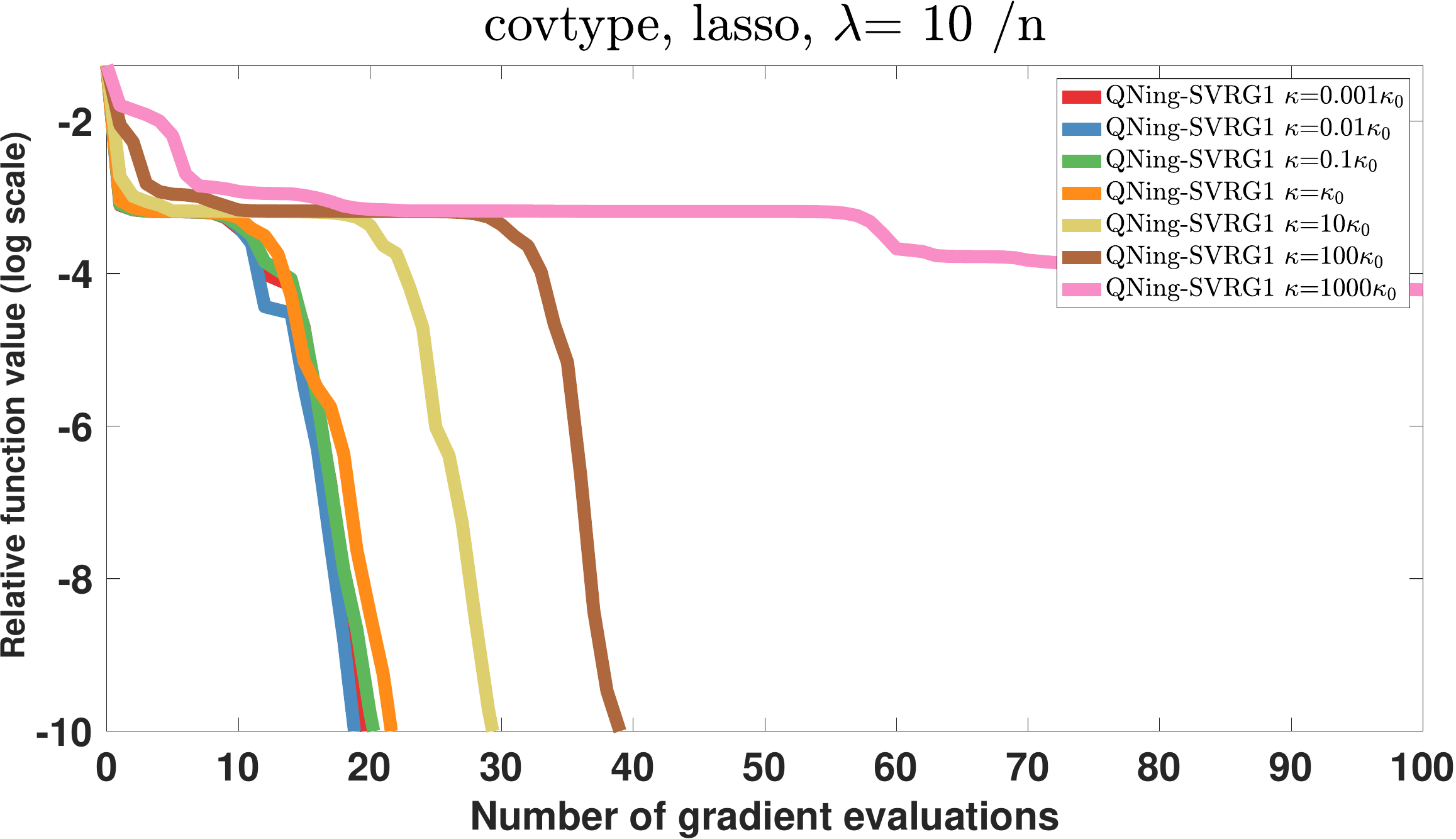} \\
   ~~\includegraphics[width=0.30\linewidth]{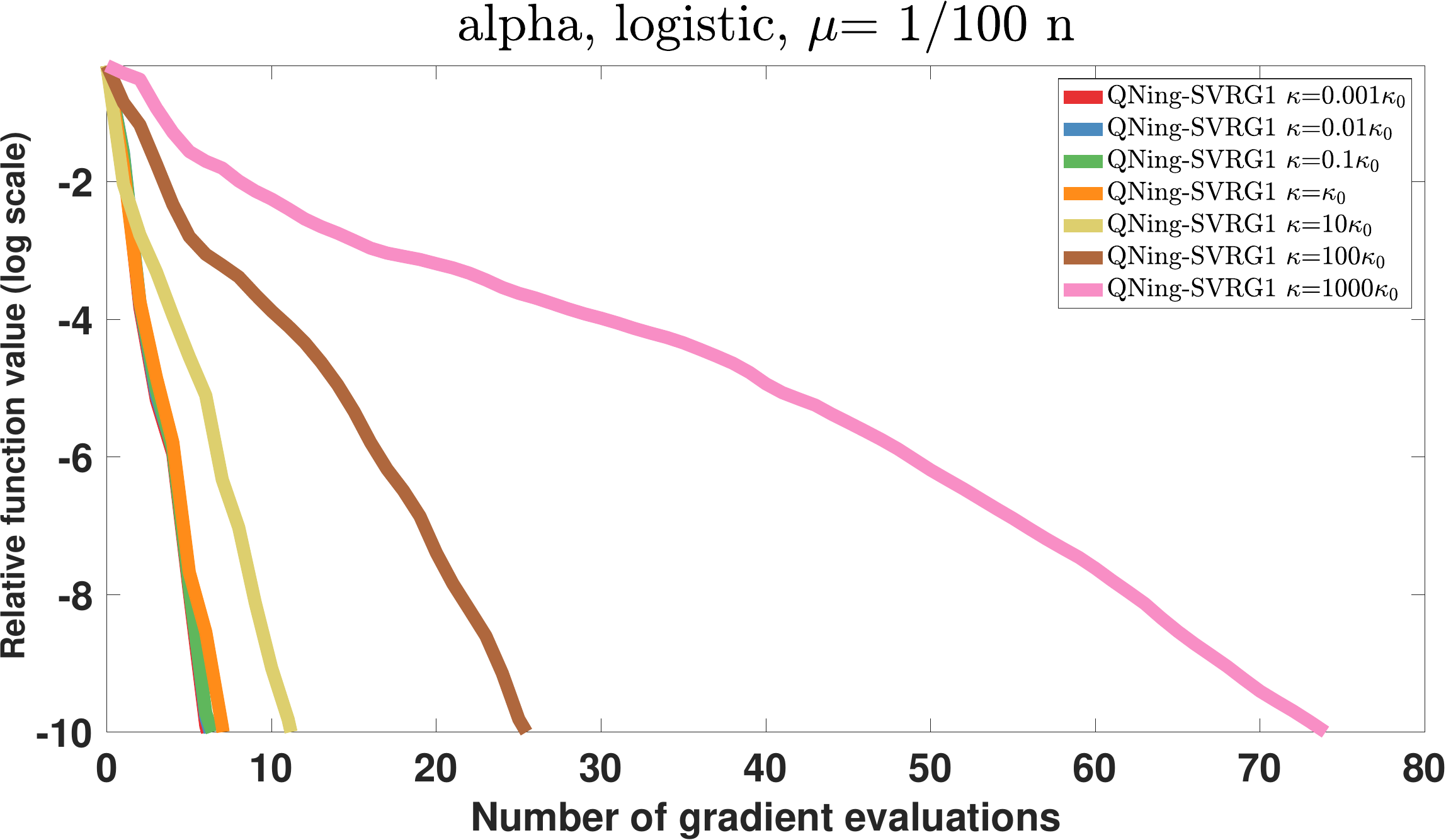} ~ 
   ~~\includegraphics[width=0.30\linewidth]{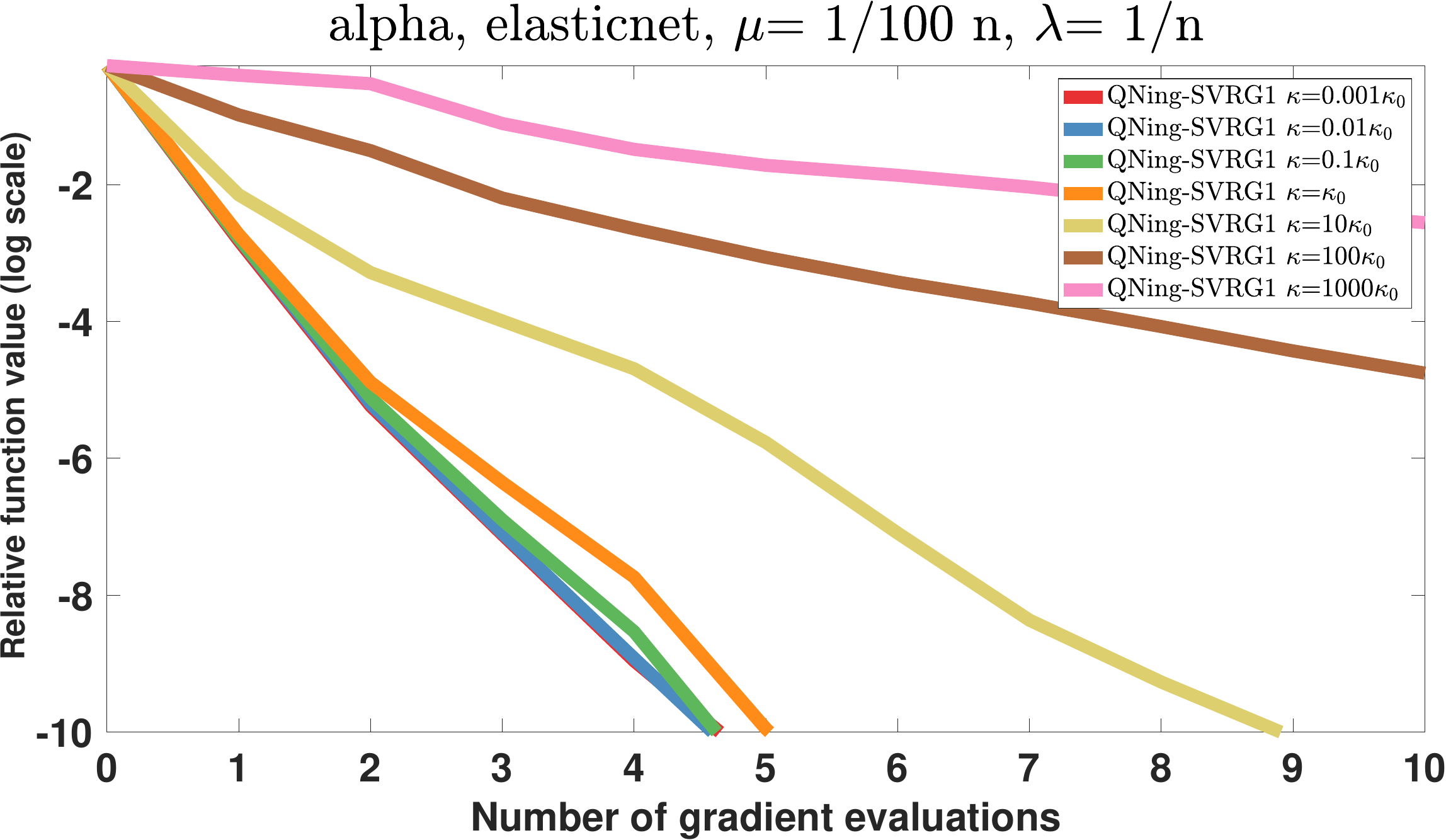} ~ 
   ~~\includegraphics[width=0.30\linewidth]{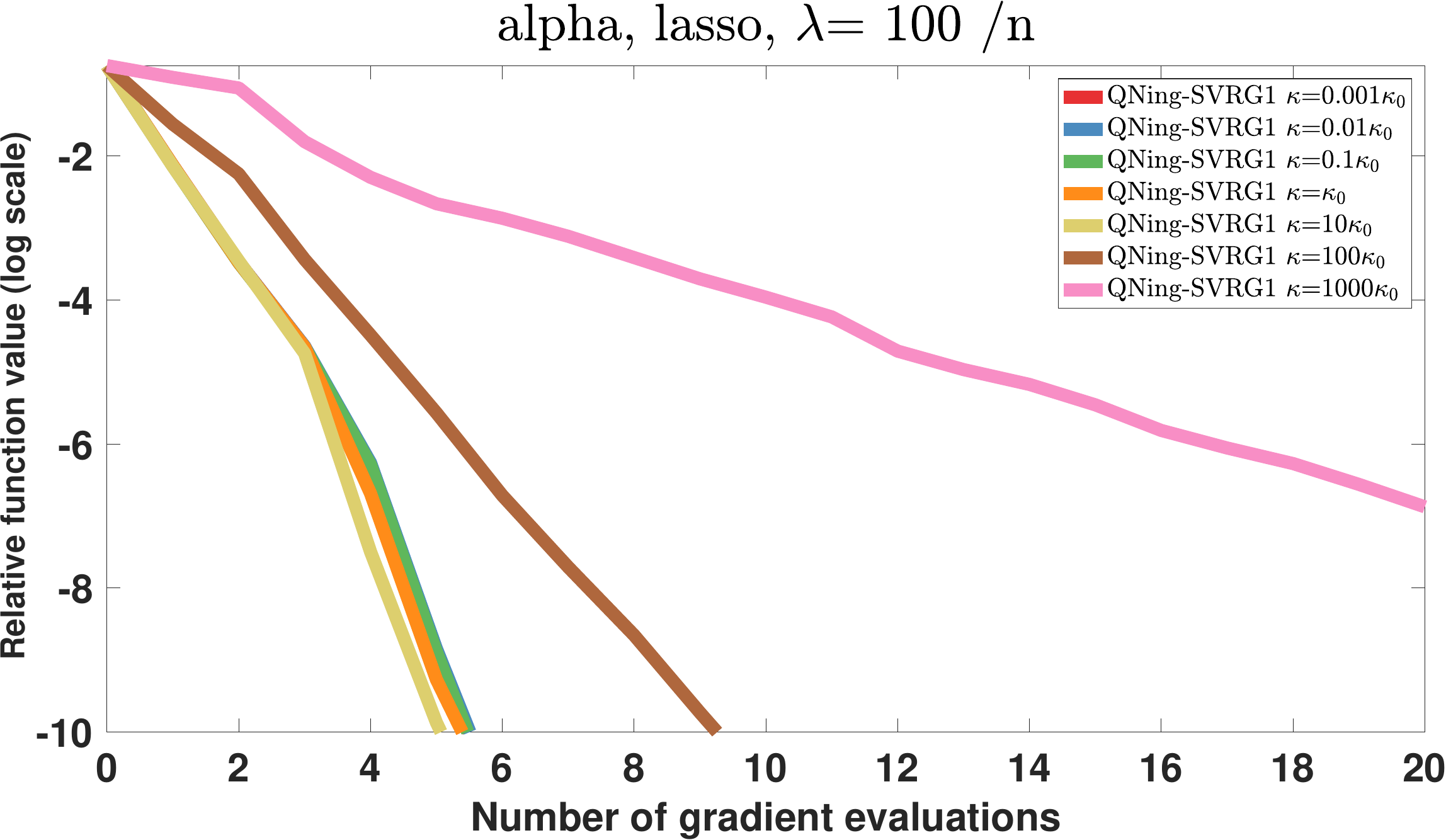}\\
   ~~\includegraphics[width=0.30\linewidth]{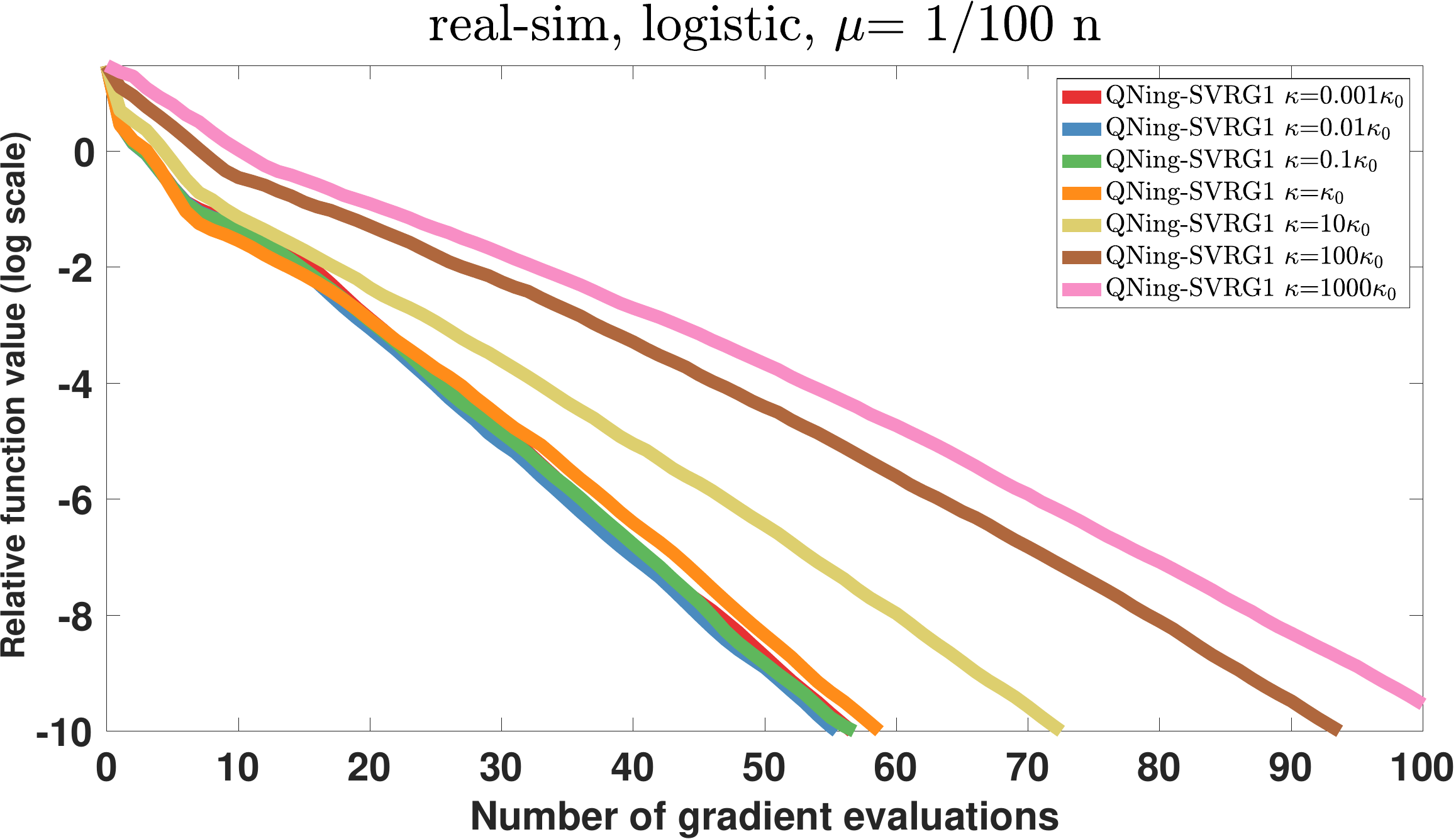} ~ 
   ~~\includegraphics[width=0.30\linewidth]{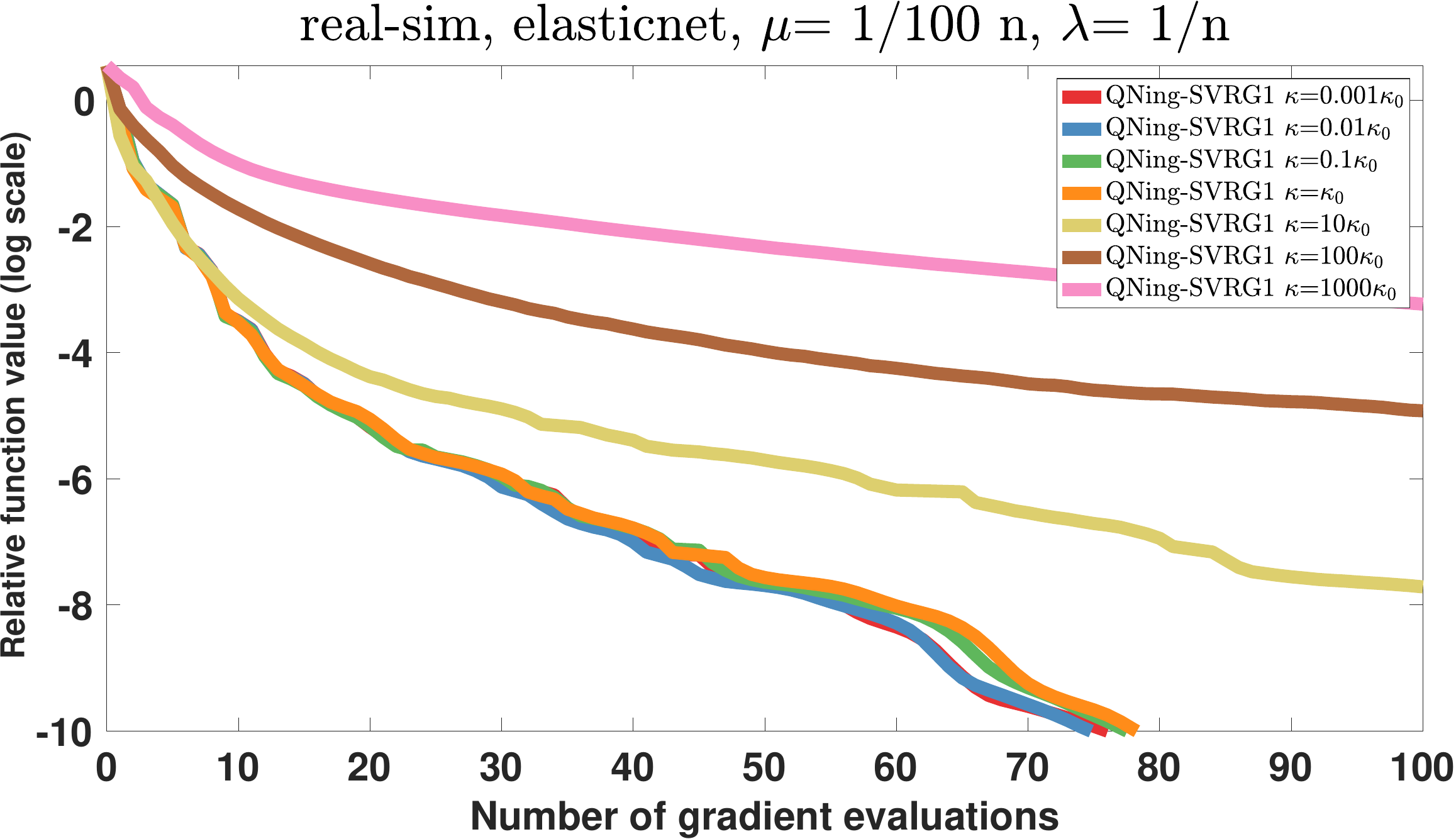} ~ 
   ~~\includegraphics[width=0.30\linewidth]{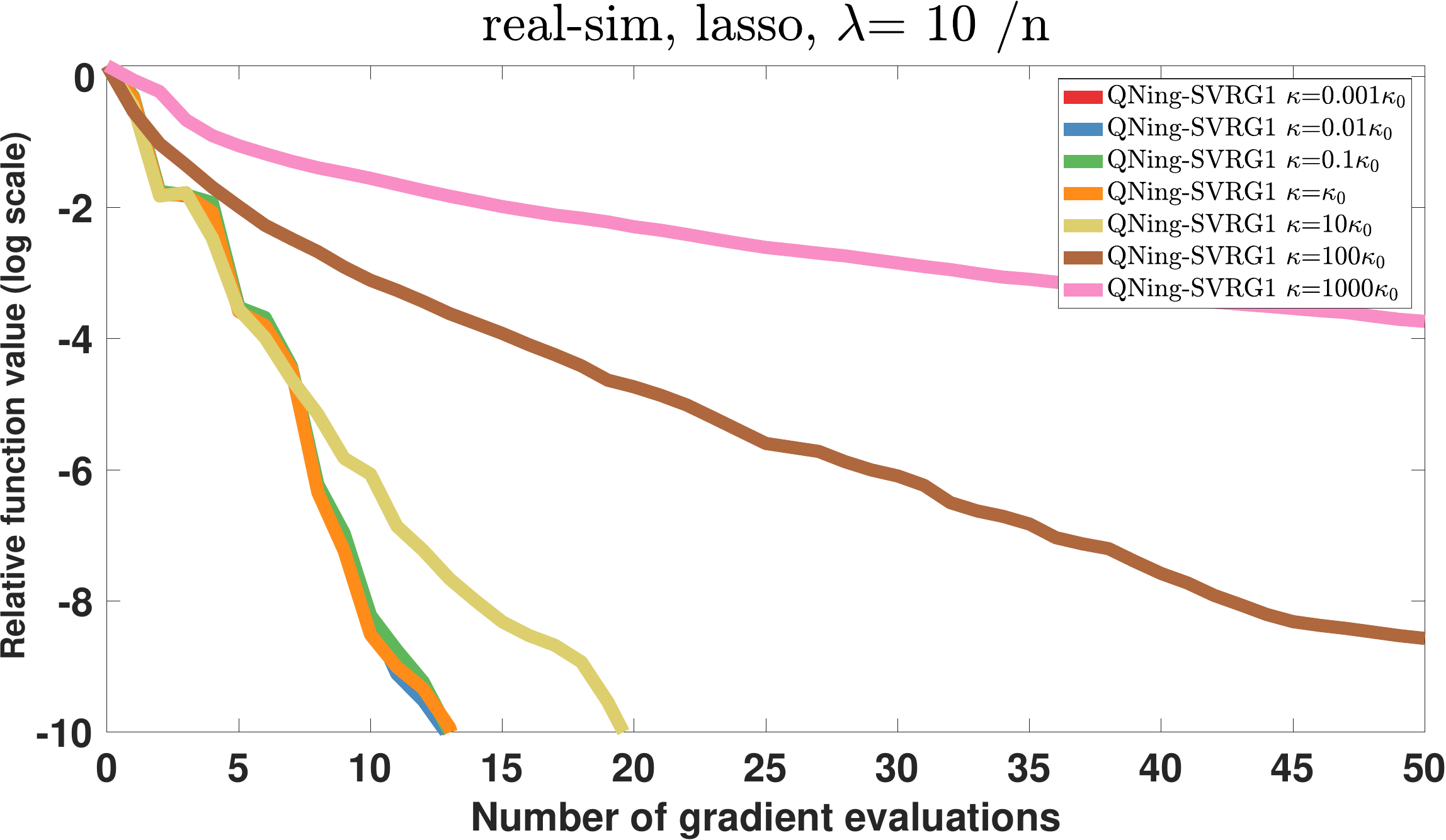}\\
   ~~\includegraphics[width=0.30\linewidth]{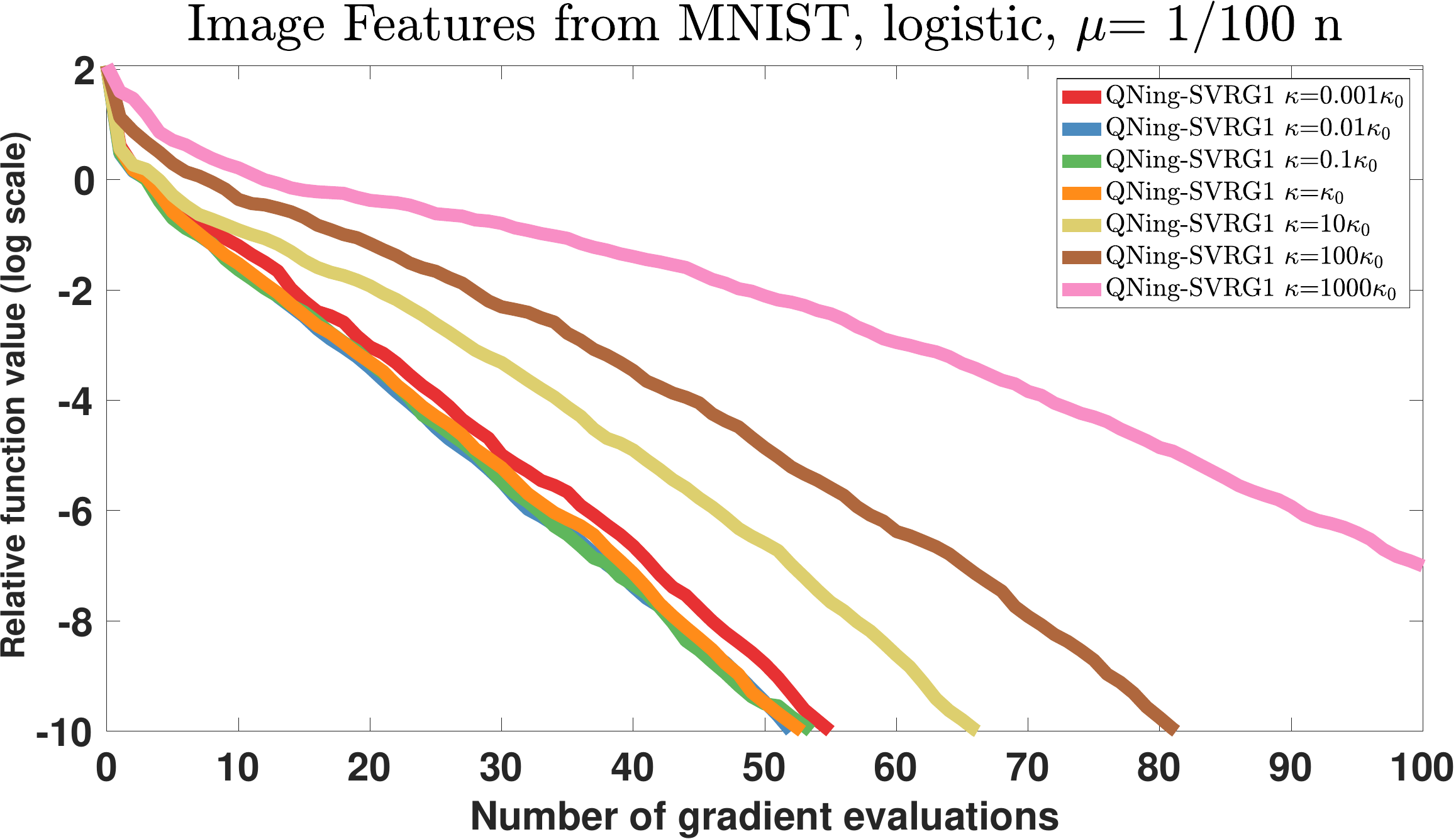} ~ 
   ~~\includegraphics[width=0.30\linewidth]{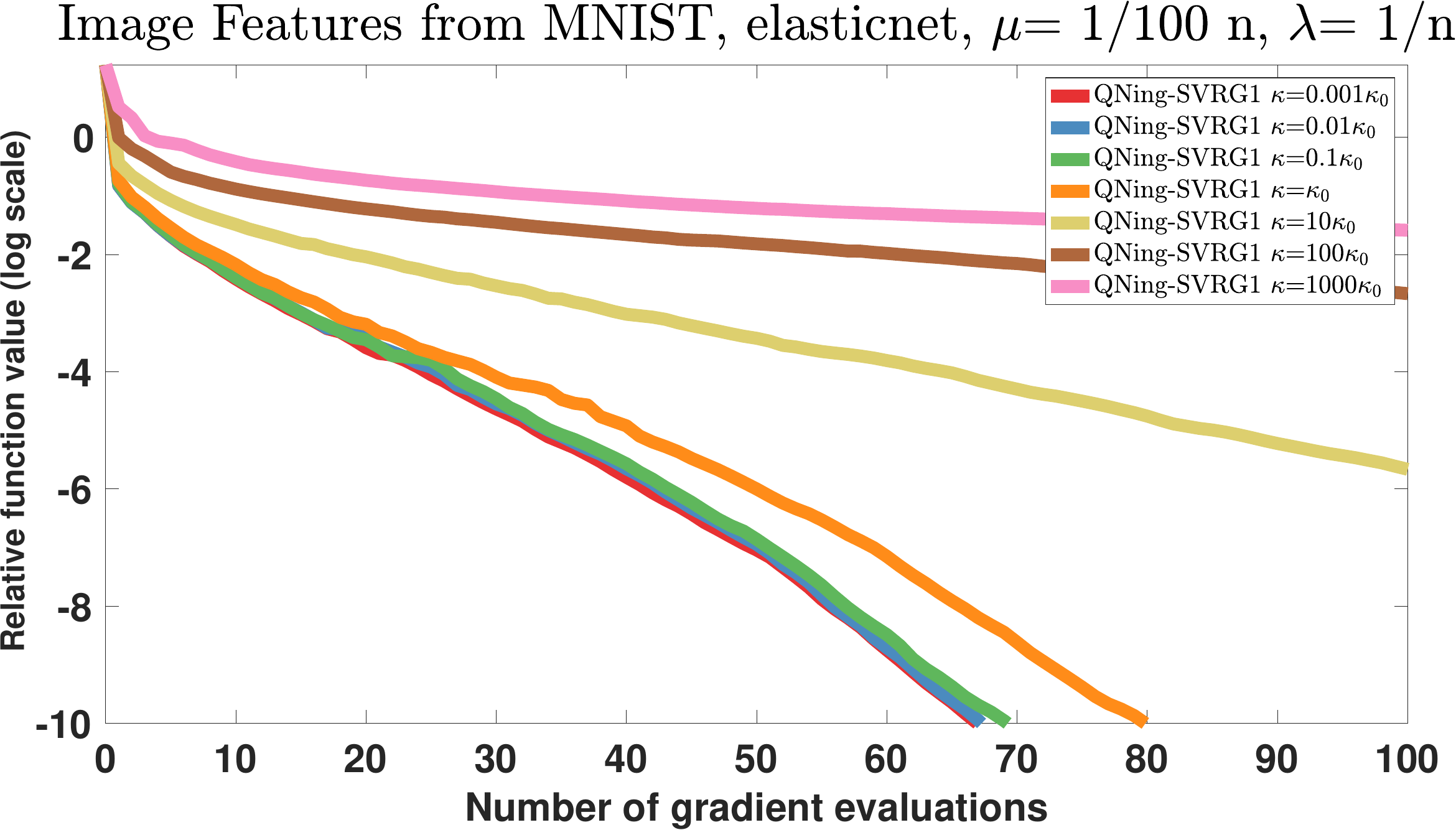} ~ 
   ~~\includegraphics[width=0.30\linewidth]{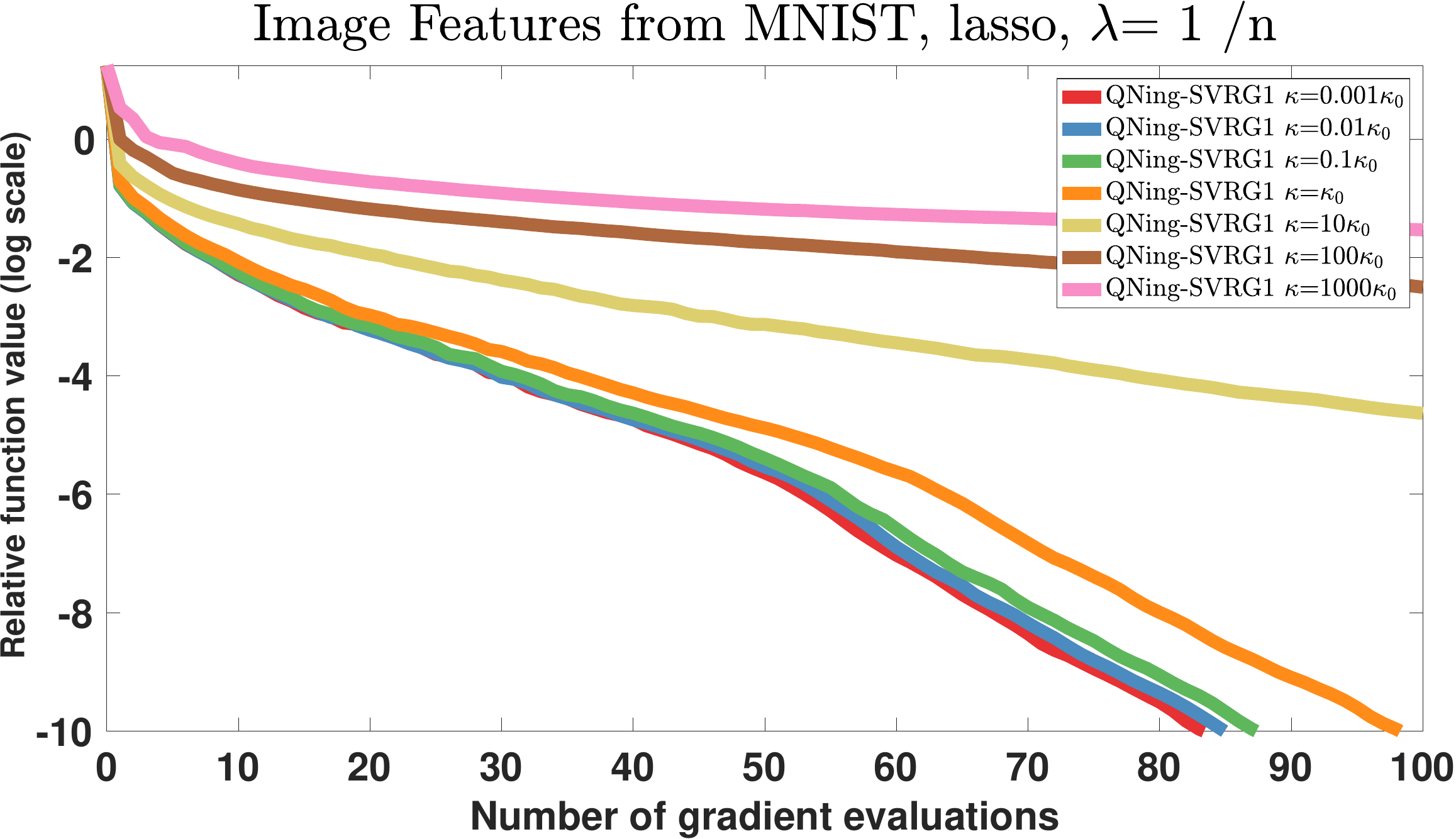}\\
   ~~\includegraphics[width=0.30\linewidth]{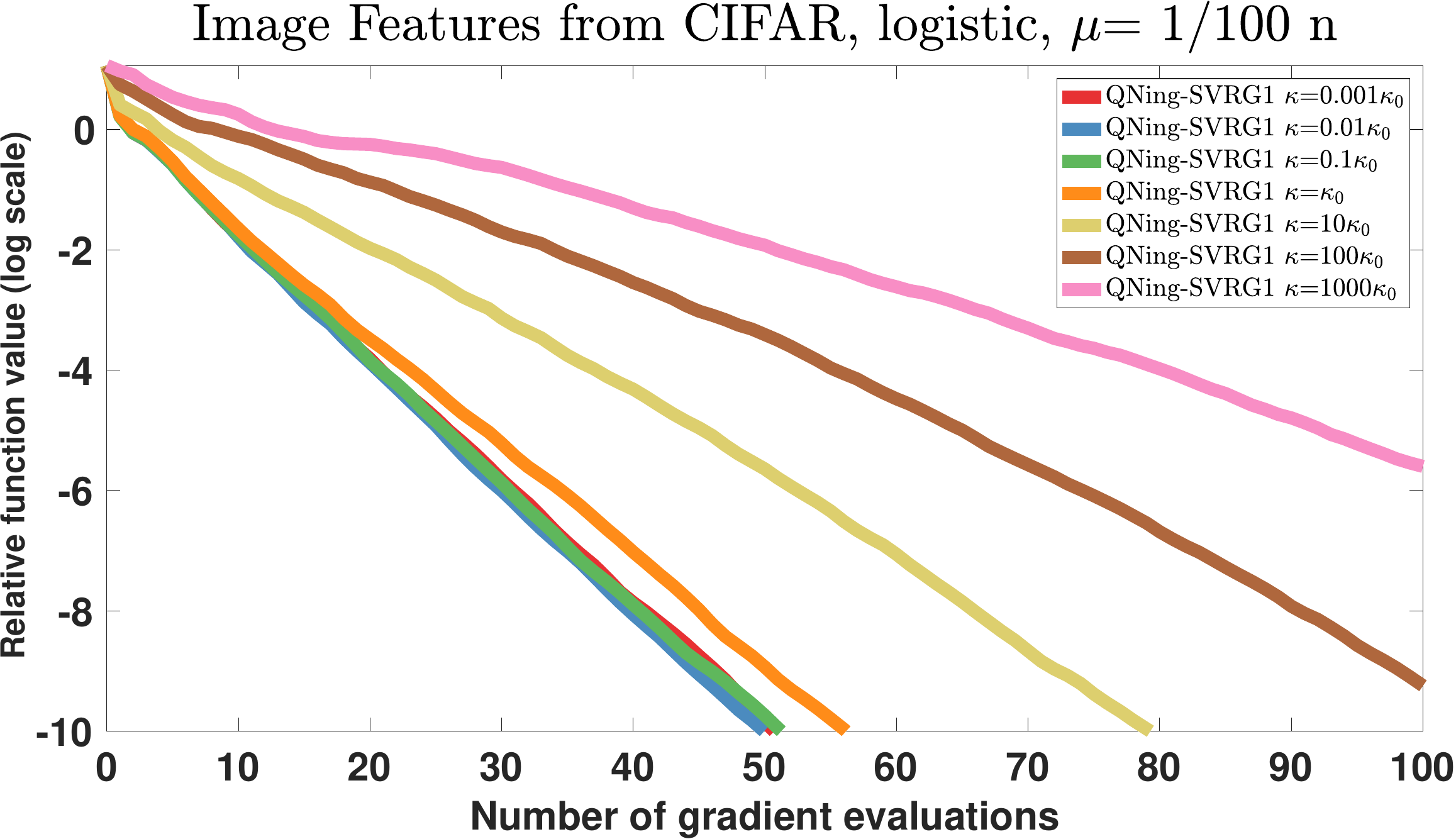} ~ 
   ~~\includegraphics[width=0.30\linewidth]{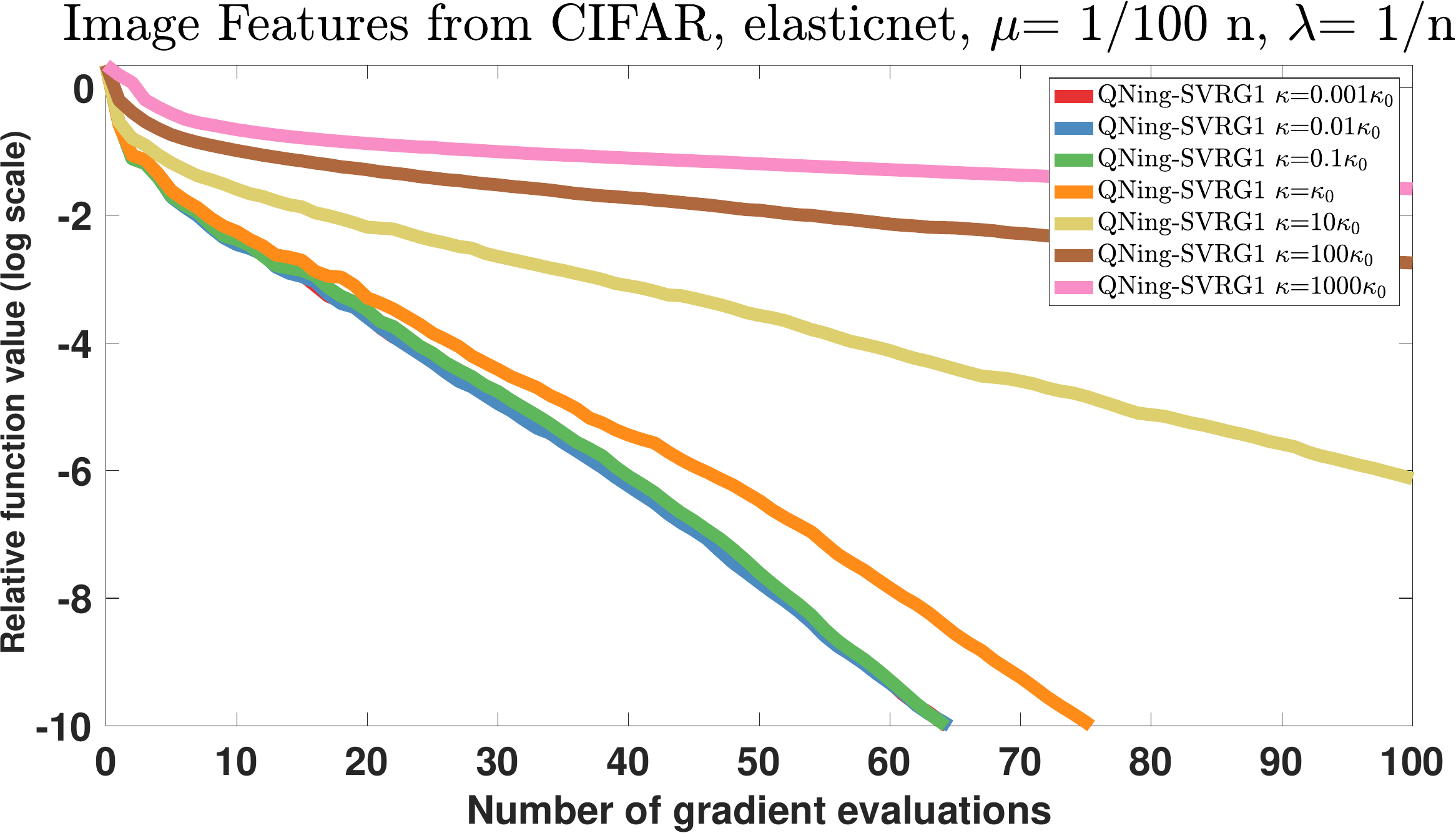} ~ 
   ~~\includegraphics[width=0.30\linewidth]{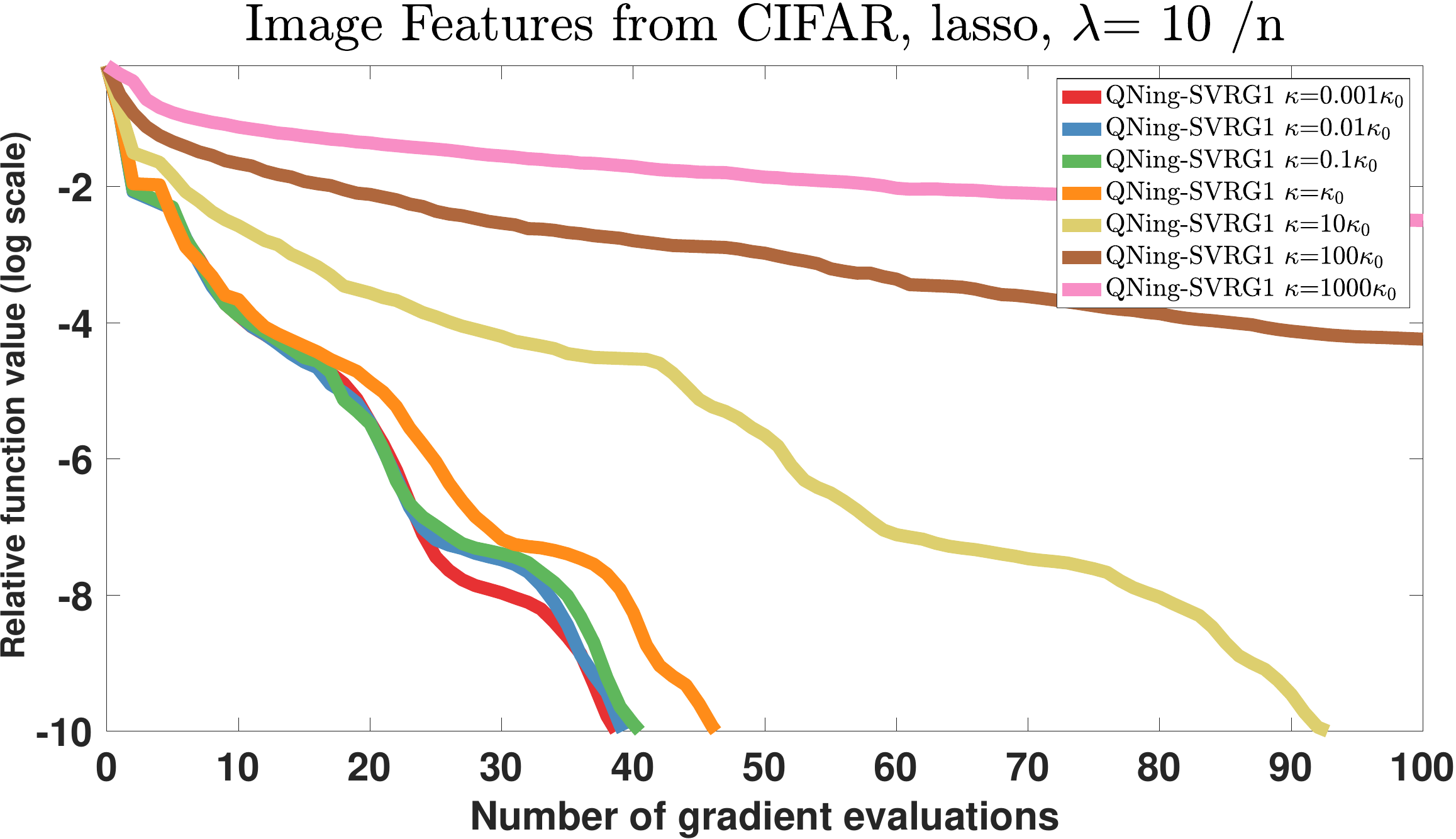} \\
   \caption{Experimental study of influence of the smoothing parameter~$\kappa$
   for \qning-SVRG1. $\kappa_0$ denotes the default choice used in the
   previous experiments. We plot the value~$F(x_k)/F^\star-1$ as a function of
   the number of gradient evaluations, on a logarithmic scale; the optimal
   value $F^\star$ is estimated with a duality gap.}\label{fig:kappa}
\end{figure}

\section{Discussions and concluding remarks}\label{sec:ccl}
A few questions naturally arise regarding the \qningsp scheme:
one may wonder whether or not our convergence rates may be improved, or if the Moreau envelope  could be replaced by another smoothing
technique. In this section, we discuss these two points and present concluding
remarks.

\subsection{Discussion of convergence rates}\label{subsec:rate}
In this paper, we have established the linear convergence of~\qningsp
for strongly convex objectives when sub-problems are solved with enough accuracy.
Since~\qningsp uses Quasi-Newton steps, one might have
expected a superlinear convergence rate as several Quasi-Newton algorithms
often enjoy~\cite{byrd1987global}. The situation is as follows. Consider the
BFGS algorithm (without limited memory), as shown
in~\cite{fukushima1999}, if the sequence $(\epsilon_k)_{k \geq 0}$ decreases
super-linearly, then, it is possible to design a scheme similar to \qningsp that indeed enjoys a super-linear
convergence rate. There is a major downside though: 
a super-linearly decreasing sequence $(\epsilon_k)_{k \geq 0}$ implies an exponentially growing number of
iterations in the inner-loops, which will degrade the global complexity of the algorithm. This issue makes the approach proposed in~\cite{fukushima1999} impractical.

Another potential strategy for obtaining a faster convergence rate consists in
interleaving a Nesterov-type extrapolation step in the~\qningsp algorithm.
Indeed, the convergence rate of~\qningsp scales linearly in the condition
number $\mu_F/L_F$, which suggests that a faster convergence rate could be
obtained using a Nesterov-type acceleration scheme. 
Empirically, we did not observe any benefit of such a strategy, probably
because of the pessimistic nature of the convergence rates that are typically
obtained for Quasi-Newton approaches based on L-BFGS. Obtaining a linear convergence
rate for an L-BFGS algorithm is still an important sanity check, but to the best of
our knowledge, the gap in performance between these worst-case rates and 
 practice has always been huge for this class of algorithms. 

\subsection{Other types of smoothing}\label{subsec:smoothing}
The Moreau envelope we considered is a particular instance of infimal
convolution smoothing~\cite{beck2012}, whose family also
encompasses the so-called Nesterov smoothing~\cite{beck2012}. Other ways to smooth a function
include randomization techniques~\cite{duchi2012} or specific
strategies tailored for the objective at hand. 

One of the main purposes of applying the Moreau envelope is to provide a 
better conditioning. As recalled in Proposition~\ref{propMY}, the gradient of
the smoothed function~$F$ is $\kappa-$Lipschitz continuous regardless of
whether the original function is continuously differentiable or not.
Furthermore, the conditioning of~$F$ is improved with respect to the original
function, with a condition number depending on the amount of smoothing. As
highlighted in~\cite{beck2012}, this property is also shared by other types of
infimal convolutions. Therefore,~\qningsp could potentially be
extended to such types of smoothing in place of the
Moreau envelope. A major advantage of our approach, though, is its
outstanding simplicity. 

\subsection{Concluding remarks}\label{subsec:ccl} 
To conclude, we have proposed a generic mechanism, \qning, to accelerate existing
first-order optimization algorithms with Quasi-Newton-type methods.
\qning's main features are the compatibility with variable metric update rule and composite optimization. Its ability of combining with incremental approaches makes it a promising tool for solving large-scale machine learning problems.
A few questions remain however open regarding the use of the method in a pure
stochastic optimization setting, and the gap in performance between worst-case
convergence analysis and practice is significant. We are planning to address
the first question about stochastic optimization in future work; the second question
is unfortunately difficult and is probably one of the main open questions in the literature about L-BFGS methods.

\subsection*{Acknowledgements}
The authors would like to thank the editor and the reviewers for their constructive and detailed comments. This work was supported by the ERC grant SOLARIS (number 714381), a grant from ANR (MACARON project ANR-14-CE23-0003-01), and the program ``Learning in Machines and Brains'' (CIFAR).

\bibliographystyle{abbrv}
\bibliography{bib}

\appendix
\section{Proof of Proposition~\ref{prop:convex}} \label{app:prop6}
First, we show that the Moreau envelope $F$ inherits the bounded level set property from $f$. 
\begin{definition}
	We say that a convex function $f$ has bounded level sets if $f$ attains its minimum at $x^*$ in $\R^d$ and for any $x$,  there exists~$R_x > 0$ such that
   $$ \forall y \in \Real^d \st f(y) \leq f(x)~~~\text{then}~~~ \Vert y-x^* \Vert \leq R_x.$$
\end{definition}
\begin{lemma}
	If $f$ has bounded level sets, then its Moreau envelope $F$ has bounded level sets as well.
\end{lemma}
\begin{proof}
	First, from Proposition~\ref{propMY}, the minimum of $F$ is attained at~$x^*$. Next,	 we reformulate the bounded level set property by contraposition: for any $x$,  there exists~$R_x > 0$ such that 
	$$ \forall y \in \Real^d \st \Vert y-x^* \Vert > R_x ~~~\text{then}~~~  f(y) > f(x).$$
	Given $x$ in $\R^d$, we show that 
	$$ \forall y \in \Real^d \st \Vert y-x^* \Vert >  \sqrt{\frac{2(f(x)-f^*)} {\kappa}}+R_x ~~~\text{then}~~~  F(y) > F(x).$$
	Let $y$ satisfies the above inequality, by definition, 
	\begin{equation*}
		F(y) = f(p(y))+ \frac{\kappa}{2} \Vert p(y) - y\Vert^2. 
	\end{equation*}
	From the triangle inequality, 
	$$\Vert y -p(y) \Vert+ \Vert p(y)- x^* \Vert \geq \Vert y-x^* \Vert>  \sqrt{\frac{2(f(x)-f^*)}{\kappa}} +R_x . $$
	Then either $\Vert y -p(y) \Vert > \sqrt{\frac{2(f(x)-f^*)}{\kappa}}$ or $\Vert p(y)- x^* \Vert >R_x$.
	\begin{itemize}
		\item If $\Vert y -p(y) \Vert > \sqrt{\frac{2(f(x)-f^*)}{\kappa}}$, then 
 	\begin{equation*}
		F(y) = f(p(y))+ \frac{\kappa}{2} \Vert p(y) - y\Vert^2 > f(p(y))+ f(x)-f^* \geq f(x) \geq F(x).
	\end{equation*}
	\item If $\Vert p(y)- x^* \Vert >R_x$, then 
	 \begin{equation*}
		F(y) = f(p(y))+ \frac{\kappa}{2} \Vert p(y) - y\Vert^2 \geq f(p(y))> f(x) \geq F(x).
	\end{equation*}
	\end{itemize}
	This completes the proof. 
\end{proof}
We are now in shape to prove the proposition.
\begin{proof}
From (\ref{eq:gd}), we have 
\begin{equation*}
	F(x_{k+1}) \leq F(x_k) - \frac{1}{32\kappa} \Vert \nabla F(x_k) \Vert^2.
\end{equation*}	
Thus $F(x_k)$ is decreasing. From the bounded level set property of $F$, there exists $R>0$ such that $\Vert x_k -x^*\Vert \leq R$ for any $k$. By the convexity of $F$, we have 
\begin{equation*}
	F(x_k) -F^* \leq \langle \nabla F(x_k), x_k -x^* \rangle \leq \Vert \nabla F(x_k) \Vert \Vert x_k -x^* \Vert  \leq R \Vert \nabla F(x_k) \Vert.
\end{equation*}
Therefore, 
\begin{align*}
	F(x_{k+1})-F^* & \leq F(x_k)-F^* - \frac{1}{32\kappa} \Vert \nabla F(x_k) \Vert^2 \\
	&  \leq F(x_k)-F^* -  \frac{(F(x_k)-F^*)^2}{32\kappa R^2}.
\end{align*}
Let us define $r_k \defin f(x_k)-f^*$. Thus, 
 $$ \frac{1}{r_{k+1}} \geq \frac{1}{r_k(1 - \frac{r_k}{32\kappa R^2})} \geq \frac{1}{r_k} \left (1+ \frac{r_k}{32\kappa R^2} \right ) 
    = \frac{1}{r_k} + \frac{1}{32\kappa R^2}.$$
Then, after exploiting the telescoping sum,
$$ \frac{1}{r_{k+1}} \geq \frac{1}{r_0} + \frac{k+1}{32\kappa R^2}  \geq \frac{k+1}{32\kappa R^2}. $$

\end{proof}

\section{Proof of Lemma ~\ref{lem:unit stepsize}}
\begin{proof}
Let us denote $\delta_k = - B_k^{-1}g_k$ and let the subproblems solved to accuracy $\varepsilon_k \leq \frac{c}{\kappa} \Vert g_k \Vert^2$. We show that when $c \leq \frac{\mu^2}{128 (\mu+\kappa)^2}$, the following two inequalities hold:
\begin{equation}\label{eq:o(gk)}
	F(x_k+\delta_k) \leq F(x_k) - \frac{3}{8\kappa} \Vert g_k \Vert^2 + o(\Vert g_k \Vert^2),
\end{equation}
and
\begin{equation}\label{eq:F_k+1}
	F_{k+1} \leq F(x_k+\delta_k) + \frac{1}{16\kappa} \Vert g_k \Vert^2 + o(\Vert g_k \Vert^2).
\end{equation}
Then summing up the above inequalities yields 
\begin{align*}
	F_{k+1} & \leq F(x_k) - \frac{5}{16\kappa} \Vert g_k \Vert^2 +o(\Vert g_k \Vert^2) \\
	& \leq F_k - \frac{1}{4\kappa} \Vert g_k \Vert^2,
\end{align*}
where the last inequality holds since $F(x_k) \leq F_k$ and $o(\Vert g_k \Vert^2) \leq \frac{1}{4\kappa} \Vert g_k \Vert^2$ when $k$ is large enough. This is the desired descent condition (\ref{suffdescent}). 

We first prove (\ref{eq:o(gk)}) which relies on the somoothness and Lipschitz Hessian assumption of $F$. More concretely,
\begin{align*}
F(x_k + \delta_k)-F(x_k) & \leq  \nabla F(x_k)^T \delta_k + \frac{1}{2}\delta_k^T \nabla^2 F(x_k) \delta_k + \frac{L_2}{6} \Vert \delta_k \Vert^3 \\
	& = (\nabla F(x_k)-g_k)^T \delta_k + g_k^T \delta_k + \frac{1}{2}\delta_k^T (\nabla^2 F(x_k)- B_k) \delta_k + \underbrace{\frac{1}{2}\delta_k^T B_k \delta_k}_{= -\frac{1}{2} g_k^T\delta_k} + \frac{L_2}{6} \Vert \delta_k \Vert^3 \\
	& = \underbrace{\frac{1}{2} g_k^T\delta_k}_{E_1}+ \underbrace{(\nabla F(x_k)-g_k)^T \delta_k}_{E_2} + \underbrace{\frac{1}{2}\delta_k^T (\nabla^2 F(x_k)- B_k) \delta_k}_{E_3} + \underbrace{\frac{L_2}{6} \Vert \delta_k \Vert^3}_{E_4} .\\
\end{align*}
We are going upper bound each term one by one. First,
\begin{align*}
	E_1 & = \frac{1}{2} g_k^T\delta_k = -\frac{1}{2} g_k B_k^{-1}g_k \\
	    & = -\frac{1}{2} g_k \nabla^2 F(x^*)^{-1}g_k -\frac{1}{2} g_k (B_k^{-1}-\nabla^2 F(x^*)^{-1}) g_k \\
	    & \leq -\frac{1}{2\kappa} \Vert g_k \Vert^2 + o(\Vert g_k \Vert^2),
\end{align*}  	
where the last inequality uses (\ref{eq:DM1}) and the $\kappa$-smoothness of $F$ which implies $\nabla^2 F(x^*) \preceq \kappa I$. Second, 
\begin{align*}
E_2 & = (\nabla F(x_k)-g_k)^T \delta_k \leq \Vert \nabla F(x_k)-g_k \Vert \Vert \delta_k \Vert \\
	& \leq \sqrt{2c} \Vert g_k \Vert \Vert B_k^{-1} g_k\Vert \quad \text{ (from (\ref{ineqgrad}))}\\
	& \leq \sqrt{2c} \Vert g_k \Vert \left [\Vert \nabla^2 F(x^*)^{-1} g_k\Vert+ \Vert \left (B_k^{-1}- \nabla^2 F(x^*)^{-1} \right )  g_k\Vert \right ] \\
	& \leq \sqrt{2c} \frac{1}{\mu_F} \Vert g_k \Vert^2 + o(\Vert g_k \Vert^2)\\
	& = \frac{1}{8 \kappa} \Vert g_k \Vert^2 + o(\Vert g_k \Vert^2). 
\end{align*}
Third,
\begin{align*}
	E_3 & = \frac{1}{2}\delta_k^T (\nabla^2 F(x_k)- B_k) \delta_k \leq \frac{1}{2} \Vert \delta_k \Vert \Vert (\nabla^2 F(x_k)- B_k) \delta_k \Vert \\
		& \leq \frac{1}{2} \Vert \delta_k \Vert \left ( \Vert (\nabla^2 F(x_k) - \nabla^2 F(x^*)) \delta_k  \Vert + \Vert (\nabla^2 F(x^*) - B_k)\delta_k \Vert \right ) \\
		& \leq \frac{L_2}{2} \Vert x_k -x^* \Vert \Vert \delta_k \Vert^2 + \Vert \nabla^2 F(x^*) \Vert \Vert (B_k^{-1} - \nabla^2 F(x^*)^{-1} )g_k \Vert \\
		& = o(\Vert g_k \Vert^2),
\end{align*}
where the last line comes from (\ref{eq:DM1}) and the fact that $\Vert x_k - x^* \Vert \rightarrow 0$. Last, since  
$$  \delta_k = \underbrace{- \nabla^2 F(x^*)^{-1}g_k}_{=O(\Vert g_k \Vert)} +  \underbrace{(\nabla^2 F(x^*)^{-1} - B_k^{-1})g_k}_{=o(\Vert g_k \Vert) \text{ by Dennis-Mor\'e condition}} $$
and $\Vert g_k \Vert \rightarrow 0$, we have 
\begin{align*}
	E_4 = \frac{L_2}{6} \Vert \delta_k \Vert^3 = o(\Vert g_k \Vert^2).
\end{align*}
Summing up above four inequalities yields (\ref{eq:o(gk)}). 
Next, we prove the other desired inequality (\ref{eq:F_k+1}). The main effort is to bound $\Vert g_{k+1} \Vert$ by a constant factor times $\Vert g_k \Vert$. From the inexactness of the subproblem, we have 
\begin{align*}
	F_{k+1} \leq F(x_{k+1}) + \frac{c}{\kappa} \Vert g_{k+1} \Vert^2 \leq F(x_{k+1}) + \frac{2c}{(1-4c)\kappa} \Vert \nabla F(x_{k+1}) \Vert^2
\end{align*}
Moreover,
\begin{align*}
  & \nabla F(x_{k+1})- \nabla F(x_k) - \nabla^2 F(x^*) (x_{k+1}-x_k) \\
= & \int_{0}^1 \left ( \nabla^2 F(x_k+ \tau(x_{k+1}-x_k) ) - \nabla^2 F(x^*) \right ) (x_{k+1}-x_k) \text{d} \tau \\
\end{align*}
Therefore,
$$ \Vert \nabla F(x_{k+1})- \nabla F(x_k) - \nabla^2 F(x^*) (x_{k+1}-x_k) \Vert \leq \max \left \{ \Vert x_k-x^*\Vert, \Vert x_{k+1}-x^* \Vert \right \} \Vert x_{k+1}-x_k \Vert = o(\Vert g_k \Vert).$$
\begin{align*}
	\Vert \nabla F(x_{k+1}) \Vert & \leq \Vert \nabla F(x_k) + \nabla^2 F(x^*) (x_{k+1}-x_k) \Vert + o(\Vert g_k \Vert)\\
	& \leq \Vert \nabla F(x_k) -g_k \Vert + \underbrace{\Vert g_k+ \nabla^2 F(x^*) (x_{k+1}-x_k) \Vert}_{=o(\Vert g_k \Vert) \text{ by Dennis-Mor\'e condition}} + o(\Vert g_k \Vert) \\
	& \leq \sqrt{2c} \Vert g_k \Vert + o(\Vert g_k \Vert). \\
\end{align*}
As a result,
\begin{align*}
	F_{k+1} & \leq F(x_{k+1}) + \frac{4c^2}{(1-4c)\kappa} \Vert \nabla g_k \Vert^2 + o(\Vert g_k \Vert^2) \\
	& \leq F(x_{k+1}) + \frac{1}{16\kappa} \Vert \nabla g_k \Vert^2 + o(\Vert g_k \Vert^2)  \quad \text{ when } \quad c\leq \frac{1}{16}.
\end{align*}
This completes the proof.
\end{proof}

\section{Additional experiments}
In this section, we provide additional experimental results including experimental comparisons in
terms of outer loop iterations, and an empirical study regarding the choice of the unit step size $\eta_k=1$.

\subsection{Comparisons in terms of outer-loop iterations} \label{appendix:iterations}
In the main paper, we have used the number of gradient evaluations as a natural
measure of complexity. Here, we also provide a comparison in terms of
outer-loop iterations, which does not take  into account the complexity of
solving the sub-problems. While interesting, the comparison artificially gives
an advantage to the stopping criteria (\ref{eq:stop condition}) since achieving
it usually requires multiple passes. 

\begin{figure}[hbtp!]
   \centering
   ~~\includegraphics[width=0.30\linewidth]{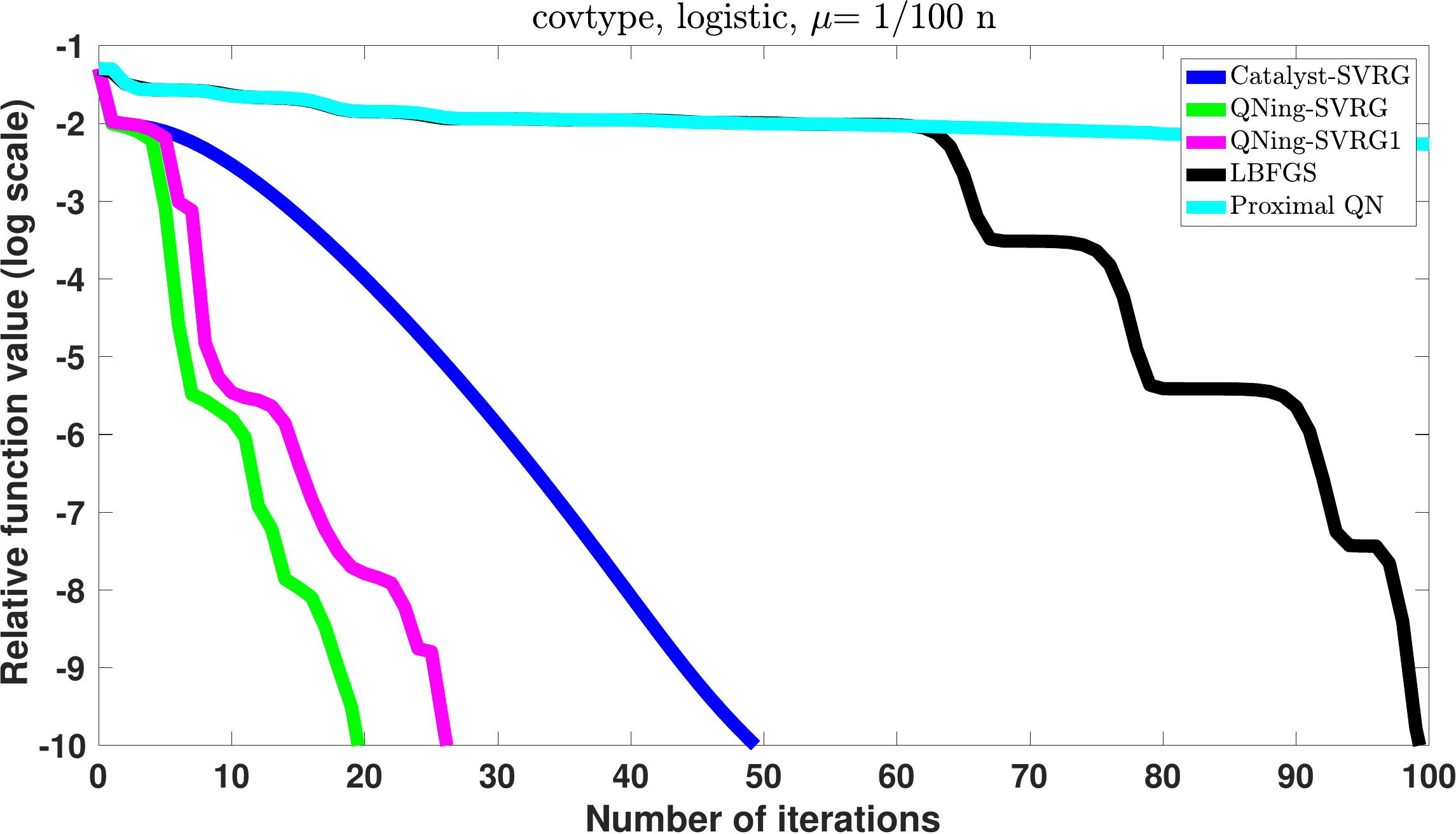} ~ 
   ~~\includegraphics[width=0.30\linewidth]{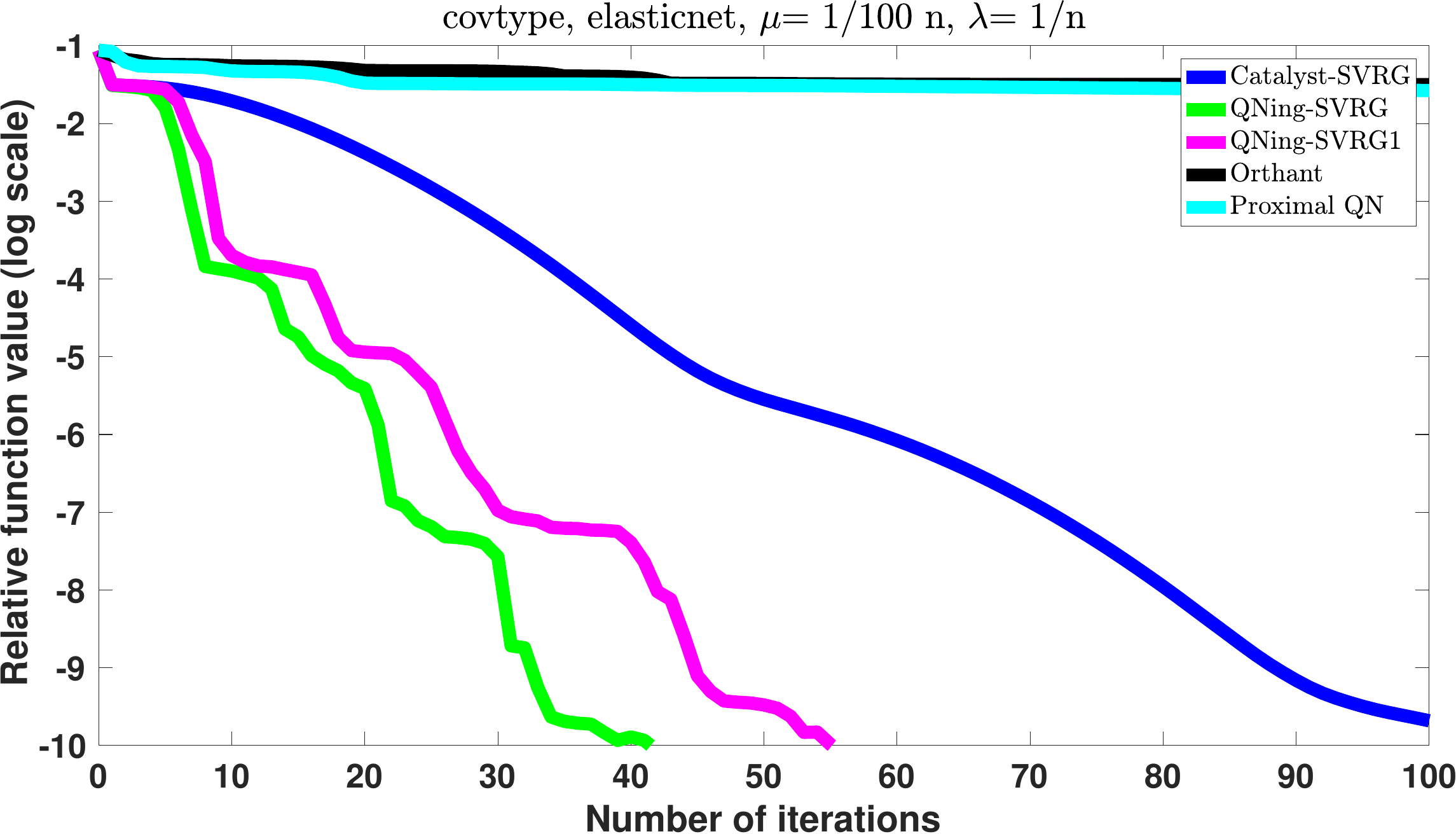} ~ 
   ~~\includegraphics[width=0.30\linewidth]{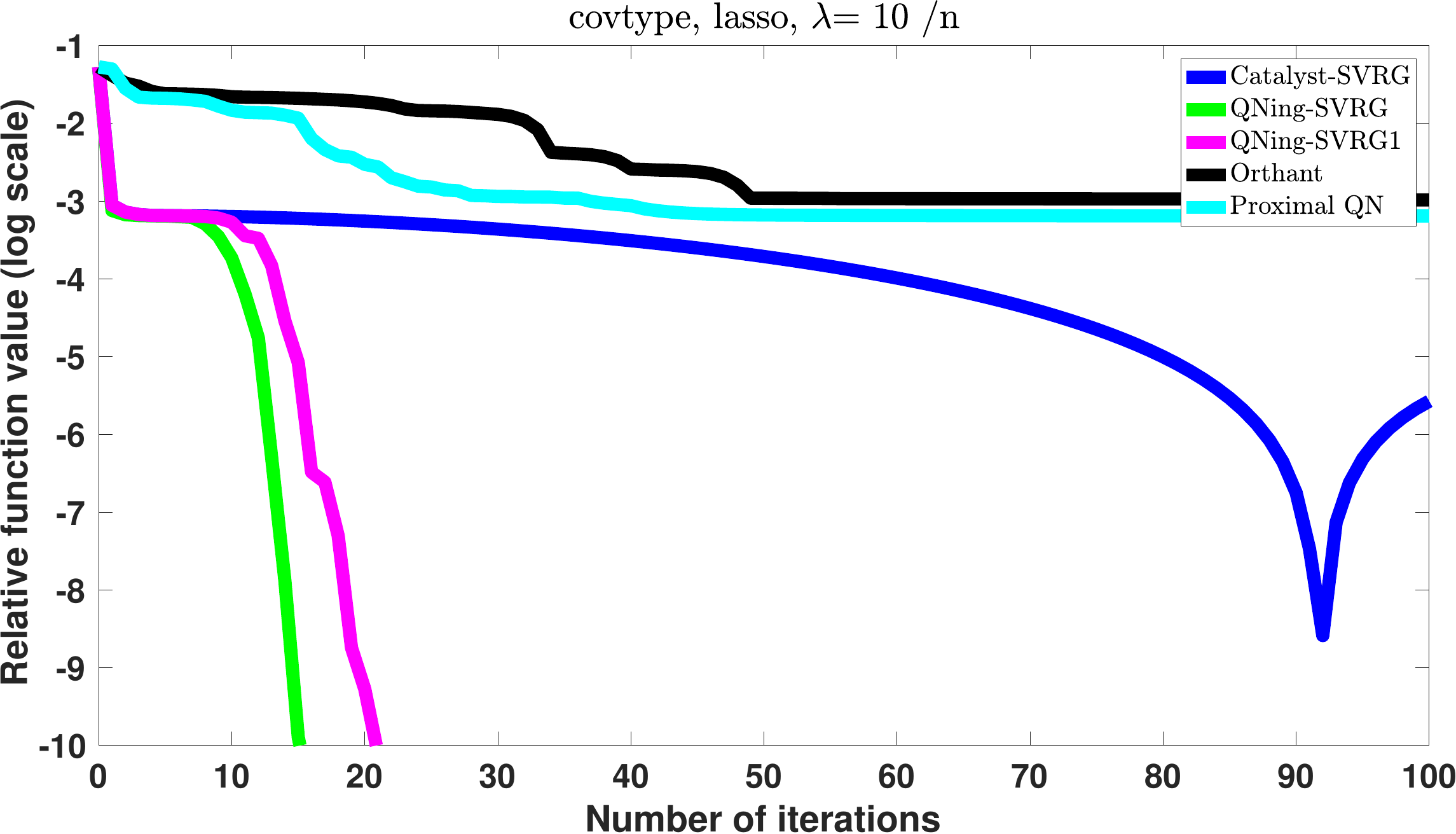} \\
   ~~\includegraphics[width=0.30\linewidth]{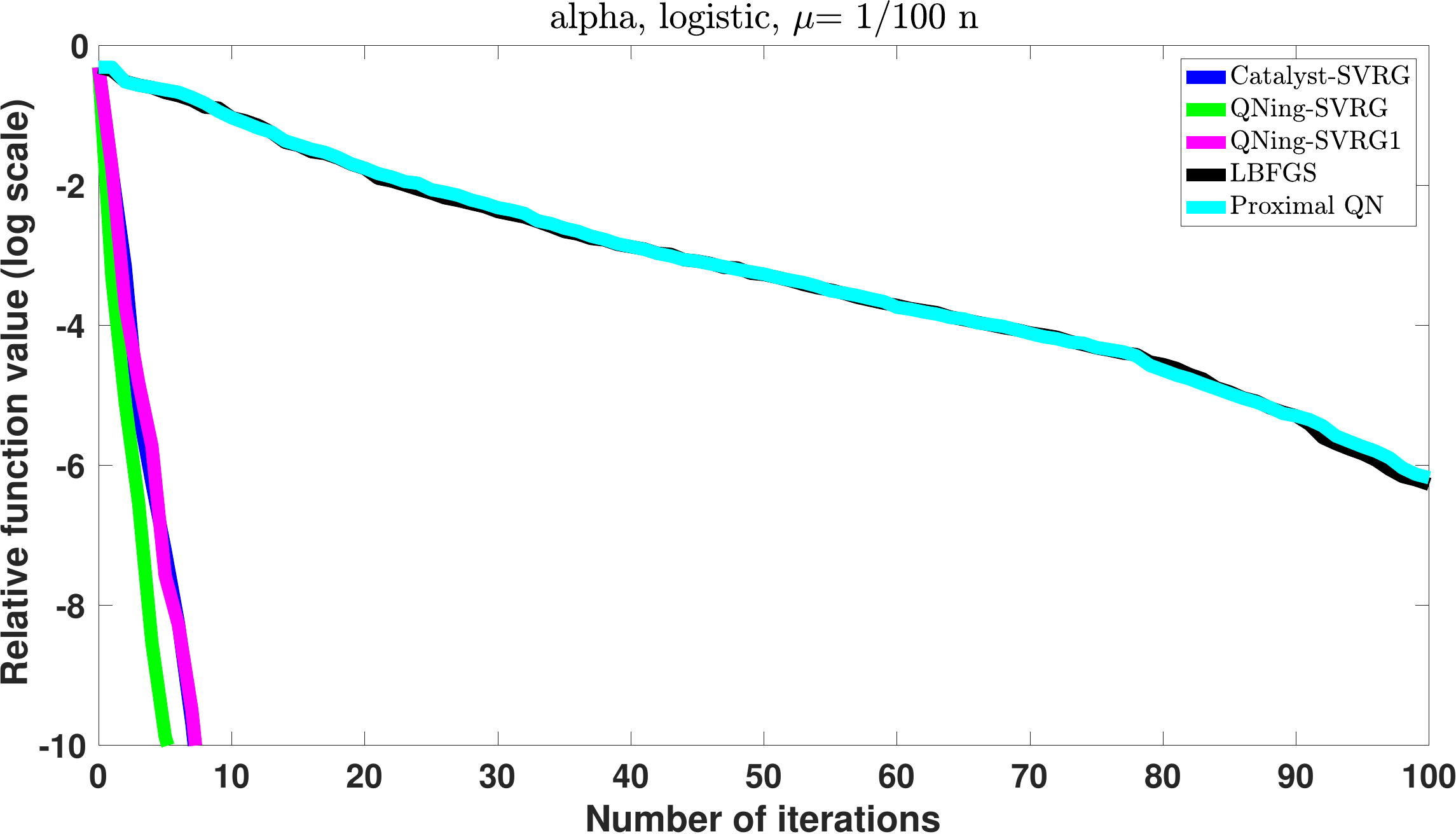} ~ 
   ~~\includegraphics[width=0.30\linewidth]{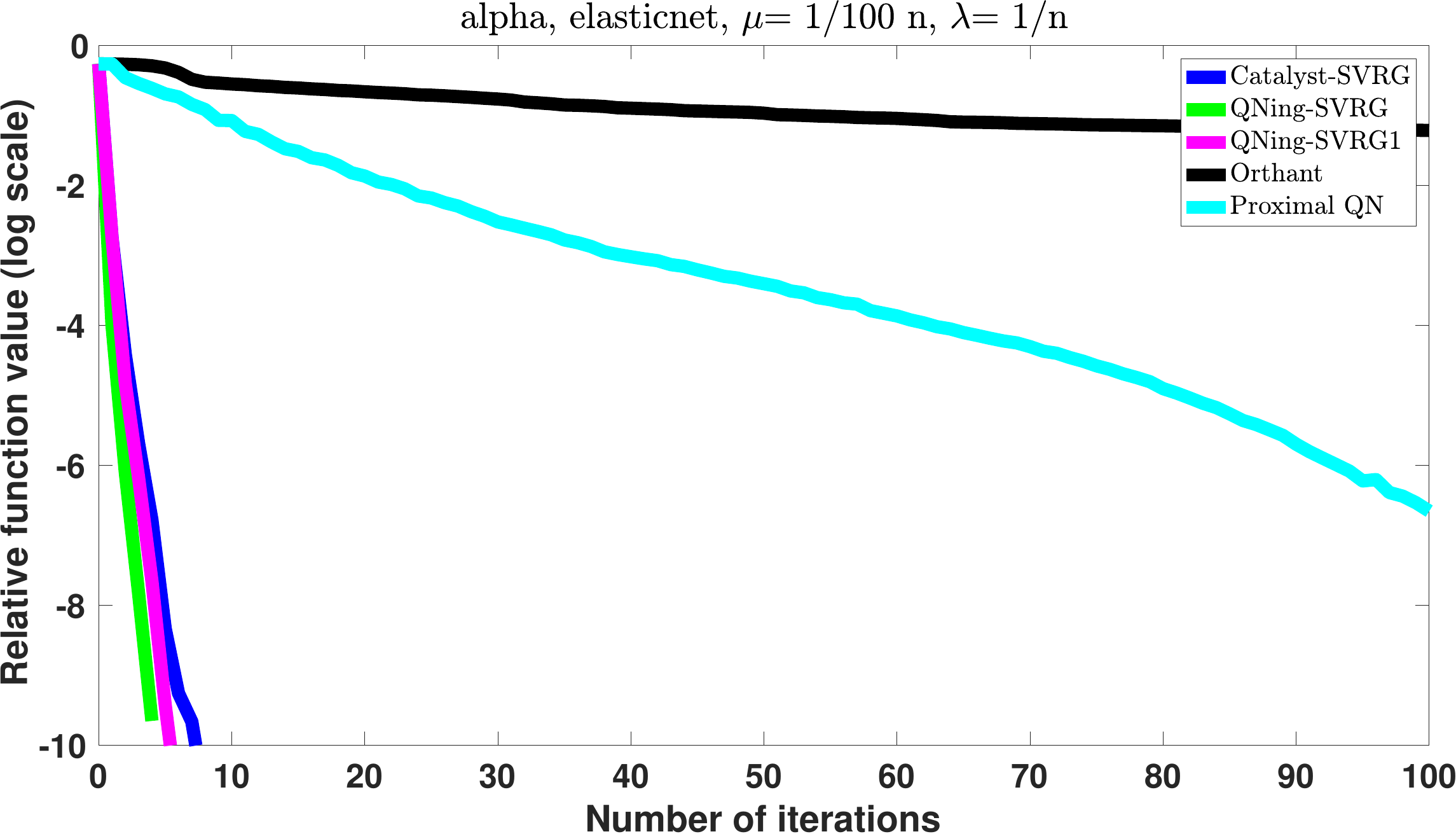} ~ 
   ~~\includegraphics[width=0.30\linewidth]{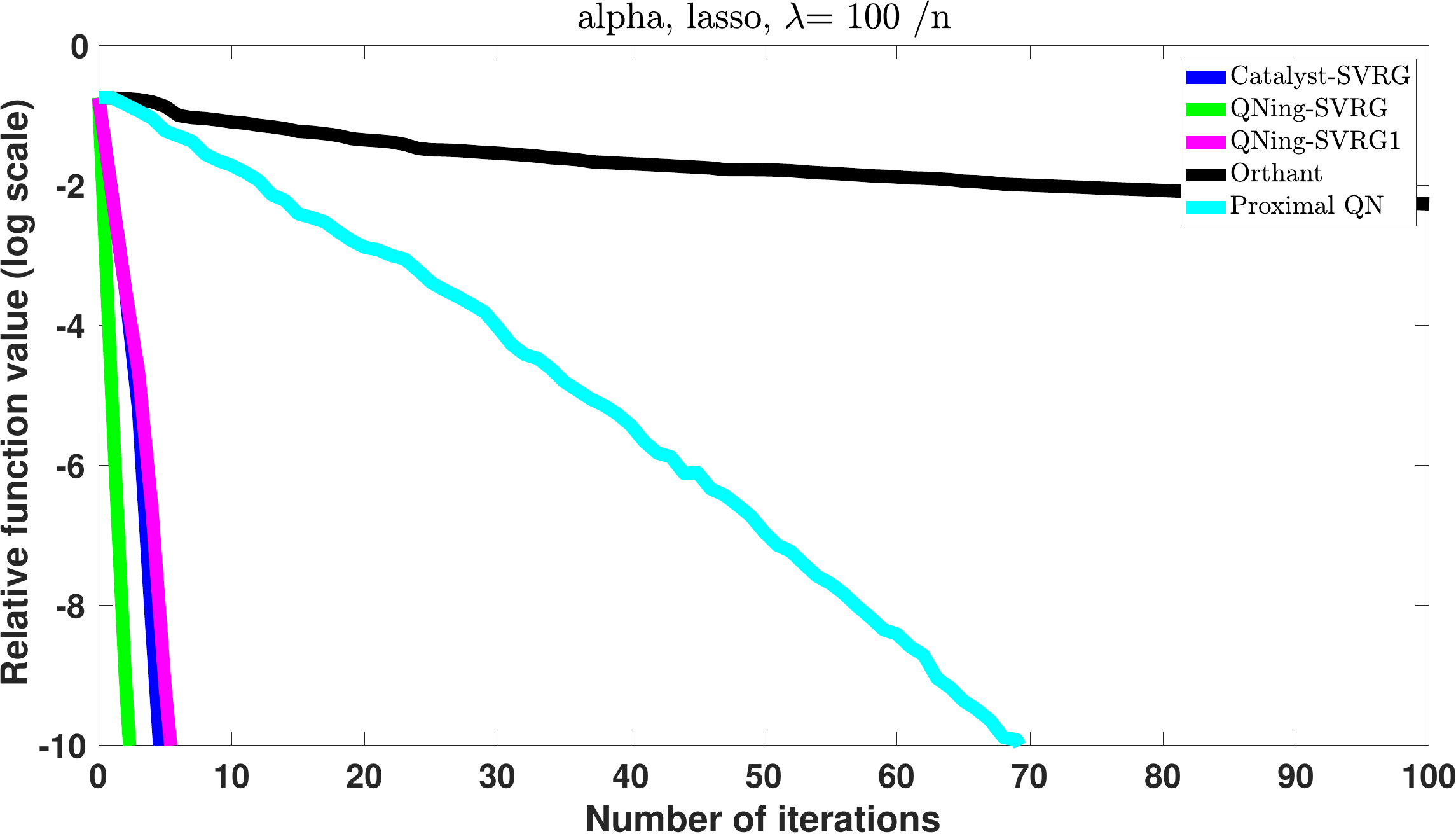}\\
   ~~\includegraphics[width=0.30\linewidth]{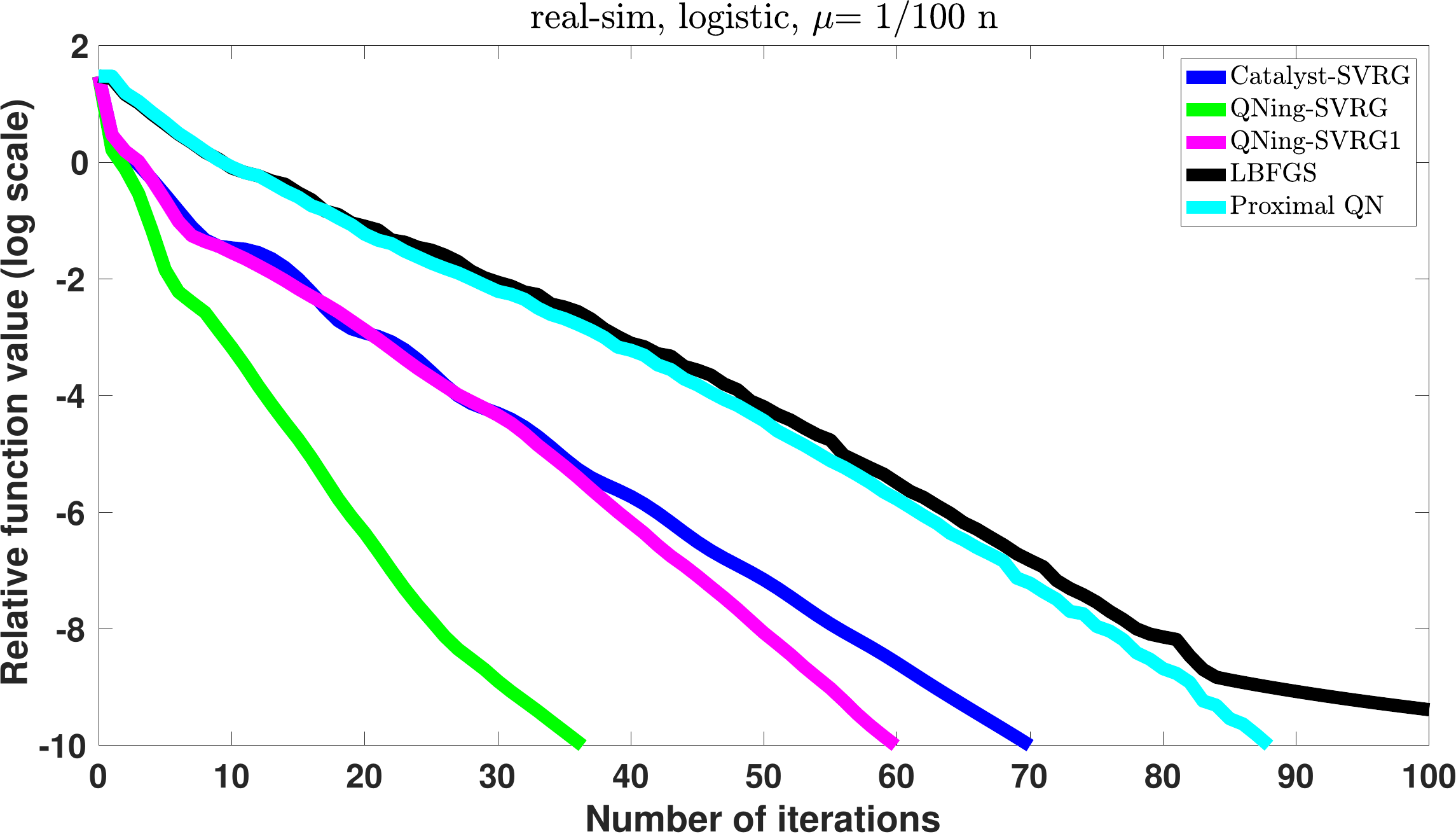} ~ 
   ~~\includegraphics[width=0.30\linewidth]{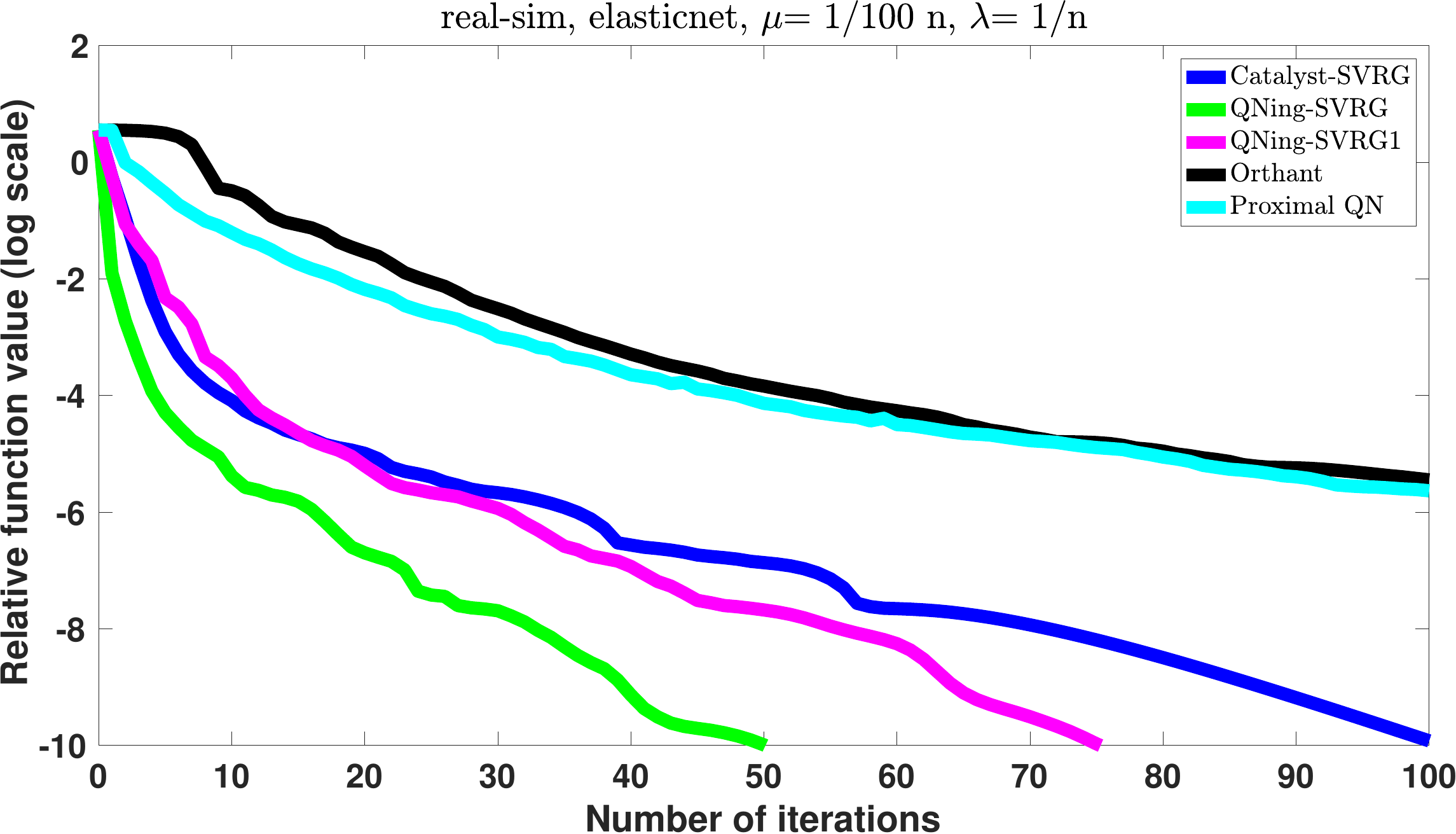} ~ 
   ~~\includegraphics[width=0.30\linewidth]{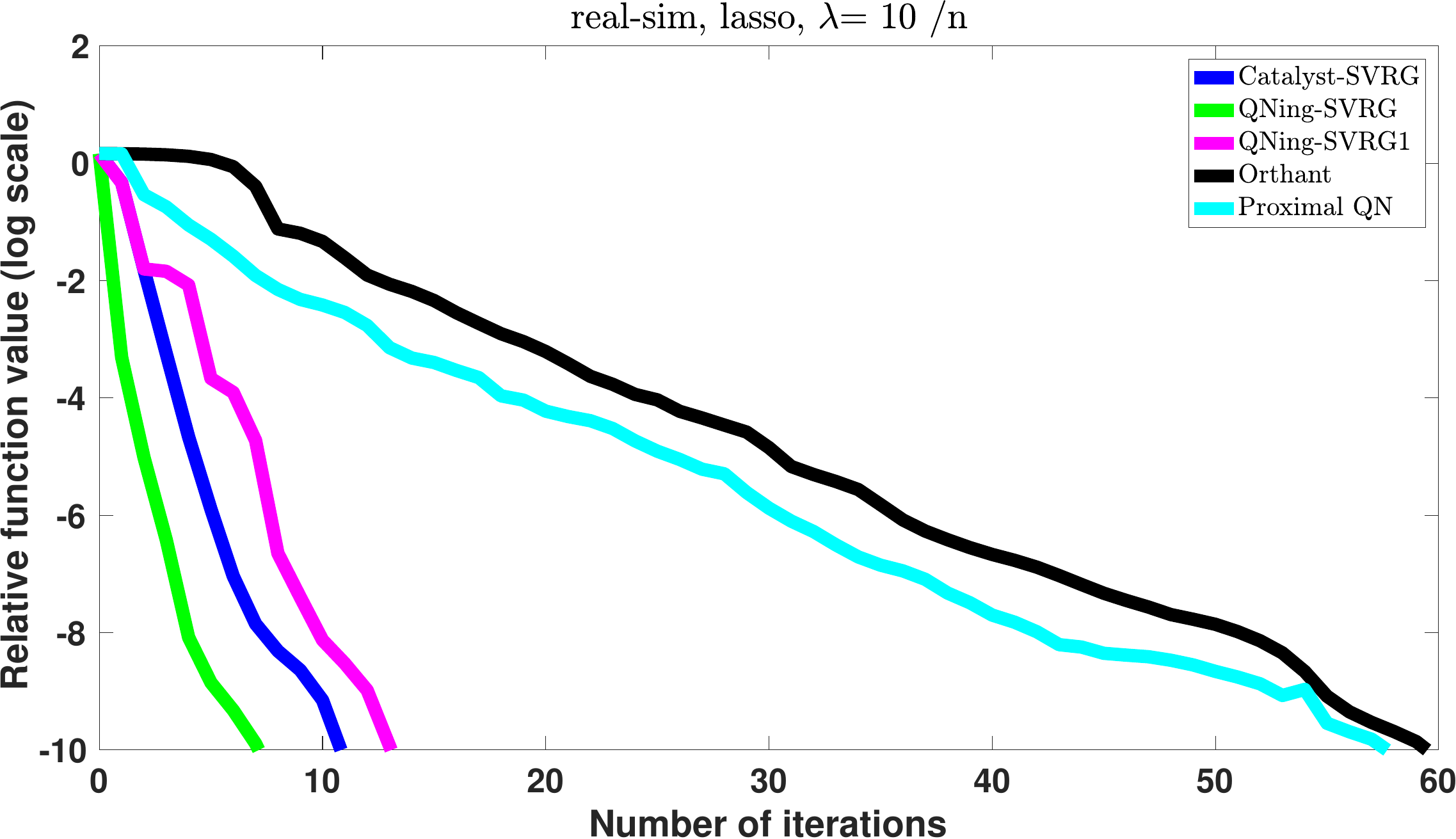} \\ 
   ~~\includegraphics[width=0.30\linewidth]{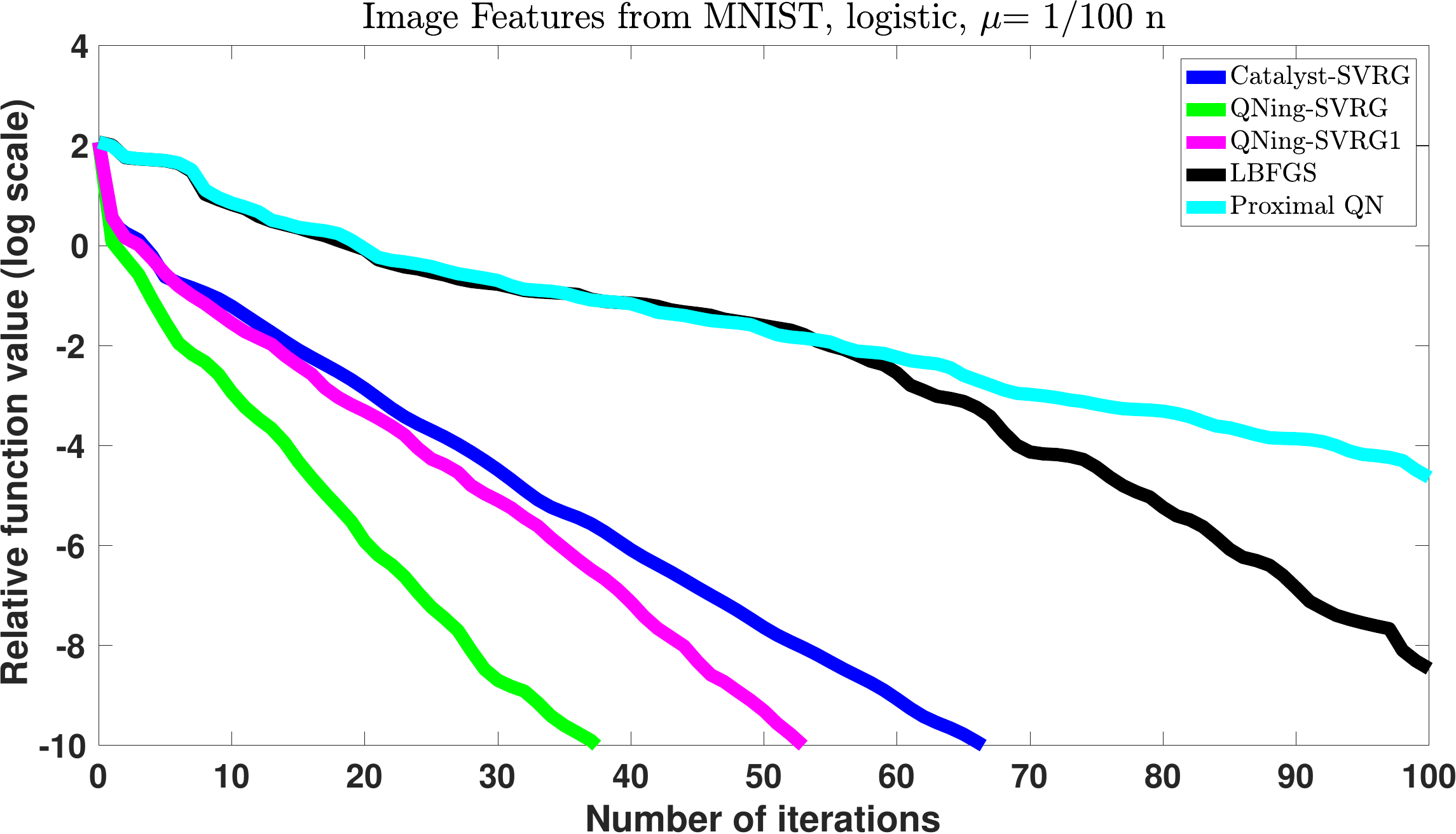} ~ 
   ~~\includegraphics[width=0.30\linewidth]{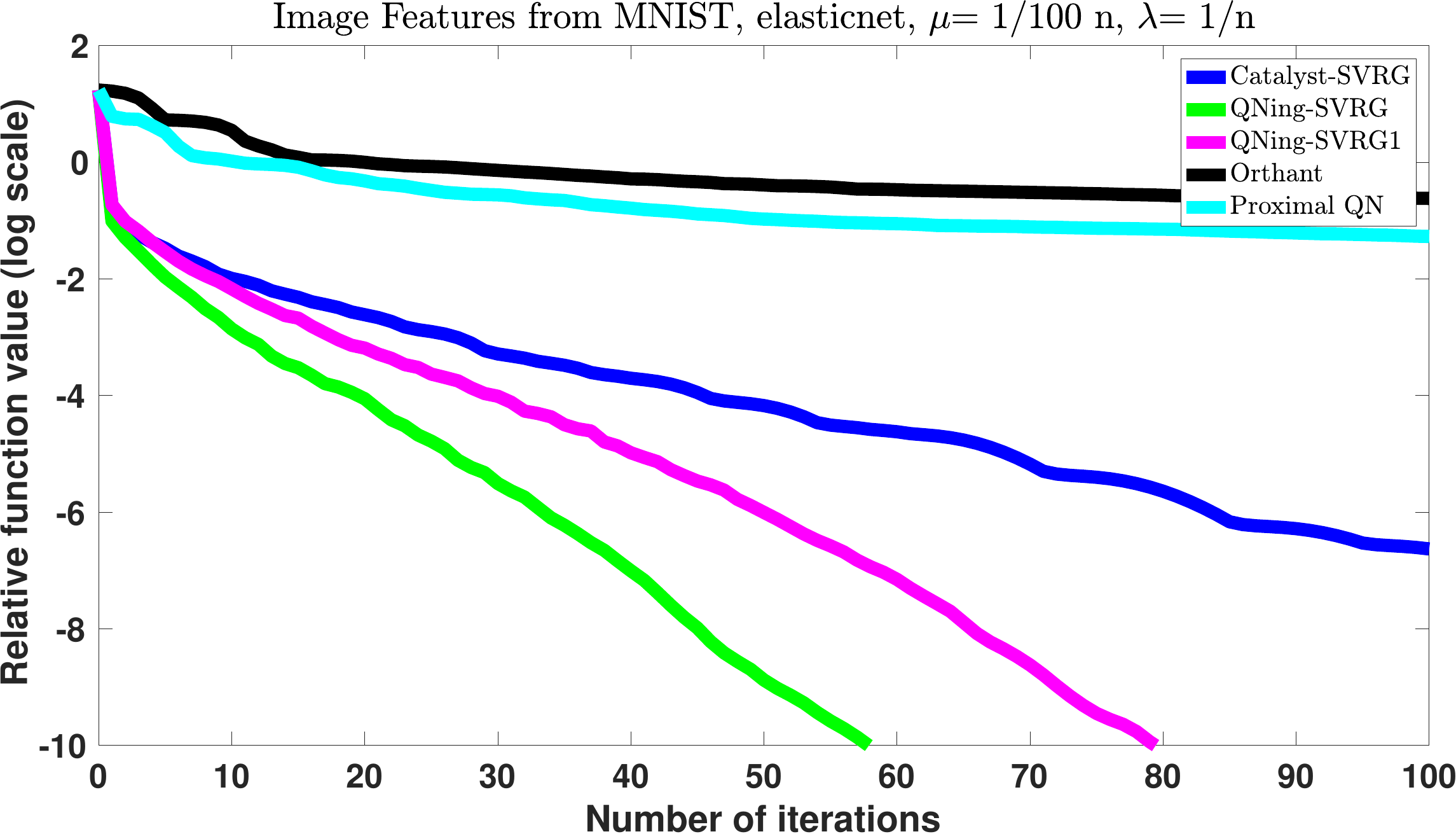} ~ 
   ~~\includegraphics[width=0.30\linewidth]{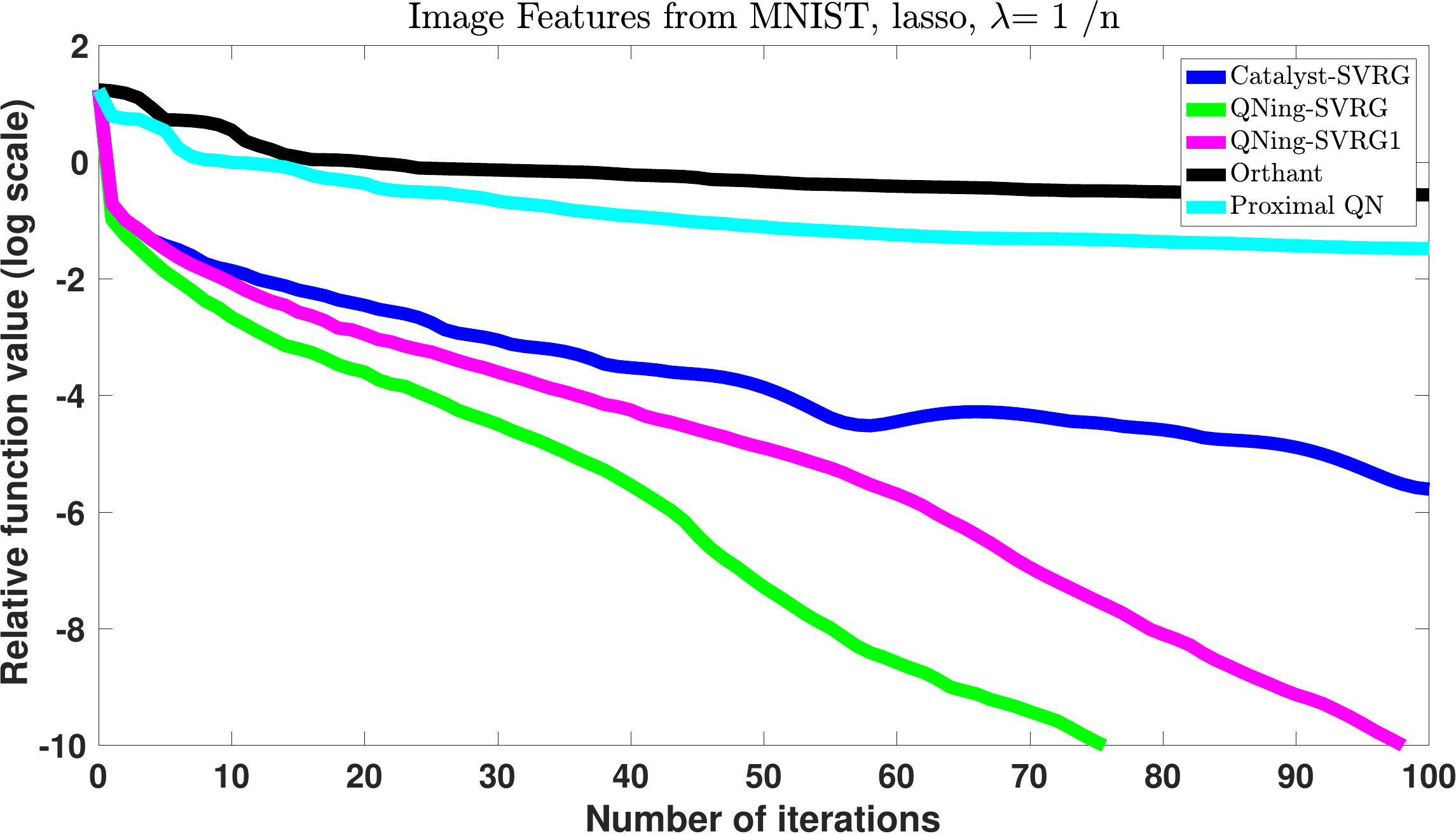} \\
   ~~\includegraphics[width=0.30\linewidth]{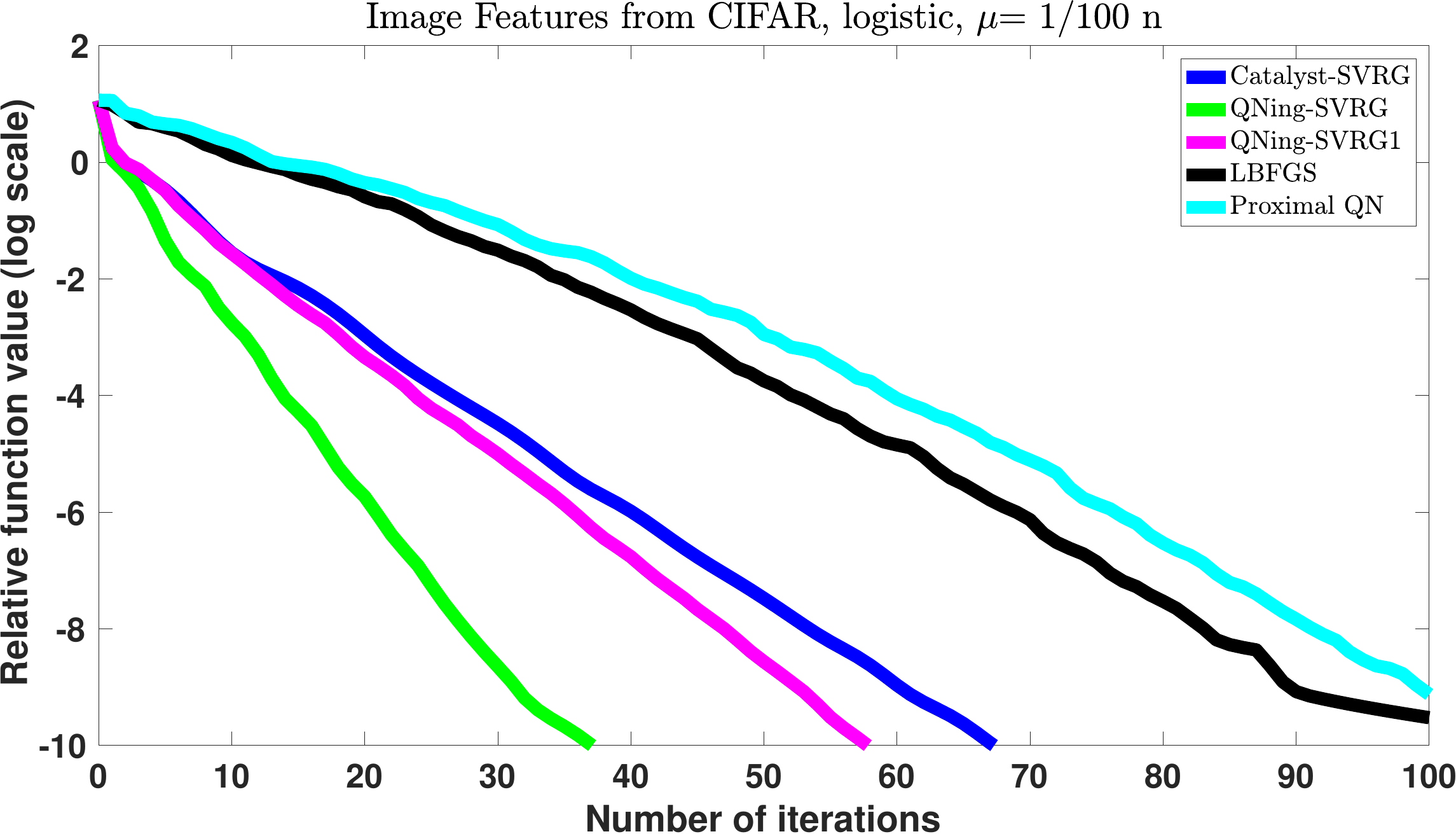} ~ 
   ~~\includegraphics[width=0.30\linewidth]{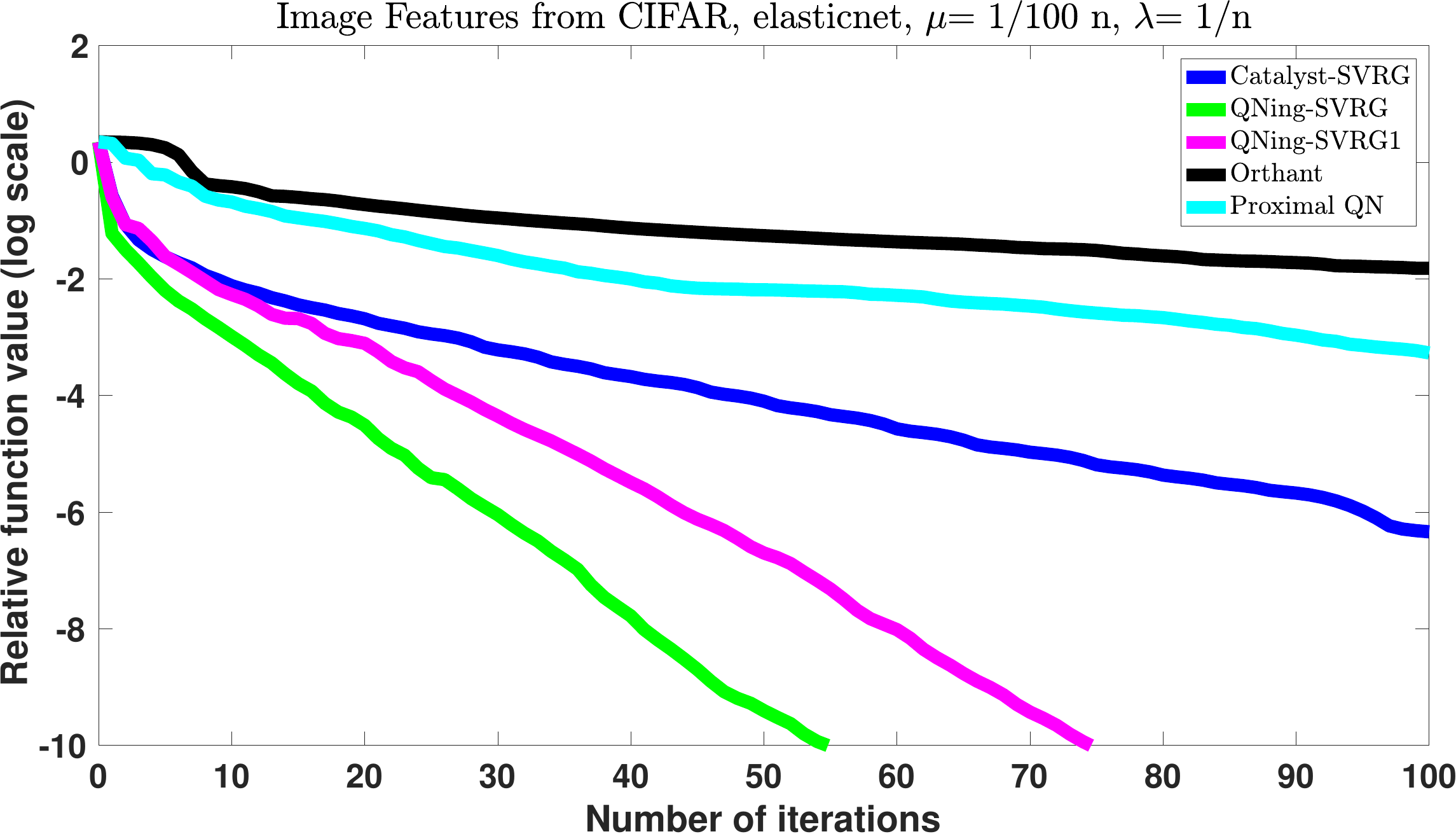} ~ 
   ~~\includegraphics[width=0.30\linewidth]{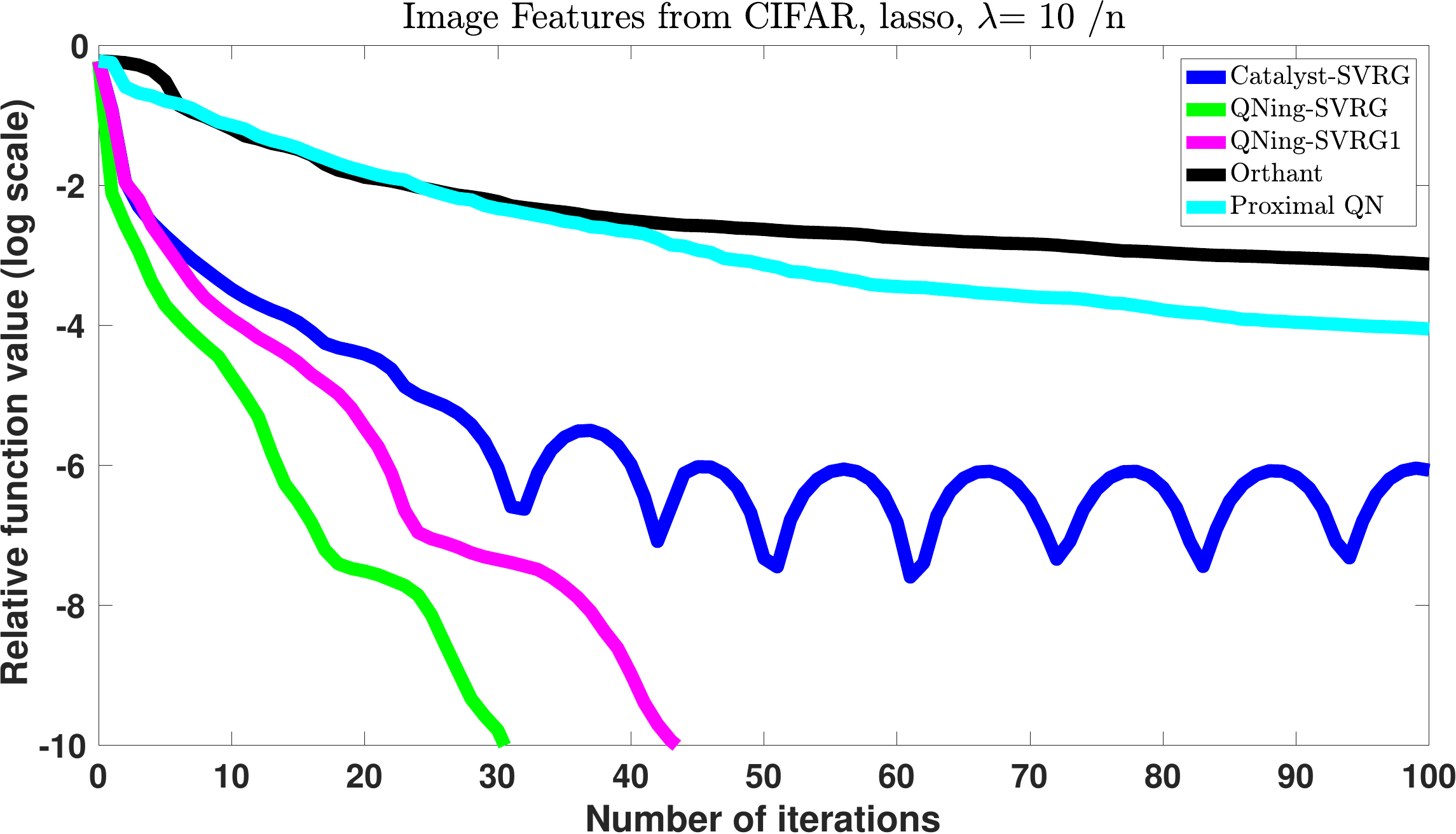} \\
   \caption{Experimental study of the performance of \qning-SVRG respect to the number of outer iterations.
   }\label{fig:svrg_it}
\end{figure}

The result of the comparison is presented in Figure~\ref{fig:svrg_it}. We observe that the theoretical grounded variant \qning-SVRG always outperform the one-pass heuristic \qning-SVRG1. This is not surprising since the sub-problems are solved more accurately in the theoretical grounded variant. However, once we take the complexity of the sub-problems into account, \qning-SVRG never outperforms \qning-SVRG1. This suggests that it is not beneficial to solve the sub-problem up to high accuracy as long as the algorithm converge. 

\subsection{Empirical frequency of choosing the unit stepsize} \label{subsec:unit}
In this section, we evaluate how often the unit stepsize is taken in the line search. When the unit stepsize is taken, the variable metric step provides sufficient decrease, which is the key for acceleration. The statistics of \qning-SVRG1 (one-pass variant) and \qning-SVRG (the sub-problems are solved until the stopping criteria (\ref{eq:stop condition}) is satisfed ) are given in Table~\ref{tab:qning1} and Table~\ref{tab:qning}, respectively. As we can see, for most of the iterations~($>90\%$), the unit stepsize is taken.

\begin{table}[htbp!]
\captionsetup{width=.8\textwidth}
\centering
  \caption{Relative frequency of picking the unit stepsize $\eta_k =1$ of QNing-SVRG1}
    \begin{tabular}{c|cc|cc|cc}
    \toprule
    \multicolumn{1}{c}{} & \multicolumn{2}{c}{\textbf{Logistic}} & \multicolumn{2}{c}{\textbf{Elastic-net}} & \multicolumn{2}{c}{\textbf{Lasso}} \\
    \midrule
    covtype & 24/27  & 89\% & 54/56   & 96\% & 19/21  & 90\% \\
    alpha   & 8/8   & 100\% & 6/6   & 100\% & 6/6   & 100\% \\
    real-sim & 60/60   & 100\% & 71/76    & 93\% & 14/14   & 100\% \\
    mnist & 53/53 & 100\% & 80/80  & 100\% & 100/100    & 100\% \\
    cifar-10 & 58/58 & 100\% & 75/75 & 100\% & 42/44    & 95\% \\
    \bottomrule
    
    \end{tabular}%
    \bigskip
    \caption*{The first column is in the form $N/D$, where $N$ is the number of times over the iterations the unit stepsize was picked and $D$ is the total number of iterations. The total number of iterations $D$ varies a lot since we stop our algorithm as soon as the relative function gap is smaller than $10^{-10}$ or the maximum number of iterations $100$ is reached. It implicitly indicates how easy the problem is.}
  \label{tab:qning1}%
\end{table}

\begin{table}[htbp!]
\captionsetup{width=.8\textwidth}
  \centering
  \caption{Relative frequency of picking the unit stepsize $\eta_k =1$ of QNing-SVRG}
    \begin{tabular}{c|cc|cc|cc}
    \toprule
    \multicolumn{1}{c}{} & \multicolumn{2}{c}{\textbf{Logistic}} & \multicolumn{2}{c}{\textbf{Elastic-net}} & \multicolumn{2}{c}{\textbf{Lasso}} \\
    \midrule
    covtype & 18/20  & 90\% & 23/25   & 92\% & 16/16  & 100\% \\
    alpha   & 6/6   & 100\% & 4/4   & 100\% &  3/3  & 100\% \\
    real-sim & 27/27 & 100\% & 20/23  & 87\% & 8/8  & 100\% \\
    mnist & 27/27 & 100\% & 28/28  & 100\% & 28/28  & 100\% \\
    cifar-10 & 25/25 & 100\% & 29/29 & 100\% & 31/31  & 100\% \\
    \bottomrule
    \end{tabular}%
        \bigskip
    \caption*{The setting are the same as in Table~\ref{tab:qning1}.}
  \label{tab:qning}%
\end{table}

\end{document}